\documentclass[11pt,twoside]{article} 

\usepackage{eqnarray,amsmath}
\usepackage[utf8]{inputenc} 
\usepackage[T1]{fontenc}    
\usepackage{booktabs}       
\usepackage{amsfonts}       
\usepackage{nicefrac}       
\usepackage{microtype}      
\usepackage{xcolor}         
\usepackage{subcaption}
\expandafter\def\csname 
ver@subfig.sty\endcsname{}
\usepackage{eqnarray,amsmath}
\usepackage{epsf}
\usepackage{epsfig}
\usepackage{fancyhdr}
\usepackage{graphics}
\usepackage{graphicx}
\usepackage{psfrag}
\usepackage{fullpage}
\usepackage{pdfpages}

\usepackage{url}

\usepackage{color}

\usepackage{amsthm}
\usepackage{amsmath}
\usepackage{amssymb,bbm}
\usepackage[skip=0pt]{caption}  
\usepackage{algorithmic}
\usepackage{algorithm}
\usepackage{textcomp}
\usepackage{siunitx}
\usepackage{wrapfig}
\usepackage{algorithmic}
\usepackage{algorithm}
\usepackage{multirow}
\usepackage{multicol}
\usepackage{makecell}
\usepackage{colortbl} 
\usepackage{changepage}

\definecolor{customyellow}{HTML}{fedf8a}
\usepackage{eqnarray,amsmath}
\usepackage{epsf}
\usepackage{epsfig}
\usepackage{fancyhdr}
\usepackage{graphics}
\usepackage{graphicx}
\usepackage{psfrag}
\usepackage{fullpage}
\usepackage{pdfpages}
\usepackage{ragged2e}

\newtheorem{assumption}{Assumption}

\newtheorem{lemma}{Lemma}

\newtheorem{theorem}{Theorem}
\newtheorem{proposition}{Proposition}
\newtheorem{definition}{Definition}
\newtheorem{corollary}{Corollary}

\usepackage{color}

\usepackage{amsthm}
\usepackage{amsmath}
\usepackage{amssymb,bbm}
\usepackage{caption}
\usepackage{algorithmic}
\usepackage{algorithm}
\usepackage{textcomp}
\usepackage{siunitx}
\usepackage{wrapfig}
\usepackage{algorithmic}
\usepackage{algorithm}
\usepackage{multirow}
\usepackage{multicol}
\usepackage{mathtools}

\def\1{\bm{1}}

\DeclareMathAlphabet{\mathsfit}{\encodingdefault}{\sfdefault}{m}{sl}
\SetMathAlphabet{\mathsfit}{bold}{\encodingdefault}{\sfdefault}{bx}{n}

\newcommand{\softmax}{\mathrm{softmax}}

\DeclareMathOperator*{\argmax}{arg\,max}
\DeclareMathOperator*{\argmin}{arg\,min}

\usepackage{eqnarray,amsmath}

\usepackage{amsthm}
\usepackage{amsmath}
\usepackage{amssymb,bbm}

\allowdisplaybreaks

\newcommand{\bbE}{\mathbb{E}}

\newcommand{\topK}{\mathrm{Top}_{K}}
\newcommand{\topKbar}{\mathrm{Top}_{\bar{K}}}

\newcommand{\dint}{\mathrm{d}}

\newcommand{\dboijn}{\Delta \beta_{1ij}^{n}}
\newcommand{\dbzijn}{\Delta \beta_{0ij}^{n}}

\newcommand{\deijn}{\Delta \eta_{ij}^{n}}
\newcommand{\deoijn}{\Delta \eta_{1ij}^{n}}
\newcommand{\dezijn}{\Delta \eta_{0ij}^{n}}
\newcommand{\dnuijn}{\Delta \nu_{ij}^{n}}
\newcommand{\dtijn}{\Delta \tau_{ij}^{n}}
\newcommand{\dkijn}{\Delta \kappa_{ij}^{n}}
\newcommand{\dkoijn}{\Delta \kappa_{1ij}^{n}}
\newcommand{\dkzijn}{\Delta \kappa_{0ij}^{n}}

\newcommand{\cdboin}{\Delta \check{\beta}_{1i}^{n}}
\newcommand{\cdbzin}{\Delta \check{\beta}_{0i}^{n}}
\newcommand{\cdein}{\Delta \check{\eta}_{i}^{n}}
\newcommand{\cdnuin}{\Delta \check{\nu}_{i}^{n}}

\newcommand{\boin}{\beta_{1i}^n}
\newcommand{\bzin}{\beta_{0i}^n}

\newcommand{\ein}{\eta_i^n}
\newcommand{\eoin}{\eta_{1i}^n}
\newcommand{\ezin}{\eta_{0i}^n}
\newcommand{\nuin}{\nu_i^n}
\newcommand{\oin}{\omega_i^n}
\newcommand{\kin}{\kappa_i^n}
\newcommand{\koin}{\kappa_{1i}^n}
\newcommand{\kzin}{\kappa_{0i}^n}
\newcommand{\tin}{\tau_i^n}

\newcommand{\boj}{\beta_{1j}^*}
\newcommand{\bzj}{\beta_{0j}^*}

\newcommand{\ej}{\eta_j^*}
\newcommand{\eoj}{\eta_{1j}^*}
\newcommand{\ezj}{\eta_{0j}^*}
\newcommand{\nuj}{\nu_j^*}
\newcommand{\oj}{\omega_j^*}
\newcommand{\kj}{\kappa_j^*}
\newcommand{\koj}{\kappa_{1j}^*}
\newcommand{\kzj}{\kappa_{0j}^*}
\newcommand{\tj}{\tau_j^*}

\newcommand{\boi}{\beta_{1i}^*}
\newcommand{\bzi}{\beta_{0i}^*}

\newcommand{\cboi}{\check{\beta}_{1i}}
\newcommand{\cbzi}{\check{\beta}_{0i}}
\newcommand{\cei}{\check{\eta}_i}
\newcommand{\cnui}{\check{\nu}_i}

\newcommand{\brj}{r_{1,j}}
\newcommand{\trj}{r_{2,j}}

\begin{document}

\begin{center}

{\bf{\LARGE{On DeepSeekMoE: Statistical Benefits of Shared Experts
and Normalized Sigmoid Gating}}}
  
\vspace*{.2in}
{\large{
\begin{tabular}{cccccc}
Huy Nguyen$^{\dagger}$ & Thong T. Doan$^{\diamond}$ & Quang Pham$^{\ddagger}$ \\
Nghi D. Q. Bui$^{\diamond}$ & Nhat Ho$^{\dagger,\star}$ & Alessandro Rinaldo$^{\dagger,\star}$
\end{tabular}
}}

\vspace*{.2in}

\begin{tabular}{cc}
$^{\dagger}$The University of Texas at Austin\\
$^{\diamond}$FPT Software AI Center\\
$^{\ddagger}$Salesforce AI Research \\
\end{tabular}

\vspace*{.2in}
\today


\begin{abstract}
    Mixture of experts (MoE) methods  
   are a key component in most large language model architectures, including the recent series of DeepSeek models. Compared to other MoE implementations, DeepSeekMoE stands out because of two unique features: the deployment of a shared expert strategy and of the normalized sigmoid gating mechanism. Despite the prominent role of DeepSeekMoE in the success of the DeepSeek series of models, there have been only a few attempts to justify theoretically the value of the shared expert strategy, while its normalized sigmoid gating has remained unexplored. 
  To bridge this gap, we undertake a comprehensive theoretical study of these two features of DeepSeekMoE from a statistical perspective. We perform a convergence analysis of the expert estimation task to highlight the gains in sample efficiency for both the shared expert strategy and the normalized sigmoid gating, offering useful insights into the design of expert and gating structures. To verify empirically  our theoretical findings, we carry out several experiments on both synthetic data and real-world datasets for (vision) language modeling tasks. Finally, we conduct an extensive empirical analysis of the router behaviors, ranging from router saturation, router change rate, to expert utilization. 
\end{abstract}
\end{center}
\let\thefootnote\relax\footnotetext{$\star$ Co-last authors.}

\section{Introduction}
\label{sec:introduction}

The recent years have witnessed a dramatic increase in the use and success of of deep learning models, leading to remarkable advances in a variety of fields, namely natural language processing \cite{jiang2024mixtral,Du_Glam_MoE,fedus2022switch,lepikhin_gshard_2021}, computer vision \cite{Riquelme2021scalingvision,liang_m3vit_2022}, multimodal learning \cite{han2024fusemoe,yun2024flexmoe},  and reinforcement learning \cite{ceron2024rl,chow_mixture_expert_2023}. However, this trend has also introduced several challenges in terms of computational efficiency. One common approach to tackle this challenge is to leverage Mixture-of-Experts (MoE) architecture, which allows to scale up the model capacity without a proportional increase in computation. 

\vspace{0.5 em}
\noindent
Originally proposed by \cite{Jacob_Jordan-1991}, MoE has been known as a form of ensemble learning that combines the power of several individual models through an adaptive gating network.
In particular, these individual models are termed experts and can be formulated as classifiers \cite{chen2022theory,nguyen2024general}, regression models \cite{faria2010regression,kwon_em_2020}, or feed-forward networks (FFNs) \cite{shazeer2017topk,dai2024deepseekmoe}. Meanwhile, the gating network is responsible for dynamically assigning input-dependent softmax weights to experts based on their specialization in the input domain. 
Then, to improve the scalability of MoE, \cite{shazeer2017topk} have recently introduced a sparse version of MoE which activates only a subset of specialized experts per input, allowing to increase the number of trainable parameters while keeping the computation overhead nearly unchanged. As a result, there has been a surge of interest in employing the sparse MoE architecture in several large-scale applications, particularly language modeling \cite{deepseekv3,grattafiori2024llama3,geminiteam2024gemini15,qwen2025}.

\vspace{0.5 em}
\noindent
Despite their widespread use in large language models, the sparse MoE architecture faces the challenge of knowledge redundancy, that is, multiple experts may end up acquiring overlapping knowledge, leading to the redundancy of expert parameters. In response to this issue, \cite{dai2024deepseekmoe} have come up with a novel DeepSeekMoE framework (see Figure~\ref{fig:deepseek_architecture}) that divides the set of experts into two disjoint subsets. Experts in the first subset are referred to as shared experts and are always activated to capture common knowledge across different domains. On the other hand, only few experts in the second subset,  called routed experts, are activated, typically via a sparse softmax gating mechanism to learn specialized knowledge. This shared expert strategy helps enhance expert specialization by encouraging experts to specialize in distinctive aspects of the data, thereby alleviating the parameter redundancy problem. The new DeepSeekMoE architecture has been adopted as a vital component in the series of high-performance DeepSeek language models, most notably DeepSeek-V2 \cite{deepseekv2} and DeepSeek-V3 \cite{deepseekv3}. Another innovative aspect of the DeepSeekMoE framework lies in the choice of gating functions. In particular, the DeepSeek-V2 language model still uses a traditional softmax gating function to determine expert weights, the DeepSeek-V3 version employs a new normalized sigmoid gating function, which partly helps the latter model outperforms the former one. Given the success of DeepSeekMoE, it is surprising that the shared expert strategy has only been briefly investigated in \cite{dai2024deepseekmoe} from the perspective of expert specialization without any rigorous exploration, while there have been no studies on the benefits of the normalized sigmoid gating in the literature.

\begin{figure}[t!]
    \centering
    \includegraphics[width=.5\linewidth]{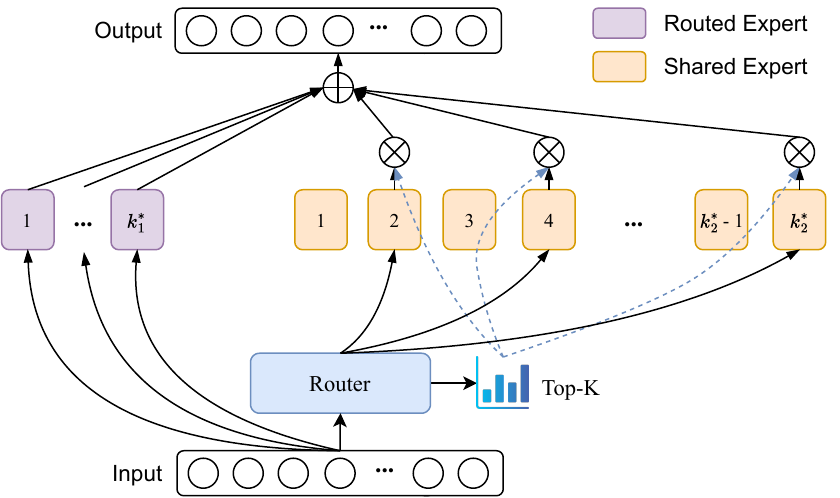}
    \caption{Illustration of the DeepSeekMoE architecture.}
    \label{fig:deepseek_architecture}
    \vspace{-0.8em}
\end{figure}
\vspace{0.5 em}
\noindent
\textbf{Contributions.} The primary goal of this paper is to provide a comprehensive theoretical study of these two distinguishing features of DeepSeekMoE. Below we perform a convergence analysis of expert estimation in order to examine the sample efficiency of the shared expert strategy, that is the rate, as a function of the number of data points, at which each expert to specialize in some aspects of the data. Furthermore, we also compare the sample efficiency of the normalized sigmoid gating used in the DeepSeek-V3 model with that of the softmax gating used in the DeepSeek-V2 model. 
Our contributions are threefold and can be summarized as follows.

\vspace{0.5em}
\emph{(1) Sample complexity of estimating experts when using the shared expert strategy.} Our analysis in Section~\ref{sec:shared_expert_strategy} 
reveals that shared experts admit significantly faster convergence rates than routed experts and experts in MoE models without the shared expert strategy, whose rates depend in a complicated manner on the solvability of certain systems of polynomial equations as well as the number of fitted experts (see Table~\ref{table:expert_rates_softmax}). As a result, a smaller amount of data are required to approximate shared experts compared to non-shared experts in DeepSeekMoE and standard MoE models to achieve the same level of statistical accuracy. 

\vspace{0.5em}
\emph{(2) Sample complexity of estimating experts when using the normalized sigmoid gating.} In Section~\ref{sec:normalized_sigmoid_gating}, when using the normalized sigmoid gating instead of the softmax gating, the convergence rates of routed experts no longer hinge on the solvability of a system of polynomial equations and, therefore, are provably faster than those of shared experts, which remain unchanged in this setting (see also Table~\ref{table:expert_rates_softmax}). Thus, the amount of data required to estimate routed experts within a given error decreases substantially, demonstrating the sample efficiency of the normalized sigmoid gating over the standard softmax gating. 

\vspace{0.5em}
\emph{(3) Empirical validation.} To validate our theoretical findings, we conduct extensive numerical experiments on simulated and real-world data. The experimental results on synthetic data are in very close agreement with our theoretical findings about the convergence rates of the shared expert strategy and the normalized sigmoid gating; see Section~\ref{subsec:numerical_exp} for detailed results. The experiments on language modeling and vision-language modeling in Sections~\ref{subsec:lm_modeling} and \ref{subsec:vlm_modeling} further demonstrate the applicability of our theoretical insights in real-world scenarios. Finally,  we perform a comprehensive router analysis in Section~\ref{sec:router_analysis}, including the router saturation, router change rate, and expert utilization.

\vspace{0.5 em}
\noindent
\textbf{Notation.} For any $n\in\mathbb{N}$, we let $[n] = \{1,2,\ldots,n\}$. For any vectors $v:=(v_i)_{i=1}^{d} \in \mathbb{R}^{d}$ and $\alpha:=(\alpha_i)_{i=1}^{d}\in\mathbb{N}^d$, we denote $v^{\alpha}:=\prod_{i=1}^{d}v_{i}^{\alpha_i}$, $|v|:=\sum_{i=1}^{d}v_{i}$ and $\alpha!:=\prod_{i=1}^{d}\alpha_{i}!$, while $\|v\|$ represents the $\ell_2$-norm of $v$. The cardinality of a set $S$ is denoted with $|S|$. Finally, for any two positive sequences $(a_n)_{n\geq 1}$ and $(b_n)_{n\geq 1}$, we write $a_n = \mathcal{O}(b_n)$ or $a_{n} \lesssim b_{n}$ if $a_n \leq C b_n$ for all $ n\in\mathbb{N}$, for some constant $C > 0$. For a sequence $(A_n)_{n\geq 1}$ of positive random variables, the notation $A_{n} = \mathcal{O}_{P}(b_{n})$ signifies $A_{n}/b_{n}$ is stochastically bounded, that is, for any $\epsilon>0$, there exists an $M>0$ such that $\mathbb{P}( A_{n}/b_{n} > M) < \epsilon $ for all $n$ large enough. We further write $A_n = \widetilde{\mathcal{O}}_{P}(b_n)$ when $A_n = \mathcal{O}_{P}(b_n \log^c(b_n))$, for some $c > 0$. Finally, for two Lebesgue probability densities on $\mathbb{R}^d$, $f_1$ and $f_2$,  $V(f_1,f_2):=\frac{1}{2}\int|f_1(y)-f_2(y)|dy$ denotes  their total variation distance. 

\begin{table*}[t!]
\caption{\small Summary of expert estimation rates in DeepSeek-V2's MoE with softmax gating (Section~\ref{sec:shared_expert_strategy}) and DeepSeek-V3's MoE with normalized sigmoid gating (Section~\ref{sec:normalized_sigmoid_gating}). Below, the function $r_{2}$ stands for the solvability of certain systems  of polynimial equations specified in Appendix~\ref{appendix:system}, while the notation $\mathcal{V}_{2,j}$ denotes a Voronoi cell defined in equation~\eqref{eq:Voronoi_cells}. For the normalized sigmoid gating setting, we consider two complementary parameter settings, namely sparse regime and dense regime (see Section~\ref{sec:normalized_sigmoid_gating} for further details).
\vspace{0.5em}
}
\centering
\begin{tabular}{| c | c | c |} 
\hline
DeepSeek-V2's MoE &\textbf{GELU FFN Experts} & \textbf{Linear Experts}\\
\hline
{\textbf{Shared Experts}}  & $\widetilde{\mathcal{O}}_{P}(n^{-1/4})$ & { $\widetilde{\mathcal{O}}_{P}(n^{-1/4})$}   \\
\hline
{\textbf{Routed Experts}}  & $\widetilde{\mathcal{O}}_{P}(n^{-1/4})$ & { $\widetilde{\mathcal{O}}_{P}(n^{-1/r_{2}(|\mathcal{V}_{2,j}|)})$}   \\
\hline
\end{tabular}
\label{table:expert_rates_softmax}
\end{table*}

\begin{table*}[t!]
\centering
\resizebox{\textwidth}{!}{
\begin{tabular}{| m{8em} | c | c | c | c |} 
\hline
\multirow{2}{8em}{DeepSeek-V3's MoE}& \multicolumn{2}{c|}{\textbf{GELU FFN Experts}} &\multicolumn{2}{c|}{\textbf{Linear Experts}}\\ \cline{2-5}
 &Sparse Regime& Dense Regime &Sparse Regime& Dense Regime  \\
\hline
{\textbf{Shared Experts}}  & \multicolumn{2}{c|}{$\widetilde{\mathcal{O}}_{P}(n^{-1/4})$}  & \multicolumn{2}{c|}{$\widetilde{\mathcal{O}}_{P}(n^{-1/4})$}  \\
\hline
{\textbf{Routed Experts}}  & $\widetilde{\mathcal{O}}_{P}(n^{-1/4})$ & { $\widetilde{\mathcal{O}}_{P}(n^{-1/2})$} & $\widetilde{\mathcal{O}}_{P}(n^{-1/r_{2}(|\mathcal{V}_{2,j}|)})$ & { $\widetilde{\mathcal{O}}_{P}(n^{-1/2})$}  \\
\hline  
\end{tabular}}
\end{table*}

\section{Related Work}
\label{appendix:related_work}
There have been two primary lines of works on understanding MoE models in the literature. 

\vspace{0.5 em}
\noindent
From a statistical perspective, \cite{zeevi1998approximation} investigated the representation power of a mixture of generalized linear experts when using this model to approximate target functions belonging to a Sobolev class. Next, \cite{mendes2011convergence} performed a convergence analysis of MLE under the MoE with
experts being polynomial regression models, offering an important insight for finding the optimal configuration of the number of experts and their sizes. After that, considering data generated from a Gaussian MoE with covariate-free gating, \cite{ho2022gaussian} established an \emph{algebraic independence} condition on the location and scale functions of the Gaussian density to characterize which choices of this pair will lead to faster convergence rates of parameter estimation. Then, this analysis was extended to more practical yet challenging settings of dense and sparse softmax gating Gaussian MoE in \cite{nguyen2023demystifying} and \cite{nguyen2024statistical}, respectively. These works demonstrated that parameter and expert estimation rates hinged on the solvability of some systems of polynomial equations and became significantly slow as the number of experts increased. Lastly, \cite{nguyen2025convergence} considered a MoE-based regression framework where the regression function took the form of MoE with standard softmax gating, dense-to-sparse gating, and hierarchical softmax gating, respectively. Their convergence analysis of least squares estimation provided critical implications on the design of expert structures. In particular, it indicated that feed-forward expert
networks equipped with the sigmoid function or the Gaussian linear error unit (GELU) activation function admitted estimation rates of polynomial orders, while experts of polynomial forms had much slower estimation rates, of exponential orders.

\vspace{0.5 em}
\noindent
From a deep learning perspective, \cite{chen2022theory} took into account a classification problem with cluster structures using MoE models. In particular, they justified the ability of the gating network to learn the cluster-center features, enabling the model to separate a big complex problem into simpler ones, each of which will be handled by the corresponding specialized experts. Next, \cite{boixadsera2025granularity} studied the effects of the number of active experts on the expressivity of sparse MoE models, while \cite{wang2025expressivepower} investigated the expressive power of MoE in modeling complex tasks. Furthermore, theories for applications of MoE in continual learning \cite{li2025cl,le2024mixture}, domain adaptation \cite{nguyen2025cosine,chi_representation_2022}, and language modeling \cite{pham2024competesmoe,diep2025zero} have also been extensively explored in the literature. Interestingly, self-attention mechanism in the Transformers architecture \cite{vaswani2017attention} has recently been shown to be represented by a mixture of linear experts with quadratic softmax gating \cite{akbarian2024quadratic,yan2025sigmoid}, leading to numerous advances in parameter-efficient fine-tuning methods \cite{truong2025replora,le2025revisiting}.

\vspace{0.5 em}
\noindent
However, to the best of our knowledge, no prior work has been done to identify the theoretical properties of the DeepSeekMoE architecture. Therefore, we aim to provide a comprehensive study on the sample complexity of estimating experts when adopting two key ingredients of DeepSeekMoE, including the shared expert strategy in Section~\ref{sec:shared_expert_strategy} and the normalized sigmoid gating in Section~\ref{sec:normalized_sigmoid_gating}.
\section{On Shared Expert Strategy}
\label{sec:shared_expert_strategy}

To begin with, we analyze the effects of the shared expert strategy on the sample complexity of estimating experts in the DeepSeek-V2's MoE. For ease of presentation, we will focus here on the dense DeepSeekMoE case, and defer the analysis of the sparse DeepSeekMoE setting to Appendix~\ref{appendix:sparse_gating}. In the sequel, after formally introducing the settings, we formulate a \emph{strong identifiability} condition on the expert functions ensuring fast expert convergence rates in Section~\ref{sec:strongly_identifiable}. We then turn to linear experts, which violate the strong identifiability condition, and prove that, in fact, they exhibit slow rates of convergence in Section~\ref{sec:linear_experts}. 

\vspace{0.5 em}
\noindent
\textbf{Problem setting.} Assume that $(X_{1}, Y_{1}),(X_{2}, Y_{2}), \ldots, (X_{n}, Y_{n}) \in \mathbb{R}^{d} \times \mathbb{R}$ are i.i.d. samples drawn from a Gaussian DeepSeekMoE model, whose conditional density function $f_{G^{*}_1,G^*_2}(y|x)$ is given by
\begin{align}
    \label{eq:density}
    f_{G^*_1,G^*_2}(y|x):= & \frac{1}{2}\sum_{i=1}^{k^*_1}\omega^*_{i}\pi(Y|h_1(x,\kappa^*_{i}),\tau^*_{i})+\frac{1}{2}\sum_{i = 1}^{k^*_2} \frac{\exp((\beta_{1i}^{*})^{\top} x + \beta_{0i}^{*})}{\sum_{j = 1}^{k^*_2} \exp((\beta_{1j}^{*})^{\top} x + \beta_{0j}^{*})}\pi(y|h_2(x,\eta^*_{i}), \nu_{i}^{*}). 
\end{align}
Above, $\pi(\cdot|\mu, \nu)$ denotes the Gaussian density function with mean $\mu$ and variance $\nu$,  $h_1(\cdot,\kappa^*_i)$ and $h_2(\cdot,\eta^*_i)$ are real-valued functions on $\mathbb{R}^d$ referred to as {\it shared} and {\it routed} experts, respectively.
The weight parameters $\omega^*_1,\omega^*_2,\ldots,\omega^*_{k^*_1}$ are positive and satisfy $\sum_{i=1}^{k^*_1}\omega^*_{i}=1$.  
We conveniently represent all the model parameters with the \emph{mixing measures} $G^*_1 : = \sum_{i=1}^{k^*_1}\omega^*_i\delta_{(\kappa^*_{i},\tau^*_{i})}$ and $G^*_2:=\sum_{i = 1}^{k^*_2} \exp(\beta_{0i}^{*}) \delta_{(\beta_{1i}^{*}, \eta_{i}^{*}, \nu_{i}^{*})}$, a combination of Dirac $\delta$-measures with mass on the unknown true parameters $\theta^*_{1i}:=(\omega^*_i,\kappa^*_i,\tau^*_i)$ in $\Theta_1\subseteq\mathbb{R}\times\mathbb{R}^{d_1}\times\mathbb{R}_+$ and $\theta^*_{2i}:=(\beta^*_{0i},\beta^*_{1i},\eta^*_{i},\nu^*_{i})$ in $\Theta_2\subseteq\mathbb{R}\times\mathbb{R}^{d}\times\mathbb{R}^{d_2}\times\mathbb{R}_+$, respectively. 
Thus, our goal is to estimate the pair of ground-truth mixing measures $(G^*_1,G^*_2)$. 

\vspace{0.5 em}
\noindent
\textbf{Maximum likelihood estimation (MLE).} 
As the numbers  $k^*_1$ and  $k^*_2$ of shared and routed experts are unknown, we consider the ground-truth model~\eqref{eq:density} with up to $k_1>k^*_1$ shared experts and $k_2>k^*_2$ routed experts. Towards that goal, we let $\mathcal{G}_{k_1,k_2}(\Theta):=\mathcal{G}_{k_1}(\Theta_1)\times\mathcal{G}_{k_2}(\Theta)$ stands for the set of mixing measure pairs $(G_1,G_2)$ with at most $k_1$ and $k_2$ atoms, respectively; that is $\mathcal{G}_{k_1}(\Theta_1):=\Big\{G_1=\sum_{i=1}^{k'_1}\omega_i\delta_{(\kappa_i,\tau_i)}:1\leq k'_1\leq k_1\Big\}$ and $\mathcal{G}_{k_2}(\Theta_2):=\Big\{G_2=\sum_{i=1}^{k'_2}\exp(\beta_{0i})\delta_{(\beta_{1i},\eta^*_{i},\nu^*_{i})}:1\leq k'_2\leq k_2\Big\}$.
Our final estimator is the MLE over $\mathcal{G}_{k_1,k_2}(\Theta)$, i.e.
\begin{align}
    \label{eq:MLE}
    (\widehat{G}^n_1,\widehat{G}^n_2)\in\argmax_{(G_1,G_2)\in\mathcal{G}_{k_1,k_2}(\Theta)}\frac{1}{n}\sum_{i=1}^{n}\log(f_{G_1,G_2}(Y_i|X_i)),
\end{align}
\textbf{Universal assumptions.} For our theoretical analysis, we impose the following three mild assumptions on the ground-truth parameters throughout the paper. 

\vspace{0.5em}
\emph{(A.1) The parameter space $\Theta$ is compact with fixed dimension, while the input space $\mathcal{X}$ is bounded.}

\vspace{0.5em}
\emph{(A.2) The last pair of gating parameters vanish, that is, $\beta^*_{1k^*_2}=0_d$ and $\beta^*_{0k^*_2}=0$ (to avoid non-identifiability due to invariance to translation of the softmax gating function). In addition, at least one among parameters $\{ \beta^*_{1i}, i\in[k^*_2]\}$, is non-zero (to maintain the dependence of the gating on the input value).}

\vspace{0.5em}
\emph{(A.3) The expert parameters $(\kappa^*_i)_{i=1}^{k^*_1}$ and $(\eta^*_i)_{i=1}^{k^*_2}$ are distinct. Meanwhile, the expert functions $h_1(\cdot,\kappa)$ and $h_2(\cdot,\eta)$ are bounded and Lipschitz continuous w.r.t $\kappa$ and $\eta$.}

\vspace{0.5 em}
\noindent
Equipped with these assumptions, we are now ready to give our first consistency result for the ground-truth conditional density $f_{G^*_1,G^*_2}$. 
\begin{proposition}
    \label{prop:density_rate}
    The maximum likelihood density estimator $f_{\widehat{G}^n_1,\widehat{G}^n_2}(Y|X)$ converges to the true density $f_{G^*_1,G^*_2}(Y|X)$ in total variation distance at the  rate
    \begin{align*}
        \bbE_X[V(f_{\widehat{G}^n_1,\widehat{G}^n_2}(\cdot|X),f_{G^*_1,G^*_2}(\cdot|X))]=\mathcal{O}_P([\log(n)/n]^{\frac{1}{2}}).
    \end{align*}
\end{proposition}
\noindent
The above result, whose proof can be found in Appendix~\ref{appendix:density_rate}, shows that the true density function $f_{G^*_1,G^*_2}(y|x)$ can be estimated at a rate that is nearly parametric. Following a strategy used in the analysis of MoE models \cite{nguyen2023demystifying}, if one can exhibit an appropriate loss function over the mixing measures, say $\mathcal{D}((G_1,G_2),(G^*_2,G^*_2))$, that, up to constant, is a lower bound on $\bbE_X[V(f_{\widehat{G}^n_1,\widehat{G}^n_2}(\cdot|X),f_{G^*_1,G^*_2}(\cdot|X))]$, Proposition~\ref{prop:density_rate} will then imply a near parametric rate also for the parameters and expert functions themselves. 

\vspace{0.5em}
\noindent
\textbf{Technical challenges.} However, the derivation of this lower bound is challenging. Specifically, a key step in establishing the aforementioned lower bound is to decompose the difference $f_{\widehat{G}^n_1,\widehat{G}^n_2}(y|x)-f_{G^*_1,G^*_2}(y|x)$ through a series of Taylor expansions of the functions $\pi(y|h_1(x,\kappa),\tau)$ and $ F(y|x;\beta_1,\eta,\nu):=\exp(\beta_1^{\top}x)\pi(y|h_2(x,\eta),\nu)$ w.r.t their parameters $(\kappa,\tau)$ and $(\beta_1,\eta,\nu)$, respectively. When the difference of the densities converges to zero (as ensured by Proposition~\ref{prop:density_rate}), then one may expect the coefficients of this Taylor expansions, which correspond to the difference between the true and estimated parameters, will also vanish. However, this is true only provided that these functions and their partial derivatives arising from the Taylor expansions remain linearly independent. To ensure that this property holds,  we formulate a new, non-trivial condition, called \emph{strong identifiability} for the expert functions $h_1$ and $h_2$.


\begin{definition}[Strong Identifiability]
    \label{def:strong_identifiability}
    We say that expert functions $x\mapsto h_1(x,\kappa)$ and $x\mapsto h_2(x,\eta)$ are strongly identifiable if they are twice differentiable w.r.t $\kappa$ and $\eta$, respectively, and if for any $k_1,k_2\geq 1$ and distinct parameters $\kappa_1,\ldots,\kappa_{k_1}$ and $\eta_1,\ldots,\eta_{k_2}$, each of the sets
    \begin{align*}
        &\Bigg\{\frac{\partial h_1}{\partial\kappa^{(u_1)}}(x,\kappa_i):i\in[k_1], \ u_1\in[d_1]\Bigg\}, \Bigg\{\frac{\partial h_1}{\partial\kappa^{(u_1)}}(x,\kappa_i)\frac{\partial h_1}{\partial\kappa^{(v_1)}}(x,\kappa_i),1 :i\in[k_1], \ u_1,v_1\in[d_1]\Bigg\},\\
        &\Bigg\{\frac{\partial h_2}{\partial\eta^{(u_2)}}(x,\eta_j), \ \frac{\partial^2h_2}{\partial\eta^{(u_2)}\partial\eta^{(v_2)}}(x,\eta_j), \ x^{(u)}\frac{\partial h_2}{\partial\eta^{(v_2)}}(x,\eta_j):j\in[k_2], \ u_2,v_2\in[d_2], \ u\in[d]\Bigg\}
    \end{align*}
    consists of linearly independent functions (in $x$).
\end{definition}
\noindent
\textbf{Examples.} Two-layer FFNs $h_1(x,(\kappa_2,\kappa_1,\kappa_0)):=\kappa_2\mathrm{GELU}(\kappa_1^{\top}x+\kappa_0)$ and $h_2(x,(\eta_2,\eta_1)):=\eta_2\mathrm{GELU}(\eta_1^{\top}x)$ are strongly identifiable. The same claim holds when replacing the $\mathrm{GELU}$ function with other activation functions such as $\mathrm{sigmoid}$ and $\tanh$. On the other hand, linear experts $h_1(x,(\kappa_1,\kappa_0)):=\kappa_1^{\top}x+\kappa_0$ and $h_2(x,(\eta_1,\eta_0)):=\eta_1^{\top}x+\eta_0$ fail to satisfy the strong identifiability condition because $\frac{\partial h_1}{\partial \kappa_0}\frac{\partial h_1}{\partial \kappa_0}=1$ and $\frac{\partial h_2}{\partial \eta_1}=x\frac{\partial h_2}{\partial \eta_0}$ for all $x$. 
\subsection{Strongly Identifiable Experts}
\label{sec:strongly_identifiable}
We first analyze the convergence behavior of strongly identifiable experts. For that purpose, it is necessary to construct a loss over pairs of mixing measures $(G_1,G_2)$ and $(G^*_1,G^*_2)$. To this end, let us revisit the concepts of Voronoi cells and Voronoi loss functions presented in \cite{manole22refined}.

\vspace{0.5em}
\noindent
\textbf{Voronoi loss.} For any pair of mixing measures $(G_1,G_2)$ with $k'_1\leq k_1$ and $k'_2\leq k_2$ atoms, we distribute their atoms to the Voronoi cells $\mathcal{V}_{1,j_1}\equiv\mathcal{V}_{1,j_1}(G)$ and $\mathcal{V}_{2,j_2}\equiv\mathcal{V}_{2,j_2}(G)$, defined as
\begin{align}
    \label{eq:Voronoi_cells}
    \mathcal{V}_{1,j_1}:=\{i_1\in[k'_1]:\|\xi_{i_1}-\xi^*_{j_1}\|\leq\|\xi_{i_1}-\xi^*_{\ell_1}\|, \ \forall \ell_1\neq j_1\},\nonumber\\
    \mathcal{V}_{2,j_2}:=\{i_2\in[k'_2]:\|\zeta_{i_2}-\zeta^*_{j_2}\|\leq\|\zeta_{i_2}-\zeta^*_{\ell_2}\|, \ \forall \ell_2\neq j_2\},
\end{align}
where we denote $\xi_{i_1}:=(\kappa_{i_1},\tau_{i_1})$, $\xi^*_{j_1}:=(\kappa^*_{j_1},\tau^*_{j_1})$ for all $j_1\in[k^*_1]$, and $\zeta_{i_2}:=(\beta_{1i_2},\eta_{i_2},\nu_{i_2})$, $\zeta^*_{j_2}:=(\beta^*_{1j_2},\eta^*_{j_2},\nu^*_{j_2})$ for all $j_2\in[k^*_2]$. Then, the proposed Voronoi loss over mixing measures is given by
\begin{align}
    \label{eq:loss_1}
    &\mathcal{D}_{1}((G_1,G_2),(G^*_1,G^*_2)):=\sum_{j=1}^{k^*_1}\Big|\sum_{i\in\mathcal{V}_{1,j}}\omega_{i}-\oj\Big|+\sum_{j=1}^{k^*_2}\Big|\sum_{i\in\mathcal{V}_{2,j}}\exp(\beta_{0i})-\exp(\bzj)\Big|\nonumber\\
    &+\sum_{\substack{j\in[k^*_1],\\|\mathcal{V}_{1,j}|=1}}\sum_{i\in\mathcal{V}_{1,j}}\omega_{i}(\|\Delta\kappa_{ij}\|+|\Delta\tau_{ij}|)+\sum_{\substack{j\in[k^*_2],\\|\mathcal{V}_{2,j}|=1}}\sum_{i\in\mathcal{V}_{2,j}}\exp(\beta_{0i})(\|\Delta\beta_{1ij}\|+\|\Delta\eta_{ij}\|+|\Delta\nu_{ij}|)\nonumber\\
    &+\sum_{\substack{j\in[k^*_1],\\|\mathcal{V}_{1,j}|>1}}\sum_{i\in\mathcal{V}_{1,j}}\omega_{i}(\|\Delta\kappa_{ij}\|^2+|\Delta\tau_{ij}|^2)+\sum_{\substack{j\in[k^*_2],\\|\mathcal{V}_{2,j}|>1}}\sum_{i\in\mathcal{V}_{2,j}}\exp(\beta_{0i})(\|\Delta\beta_{1ij}\|^2+\|\Delta\eta_{ij}\|^2+|\Delta\nu_{ij}|^2),
\end{align}
where we denote $\Delta\kappa_{ij}:=\kappa_{i}-\kj$, $\Delta\tau_{ij}:=\tau_{i}-\tj$, $\Delta\beta_{1ij}:=\beta_{1i}-\boj$, $\Delta\eta_{ij}:=\eta_{i}-\ej$, and $\Delta\nu_{ij}:=\nu_{i}-\nuj$. It is clear that convergence of the mixing measures in the $\mathcal{D}_1$ loss is equivalent to convergence of their respective parameters. Thus, though not a metric over mixing measures, the $\mathcal{D}_1$ loss can be used to characterize parameter and expert estimation rates.
\begin{theorem}
    \label{theorem:strongly_identifiable_experts}
    Assume that the expert functions $h_1$ and $h_2$ are strongly identifiable. Then, the lower bound $\bbE_X[V(f_{G_1,G_2}(\cdot|X),f_{G^*_1,G^*_2}(\cdot|X))]\gtrsim \mathcal{D}_1((G_1,G_2),(G^*_1,G^*_2))$ holds for all pairs of mixing measures $(G_1,G_2)\in\mathcal{G}_{k_1,k_2}(\Theta)$. As a consequence, we obtain 
    \begin{align}
        \label{eq:bound_D_1}
        \mathcal{D}_1((\widehat{G}^n_1,\widehat{G}^n_2),(G^*_1,G^*_2))=\mathcal{O}_P([\log(n)/n]^{\frac{1}{2}}).
    \end{align}
\end{theorem}
\noindent
The combination of Theorem~\ref{theorem:strongly_identifiable_experts}, whose proof is in Appendix~\ref{appendix:strongly_identifiable_experts}, and of the form of the loss $\mathcal{D}_1$ leads to various estimation rates. Below we say that a parameter is {\it exactly-specified} or {\it over-specified} depending on whether the associated Voronoi cell has one or more elements, respectively. 

\vspace{0.5em}
\emph{(i) Shared experts.} For shared experts, we see that  the estimation rate for exactly-specified parameters $\kappa^*_j$, $\tau^*_j$, 
is nearly parameteric, i.e. of order $\widetilde{\mathcal{O}}_P( n^{-1/2})$. On the other hand, over-specified parameters $\kappa^*_j$, $\tau^*_j$, 
admit slightly slower estimation rates, of order $\widetilde{\mathcal{O}}_P(n^{-1/4})$. As for the expert estimation rates, since the shared expert function $h_1(\cdot,\kappa)$ is Lipschitz continuous with respect to its paramter $\kappa$, we have that $|h_1(x,\hat{\kappa}^n_i)-h_1(x,\kappa^*_j)| \lesssim \|\hat{\kappa}^n_i-\kappa^*_j\|$ for almost every $x$.
It then follows that the estimation rates for exactly-specified and over-specified shared experts $h_1(x,\kappa^*_j)$ are also of the order $\widetilde{\mathcal{O}}_P(n^{-1/2})$ and $\widetilde{\mathcal{O}}_P(n^{-1/4})$, respectively. Thus, polynomially many data points of orders $\mathcal{O}(\epsilon^{-2})$ and $\mathcal{O}(\epsilon^{-4})$ are needed to estimate these experts within a error $\epsilon>0$.

\vspace{0.5em}
\emph{(ii) Routed experts.} Likewise, exactly-specified and over-specified parameters $\beta^*_{1j}$, $\eta^*_{j}$, $\nu^*_{j}$, for $j\in[k^*_2]$, have estimation rates of order $\widetilde{\mathcal{O}}_P(n^{-1/2})$ and $\widetilde{\mathcal{O}}_P(n^{-1/4})$, respectively. As the routed expert function $h_2(\cdot,\eta)$ is Lipschitz continuous, we deduce that the rates for estimating routed experts $h_2(x,\eta^*_j)$ also vary between $\widetilde{\mathcal{O}}_P(n^{-1/2})$ and $\widetilde{\mathcal{O}}_P(n^{-1/4})$ depending on the cardinality of the corresponding Voronoi cell $\mathcal{V}_{2,j}$. In summary, when both shared and routed expert functions are strongly identifiable, they enjoy the same estimation rates. 
\subsection{Linear Experts}
\label{sec:linear_experts}
In this section, we consider linear expert functions of the form $h_1(x,(\kappa_1,\kappa_0)):=\kappa_1^{\top}x+\kappa_0$ and $h_2(x,(\eta_1,\eta_0)):=\eta_1^{\top}x+\eta_0$. 
Then, 
the pair of ground-truth mixing measures $(G^*_1,G^*_2)$ become $G^*_1 : = \sum_{i=1}^{k^*_1}\omega^*_i\delta_{(\kappa^*_{1i},\kappa^*_{0i},\tau^*_{i})}$ and $G^*_2:=\sum_{i = 1}^{k^*_2} \exp(\beta_{0i}^{*}) \delta_{(\beta_{1i}^{*}, \eta_{1i}^{*}, \eta^*_{0i}, \nu_{i}^{*})}$. 

\vspace{0.5em}
\noindent
\textbf{Parameter interaction issues.} As discussed below Definition~\ref{def:strong_identifiability}, 
linear experts violate the strong identifiability condition due to the PDEs
\begin{align*}
    \frac{\partial h_1}{\partial \kappa_0}(x,(\kappa_1,\kappa_0))\cdot\frac{\partial h_1}{\partial \kappa_0}(x,(\kappa_1,\kappa_0))=1, \quad \frac{\partial h_2}{\partial \eta_1}(x,(\eta_1,\eta_0))=x\cdot\frac{\partial h_2}{\partial \eta_0}(x,(\eta_1,\eta_0)).
\end{align*}
Furthermore, these PDEs lead to linear dependencies among the partial derivatives of the Gaussian density function $\pi(y|h_1(x,(\kappa_1,\kappa_0)),\tau)$ and of the function $F(y|x;\beta_1,\eta,\nu)=\exp(\beta_1^{\top}x)\pi(y|h_2(x,(\eta_1,\eta_0)),\nu)$, given by
\begin{align*}
    \frac{\partial^2\pi}{\partial \kappa_0^2}(y|h_1(x,(\kappa_1,\kappa_0)),\tau)&=2\frac{\partial\pi}{\partial\tau}(y|h_1(x,(\kappa_1,\kappa_0)),\tau),\\
    \frac{\partial F}{\partial \eta_1}(y|x;\beta_1,\eta,\nu)&=\frac{\partial^2F}{\partial \beta_1\partial \eta_0}(y|x;\beta_1,\eta,\nu). 
\end{align*}
These delicate relationships, which can be intuitively interpreted as interactions between the parameters $\kappa_0$ and $\tau$, and among the parameters $\eta_1$, $\beta_1$ and $\eta_0$,
negatively affect the parameter and expert estimation rates. To overcome this issue, we consider instead a new Voronoi loss, given by
\begin{align}
    \label{eq:loss_2}
    &\mathcal{D}_{2}((G_1,G_2),(G^*_1,G^*_2)):=\sum_{j=1}^{k^*_1}\Big|\sum_{i\in\mathcal{V}_{1,j}}\omega_{i}-\oj\Big|+\sum_{j=1}^{k^*_2}\Big|\sum_{i\in\mathcal{V}_{2,j}}\exp(\beta_{0i})-\exp(\bzj)\Big|\nonumber\\
    &+\sum_{\substack{j\in[k^*_1],\\|\mathcal{V}_{1,j}|=1}}\sum_{i\in\mathcal{V}_{1,j}}\omega_{i}(\|\Delta\kappa_{1ij}\|+|\Delta\kappa_{0ij}|+|\Delta\tau_{ij}|)+\sum_{\substack{j\in[k^*_1],\\|\mathcal{V}_{1,j}|>1}}\sum_{i\in\mathcal{V}_{1,j}}\omega_{i}(\|\Delta\kappa_{ij}\|^2+|\Delta\kappa_{0ij}|^{r_{1,j}}+|\Delta\tau_{ij}|^{r_{1,j}/2})\nonumber\\
    &+\sum_{\substack{j\in[k^*_2]:|\mathcal{V}_{2,j}|=1}}\sum_{i\in\mathcal{V}_{2,j}}\exp(\beta_{0i})(\|\Delta\beta_{1ij}\|+\|\Delta\eta_{1ij}\|+|\Delta\eta_{0ij}|+|\Delta\nu_{ij}|)\nonumber\\
    &+\sum_{\substack{j\in[k^*_2]:|\mathcal{V}_{2,j}|>1}}\sum_{i\in\mathcal{V}_{2,j}}\exp(\beta_{0i})(\|\Delta\beta_{1ij}\|^{r_{2,j}}+\|\Delta\eta_{1ij}\|^{r_{2,j}/2}+|\Delta\eta_{0ij}|^{r_{2,j}}+|\Delta\nu_{ij}|^{r_{2,j}/2}),
\end{align}
where we denote $\Delta\kappa_{1ij}:=\kappa_{1i}-\kappa^*_{1j}$, $\Delta\kappa_{0ij}:=\kappa_{0i}-\kappa^*_{0j}$, $\Delta\eta_{1ij}:=\eta_{1i}-\eta^*_{1j}$ and $\Delta\eta_{0ij}:=\eta_{0i}-\eta^*_{0j}$. In addition, we define $r_{1,j}:=r_{1}(|\mathcal{V}_{1,j}|)$ and $r_{2,j}:=r_{2}(|\mathcal{V}_{2,j}|)$, where the functions $r_1$ and $r_2$ stand for the solvability of polynomial equation systems specified in Appendix~\ref{appendix:system}. 
In particular, we have $r_1(2)=r_2(2)=4$, $r_1(3)=r_2(3)=6$, and $r_1(m),r_2(m)\geq 7$ for all $m\geq 4$. Intuitively, these functions are involved to capture the parameter interactions expressed in the language of PDEs.
\begin{theorem}
    \label{theorem:linear_experts}
   Assume that the expert functions $h_1$ and $h_2$ take linear forms. Then, the lower bound $\bbE_X[V(f_{G_1,G_2}(\cdot|X),f_{G^*_1,G^*_2}(\cdot|X))]\gtrsim \mathcal{D}_2((G_1,G_2),(G^*_1,G^*_2))$ holds for any pair of mixing measures $(G_1,G_2)\in\mathcal{G}_{k_1,k_2}(\Theta)$. As a consequence, we obtain
    \begin{align}
        \mathcal{D}_2(\widehat{G}^n_1,\widehat{G}^n_2),(G^*_1,G^*_2))=\mathcal{O}_P([\log(n)/n]^{\frac{1}{2}}).
    \end{align}    
\end{theorem}
\noindent
By comparing the  Voronoi losses $\mathcal{D}_1$ and $\mathcal{D}_2$, we see that the estimation rates for exactly-specified shared and routed experts remain of parametric order $\widetilde{\mathcal{O}}_P(n^{-1/2})$. By contrast, there are changes in the estimation rates for the over-specified experts.

\vspace{0.5em}
\emph{(i) Shared experts.} The estimation rates for over-specified parameters $\kappa^*_{1j}$, $\kappa^*_{0j}$, $\tau^*_{j}$ are heterogeneous, of orders $\widetilde{\mathcal{O}}_P(n^{-1/4})$, $\widetilde{\mathcal{O}}_P(n^{-1/2r_{1,j}})$, $\widetilde{\mathcal{O}}_P(n^{-1/r_{1,j}})$, respectively. Since the input space is bounded, we have $|(\hat{\kappa}^n_{1i})^{\top}x+\hat{\kappa}^n_{0i}-(\kappa^*_{1j})^{\top}x-\kappa^*_{0j}|\lesssim\|\hat{\kappa}^n_{1i}-\kappa^*_{1j}\|+|\hat{\kappa}^n_{0i}-\kappa^*_{0j}|$. Then, it follows that the shared experts $(\kappa^*_{1j})^{\top}x+\kappa^*_{0j}$ admit estimation rates of orders $\widetilde{\mathcal{O}}_P(n^{-1/2r_{1,j}})$. However, note that the rates for estimating their input-dependent terms $(\kappa^*_{1j})^{\top}x$ are much faster, of order $\widetilde{\mathcal{O}}_P(n^{-1/4})$.

\vspace{0.5em}
\emph{(ii) Routed experts.} The estimation rates for over-specified parameters $\eta^*_{1j}$, $\nu^*_{j}$ are of orders $\widetilde{\mathcal{O}}_P(n^{-1/r_{2,j}})$, while those for $\beta^*_{1j}$, $\eta^*_{0j}$ are slower, of orders $\widetilde{\mathcal{O}}_P(n^{-1/2r_{2,j}})$. By arguing similarly to the case of shared experts, the rates for estimating the routed experts $(\eta^*_{1j})^{\top}x+\eta^*_{0j}$ and their input-dependent terms $(\eta^*_{1j})^{\top}x$ depend on the parameter $r_2$ (related to the solvability of a certain system of polynomial equations) and are  of orders $\widetilde{\mathcal{O}}_P(n^{-1/2r_{2,j}})$ and $\widetilde{\mathcal{O}}_P(n^{-1/r_{2,j}})$, respectively. Notably, these rates become increasingly slow with the cardinality of the corresponding Voronoi cell $\mathcal{V}_{2,j}$. In particular, when $|\mathcal{V}_{2,j}|=3$, they become $\widetilde{\mathcal{O}}_P(n^{-1/12})$ and $\widetilde{\mathcal{O}}_P(n^{-1/6})$, respectively.

\vspace{0.5em}
\emph{(iii) Sample complexity of estimating experts when using the shared expert strategy.} From the above observations, we see that shared experts have faster estimation rates than routed experts, i.e., $\widetilde{\mathcal{O}}_P(n^{-1/4})$ compared to $\widetilde{\mathcal{O}}_P(n^{-1/r_{2,j}})$. Furthermore, the estimation rates for shared experts in DeepSeekMoE are also faster than those for experts in MoE models without the shared expert strategy \cite{nguyen2023demystifying}, which are also of the order $\widetilde{\mathcal{O}}_P(n^{-1/r_{2,j}})$. As a result, for a given approximation error $\epsilon>0$, we only need $\mathcal{O}(\epsilon^{-4})$ data points to estimate shared experts, while the number of data points required to estimate either routed experts or experts in the standard MoE is of order $\mathcal{O}(\epsilon^{-r_{2,j}})$, which would become significantly large of order $\mathcal{O}(\epsilon^{-12})$ when the corresponding Voronoi cell $\mathcal{V}_{2,j}$ contains 3 elements. The punchline is that fewer data points are needed to estimate shared experts. 

\section{On Normalized Sigmoid Gating}
\label{sec:normalized_sigmoid_gating}
In this section, we conduct a convergence analysis of expert estimation in DeepSeek-V3's MoE to investigate the sample efficiency of the normalized sigmoid gating. For that purpose, we will reuse the problem setting for analyzing the shared expert strategy in Section~\ref{sec:shared_expert_strategy}, which is formally stated as follows.

\vspace{0.5em}
\noindent
\textbf{Problem setting.} Assume that $(X_{1}, Y_{1}), (X_{2}, Y_{2}), \ldots, (X_{n}, Y_{n}) \in \mathbb{R}^{d} \times \mathbb{R}$ are i.i.d. samples drawn from the Gaussian DeepSeek-V3's MoE whose conditional density function $g_{G_{*}}(y|x)$ is given by:
\begin{align}
    \label{eq:density_sigmoid}
    g_{G^*_1,G^*_2}(y|x):=\frac{1}{2}&\sum_{i=1}^{k^*_1}\omega^*_{i}\pi(y|h_1(x,\kappa^*_{i}),\tau^*_{i})\nonumber\\
    &+\frac{1}{2}\sum_{i = 1}^{k^*_2} \frac{\sigma((\beta_{1i}^{*})^{\top} x + \beta_{0i}^{*})}{\sum_{j = 1}^{k^*_2} \sigma((\beta_{1j}^{*})^{\top} x + \beta_{0j}^{*})}\cdot \pi(y|h_2(x,\eta^*_{i}), \nu_{i}^{*}), 
\end{align}
where $\sigma:\mathbb{R}\to(0,\infty)$ stands for the sigmoid function, that is, $\sigma(z):=\frac{1}{1+\exp(-z)}$, for all $z\in\mathbb{R}$. By abuse of notations, we define the pair of ground-truth mixing measures $(G^*_1,G^*_2)$ under this setting as $G^*_1 : = \sum_{i=1}^{k^*_1}\omega^*_i\delta_{(\kappa^*_{i},\tau^*_{i})}$ and $G^*_2:=\sum_{i = 1}^{k^*_2} \sigma(\beta_{0i}^{*}) \delta_{(\beta_{1i}^{*}, \eta_{i}^{*}, \nu_{i}^{*})}$. Here, we still leverage all the assumptions presented in Section~\ref{sec:shared_expert_strategy} for this analysis.

\vspace{0.5em}
\noindent
\textbf{Maximum likelihood estimation (MLE).} Under the above setting, the MLE defined in equation~\eqref{eq:MLE} is rewritten as 
\begin{align}
    \label{eq:MLE_sigmoid}
    (\widetilde{G}^n_1,\widetilde{G}^n_2)\in\argmax_{(G_1,G_2)\in\mathcal{G}_{k_1,k_2}(\Theta)}\frac{1}{n}\sum_{i=1}^{n}\log(g_{G_1,G_2}(Y_i|X_i)),
\end{align}
where $\mathcal{G}_{k_1,k_2}(\Theta):=\mathcal{G}_{k_1}(\Theta_1)\times\mathcal{G}_{k_2}(\Theta)$ denotes the set of mixing measure pairs $(G_1,G_2)$ with at most $k_1$ and $k_2$ atoms, respectively, that is,
\begin{align*}
    \mathcal{G}_{k_1}(\Theta_1)&:=\Big\{G_1=\sum_{i=1}^{k'_1}\omega_i\delta_{(\kappa_i,\tau_i)}:1\leq k'_1\leq k_1\Big\},\\
    \mathcal{G}_{k_2}(\Theta_2)&:=\Big\{G_2=\sum_{i=1}^{k'_2}\sigma(\beta_{0i})\delta_{(\beta_{1i},\eta^*_{i},\nu^*_{i})}:1\leq k'_2\leq k_2\Big\}.
\end{align*}
Given the MLE $(\widetilde{G}^n_1,\widetilde{G}^n_2)$ in equation~\eqref{eq:MLE_sigmoid}, we proceed to establish the convergence rate of density estimation $g_{\widetilde{G}^n_1,\widetilde{G}^n_2}$. However, there are some changes in the gating convergence behavior compared to that in DeepSeekMoE due to the structure of the sigmoid function.

\vspace{0.5em}
\noindent
\textbf{The convergence behavior of normalized sigmoid gating.} Recall that  we fit the ground-truth DeepSeek-V3's MoE model~\eqref{eq:density_sigmoid} with a mixture of $k_1>k^*_1$ shared experts and $k_2>k^*_2$ routed experts. Then, there must be some gorund-truth routed experts approximated by more than one fitted routed experts. As a result, the sum of weights of these fitted routed experts is expected to converge to the weight of the ground-truth routed experts, for example,
\begin{align}
     \label{eq:gating_convergence}
    \sum_{i=1}^{2}\frac{\sigma((\hat{\beta}^n_{1i})^{\top}x+\hat{\beta}^n_{0i})}{\sum_{j=1}^{k^n_2}\sigma((\hat{\beta}^n_{1j})^{\top}x+\hat{\beta}^n_{0j})}\to\frac{\sigma((\beta^*_{11})^{\top}x+\beta^*_{01})}{\sum_{j=1}^{k^*_2}\sigma((\beta^*_{1j})^{\top}x+\beta^*_{0j})},
\end{align}
for almost every $x$. Since the denominator $\sum_{j=1}^{k^n_2}\sigma((\hat{\beta}^n_{1j})^{\top}x+\hat{\beta}^n_{0j})$ should converge to its counterpart $\sum_{j=1}^{k^*_2}\sigma((\beta^*_{1j})^{\top}x+\beta^*_{0j})$. Then, it must hold that 
\begin{align*}
    \sum_{i=1}^{2}\sigma((\hat{\beta}^n_{1i})^{\top}x+\hat{\beta}^n_{0i})\to\sigma((\beta^*_{11})^{\top}x+\beta^*_{01}),
\end{align*}
as $n\to\infty$, for almost every $x$. This result occurs only if $\beta^*_{11}= 0_{d}$. Therefore, we will divide our analysis into two complement regimes for the over-specified parameters $\beta^*_{1i}$:
\begin{itemize}
    \item \textbf{Sparse regime.} All over-specified parameters $\beta^*_{1i}$ equal zero vector; 
    \item \textbf{Dense regime.} Not all over-specified parameters $\beta^*_{1i}$ equal zero vector.
\end{itemize}
Note that the sparse regime of parameters rarely occurs in practice. This is because when all the over-specified parameters $\beta^*_{1i}$ vanish, the corresponding mixture weights become static (input-independent) rather than dynamic (input-dependent) as in the concept of MoE.
However, for completeness, we will perform the convergence analysis of expert estimation under both the sparse and dense regimes in Section~\ref{appendix:sparse_regime} and Section~\ref{appendix:dense_regime}, respectively.
\subsection{Sparse Regime}
\label{appendix:sparse_regime}
To begin with, let us derive the density estimation rate for the sparse regime in Proposition~\ref{prop:density_rate_sigmoid}.
\begin{proposition}
    \label{prop:density_rate_sigmoid}
    Under the sparse regime, the density estimation $g_{\widetilde{G}^n_1,\widetilde{G}^n_2}(Y|X)$ converges to the true density $g_{G^*_1,G^*_2}(Y|X)$ at the following rate:
    \begin{align*}
        \bbE_X[V(g_{\widetilde{G}^n_1,\widetilde{G}^n_2}(\cdot|X),g_{G^*_1,G^*_2}(\cdot|X))]=\mathcal{O}_P([\log(n)/n]^{\frac{1}{2}}).
    \end{align*}
\end{proposition}
\noindent
Since the sigmoid function is Lipschitz continuous, the proof of this proposition can be done similarly to that of Proposition~\ref{prop:density_rate}, which is provided in Appendix~\ref{appendix:density_rate}. The result of Proposition~\ref{prop:density_rate_sigmoid} indicates that the density estimation $g_{\widetilde{G}^n_1,\widetilde{G}^n_2}$ converges to the ground-truth density $g_{G^*_1,G^*_2}$ under the Total Variation distance at the parametric rate of order $\widetilde{\mathcal{O}}_P(n^{-1/2})$. 

\vspace{0.5em}
\noindent
\textbf{Voronoi loss.} Next, we construct Voronoi loss tailored to the sparse regime as 
\begin{align}
    \label{eq:loss_3}
    &\mathcal{D}_{3}((G_1,G_2),(G^*_1,G^*_2)):=\sum_{j=1}^{k^*_1}\Big|\sum_{i\in\mathcal{V}_{1,j}}\omega_{i}-\oj\Big|+\sum_{j\in[k^*_2]:|\mathcal{V}_{2,j}|>1}\Big|\sum_{i\in\mathcal{V}_{2,j}}\sigma(\beta_{0i})-\sigma(\bzj)\Big|\nonumber\\
    &+\sum_{\substack{j\in[k^*_1],\\|\mathcal{V}_{1,j}|=1}}\sum_{i\in\mathcal{V}_{1,j}}\omega_{i}(\|\Delta\kappa_{ij}\|+|\Delta\tau_{ij}|)+\sum_{\substack{j\in[k^*_2],\\|\mathcal{V}_{2,j}|=1}}\sum_{i\in\mathcal{V}_{2,j}}(\|\Delta\beta_{1ij}\|+|\Delta\beta_{0ij}|+\|\Delta\eta_{ij}\|+|\Delta\nu_{ij}|)\nonumber\\
    &+\sum_{\substack{j\in[k^*_1],\\|\mathcal{V}_{1,j}|>1}}\sum_{i\in\mathcal{V}_{1,j}}\omega_{i}(\|\Delta\kappa_{ij}\|^2+|\Delta\tau_{ij}|^2)+\sum_{\substack{j\in[k^*_2],\\|\mathcal{V}_{2,j}|>1}}\sum_{i\in\mathcal{V}_{2,j}}(\|\Delta\beta_{1ij}\|^2+\|\Delta\eta_{ij}\|^2+|\Delta\nu_{ij}|^2),
\end{align}
where we denote $\Delta\beta_{0ij}:=\beta_{0i}-\bzj$. Given the above loss function, we are now able to capture parameter and expert estimation rates under the sparse regime in the following theorem.
\begin{theorem}
    \label{theorem:strongly_identifiable_experts_sigmoid}
    Suppose that the expert functions $h_1$ and $h_2$ are strongly identifiable. Then, the lower bound $\bbE_X[V(g_{G_1,G_2}(\cdot|X),g_{G^*_1,G^*_2}(\cdot|X))]\gtrsim \mathcal{D}_3((G_1,G_2),(G^*_1,G^*_2))$ holds for any $(G_1,G_2)\in\mathcal{G}_{k_1,k_2}(\Theta)$. As a consequence, we have
    \begin{align*}
        \mathcal{D}_3(\widetilde{G}^n_1,\widetilde{G}^n_2),(G^*_1,G^*_2))=\mathcal{O}_P([\log(n)/n]^{\frac{1}{2}}).
    \end{align*}
\end{theorem}
\noindent
The proof of Theorem~\ref{theorem:strongly_identifiable_experts_sigmoid} is provided in Appendix~\ref{appendix:strongly_identifiable_experts_sigmoid}. From the formulations of Voronoi losses $\mathcal{D}_1$ and $\mathcal{D}_3$ in equations~\eqref{eq:loss_1} and \eqref{eq:loss_3}, respectively, we observe that shared experts and routed experts which satisfy the strong identifiability condition admit the same estimation rates as those in Theorem~\ref{theorem:strongly_identifiable_experts}. In particular, the rates for estimating both types of experts are of orders $\widetilde{\mathcal{O}}_P(n^{-1/2})$ and $\widetilde{\mathcal{O}}_P(n^{-1/4})$ when they are exactly-specified and over-specified, respectively. In other words, the normalized sigmoid gating does not have clear advantages over the standard softmax gating under the sparse regime. However, it should be noted that the sparse regime is less likely to occur in practice than the dense regime. Thus, we continue to compare the two gatings under the dense regime in the next section.

\subsection{Dense Regime}
\label{appendix:dense_regime}
Recall that under the dense regime, at least one among over-specified parameters $\beta^*_{1j}$ is non-zero. Then, the convergence of the normalized sigmoid gating in equation~\eqref{eq:gating_convergence}, which occurs only if $\beta^*_{1j}=0$, becomes invalid. Therefore, the ground-truth model is misspecified, that is, the density estimation $g_{\widetilde{G}^n_1,\widetilde{G}^n_2}$ converges to the missepcified density function $g_{G^*_1,\check{G}_2}$, where $\check{G}_2\in\overline{\mathcal{G}}_{k_2}(\Theta_2):=\argmin_{G_2\in\mathcal{G}_{k_2}(\Theta_2)\setminus\mathcal{G}_{k^*_2}(\Theta_2)}\mathrm{KL}(g_{G^*_1,G_2}\|g_{G^*_1,G^*_2})$ \cite{dwivedi2018misspecified}, rather than the ground-truth density $g_{G^*_1,G^*_2}$, where $\mathrm{KL}$ denotes the Kullback-Leibler divergence.
Following the result of Proposition~\ref{prop:density_rate_sigmoid}, we are also able to establish the parametric density estimation rate under the dense regime in the following corollary.
\begin{corollary}
    \label{corollary:over_regression_regime2}
     Under the dense regime, the density estimation $g_{\widetilde{G}^n_1,\widetilde{G}^n_2}$ converges to the density $g_{G^*_1,\check{G}_2}$ at the rate: $\inf_{\check{G}_2\in\overline{\mathcal{G}}_{k_2}(\Theta_2)}\bbE_X[V(g_{\widetilde{G}^n_1,\widetilde{G}^n_2}(\cdot|X),g_{G^*_1,\check{G}_2}(\cdot|X))]=\mathcal{O}_{P}([\log(n)/n]^{\frac{1}{2}})$.
\end{corollary}
\noindent
Subsequently, we focus on characterizing parameter and expert estimation rates under the dense regime by establishing the Total Variation lower bound 
\begin{align*}
    \inf_{(G^*_1,\check{G}_2)\in\overline{\mathcal{G}}_{k_1,k_2}(\Theta)}\bbE_X[V(g_{G_1,G_2}(\cdot|X),g_{G^*_1,\check{G}_2}(\cdot|X))]\gtrsim \mathcal{D}_4((G_1,G_2),(G^*_1,\check{G}_2)),
\end{align*}
where $\mathcal{D}_4$ is a Voronoi loss that will be defined later in equation~\eqref{eq:loss_4}. Recall that a key step in deriving this lower bound is to decompose the density difference $g_{\widetilde{G}^n_1,\widetilde{G}^n_2}(Y|X)-g_{G^*_1,\check{G}_2}(Y|X)$ into linearly independent terms using Taylor expansions to the functions $x\mapsto\pi(Y|h_1(x,\kappa),\tau)$ and $x\mapsto \sigma(\beta_1^{\top}x+\beta_0)\pi(Y|h_2(x,\eta),\nu)$ w.r.t their parameters $(\kappa,\tau)$ and $(\beta_1,\beta_0,\eta,\nu)$, respectively. Due to the gating change, it is necessary to introduce a new condition on the routed expert function $h_2$ to ensure linear independence among terms in the Taylor expansions. 
\begin{definition}[Weak Identifiability]
    \label{def:weak_identifiability}
    We say that a routed expert function $x\mapsto h_2(x,\eta)$ is weakly identifiable if it is differentiable w.r.t its parameter $\eta$, and if for any $k_2\geq 1$ and distinct parameters $\eta_1,\eta_2,\ldots,\eta_{k_2}$, the following set is linearly independent w.r.t $x$:
    \begin{align*}
        \Bigg\{\frac{\partial h_2}{\partial\eta^{(u_2)}}(x,\eta_i) :i\in[k_2], \ u_2\in[d_2]\Bigg\}.
    \end{align*}
\end{definition}
\noindent
\textbf{Examples.} 
It can be validated that even linear experts of the form $h_2(x,(\eta_1,\eta_0)):=\eta_1^{\top}x+\eta_0$ satisfy the weak identifiability condition. Note that the strong identifiability condition in Definition~\ref{def:strong_identifiability} implies the weak identifiability condition. Therefore,
two-layer FFNs $h_2(x,(\eta_2,\eta_1,\eta_0)):=\eta_2\mathrm{GELU}(\eta_1^{\top}x+\eta_0)$ are also weakly identifiable. On the other hand, input-free experts $h_2(x,\eta)=c(\eta)$ does not meet the weak identifiability condition.

\vspace{0.5em}
\noindent
\textbf{Voronoi loss.} Now, we build a Voronoi loss to capture parameter estimation rates under the dense regime, which is given by
\begin{align}
    \label{eq:loss_4}
    &\mathcal{D}_{4}((G_1,G_2),(G^*_1,\check{G}_2)):=\sum_{j=1}^{k^*_1}\Big|\sum_{i\in\mathcal{V}_{1,j}}\omega_{i}-\oj\Big|+\sum_{j\in[k^*_1]:|\mathcal{V}_{1,j}|=1}\sum_{i\in\mathcal{V}_{1,j}}\omega_{i}(\|\Delta\kappa_{ij}\|+|\Delta\tau_{ij}|)\nonumber\\
    &+\sum_{\substack{j\in[k^*_1]:|\mathcal{V}_{1,j}|>1}}\sum_{i\in\mathcal{V}_{1,j}}\omega_{i}(\|\Delta\kappa_{ij}\|^2+|\Delta\tau_{ij}|^2)+\sum_{j=1}^{k^*_2}\sum_{i\in\mathcal{V}_{2,j}}(\|\beta_{1i}-\check{\beta}_{1j}\|+|\beta_{0i}-\check{\beta}_{0j}|\nonumber\\
    &\hspace{9cm}+\|\eta_{i}-\check{\eta}_{j}\|+|\nu_{i}-\check{\nu}_{j}|).
\end{align}
Given the above loss, we are now ready to present results for the convergence rates of parameter estimation and expert estimation in Theorem~\ref{theorem:weakly_identifiable_experts_sigmoid}, whose proof can be found in Appendix~\ref{appendix:weakly_identifiable_experts_sigmoid}.
\begin{theorem}
    \label{theorem:weakly_identifiable_experts_sigmoid}
    Suppose that the shared expert function $h_1$ is strongly identifiable, while the routed expert function $h_2$ is weakly identifiable. Then, the lower bound
    \begin{align*}
        \inf_{(G^*_1,\check{G}_2)\in\overline{\mathcal{G}}_{k_1,k_2}(\Theta)}\bbE_X[V(g_{G_1,G_2}(\cdot|X),g_{G^*_1,\check{G}_2}(\cdot|X))]\gtrsim \mathcal{D}_4((G_1,G_2),(G^*_1,\check{G}_2))
    \end{align*}
    holds for any $(G_1,G_2)\in\mathcal{G}_{k_1,k_2}(\Theta)$. As a consequence, we have
    \begin{align*}
        \inf_{(G^*_1,\check{G}_2)\in\overline{\mathcal{G}}_{k_1,k_2}(\Theta)}\mathcal{D}_4(\widetilde{G}^n_1,\widetilde{G}^n_2),(G^*_1,\check{G}_2))=\mathcal{O}_P([\log(n)/n]^{\frac{1}{2}}).
    \end{align*}
\end{theorem}
\noindent
A few comments regarding the results of the above theorem are in order.

\vspace{0.5em}
\noindent
\emph{(i) Shared experts:} It can be seen from the formulation of the Voronoi loss $\mathcal{D}_{4}$ that the estimation rates for shared experts remain unchanged compared to those in Theorem~\ref{theorem:strongly_identifiable_experts_sigmoid}, which are of the orders $\widetilde{\mathcal{O}}_P(n^{-1/2})$ for exactly-specified ones and $\widetilde{\mathcal{O}}_P(n^{-1/4})$ for over-specified ones. However, there are changes in the estimation rates for routed experts.

\vspace{0.5em}
\noindent
\emph{(ii) Routed experts:} In particular, the convergence rates of parameter estimation $\widetilde{\eta}^n_{i}$ are of parametric order $\widetilde{\mathcal{O}}_P(n^{-1/2})$. Since the routed expert function $h_2(x,\eta)$ is Lipschitz continuous w.r.t its parameter $\eta$, then the rates for estimating both exactly-specified and over-specified routed experts are of order $\widetilde{\mathcal{O}}_P(n^{-1/2})$. These rates are substantially faster than those when using the standard softmax gating in Theorem~\ref{theorem:strongly_identifiable_experts} and Theorem~\ref{theorem:linear_experts}, which are of orders $\widetilde{\mathcal{O}}_P(n^{-1/4})$ and $\widetilde{\mathcal{O}}_P(n^{-1/r_2(|\mathcal{V}_{2,j}|)})$, respectively.

\vspace{0.5em}
\noindent
\emph{(iii) Sample efficiency of the normalized sigmoid gating:} As a result, when using the normalized sigmoid gating, then we need only $\mathcal{O}(\epsilon^{-2})$ to approximate routed experts with a given error $\epsilon$, even if they are of linear form. On the other hand, when using the softmax gating, it requires $\mathcal{O}(\epsilon^{-4})$ data points to estimate strongly identifiable experts. Furthermore, if the routed experts are of linear form, then we need $\mathcal{O}(\epsilon^{-r_2(|\mathcal{V}_{2,j}|)})$ data points to estimate, which is equivalent to $\mathcal{O}(\epsilon^{-12})$ when these routed experts have three fitted experts, that is, $|\mathcal{V}_{2,j}|=3$. Hence, the key finding is that when using the normalized sigmoid gating, fewer data points are needed to estimate routed experts.

\section{Experiments} \label{sec:experiments}
\vspace{-0.3em}

In this section, we empirically validate the theoretical findings in the previous section. Using synthetic data, we demonstrate the convergence behavior of the maximum likelihood estimator $(\widehat{G}^n_1,\widehat{G}^n_2)$  towards the true mixing measure $(G^*_1,G^*_2)$ (Section~\ref{subsec:numerical_exp}). In real-world scenarios, we evaluate our methodology on language modeling tasks using the SlimPajama corpus \cite{cerebras2023slimpajama} (Section~\ref{subsec:lm_modeling}), and extend our evaluation to vision-language modeling benchmarks using the LLaVA architecture \cite{liu2023llava} integrated within the LIBMoE framework \cite{nguyen2024libmoe} (Section~\ref{subsec:vlm_modeling}). Our empirical study compares four model configurations: Vanilla SMoE, DeepSeek-V3 (shared experts combined with normalized sigmoid gating), DeepSeek-V2 (shared experts with softmax routing), and SMoE Sigmoid Gating (normalized sigmoid gating without shared experts).

\subsection{Numerical Experiments}  \label{subsec:numerical_exp}

\begin{figure}[t!]
    \centering

  \begin{subfigure}{0.49\textwidth}
    \centering
    \includegraphics[width=\linewidth]{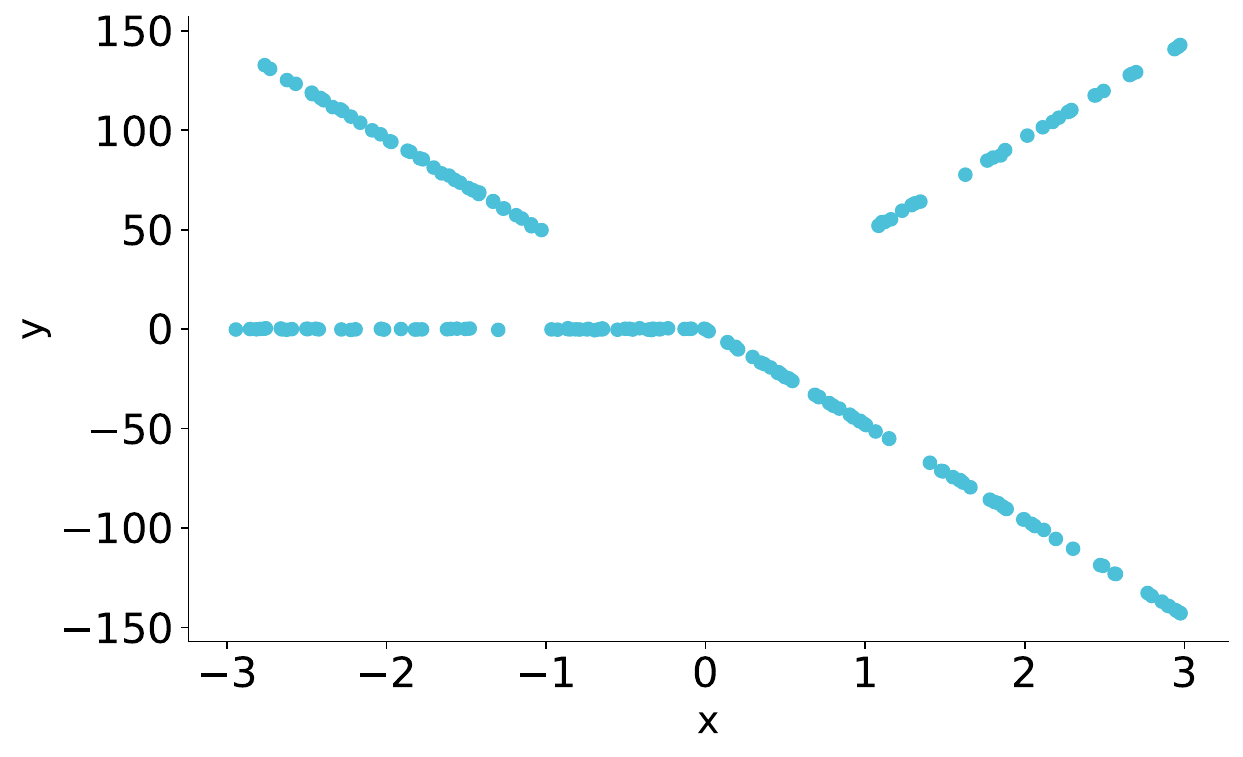}
    \caption{Theorem 1}
    \label{fig:data_viz_theorem1}
  \end{subfigure}
  \hfill
  \begin{subfigure}{0.49\textwidth}
    \centering
    \includegraphics[width=\linewidth]{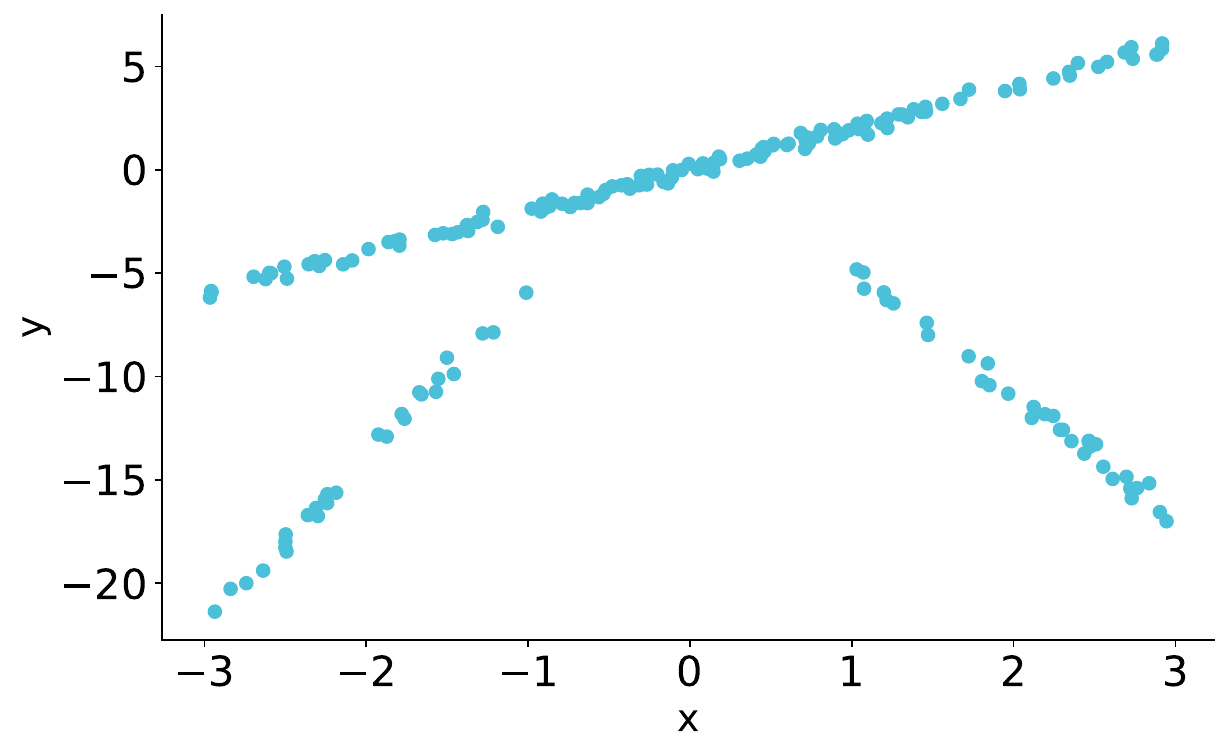}
    \caption{Theorem 2}
    \label{fig:data_viz_theorem2}
  \end{subfigure}

  \begin{subfigure}{0.49\textwidth}
    \centering
    \includegraphics[width=\linewidth]{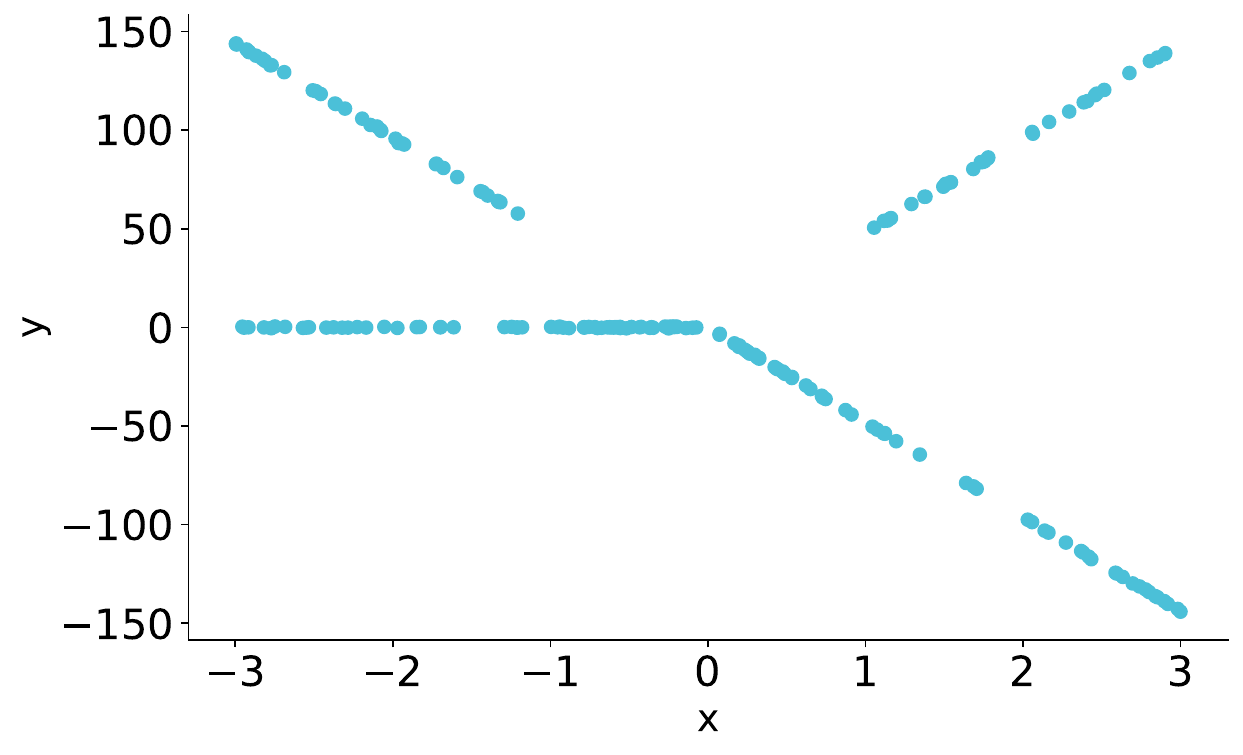}
    \caption{Theorem 3}
    \label{fig:data_viz_theorem3}
  \end{subfigure}
  \hfill
  \begin{subfigure}{0.49\textwidth}
    \centering
    \includegraphics[width=\linewidth]{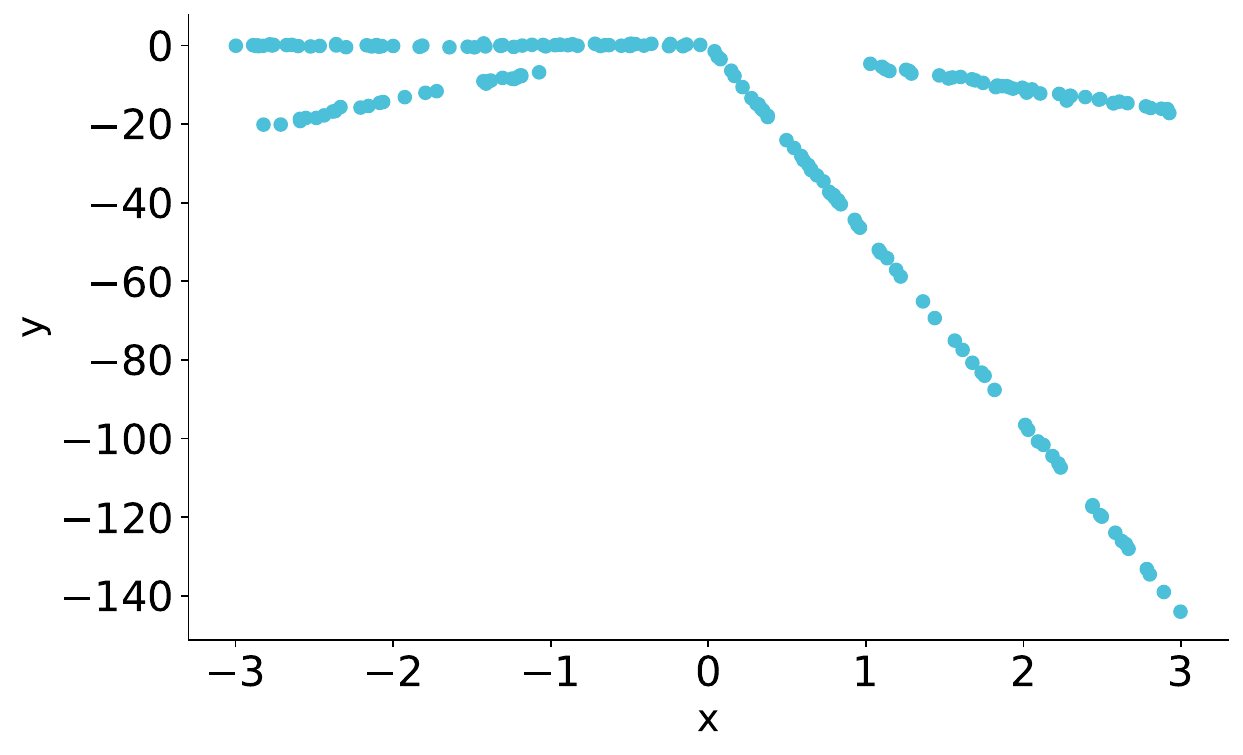}
    \caption{Theorem 4}
    \label{fig:data_viz_theorem4}
  \end{subfigure}

    \caption{Empirical illustration of the input - output relationship $(X, Y)$ under synthetic conditions for each theoretical result. Each subplot corresponds to a different theoretical setting: (a) Theorem 1, (b) Theorem 2, (c) Theorem 3, and (d) Theorem 4.}
    \label{fig:data_viz_numerical_exp}
\end{figure}


\subsubsection{Experimental Setup}

\textbf{Synthetic Data.} For each sample size $n$, we generate i.i.d samples $\{(X_i, Y_i)\}^n_{i=1}$ by first sampling $X_i$’s from the uniform distribution $\text{Uniform}[-3, 3]$ and then sampling $Y_i$’s from the true conditional density $f_{G^*_1,G^*_2}(Y|X)$ or $g_{G^*_1,G^*_2}(Y|X)$ of Gaussian mixture of experts (MoE) model setting of each theorem configuration. Figure~\ref{fig:data_viz_numerical_exp} shows the visualization of the relationship between $X$ and $Y$ in each experiment.

\vspace{0.5em}
\noindent
\textbf{Maximum Likelihood Estimation (MLE).} A popular approach to determining the MLE $(\widehat{G}^n_1,\widehat{G}^n_2)$ for each set of samples is to use the Expectation-Maximization (EM) algorithm \cite{dempster1977maximum}. However, since there are not any closed-form expressions for updating the gating parameters $\beta_{0i}$, $\beta_{1i}$ in the maximization steps, we have to leverage an EM-based numerical scheme, which was previously used in \cite{chamroukhi2009time}. We select the convergence criterion of $\epsilon = 10^{-6}$ and run a maximum of 1000 EM iterations.

\vspace{0.5em}
\noindent
\textbf{Experiment Design. }  Our empirical investigation systematically examines four experimental configurations, each precisely corresponding to the theoretical scenarios elaborated in our main paper. Each configuration includes 40 independent sample generations over a comprehensive range of sample sizes $n$, specifically $n \in [10^2, 10^5]$. To ensure consistency and comparative clarity across experiments, we uniformly adopt an architecture consisting of one shared expert ($k_1^* = 1$) complemented by two routed experts ($k_2^* = 2$), where we fit two shared experts ($k_1 = 2$) and three routed experts ($k_2 = 3$) in our experiment settings.

\subsubsection{Theorem 1}

The problem setting is defined in equation~\eqref{eq:density}, where we choose expert functions $h_1$ and $h_2$ to satisfy the strong identifiability condition, specifically $h_1(x,(\kappa_2,\kappa_1,\kappa_0)):=\kappa_2\mathrm{GELU}(\kappa_1^{\top}x+\kappa_0)$ and $h_2(x,(\eta_2,\eta_1)):=\eta_2\mathrm{GELU}(\eta_1^{\top}x)$.  The ground-truth parameters employed in our experiments are presented as follows:

\begin{align*}
\omega^* &= 1.0, & \kappa_0^* &= 0, & \kappa_1^* &= 6, & \kappa_2^* &= -8, & \tau^* &= 0.25, \\
\beta_{01}^* &= -0.5, & \beta_{11}^* &= 5, & \eta_{11}^* &= -12, & \eta_{21}^* &= 4, & \nu_1^* &= 0.4, \\
\beta_{02}^* &= 0.5, & \beta_{12}^* &= 5, & \eta_{12}^* &= 12, & \eta_{22}^* &= 4, & \nu_2^* &= 0.4, \\
\end{align*}

\noindent
As illustrated in Figure~\ref{fig:numerical_exp_theorem1}, the MLE $(\widehat{G}^n_1,\widehat{G}^n_2)$  exhibits empirical convergence to the ground-truth counterpart $(G^*_1,G^*_2)$ under the Voronoi metric $\mathcal{D}_1$ (see equation~\eqref{eq:loss_1}) at the rate of order $\mathcal{O}(n^{-0.45})$. This empirically observed rate closely matches the theoretical convergence rate of order $\mathcal{O}_P([\log(n)/n]^{1 / 2})$ established in Theorem~\ref{theorem:strongly_identifiable_experts}, thus validating our theoretical results under the assumptions of Theorem~\ref{theorem:strongly_identifiable_experts}.

\begin{figure}[t!]
    \centering

  \begin{subfigure}{0.49\textwidth}
    \centering
    \includegraphics[width=\linewidth]{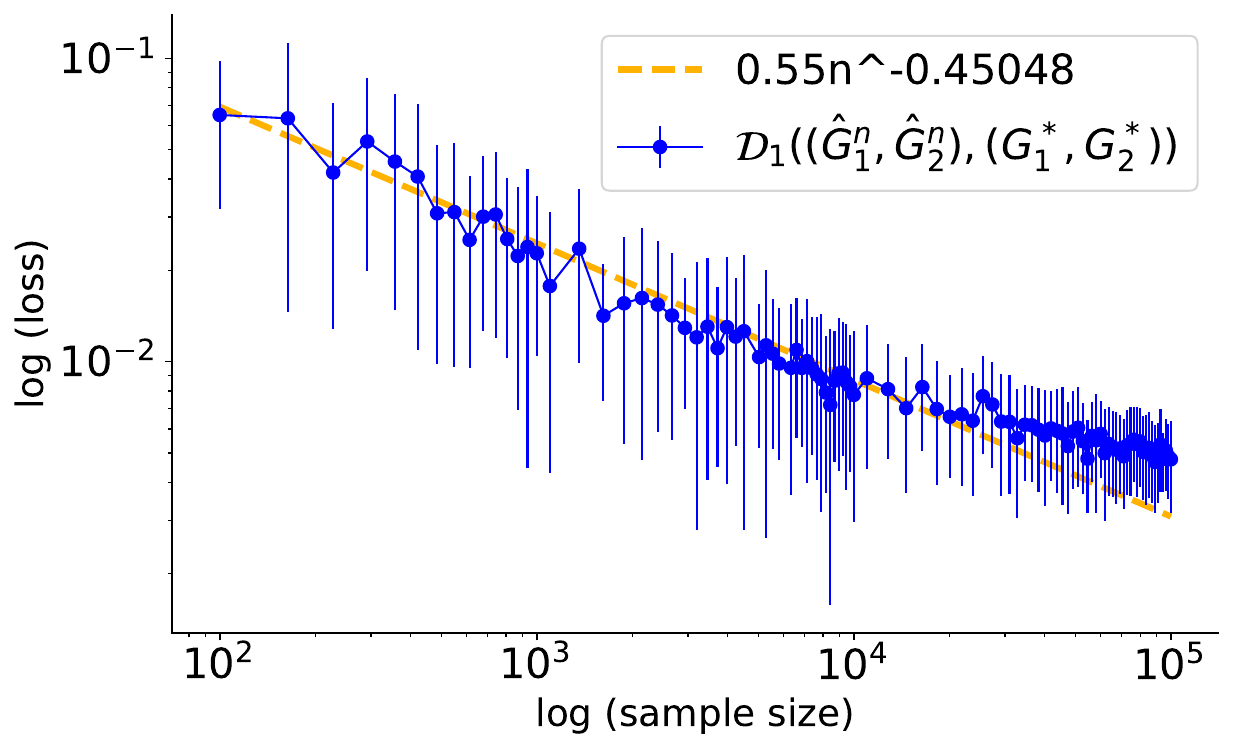}
    \caption{Theorem 1}
    \label{fig:numerical_exp_theorem1}
  \end{subfigure}
  \hfill
  \begin{subfigure}{0.49\textwidth}
    \centering
    \includegraphics[width=\linewidth]{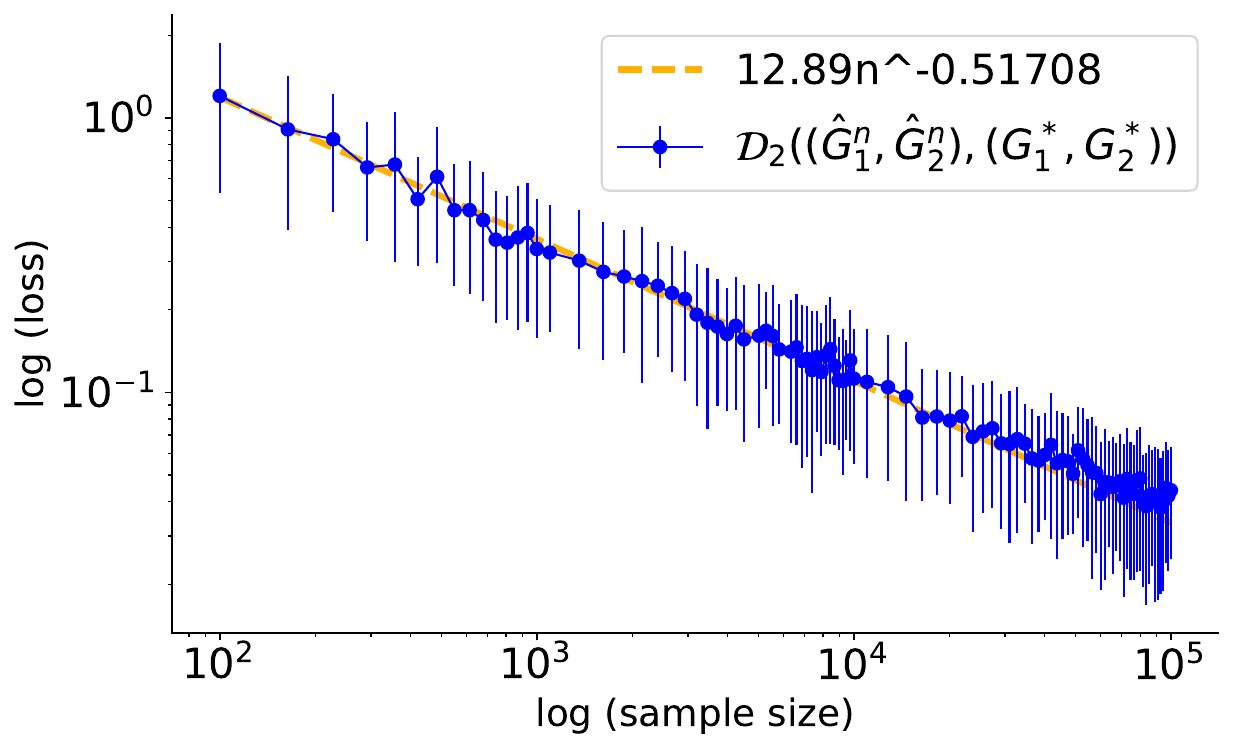}
    \caption{Theorem 2}
    \label{fig:numerical_exp_theorem2}
  \end{subfigure}

  \begin{subfigure}{0.49\textwidth}
    \centering
    \includegraphics[width=\linewidth]{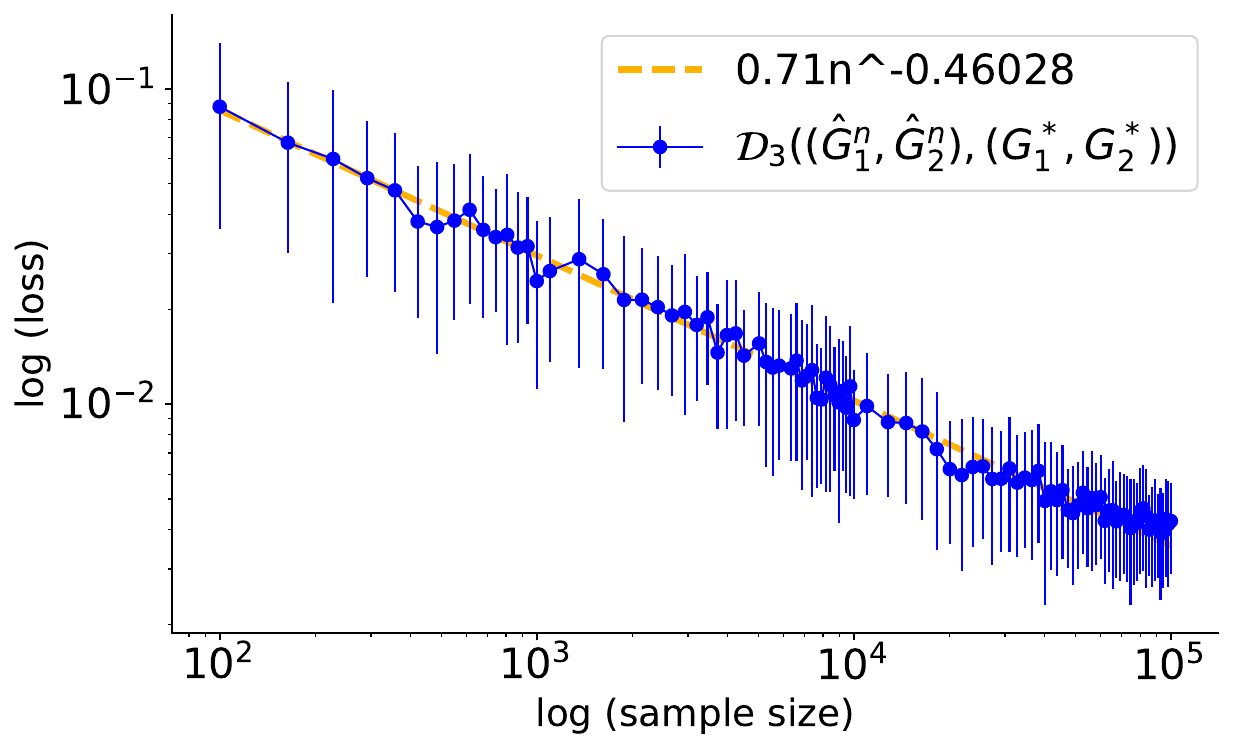}
    \caption{Theorem 3}
    \label{fig:numerical_exp_theorem3}
  \end{subfigure}
  \hfill
  \begin{subfigure}{0.49\textwidth}
    \centering
    \includegraphics[width=\linewidth]{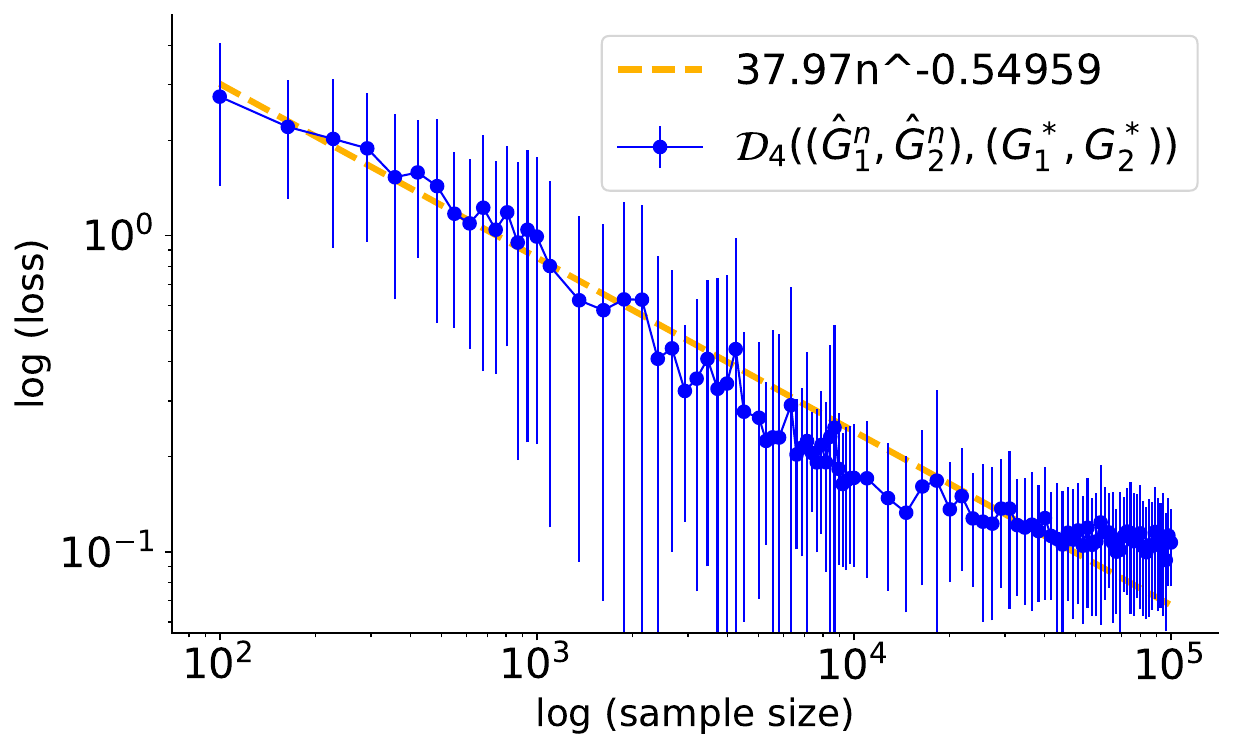}
    \caption{Theorem 4}
    \label{fig:numerical_exp_theorem4}
  \end{subfigure}

    \caption{Log-log scaled plots illustrating simulation results with different model settings. The blue curves depict the mean discrepancy between the MLE $(\widehat{G}^n_1,\widehat{G}^n_2)$  and the true mixing measure $(G^*_1,G^*_2)$ accompanied by error bars representing the standard deviation over 40 times of experiments for each sample size $n$. Additionally, an orange dash-dotted line represents the least-squares fitted linear regression line for these data points.}
    \label{fig:numerical_exp}
\end{figure}

\subsubsection{Theorem 2}

In this experiment, we consider expert functions $h_1$ and $h_2$ of linear forms as in Section~\ref{sec:linear_experts}. Specifically, we set $h_1(x,(\kappa_1,\kappa_0)):=\kappa_1^{\top}x+\kappa_0$ and $h_2(x,(\eta_1,\eta_0)):=\eta_1^{\top}x+\eta_0$, with the associated ground-truth parameters defined as follows:

\begin{align*}
\omega^* &= 1.0, & \kappa_0^* &= 0, & \kappa_1^* &= 2, & \tau^* &= 0.2, \\
\beta_{01}^* &= -0.5, & \beta_{11}^* &= 5, & \eta_{11}^* &= 8, & \eta_{01}^* &= 2, & \nu_1^* &= 0.4, \\
\beta_{02}^* &= 0.5, & \beta_{12}^* &= 5, & \eta_{12}^* &= -6, & \eta_{02}^* &= 1, & \nu_2^* &= 0.4, \\
\end{align*}

\noindent
The result is shown in Figure~\ref{fig:numerical_exp_theorem2}. Under the linear expert setting and the Voronoi loss $\mathcal{D}_2$ (see equation~\eqref{eq:loss_2}), the MLE admits the empirical convergence rate of order $\mathcal{O}(n^{-0.517})$, which totally aligns with the theoretical rate $\mathcal{O}_P([\log(n)/n]^{1 / 2})$ in Theorem~\ref{theorem:linear_experts}.

\subsubsection{Theorem 3}

This experiment is designed to empirically validate Theorem~\ref{theorem:strongly_identifiable_experts_sigmoid} under the setting of normalized sigmoid gating in Section~\ref{sec:normalized_sigmoid_gating}. Under the sparse regime, we set all over-specified parameters $\beta^*_{1i}$ to be zero vectors. Here, we select strongly identifiable experts $h_1(x,(\kappa_2,\kappa_1,\kappa_0)):=\kappa_2\mathrm{GELU}(\kappa_1^{\top}x+\kappa_0)$ and $h_2(x,(\eta_2,\eta_1)):=\eta_2\mathrm{GELU}(\eta_1^{\top}x)$ as required in Theorem~\ref{theorem:strongly_identifiable_experts_sigmoid}. The complete set of ground-truth parameters used in this experiment is given by

\begin{align*}
\omega^* &= 1.0, & \kappa_0^* &= 0, & \kappa_1^* &= 6, & \kappa_2^* &= -8, & \tau^* &= 0.25, \\
\beta_{01}^* &= -0.5, & \beta_{11}^* &= 0, & \eta_{11}^* &= -12, & \eta_{21}^* &= 4, & \nu_1^* &= 0.4, \\
\beta_{02}^* &= 0.5, & \beta_{12}^* &= 0, & \eta_{12}^* &= 12, & \eta_{22}^* &= 4, & \nu_2^* &= 0.4, \\
\end{align*}


\noindent
Figure~\ref{fig:numerical_exp_theorem3} presents the experimental results for the convergence analysis under the sparse regime utilizing normalized sigmoid gating. The MLE $(\widehat{G}^n_1,\widehat{G}^n_2)$  empirically converges to the true mixing measure $(G^*_1,G^*_2)$ at a rate of $\mathcal{O}(n^{-0.46})$ under the Voronoi loss $\mathcal{D}_3$ (see equation~\eqref{eq:loss_3}). This empirical convergence rate is closely aligned with the theoretical rate of order $\mathcal{O}_P([\log(n)/n]^{1/2})$  presented in Theorem~\ref{theorem:strongly_identifiable_experts_sigmoid}. 

\subsubsection{Theorem 4}

In this experiment, we adopt the same problem setting of Theorem~\ref{theorem:weakly_identifiable_experts_sigmoid}. With the normalized sigmoid gating under the dense regime, we chosse the shared expert function $h_1$ to be strongly identifiable, while the routed expert function $h_2$ is weakly identifiable. Specifically, we set $h_1(x,(\kappa_2,\kappa_1,\kappa_0)):=\kappa_2\mathrm{GELU}(\kappa_1^{\top}x+\kappa_0)$ and $h_2(x,(\eta_1,\eta_0)):=\eta_1^{\top}x+\eta_0$. The complete set of ground-truth parameters used in this experiment is detailed below:

\begin{align*}
\omega^* &= 1.0, & \kappa_0^* &= 0, & \kappa_1^* &= 6, & \kappa_2^* &= -8, & \tau^* &= 0.25, \\
\beta_{01}^* &= -0.5, & \beta_{11}^* &= 5, & \eta_{11}^* &= 8, & \eta_{01}^* &= 2, & \nu_1^* &= 0.4, \\
\beta_{02}^* &= 0.5, & \beta_{12}^* &= 5, & \eta_{12}^* &= -6, & \eta_{02}^* &= 1, & \nu_2^* &= 0.4, \\
\end{align*}

\noindent
Figure~\ref{fig:numerical_exp_theorem4} presents the numerical results corresponding to Theorem~\ref{theorem:weakly_identifiable_experts_sigmoid}. Under the dense regime, the MLE achieves an empirical convergence rate of order $\mathcal{O}(n^{-0.55})$, which matches the theoretical rate of order $\mathcal{O}_P([\log(n)/n]^{1/2})$. This empirical evidence substantiates the results of  Theorem~\ref{theorem:weakly_identifiable_experts_sigmoid}. 


\renewcommand{\arraystretch}{1}
\begin{table}[t!]
\centering
\small

\caption{\small Performance comparisons of different Sparse Mixture of Experts (SMoE) models on subsets of the SlimPajama dataset using a small-scale model with 158M parameters and large-scale model with 679M parameters. (SMoE-SG refers to SMoE Sigmoid Gating).  PPL indicates the perplexity score.}

\resizebox{\textwidth}{!}{
\begin{tabular}{l cccc| cccc}
\toprule
 & \multicolumn{4}{c}{\textbf{Small Models (158M)}} & \multicolumn{4}{c}{\textbf{Large Models (679M)}}\\
\cmidrule(lr){2-5}\cmidrule(lr){6-9}

 & \makecell{\textbf{SMoE}} & \makecell{\textbf{DeepSeek-V3}} & \makecell{\textbf{DeepSeek-V2}} & \makecell{\textbf{SMoE-SG}} & \makecell{\textbf{SMoE}} & \makecell{\textbf{DeepSeek-V3}} & \makecell{\textbf{DeepSeek-V2}} & \makecell{\textbf{SMoE-SG}} \\

\midrule
\textbf{PPL} $\downarrow$ & 13.63 & \textbf{13.42} & \underline{13.49} & 13.61 &  9.51 & \underline{9.49} & 9.52 & \textbf{9.46} \\
\midrule
\textbf{LAMBADA} & 25.27\% & 25.49\% & 25.29\% & 25.43\% & 37.13\% & 36.88\% & 37.11\% & 37.56\% \\
\textbf{BLiMP} & 77.71\% & 77.20\% & 77.37\% & 77.38\% & 80.47\% & 81.28\% & 80.98\% & 81.08\% \\
\textbf{CBT} & 84.18\% & 84.40\% & 84.33\% & 84.23\% & 89.83\% & 89.65\% & 89.93\% & 89.57\% \\
\textbf{HellaSwag} & 29.43\% & 29.38\% & 29.38\% & 29.13\% & 37.49\% & 37.32\% & 37.14\% & 37.52\% \\
\textbf{PIQA} & 57.94\% & 59.14\% & 60.17\% & 58.92\% & 64.36\% & 65.72\% & 64.36\% & 64.91\% \\
\textbf{ARC-Easy} & 32.68\% & 32.52\% & 33.83\% & 32.73\% & 38.22\% & 38.86\% & 38.06\% & 39.15\% \\
\textbf{RACE} & 30.11\% & 30.60\% & 31.02\% & 31.05\% & 33.03\% & 33.12\% & 33.17\% & 32.68\% \\
\textbf{SIQA} & 35.62\% & 35.57\% & 34.90\% & 34.90\% & 37.41\% & 38.59\% & 36.95\% & 37.67\% \\
\textbf{CommonSenseQA} & 24.65\% & 25.47\% & 24.98\% & 24.90\% & 26.54\% & 28.09\% & 27.35\% & 28.50\% \\
\midrule
\rowcolor{customyellow!80} 
\textbf{Average} & 44.18\% & \underline{44.42\%} & \textbf{44.66\%} & 44.30\% & 49.39\% & \textbf{49.95\%} & 49.45\%  & \underline{49.85\%}  \\
\bottomrule
\end{tabular}}

\label{tab:lm_result}
\end{table}

\begin{figure}[t!]
    \centering
    \includegraphics[width=.9\linewidth]{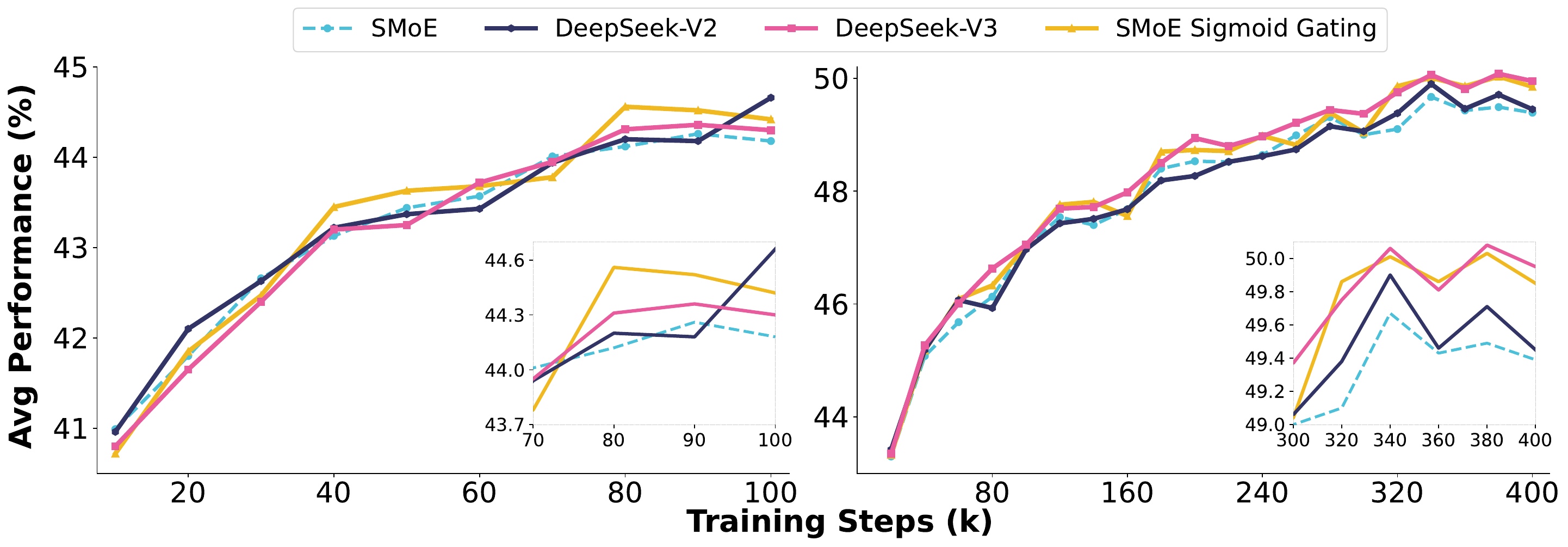}
    \caption{\small Average performance (\%) over training steps in language modeling tasks. \textbf{Left:} Model with 158M parameters; \textbf{Right:} Model with 679M parameters.}
    \label{fig:lm_overtime}
    \vspace{-0.8em}
\end{figure}

\subsection{Language Modeling}  \label{subsec:lm_modeling}

\noindent
\textbf{Experimental Setup.} We conduct the experiments on language modeling using subsets of the popular SLimPajama \cite{cerebras2023slimpajama} dataset using Switch Transformer \cite{fedus2022switch} baseline in two scales: small (158M parameters trained on 6.5B tokens) and large (679M parameters trained on 26.2B tokens). The models are configured with 66 total experts, utilizing top-8 expert routing in the baseline and a top-6 plus 2 shared experts routing scheme in the DeepSeek variants. We measure model performance in terms of perplexity and zero-shot accuracy across nine diverse downstream evaluation tasks \cite{paperno2016lambada,warstadt2020blimp,hill2015goldilocks,zellers2019hellaswag,bisk2020piqa,clark2018think,lai-etal-2017-race,sap2019socialiqa,talmor2018commonsenseqa}. Full experimental details are provided in Appendix \ref{appx:lm_setting}.


\vspace{0.5em}
\noindent 
\textbf{Zero-shot performance on downstream tasks.} Table~\ref{tab:lm_result} summarizes our primary experimental results for two model sizes trained on the SlimPajama dataset \cite{cerebras2023slimpajama}. The results clearly demonstrate that both DeepSeek-V3 and DeepSeek-V2 consistently outperform the Vanilla SMoE baseline, achieving lower perplexity (PPL) scores and higher average accuracy across various downstream tasks for both model scales. Additionally, we integrated the normalized sigmoid gating into the Vanilla SMoE architecture and observed that the SMoE Sigmoid Gating achieves superior performance compared to the Vanilla SMoE and, in some benchmarks, even surpasses the DeepSeek variants.

\vspace{0.5em}
\noindent 
\textbf{Convergence Rate.} Figure \ref{fig:lm_overtime} presents the average performance across various downstream tasks for DeepSeek-V3 and DeepSeek-V2 compared to the Vanilla SMoE. Across both model sizes, the DeepSeek variants demonstrate substantially faster convergence.  Specifically, in both 158M and 679M parameter scales, DeepSeek-V3 and DeepSeek-V2 consistently reach the final performance of Vanilla SMoE using only 70-80\% of the total training steps. Notably, DeepSeek-V3, which incorporates normalized sigmoid gating, demonstrates marginal improvements over DeepSeek-V2 in both convergence speed and final task performance. These results highlight the efficiency gains introduced by the shared expert and normalized sigmoid gating mechanisms and provide empirical support for our theoretical findings.

\begin{figure}[ht]
    \centering
    \includegraphics[width=\linewidth]{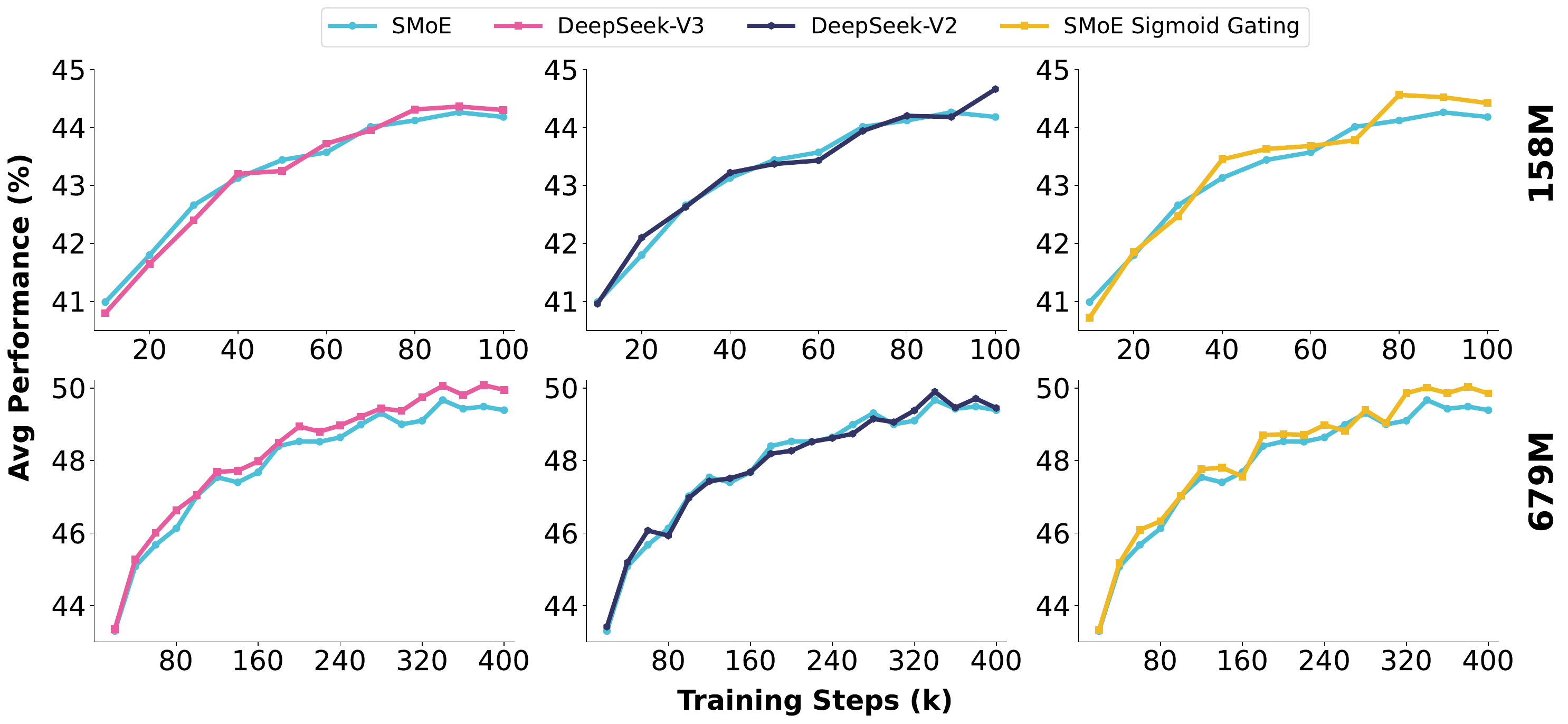}
    \caption{Average performance (\%) compared in pairs with Vanilla SMoE across three model settings over training steps on language modeling tasks. \textbf{Left:} Vanilla SMoE vs. DeepSeek-V3; \textbf{Center:} Vanilla SMoE vs. DeepSeek-V2; \textbf{Right:} Vanilla SMoE vs. SMoE Sigmoid Gating.}
    \label{fig:lm_performancel_pair}
\end{figure}

\vspace{0.5em}
\noindent
To further substantiate these observations, Figure~\ref{fig:lm_performancel_pair} presents a pairwise comparison between DeepSeek-V3, DeepSeek-V2, SMoE Sigmoid Gating, and the baseline Vanilla SMoE. Remarkably, across both model scales,  by integrating normalized sigmoid gating into SMoE, SMoE Sigmoid Gating yields a substantial improvement in convergence rate compared to the softmax-gated baseline. Notably, in several training trajectories, SMoE Sigmoid Gating achieves a convergence rate comparable to that of DeepSeek-V2.  For a more detailed examination, we provide the full training benchmark curves for both the 158M and 679M parameter language modeling settings in Figure~\ref{fig:lm_each_benchmark_158m} and Figure~\ref{fig:lm_each_benchmark_679m}, respectively.

\subsection{Vision-Language Modeling}   \label{subsec:vlm_modeling}

\textbf{Experimental Setup.} We conduct experiments on the visual instruction tuning tasks \cite{liu2024improved} using the popular LLaVA architecture \cite{liu2023llava}. Building upon the LIBMoE framework \cite{nguyen2024libmoe}, we adopt Phi3.5-mini \cite{abdin2024phi} as the language model and SigLIP \cite{zhai2023sigmoid} as the vision encoder. Unlike LIBMoE, we sparse-upcycled~\cite{komatsuzaki2022sparse} only the MLP Connector into 8 experts, employing a top-4 expert routing strategy, while the DeepSeek variants adopt a top-3 expert routing scheme with an additional shared expert, making our model approximately 4.4B parameters. To compare different SMoE algorithms, we use a subset of the LLaVA 1.5 dataset~\cite{liu2024improved} (332K samples and 287M tokens) to train the models in the Visual Instruction Tuning (VIT) stage. Evaluation covers diverse benchmarks containing various vision-language capabilities, including perception, reasoning, OCR, instruction following, and more \cite{kembhavi2016diagram,chen2024we,li2023evaluating,lu2022learn,singh2019towards,hudson2019gqa,zhang2024mme,yue2024mmmu,liu2024ocrbench}. See Appendix \ref{appx:vlm_setting}.


\renewcommand{\arraystretch}{1.1}
\begin{table}[t!]
\caption{\small Vision-language model performance across benchmarks. (SMoE-SG refers to SMoE Sigmoid Gating)}
\centering
\small

\resizebox{\textwidth}{!}{
\begin{tabular}{l ccccccccc >{\columncolor{customyellow!80}}c}
\toprule
 & \makecell{\textbf{AI2D}} & \makecell{\textbf{MMStar}} & \makecell{\textbf{POPE}} & \makecell{\textbf{Science}\\ \textbf{QA}} & \makecell{\textbf{TextVQA}} & \makecell{\textbf{GQA}} & \makecell{\textbf{MME-RW}\\ \textbf{-Lite}} & \makecell{\textbf{MMMU}\\ \textbf{Pro-S}} & \makecell{\textbf{OCR}\\ \textbf{Bench}} & \makecell{\textbf{Average}} \\
 \midrule

\textbf{SMoE} & 64.90\% & 41.66\% & 85.67\% & 81.61\% & 40.92\% & 60.19\% & 31.79\% & 25.61\% & 30.90\% & 51.47\% \\
\textbf{DeepSeek-V3} & 65.45\% & 41.40\% & 85.44\% & 81.94\% & 40.69\% & 60.01\% & 32.20\% & 26.01\% & 32.60\% & \textbf{51.75\%} \\
\textbf{DeepSeek-V2} & 64.70\% & 41.55\% & 85.80\% & 82.20\% & 40.51\% & 60.15\% & 31.11\% & 25.72\% & 31.00\% & 51.41\%   \\
\textbf{SMoE-SR} & 64.64\% & 41.51\% & 85.87\% & 82.17\% & 40.54\% & 60.07\% & 31.68\% & 25.95\% & 31.00\% & \underline{51.49\%} \\
\bottomrule
\end{tabular}}

\label{tab:vlm_result}
\end{table}

\begin{figure}[t!]
    \centering
            \includegraphics[width=\linewidth]{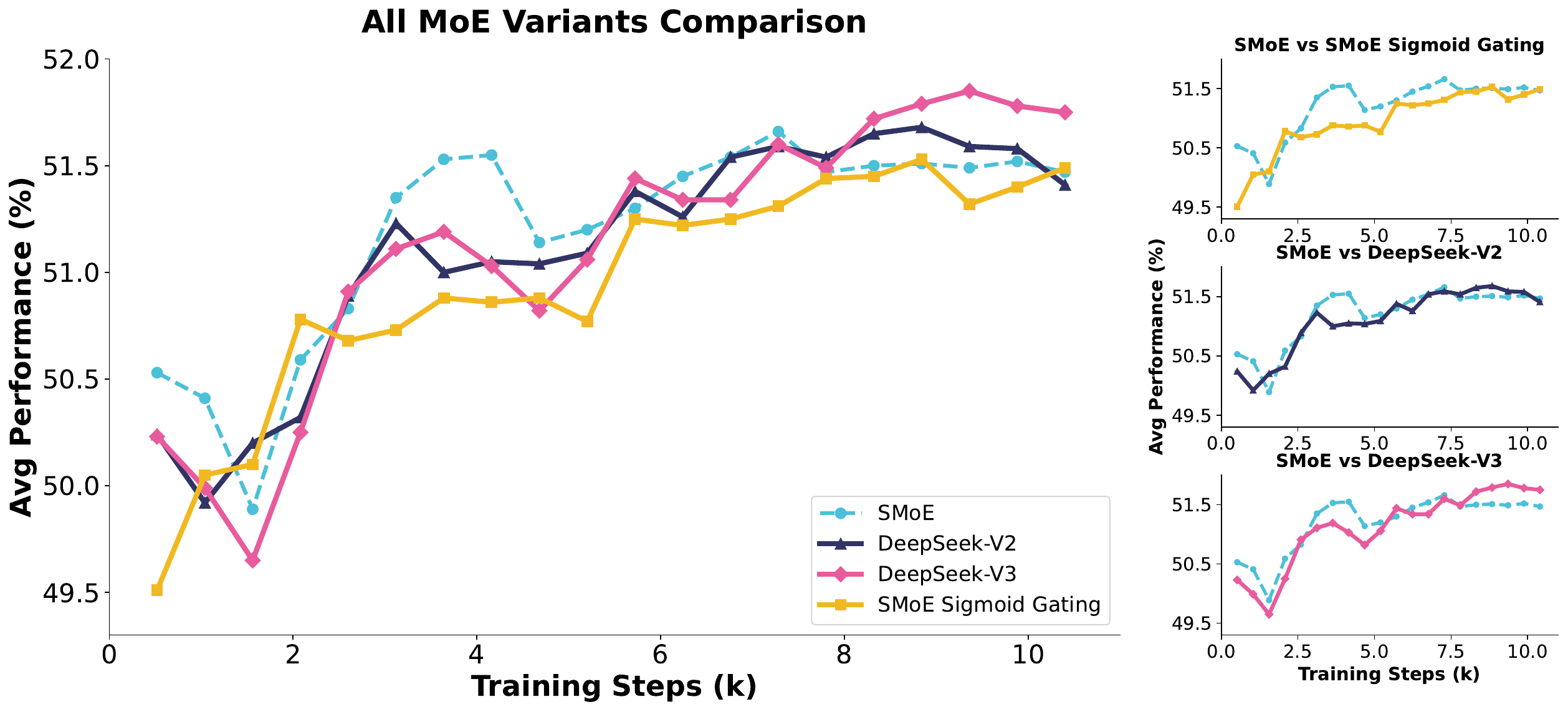}
    \caption{\small Average performance (\%) over training steps on vision-language pretraining tasks. \textbf{Left:} Full comparison among Vanilla SMoE, SMoE with Sigmoid Gating, DeepSeek-V2, and DeepSeek-V3; \textbf{Right:} Pairwise comparison with Vanilla SMoE across three settings.}
    \label{fig:vlm_overtime}
    \vspace{-0.5em}
\end{figure}


\vspace{0.5em}
\noindent
\textbf{Performance.} As summarized in Table~\ref{tab:vlm_result}, DeepSeek-V3 achieves the highest average score (51.75\%), outperforming the Vanilla SMoE (51.47\%) and other model variants. Although DeepSeek-V2 yields slightly lower performance compared to other models, the difference remains marginal. Consistent with observations from language modeling experiments, additional evaluation conducted with Vanilla SMoE and the SMoE Sigmoid Gating reveals a similar pattern, confirming that the normalized sigmoid routing mechanism consistently enhances the performance of the standard SMoE architecture.

\vspace{0.5em}
\noindent
\textbf{Convergence Rate.} Figure~\ref{fig:vlm_overtime} (\textit{left}) illustrates the performance progression over training steps, where both DeepSeek variants exhibit faster and more stable convergence compared to Vanilla SMoE. Notably, both DeepSeek-V2 and DeepSeek-V3 demonstrate accelerated convergence during the final stages of training. These results suggest that both shared expert integration and normalized routing significantly contribute to faster learning in vision-language pretraining.

\vspace{0.5em}
\noindent
Complementing this analysis, Figure~\ref{fig:vlm_overtime} (\textit{right}) presents a pairwise comparison among DeepSeek-V3, DeepSeek-V2, SMoE with Sigmoid Gating, and the baseline Vanilla SMoE. On vision-language pretraining tasks, SMoE Sigmoid Gating exhibits a comparable convergence rate and final performance to the Vanilla SMoE. However, similar to the DeepSeek variants, it demonstrates faster convergence during the later stages of training and achieves greater training stability. To facilitate a finer-grained analysis, we provide benchmark-specific performance trajectories in Figure~\ref{fig:vlm_each_benchmark}.
\subsection{Router Analysis} \label{sec:router_analysis}

In this section, we now explore the router behavior by empirically examining router behavior: Router Saturation (Section~\ref{subsec:router_saturation}), Router Change Rate (Section~\ref{subsec:router_change_rate}), and Expert Utilization (Section~\ref{subappx:expert_utilization}). For consistency, all router statistics are computed on the validation set using naturally occurring contextual sequences rather than isolated tokens. All analytical results are derived over the full 8,000-token vocabulary, following the evaluation protocol of OLMoE~\cite{muennighoff2024olmoe}.



\begin{figure}[ht!]
    \centering
    \includegraphics[width=\linewidth]{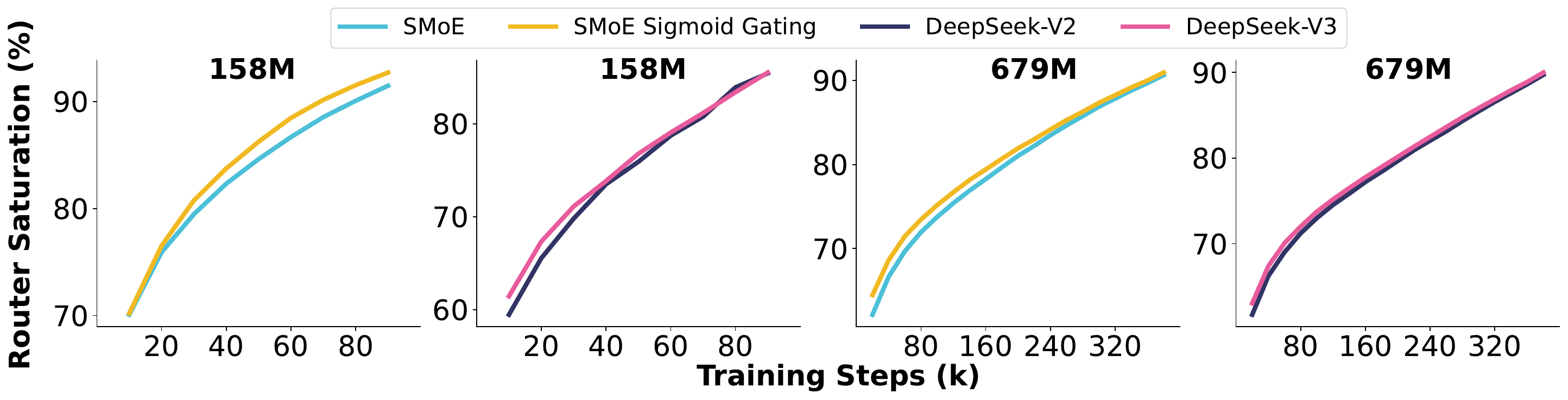}
    \caption{\small Evolution of router saturation (averaged across all layers) during training for language-modeling tasks with 158M (left) and 679 M (right) parameter models. We compute saturation by comparing the routing to the top-8 experts with SMoE and SMoE Sigmoid Gating, and the top-6 experts with DeepSeek variants.}
    \label{fig:router_saturation}
\end{figure}

\subsubsection{Router Saturation}   \label{subsec:router_saturation}

Router Saturation, first introduced in OLMoE \cite{muennighoff2024olmoe}, measures the degree of overlap in activated experts between an intermediary checkpoint at time $t$ and the final checkpoint $T$. It quantifies the router’s convergence during training, where a higher saturation value indicates greater consistency in expert selection. Formally, router saturation is defined as the proportion of expert activations at checkpoint $t$ that match those at $T$ over the same dataset:

\begin{align*}
    \text{Router Saturation}(t) = \frac{1}{N} \sum_{i=1}^{N} \frac{\left| \mathcal{E}_i^{(t)} \cap \mathcal{E}_i^{(T)} \right|}{k},
\end{align*}
where
\begin{itemize}
    \item $N$: The total number of tokens in the dataset;
    \item $k$: The number of top-k experts activated per input token;
    \item $\mathcal{E}_i^{(t)}$: The set of $k$ experts activated for the $i$-th token at the $t$-th checkpoint;
    \item $\mathcal{E}_i^{(T)}$: The set of $k$ experts activated for the $i$-th token at the final checkpoint $T$;
    \item $\left|\mathcal{E}_i^{(t)} \cap \mathcal{E}_i^{(T)}\right|$: The number of common experts activated for the $i$-th token between the $t$-th and final checkpoints.
\end{itemize}
\noindent
Router saturation provides a quantitative measure of how early the routing decisions converge during training. A saturation value of 100\% indicates that the router at an intermediate checkpoint routes to the same set of experts as at the final checkpoint. High saturation values at early checkpoints reflect early convergence in expert selection, indicating that the router has rapidly settled into a stable assignment pattern. In contrast, low saturation values suggest ongoing exploration or adaptation in expert allocations, signaling that the routing mechanism is still undergoing significant adjustments.

\vspace{0.5em}
\noindent
Figure \ref{fig:router_saturation} shows that, after 5\% of training, up to $\sim 60\%$ of router decisions have already saturated. This early saturation aligns with prior findings in OLMoE \cite{muennighoff2024olmoe} and OpenMoE \cite{xue2024openmoe}, supporting the validity of our experimental setup. When comparing model configurations, we observe that models equipped with normalized sigmoid gating achieve noticeably faster saturation than those using softmax gating. In particular, the SMoE Sigmoid Gating exhibits consistently steeper saturation curves compared to Vanilla SMoE, reflecting more rapid convergence in expert selection. A similar pattern is observed in the comparison between DeepSeek-V3 and DeepSeek-V2 under the shared expert configuration. These findings highlight the effectiveness of normalized sigmoid gating in accelerating router convergence, potentially reducing the training time required for convergence.

\begin{figure}[ht!]
    \centering
    \includegraphics[width=\linewidth]{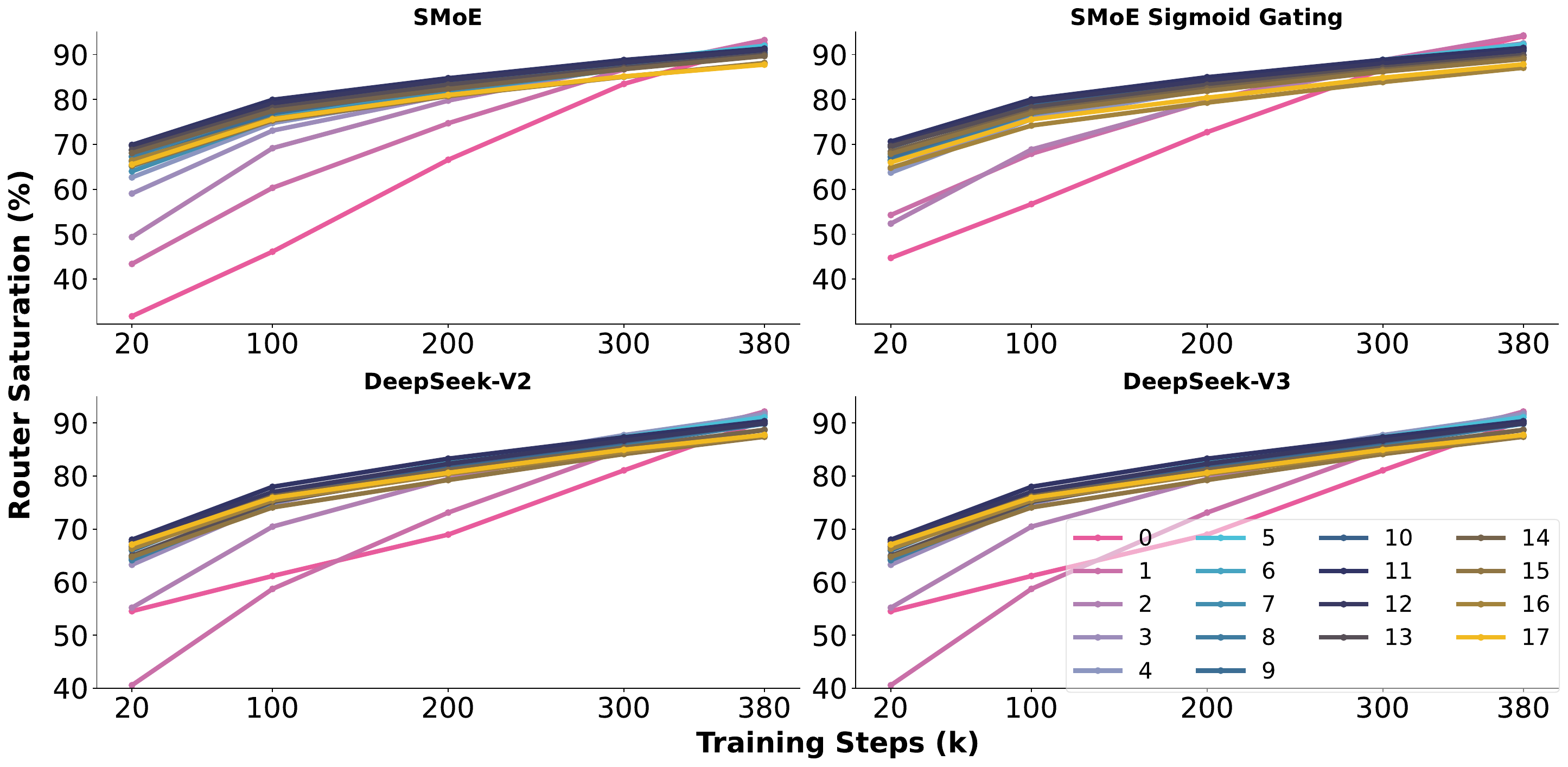}
    \caption{Router saturation across layers for 679M-parameter model in language modeling tasks. We compute saturation by comparing the routing to the top-8 experts with SMoE and SMoE Sigmoid Gating, and the top-6 experts with DeepSeek variants.}
    \label{fig:router_saturation_each_layer_679m}
\end{figure}

\begin{figure}[t!]
    \centering
    \includegraphics[width=\linewidth]{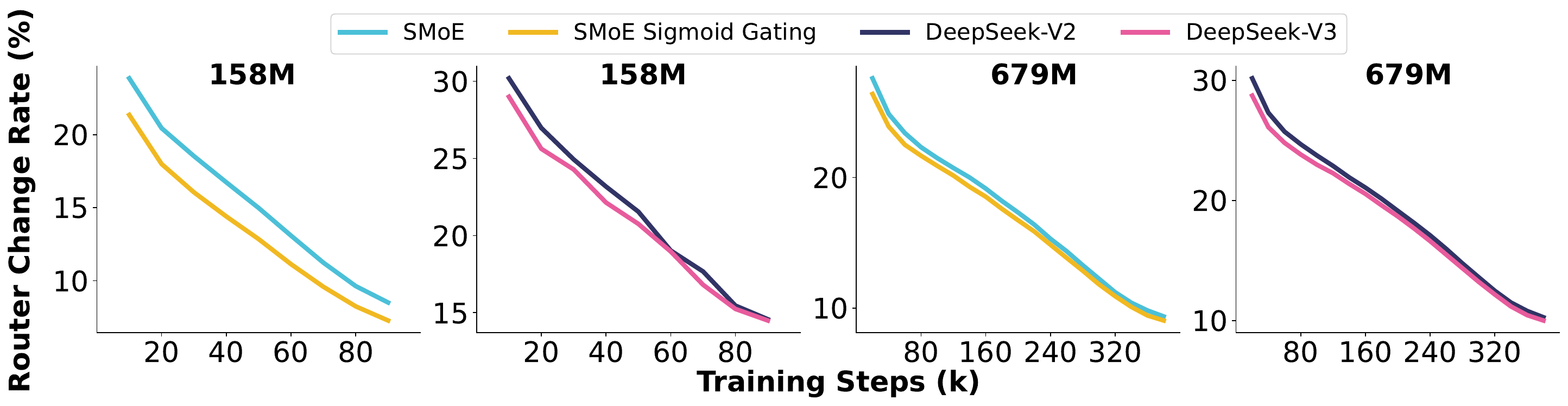}
    \caption{\small Router Change Rate (averaged across all layers) during training for language-modeling tasks with 158M (left) and 679M (right) parameter models. We compute router change rate by comparing the routing to the top-8 experts with SMoE and SMoE Sigmoid Gating, and the top-6 experts with DeepSeek variants.}
    \vspace{-0.8em}
    \label{fig:router_change_rate}
\end{figure}

\noindent
The layer-wise router saturation dynamics for the 679M-parameter model is presented in Figure~\ref{fig:router_saturation_each_layer_679m}. The result shows that the later layer tends to saturate earlier during training, where layer 0 is an outlier and saturates significantly slower than the others. Additionally, we observe that in shared layer settings (DeepSeek-V2 and DeepSeek-V3), the gap between saturation of different layers is smaller than SMoE and SMoE Sigmoid Gating. When comparing gating mechanisms, the model with normalized sigmoid gating exhibits a more uniform saturation profile across layers than the softmax-gated counterpart, demonstrating that normalized sigmoid gating promotes more balanced expert utilization and faster router stabilization.

\subsubsection{Router Change Rate} \label{subsec:router_change_rate}

To assess the stability of the routing mechanism in MoE during training, we introduce the Router Change Rate metric. This metric quantifies the proportion of expert activation decisions that differ between consecutive checkpoints, providing a direct measure of gating fluctuation over time. A lower router change rate indicates more consistent routing behavior, reflecting improved training stability. Formally, the router change rate at step $t$ is defined as:

\begin{align*}
    \text{Router Change Rate}(t) = \frac{1}{N} \sum_{i=1}^{N} \frac{\left| \mathcal{E}_i^{(t + 1)} \backslash \mathcal{E}_i^{(t)} \right|}{k},
\end{align*}
where
\begin{itemize}
    \item $N$: The total number of tokens in the dataset;
    \item $k$: The number of top-k experts activated per input token;
    \item $\mathcal{E}_i^{(t)}$: The set of $k$ experts activated for the $i$-th token at the $t$-th checkpoint;
    \item $\mathcal{E}_i^{(t + 1)}$: The set of $k$ experts activated for the $i$-th token at the $(t+1)$-th checkpoint;
    \item $\left|\mathcal{E}_i^{(t + 1)} \backslash \mathcal{E}_i^{(t)}\right|$: The number of non-intersecting experts activated for the $i$-th token between the $(t+1)$-th and the $t$-th checkpoint.
\end{itemize}
\noindent
Router Change Rate is a quantitative metric to measure the stability of routing mechanism in MoE during training. Unlike router saturation, which assesses convergence towards a final routing decision, the router change rate evaluates fluctuations between consecutive checkpoints. A low router change rate indicates stable routing decisions across training intervals, implying that the gating mechanism has achieved consistent expert assignments, minimizing disruptions and promoting steady specialization of experts. Conversely, a high router change rate suggests volatility in routing decisions, reflecting ongoing exploration or adjustment, potentially introducing training inefficiencies and hindering expert specialization. Thus, monitoring the router change rate provides valuable insights into the dynamics of expert allocation stability, enabling deeper understanding and optimization of the routing strategy in MoE architectures.

\vspace{0.5em}
\noindent
Figure~\ref{fig:router_change_rate} presents the router change rate comparison of different model configurations. We find that models employing normalized sigmoid gating have significantly lower change rates in both non-shared and shared expert settings. These findings underscore the efficiency of normalized sigmoid gating in stabilizing routing decisions throughout training. By reducing the routing fluctuation problem \cite{dai2022stablemoe}, this mechanism promotes a more consistent expert specialization, indicating that stable routing is critical in enhancing both optimization efficiency and final model performance.

\begin{figure}[t!]
    \centering
    \includegraphics[width=\linewidth]{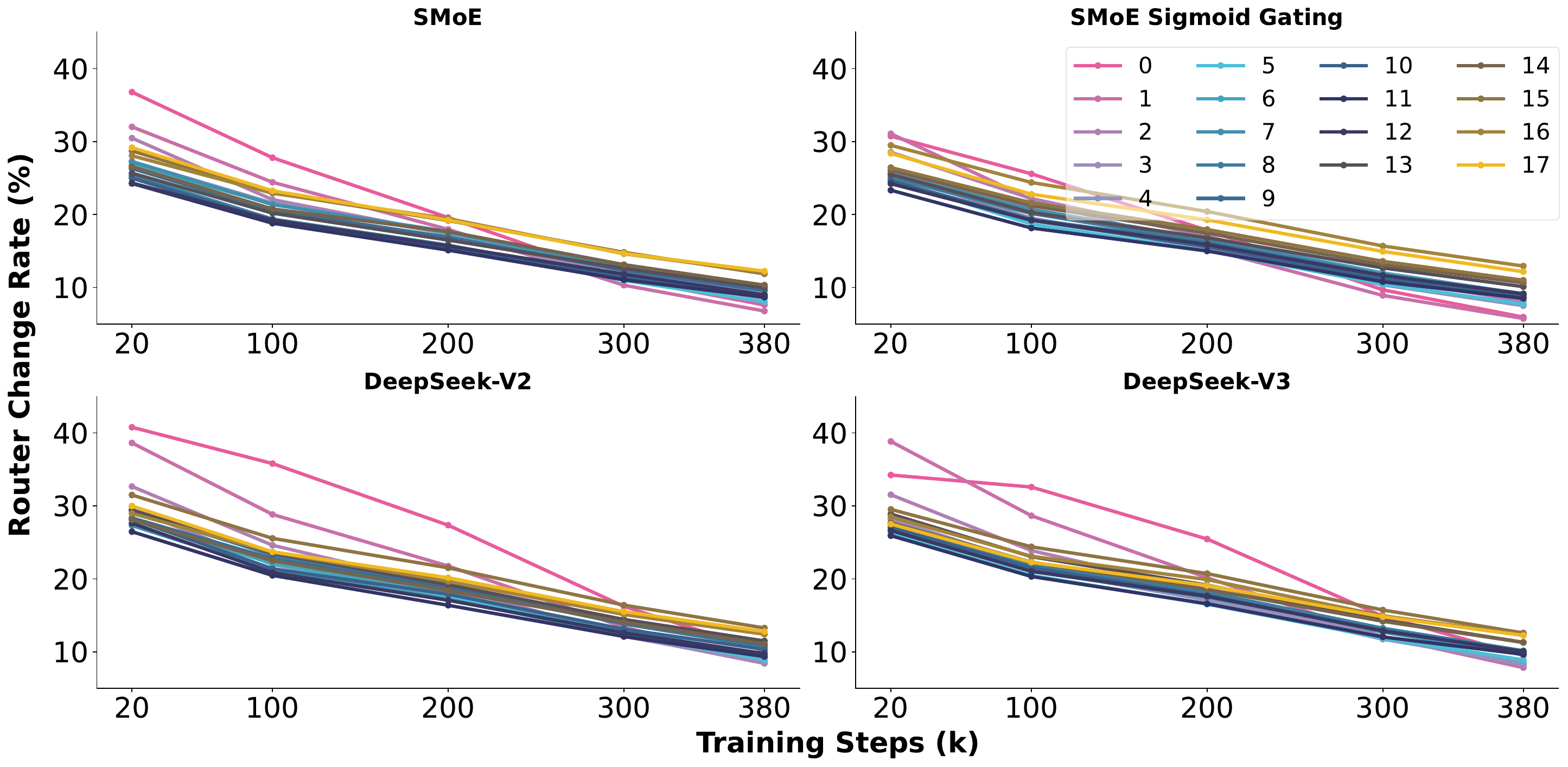}
    \caption{Router change rate across layers for 679M-parameter models in language modeling tasks. We compute router change rate by comparing the routing to the top-8 experts with SMoE and SMoE Sigmoid Gating, and the top-6 experts with DeepSeek variants.}
    \label{fig:router_change_rate_each_layer_679m}
\end{figure}

\vspace{0.5em}
\noindent
Router change rate for each layer for each layer of the 679M-parameter model is shown at Figure~\ref{fig:router_change_rate_each_layer_679m}
Similar to router saturation, later layers show more stability with lower router change rate. However, the router change rate between layers show more consistency compared to router saturation. While layer 0 still deviates slightly, its difference remains modest, suggesting that despite its slower saturation, it maintains stable routing behavior during training. Across different model configurations, those employing normalized sigmoid gating (SMoE Sigmoid Gating and DeepSeek-V3) consistently demonstrate lower and more uniform router change rates compared to their softmax-gated counterparts (SMoE and DeepSeek-V2). These results confirm that normalized sigmoid gating contributes to improved routing stability throughout training.

\subsubsection{Expert Utilization} \label{subappx:expert_utilization}

\begin{figure}[h!]
    \centering
    \includegraphics[width=\linewidth]{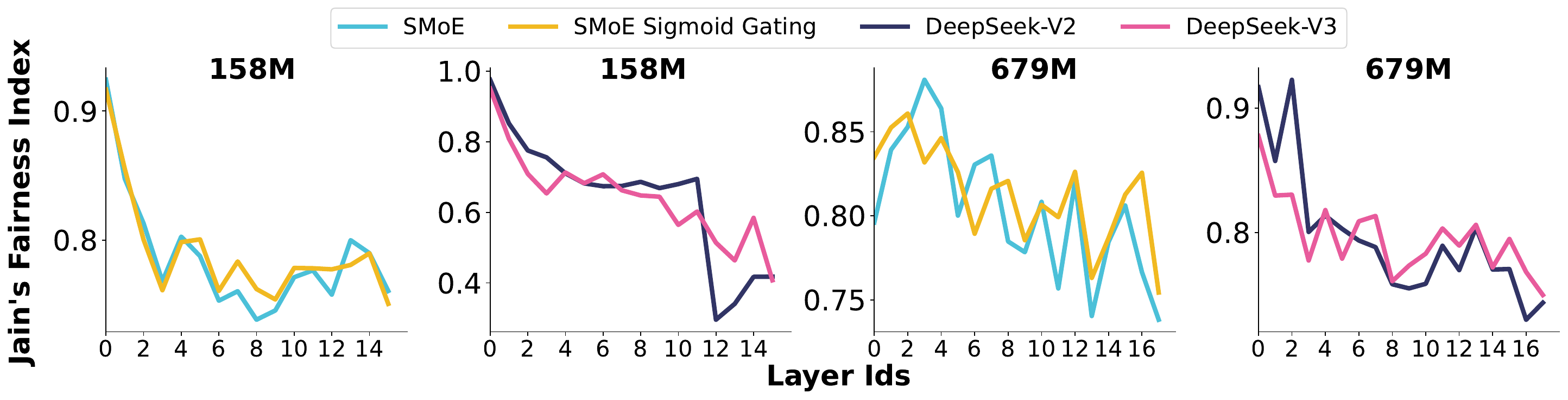}
    \caption{Jain's Fairness Index across MoE layers for language-modeling tasks with 158M (left) and 679M (right) parameter models.}
    \label{fig:jain_fairness_index}
\end{figure}

\noindent
To quantify the fairness of expert utilization in the MoE model, we apply Jain's Fairness Index to the router’s resource allocation across $n$ experts. Let $R = (r_1, r_2, \dots, r_n)$ denote the utilization vector, where $r_i \geq 0$ represents the proportion of input tokens (or total routing weight) assigned to expert $i$ over a given evaluation window. The Jain's Fairness Index $J(R)$ is computed as:

\begin{align*}
    J(R) = J(r_1, r_2, ..., r_n) = \frac{(\sum_{i=1}^n r_i)^2}{n\sum_{i=1}^n r_i^2},
\end{align*}
This index ranges from $[1/n, 1]$, where $J(R) = 1$ indicates perfectly uniform expert usage, (i.e., all experts are used equally), where $J(R) = 1/n$ signifies complete imbalance, with only one expert active.  Thus, higher values of $J(R)$ correspond to fairer and more evenly distributed expert selection.

\vspace{0.5em}
\noindent
Figure~\ref{fig:jain_fairness_index} presents a comparison of Jain’s Fairness Index \cite{jain1984quantitative} across different MoE model configurations and scales. Across both 158M and 679M parameter models, all configurations exhibit a consistent pattern: fairness in expert utilization is highest in the initial layers and declines in subsequent layers, suggesting that earlier layers facilitate broader expert utilization. Notably, models employing normalized sigmoid gating (SMoE Sigmoid Gating and DeepSeek-V3) maintain a higher fairness index, especially in the later layers, indicating better expert utilization. These results highlight the efficacy of normalized sigmoid gating in promoting more balanced expert utilization throughout the network.

\section{Discussion}
\label{sec:conclusion}
In this paper, we have presented an extensive study on the benefits of two fundamental ingredients of DeepSeekMoE architecture, namely the shared expert strategy and the normalized sigmoid gating mechanism. From the theoretical side, we perform a convergence analysis of expert estimation to investigate differences in sample efficiency. Our analysis reveals that the shared expert strategy leads to faster estimation rates for shared experts compared to routed experts and experts in the standard MoE. Furthermore, the estimation rates for routed experts become dramatically faster when replacing the softmax gating with the normalized sigmoid gating in DeepSeekMoE. Therefore, the incorporation of these two key factors into DeepSeekMoE significantly reduces the overall sample complexity for the expert estimation tasks. 

\vspace{0.5em}
\noindent
From the empirical side, we validate our theoretical findings through extensive experiments and analysis on both synthetic and real-world datasets. Our results consistently demonstrate that both the shared experts strategy and the normalized sigmoid gating mechanism substantially improve the sample efficiency of estimating experts and the performance of (vision)-language models in real-world scenarios. Moreover, these two ingredients also yield substantial gains in router convergence, routing stability, and expert utilization. Overall, our work provides both a principled understanding and robust empirical evidence for the effectiveness of these two components, offering valuable guidance for the design of future sparse MoE models. 

\vspace{0.5em}
\noindent
Nevertheless, there is still an open problem of model selection for DeepSeekMoE that we have not explored in this work. In particular, although our analysis confirms that the usage of shared experts improves the performance and sample efficiency of MoE models, it does not indicate how many shared experts should be employed to achieve optimal performance
given a fixed computational budget. The design choice directly affects the balance between generalization and specialization - using too few shared experts may lead to insufficient knowledge sharing and the reduction of model generalization across different tasks, whereas using too many of them will lose the benefits of sparsity and cause the redundancy of routed experts. In addition, the optimal configuration also depends on factors such as data heterogeneity, model scale, and routing sparsity. Therefore, a potential approach to this problem is to establish a scaling law involving these quantities through several extensive experiments as done in \cite{ludziejewski2024scaling}. Since this research direction goes beyond the scope of our work, we leave it for future development.



\newpage
\appendix
\addtocontents{toc}{\protect\setcounter{tocdepth}{10}}
\begin{center}
{}\textbf{\Large{Appendices for \\ \vspace{.2em}
``On DeepSeekMoE: Statistical Benefits of Shared Experts \\ \vspace{.2em}
and Normalized Sigmoid Gating''}}
\end{center}
\tableofcontents

\section{Systems of Polynomial Equations}
\label{appendix:system}
In this appendix, we will provide a formal definition of the functions $r_1$ and $r_2$ involved in the Voronoi loss $\mathcal{D}_2$ defined in equation~\eqref{eq:loss_2}.

\vspace{0.5em}
\noindent
\textbf{Definition of the function $r_1$.} To capture estimation rates for shared expert parameters in Section~\ref{sec:linear_experts}, it is necessary to consider the solvability of a system of polynomial equations previously studied in \cite{ho_convergence_2016}. More specifically, for each $m\geq 2$, let $r_1(m)$ be the smallest natural number $r$ such that the system:
\begin{align}
\label{eq:system_r_1}
\sum_{i=1}^{m}\sum_{\substack{n_1,n_2\in\mathbb{N}:\\n_1+2n_2=\ell}}\dfrac{s^2_{3i}~s^{n_1}_{1i}~s^{n_2}_{2i}}{n_1!~n_2!}=0, \quad \ell=1,2,\ldots,r,
\end{align}
does not admit any non-trivial solutions for the unknown variables $\{s_{1i},s_{2i},s_{3i}\}_{i=1}^{m}$. Here, we call a solution non-trivial if all the values of $s_{3i}$ are non-zero, whereas at least one among $s_{1i}$ is different from zero. In the following proposition, we provide the values of the function $r_1$ at some specific points $m\in\mathbb{N}$.
\begin{proposition}[Proposition 2.1, \cite{ho_convergence_2016}]
    \label{lemma_r_bar}
    For $m=2$, we get $r_1(m)=4$, while for $m=3$, we have $r_1(m)=6$. When $m\geq 4$, we have $r_1(m)\geq 7$.
\end{proposition}
\noindent
The proof of Proposition~\ref{lemma_r_bar} can be found in \cite{ho_convergence_2016}.

\vspace{0.5em}
\noindent
\textbf{Definition of the function $r_2$.} To characterize estimation rates for routed expert parameters in Section~\ref{sec:linear_experts}, we need to take into account the solvability of another system of polynomial equations studied in \cite{nguyen2023demystifying}, which is given by  
\begin{align}
\label{eq:new_system}
\sum_{i=1}^{m}\sum_{\alpha\in\mathcal{I}_{\ell_1,\ell_2}}\dfrac{t_{5i}^2~t_{1i}^{\alpha_1}~t_{2i}^{\alpha_2}~t_{3i}^{\alpha_3}~t_{4i}^{\alpha_4}}{\alpha_1!~\alpha_2!~\alpha_3!~\alpha_4!}=0,
\end{align}
for all $\ell_1,\ell_2\geq 0$ satisfying $1\leq \ell_1+\ell_2\leq r$, where
\begin{align*}
    \mathcal{I}_{\ell_1,\ell_2}:=\{\alpha=(\alpha_i)_{i=1}^{4}\in\mathbb{N}^{d}\times\mathbb{N}^{d}\times\mathbb{N}\times\mathbb{N}:\alpha_1+\alpha_2=\ell_1, \alpha_3+2\alpha_4=\ell_2-|\alpha_2|\}.
\end{align*}
Then, we define $r_2(m)$ as the smallest natural number $r$ such that the system in equation~\eqref{eq:new_system} has no non-trivial solutions for the unknown variables $\{t_{5i},t_{1i},t_{2i},t_{3i},t_{4i}\}_{i=1}^m$. Here, a solution is called non-trivial if all the values of $t_{5i}$ are different from, while at least one among $t_{4i}$ is non-zero. The following proposition provides a relation between the two functions $r_1$ and $r_2$ as well as specify the values of $r_2(m)$ at some points $m\in\mathbb{N}$.
\begin{proposition}[Lemma 1, \cite{nguyen2023demystifying}]
    \label{lemma_r_tilde}
    The function $r_2$ is upper bounded by the function $r_1$, that is, $r_2(m)\leq r_1(m)$, for all $m\in\mathbb{N}$. In addition, we have $r_2(2)=4$, $r_2(3)=6$ and $r(m)\geq 7$ when $m\geq 4$. 
\end{proposition}
\noindent
The proof of Lemma~\ref{lemma_r_tilde} can be found in \cite{nguyen2023demystifying}.

\section{Proof of Main Results}
\label{appendix:proof_main_results}
\subsection{Proof of Theorem~\ref{theorem:strongly_identifiable_experts}}
\label{appendix:strongly_identifiable_experts}
\textbf{Proof overview.} Recall that our goal is to demonstrate that the following lower bound holds for any $G\in\mathcal{G}_{k_1,k_2}(\Theta)$:
\begin{align}
    \label{eq:target_bound}
    \bbE_X[V(f_{G_1,G_2}(\cdot|X),f_{G^*_1,G^*_2}(\cdot|X))]\gtrsim \mathcal{D}_1((G_1,G_2),(G^*_1,G^*_2)).
\end{align}
Our proof will be divided into two main parts. Firstly, we aim to establish the local part of the bound~\eqref{eq:target_bound}, that is,
\begin{align}
    \label{eq:local_part}
    \lim_{\varepsilon\to0}\inf_{(G_1,G_2)\in\mathcal{G}_{k_1,k_2}(\Theta):\mathcal{D}_1((G_1,G_2),(G^*_1,G^*_2))\leq\varepsilon}\dfrac{\bbE_X[V(f_{G_1,G_2}(\cdot|X),f_{G^*_1,G^*_2}(\cdot|X))]}{\mathcal{D}_1((G_1,G_2),(G^*_1,G^*_2))}>0.
\end{align}
The above result implies that there exists a positive constant $\varepsilon'$ such that
\begin{align*}
    \inf_{(G_1,G_2)\in\mathcal{G}_{k_1,k_2}(\Theta):\mathcal{D}_1((G_1,G_2),(G^*_1,G^*_2))\leq\varepsilon'}\dfrac{\bbE_X[V(f_{G_1,G_2}(\cdot|X),f_{G^*_1,G^*_2}(\cdot|X))]}{\mathcal{D}_1((G_1,G_2),(G^*_1,G^*_2))}>0.
\end{align*}
Then, we complete the proof by deriving the following global part of the bound~\eqref{eq:target_bound}:
\begin{align}
    \label{eq:global_part}
    \inf_{(G_1,G_2)\in\mathcal{G}_{k_1,k_2}(\Theta):\mathcal{D}_1((G_1,G_2),(G^*_1,G^*_2))>\varepsilon'}\dfrac{\bbE_X[V(f_{G_1,G_2}(\cdot|X),f_{G^*_1,G^*_2}(\cdot|X))]}{\mathcal{D}_1((G_1,G_2),(G^*_1,G^*_2))}>0.
\end{align}
\textbf{Proof for the local part~\eqref{eq:local_part}:} Assume by contrary that the claim in equation~\eqref{eq:local_part} does not hold. Then, we can find a sequence of mixing measure pairs $(G^n_1,G^n_2)$ taking the form $G^n_1:=\sum_{i=1}^{k^n_1}\oin\delta_{(\kin,\tin)}$, $G^n_2:=\sum_{i=1}^{k^n_2}\exp(\bzin)\delta_{(\boin,\ein,\nuin)}$ for $n\in\mathbb{N}$ such that $\mathcal{D}_{1n}:=\mathcal{D}_{1}((G^n_1,G^n_2),(G^*_1,G^*_2))\to0$ and
\begin{align}
    \label{eq:expectation_zero}
    \bbE_X[V(f_{G^n_1,G^n_2}(\cdot|X),f_{G^*_1,G^*_2}(\cdot|X))]/\mathcal{D}_{1n}\to0,
\end{align}
as $n\to\infty$. As our proof argument is asymptotic, we may assume that the number of shared and routed experts $k^n_1,k^n_2$ do not vary with the sample size $n$. In addition, we also assume that Voronoi cells are independent of $n$, that is, $\mathcal{V}_{1,j_1}=\mathcal{V}_{1,j_1}(G^n_1)$ and $\mathcal{V}_{2,j_2}=\mathcal{V}_{2,j_2}(G^n_2)$, for all $j_1\in[k^*_1]$ and $j_2\in[k^*_2]$. Then, we can represent the Voronoi loss $\mathcal{D}_{1n}$ as
\begin{align}
    \label{eq:loss_1n}
    &\mathcal{D}_{1n}=\sum_{j=1}^{k^*_1}\Big|\sum_{i\in\mathcal{V}_{1,j}}\oin-\oj\Big|+\sum_{j=1}^{k^*_2}\Big|\sum_{i\in\mathcal{V}_{2,j}}\exp(\bzin)-\exp(\bzj)\Big|\nonumber\\
    &+\sum_{j\in[k^*_1]:|\mathcal{V}_{1,j}|=1}\sum_{i\in\mathcal{V}_{1,j}}\oin(\|\dkijn\|+|\dtijn|)+\sum_{j\in[k^*_2]:|\mathcal{V}_{2,j}|=1}\sum_{i\in\mathcal{V}_{2,j}}\exp(\bzin)(\|\dboijn\|+\|\deijn\|+|\dnuijn|)\nonumber\\
    &+\sum_{j\in[k^*_1]:|\mathcal{V}_{1,j}|>1}\sum_{i\in\mathcal{V}_{1,j}}\oin(\|\dkijn\|^2+|\dtijn|^2)+\sum_{j\in[k^*_2]:|\mathcal{V}_{2,j}|>1}\sum_{i\in\mathcal{V}_{2,j}}\exp(\bzin)(\|\dboijn\|^2+\|\deijn\|^2+|\dnuijn|^2),
\end{align}
where we denote $\dkijn:=\kappa^n_{i}-\kj$, $\dtijn:=\tau^n_{i}-\tj$, $\dboijn:=\beta^n_{1i}-\boj$, $\deijn:=\eta^n_{i}-\ej$, and $\dnuijn:=\nu^n_{i}-\nuj$. Recall that $\mathcal{D}_{1n}\to0$ as $n\to\infty$, then it follows that $\sum_{i\in\mathcal{V}_{1,j}}\oin\to\oj$, $(\kin,\tin)\to(\kj,\tj)$ as $n\to\infty$ for all $i\in\mathcal{V}_{1,j}$ and $j\in[k^*_1]$. Furthermore, we also have $\sum_{i\in\mathcal{V}_{2,j}}\exp(\bzin)-\exp(\bzj)$, $(\boin,\ein,\nuin)\to(\boj,\ej,\nuj)$ as $n\to\infty$ for all $i\in\mathcal{V}_{2,j}$ and $j\in[k^*_2]$.

Subsequently, we partition the rest of this proof into three main stages:

\textbf{Stage 1 - Density Decomposition:} In this stage, we focus on decomposing the density difference $f_{G^n_1,G^n_2}(Y|X)-f_{G^*_1,G^*_2}(Y|X)$. For ease of presentation, let us denote
\begin{align*}
    q_{G^n_1}(Y|X)&:=\sum_{i=1}^{k^n_1}\omega^n_{i}\pi(Y|h_1(X,\kappa^n_{i}),\tau^n_{i}),\\
    q_{G^*_1}(Y|X)&:=\sum_{i=1}^{k^*_1}\omega^*_{i}\pi(Y|h_1(X,\kappa^*_{i}),\tau^*_{i}),\\
    p_{G^n_2}(Y|X)&:=\sum_{i = 1}^{k^n_2} \frac{\exp((\beta_{1i}^{n})^{\top} X + \beta_{0i}^{n})}{\sum_{j = 1}^{k^n_2} \exp((\beta_{1j}^{n})^{\top} X + \beta_{0j}^{n})}\cdot \pi(Y|h_2(X,\eta^n_{i}), \nu_{i}^{n}),\\
    p_{G^*_2}(Y|X)&:=\sum_{i = 1}^{k^*_2} \frac{\exp((\beta_{1i}^{*})^{\top} X + \beta_{0i}^{*})}{\sum_{j = 1}^{k^*_2} \exp((\beta_{1j}^{*})^{\top} X + \beta_{0j}^{*})}\cdot \pi(Y|h_2(X,\eta^*_{i}), \nu_{i}^{*}).
\end{align*}
Then, we have
\begin{align*}
    f_{G^n_1,G^n_2}(Y|X)-f_{G^*_1,G^*_2}(Y|X)=\frac{1}{2}\left[(q_{G^n_1}(Y|X)-q_{G^*_1}(Y|X))+(p_{G^n_2}(Y|X)-p_{G^*_2}(Y|X))\right].
\end{align*}
\textbf{Stage 1.1:} In this step, we decompose the term $q_{G^n_1}(Y|X)-q_{G^*_1}(Y|X)$ as
\begin{align*}
    q_{G^n_1}(Y|X)-q_{G^*_1}(Y|X)&=\sum_{j\in[k^*_1]:|\mathcal{V}_{1,j}|=1}\sum_{i\in\mathcal{V}_{1,j}}\oin[\pi(Y|h_1(X,\kin),\tin)-\pi(Y|h_1(X,\kj),\tj)]\\
    &+\sum_{j\in[k^*_1]:|\mathcal{V}_{1,j}|>1}\sum_{i\in\mathcal{V}_{1,j}}\oin[\pi(Y|h_1(X,\kin),\tin)-\pi(Y|h_1(X,\kj),\tj)]\\
    &+\sum_{j=1}^{k^*_1}\Big(\sum_{i\in\mathcal{V}_{1,j}}\oin-\oj\Big)\pi(Y|h_1(X,\kappa^*_j),\tau^*_j)\\
    &:=A_{n,1}(Y|X)+A_{n,2}(Y|X)+A_{n,0}(Y|X).
\end{align*}
By applying the first-order and second-order Taylor expansions to the function $\pi(Y|h_1(X,\kappa^n_i),\tau^n_i))$ around the point $(\kappa^*_j,\tau^*_j)$, respectively, we have
\begin{align*}
    A_{n,1}(Y|X)&=\sum_{j\in[k^*_1]:|\mathcal{V}_{1,j}|=1}\sum_{i\in\mathcal{V}_{1,j}}\oin\sum_{|\alpha|=1}\frac{1}{\alpha!}(\dkijn)^{\alpha_1}(\dtijn)^{\alpha_2}\cdot\frac{\partial \pi}{\partial\kappa^{\alpha_1}\partial\tau^{\alpha_2}}(Y|h_1(X,\kj),\tj)+R_{n,1}(Y|X),\\
    A_{n,2}(Y|X)&=\sum_{j\in[k^*_1]:|\mathcal{V}_{1,j}|>1}\sum_{i\in\mathcal{V}_{1,j}}\oin\sum_{|\alpha|=1}^{2}\frac{1}{\alpha!}(\dkijn)^{\alpha_1}(\dtijn)^{\alpha_2}\cdot\frac{\partial^{|\alpha|} \pi}{\partial\kappa^{\alpha_1}\partial\tau^{\alpha_2}}(Y|h_1(X,\kj),\tj)+R_{n,2}(Y|X),
\end{align*}
where $R_{n,1}(Y|X)$ and $R_{n,2}(Y|X)$ are the Taylor remainders such that $R_{n,1}(Y|X)/\mathcal{D}_{1n}\to0$ as $n\to\infty$. By the chain rule, the first-order derivatives of the function $\pi$ with respect to its parameters $\kappa$ and $\tau$ are given by
\begin{align*}
    \frac{\partial\pi}{\partial\kappa^{(u_1)}}(Y|h_1(X,\kj),\tj)&=\frac{\partial h_1}{\partial\kappa^{(u_1)}}(X,\kj)\frac{\partial\pi}{\partial h_1}(Y|h_1(X,\kj),\tj),\\
    \frac{\partial\pi}{\partial\tau}(Y|h_1(X,\kj),\tj)&=\frac{1}{2}\frac{\partial^2\pi}{\partial h_1^2}(Y|h_1(X,\kj),\tj),
\end{align*}
for all $u_1\in[d_1]$. Analogously, the second-order derivatives of the function $\pi$ w.r.t its parameters are calculated as
\begin{align*}
    \frac{\partial^2\pi}{\partial\kappa^{(u_1)}\partial\kappa^{(v_1)}}(Y|h_1(X,\kj),\tj)&=\frac{\partial^2 h_1}{\partial\kappa^{(u_1)}\partial\kappa^{(v_1)}}(X,\kj)\frac{\partial\pi}{\partial h_1}(Y|h_1(X,\kj),\tj)\\
    &\quad+\frac{\partial h_1}{\partial\kappa^{(u_1)}}(X,\kj)\frac{\partial h_1}{\partial\kappa^{(v_1)}}(X,\kj)\frac{\partial^2\pi}{\partial h_1^2}(Y|h_1(X,\kj),\tj),\\
    \frac{\partial^2\pi}{\partial\tau^2}(Y|h_1(X,\kj),\tj)&=\frac{1}{4}\frac{\partial^4\pi}{\partial h_1^4}(Y|h_1(X,\kj),\tj),\\
    \frac{\partial^2\pi}{\partial\kappa^{(u_1)}\partial\tau}(Y|h_1(X,\kj),\tj)&=\frac{1}{2}\frac{\partial h_1}{\partial\kappa^{(u_1)}}(X,\kj)\frac{\partial^3\pi}{\partial h_1^3}(Y|h_1(X,\kj),\tj),
\end{align*}
for all $u_1,v_1\in[d_1]$. Combine the above results, we can rewrite $A_{n,1}(Y|X)$ as
\begin{align*}
    A_{n,1}(Y|X)&=\sum_{j\in[k^*_1]:|\mathcal{V}_{1,j}|=1}\Big[A^{(j)}_{n,1,1}(X)\frac{\partial\pi}{\partial h_1}(Y|h_1(X,\kj),\tj)+A^{(j)}_{n,1,2}(X)\frac{\partial^2\pi}{\partial h_1^2}(Y|h_1(X,\kj),\tj)\Big]+R_{n,1}(Y|X),
\end{align*}
where we denote
\begin{align*}
    A^{(j)}_{n,1,1}(X)&:=\sum_{i\in\mathcal{V}_{1,j}}\oin\sum_{u_1=1}^{d_1}(\dkijn)^{(u_1)}\frac{\partial h_1}{\partial\kappa^{(u_1)}}(X,\kj),\\
    A^{(j)}_{n,1,2}(X)&:=\sum_{i\in\mathcal{V}_{1,j}}\oin\frac{1}{2}(\dtijn),
\end{align*}
for all $j\in[k^*_1]$ such that $|\mathcal{V}_{1,j}|=1$. Similarly, the quantity $A_{n,2}(Y|X)$ can be represented as
\begin{align*}
    A_{n,2}(Y|X)&=\sum_{j\in[k^*_1]:|\mathcal{V}_{1,j}|>1}\Big[~A^{(j)}_{n,2,1}(X)\frac{\partial\pi}{\partial h_1}(Y|h_1(X,\kj),\tj)+A^{(j)}_{n,2,2}(X)\frac{\partial^2\pi}{\partial h_1^2}(Y|h_1(X,\kj),\tj)\\
    &+A^{(j)}_{n,2,3}(X)\frac{\partial^3\pi}{\partial h_1^3}(Y|h_1(X,\kj),\tj)+A^{(j)}_{n,2,4}(X)\frac{\partial^4\pi}{\partial h_1^4}(Y|h_1(X,\kj),\tj)\Big]+R_{n,2}(Y|X),
\end{align*}
where we denote
\begin{align*}
    A^{(j)}_{n,2,1}(X)&:=\sum_{i\in\mathcal{V}_{1,j}}\oin\Big(\sum_{u_1=1}^{d_1}(\dkijn)^{(u_1)}\frac{\partial h_1}{\partial\kappa^{(u_1)}}(X,\kj)+\sum_{u_1,v_1=1}^{d_1}\frac{(\dkijn)^{(u_1)}(\dkijn)^{(v_1)}}{1+1_{\{u_1=v_1\}}}\frac{\partial^2h_1}{\partial\kappa^{(u_1)}\partial\kappa^{(v_1)}}(X,\kj)\Big),\\
    A^{(j)}_{n,2,2}(X)&:=\sum_{i\in\mathcal{V}_{1,j}}\oin\Big(\frac{1}{2}(\dtijn)+\sum_{u_1,v_1=1}^{d_1}\frac{(\dkijn)^{(u_1)}(\dkijn)^{(v_1)}}{1+1_{\{u_1=v_1\}}}\frac{\partial h_1}{\partial\kappa^{(u_1)}}(X,\kj)\frac{\partial h_1}{\partial\kappa^{(v_1)}}(X,\kj)\Big),\\
    A^{(j)}_{n,2,3}(X)&:=\sum_{i\in\mathcal{V}_{1,j}}\oin\sum_{u_1=1}^{d_1}\frac{1}{2}(\dkijn)^{(u_1)}(\dtijn)\frac{\partial h_1}{\partial\kappa^{(u_1)}}(X,\kj),\\
    A^{(j)}_{n,2,4}(X)&:=\sum_{i\in\mathcal{V}_{1,j}}\oin\frac{1}{8}(\dtijn)^2,
\end{align*}
for all $j\in[k^*_1]$ such that $|\mathcal{V}_{1,j}|>1$. 

\vspace{0.5 em}
\noindent
\textbf{Stage 1.2:} In this step, we decompose the term $Q_n(Y|X):=\Big[\sum_{j = 1}^{k^*_2} \exp((\beta_{1j}^{*})^{\top} X + \beta_{0j}^{*})\Big]\cdot[p_{G^n_2}(Y|X)-p_{G^*_2}(Y|X)]$. By denoting $F(Y|X;\beta_1,\eta,\nu):=\exp(\beta_{1}^{\top}X)\pi(Y|h_2(X,\eta),\nu)$ and $H(Y|X;\beta_1):=\exp(\beta_1^{\top}X)p_{G_2}(Y|X)$, we can represent $Q_n(Y|X)$ as
\begin{align*}
    Q_n(Y|X)&=\sum_{j=1}^{k^*_2}\sum_{i\in\mathcal{V}_{2,j}}\exp(\bzin)[F(Y|X;\boin,\ein,\nuin)-F(Y|X;\boj,\ej,\nuj)]\\
    &-\sum_{j=1}^{k^*_2}\sum_{i\in\mathcal{V}_{2,j}}\exp(\bzin)[H(Y|X;\boin)-H(Y|X;\boj)]\\
    &+\sum_{j=1}^{k^*_2}\Big(\sum_{i\in\mathcal{V}_{2,j}}\exp(\bzin)-\exp(\bzj)\Big)[F(Y|X;\boj,\ej,\nuj)-H(Y|X;\boj)]\\
    &:=B_{n}(Y|X)-C_{n}(Y|X)+E_{n}(Y|X).
\end{align*}
\textbf{Stage 1.2.1:} In this step, we decompose the term $B_{n}(Y|X)$:
\begin{align*}
    B_{n}(Y|X)&=\sum_{j\in[k^*_2]:|\mathcal{V}_{2,j}|=1}\sum_{i\in\mathcal{V}_{2,j}}\exp(\bzin)[F(Y|X;\boin,\ein,\nuin)-F(Y|X;\boj,\ej,\nuj)]\\
    &+\sum_{j\in[k^*_2]:|\mathcal{V}_{2,j}|>1}\sum_{i\in\mathcal{V}_{2,j}}\exp(\bzin)[F(Y|X;\boin,\ein,\nuin)-F(Y|X;\boj,\ej,\nuj)]\\
    &:=B_{n,1}(Y|X) + B_{n,2}(Y|X).
\end{align*}
By applying the first-order and second-order Taylor expansions to the function $F(Y|X;\boin,\ein,\nuin)$ around the point $(\boj,\ej,\nuj)$, we have
\begin{align*}
    B_{n,1}(Y|X)=\sum_{j\in[k^*_2]:|\mathcal{V}_{2,j}|=1}\sum_{i\in\mathcal{V}_{2,j}}\exp(\bzin)\sum_{|\alpha|=1}\frac{1}{\alpha!}(\dboijn)^{\alpha_1}(\deijn)^{\alpha_2}(\dnuijn)^{\alpha_3}\\
    \times \frac{\partial F}{\partial\beta_1^{\alpha_1}\partial\eta^{\alpha_2}\partial\nu^{\alpha_3}}(Y|X;\boj,\ej,\nuj)+R_{n,3}(Y|X),\\
    B_{n,2}(Y|X)=\sum_{j\in[k^*_2]:|\mathcal{V}_{2,j}|>1}\sum_{i\in\mathcal{V}_{2,j}}\exp(\bzin)\sum_{|\alpha|=1}^{2}\frac{1}{\alpha!}(\dboijn)^{\alpha_1}(\deijn)^{\alpha_2}(\dnuijn)^{\alpha_3}\\
    \times \frac{\partial^{|\alpha|} F}{\partial\beta_1^{\alpha_1}\partial\eta^{\alpha_2}\partial\nu^{\alpha_3}}(Y|X;\boj,\ej,\nuj)+R_{n,4}(Y|X),
\end{align*}
where $R_{n,3}(Y|X)$ and $R_{n,4}(Y|X)$ are the Taylor remainders such that $R_{n,3}(Y|X)/\mathcal{D}_{1n}\to0$ and $R_{n,4}(Y|X)/\mathcal{D}_{1n}\to0$ as $n\to\infty$. By means of the chain rule, the first-order derivatives of the function $F$ w.r.t its parameters $\beta_1,\eta,\nu$ are given by
\begin{align*}
    \frac{\partial F}{\partial\beta_1^{(u)}}(Y|X;\boj,\ej,\nuj)&=X^{(u)}\exp((\boj)^{\top}X)\pi(Y|h_2(X,\ej),\nuj),\\
    \frac{\partial F}{\partial\eta^{(u_2)}}(Y|X;\boj,\ej,\nuj)&=\frac{\partial h_2}{\partial\eta^{(u_2)}}(X,\ej)\exp((\boj)^{\top}X)\frac{\partial\pi}{\partial h_2}(Y|h_2(X,\ej),\nuj),\\
    \frac{\partial F}{\partial\nu}(Y|X;\boj,\ej,\nuj)&=\frac{1}{2}\exp((\boj)^{\top}X)\frac{\partial^2\pi}{\partial h_2^2}(Y|h_2(X,\ej),\nuj),
\end{align*}
for all $u_2\in[d_2]$. Similarly, we can derive the second-order derivatives of the function $F$ w.r.t its parameters as follows:
\begin{align*}
    \frac{\partial^2 F}{\partial\beta_1^{(u)}\partial\beta_1^{(v)}}(Y|X;\boj,\ej,\nuj)&=X^{(u)}X^{(v)}\exp((\boj)^{\top}X)\pi(Y|h_2(X,\ej),\nuj),\\
    \frac{\partial^2 F}{\partial\eta^{(u_2)}\partial\eta^{(v_2)}}(Y|X;\boj,\ej,\nuj)&=\frac{\partial^2 h_2}{\partial\eta^{(u_2)}\partial\eta^{(v_2)}}(X,\ej)\exp((\boj)^{\top}X)\frac{\partial\pi}{\partial h_2}(Y|h_2(X,\ej),\nuj)\\
    &+\frac{\partial h_2}{\partial\eta^{(u_2)}}(X,\ej)\frac{\partial h_2}{\partial\eta^{(v_2)}}(X,\ej)\exp((\boj)^{\top}X)\frac{\partial^2\pi}{\partial h_2^2}(Y|h_2(X,\ej),\nuj),\\
    \frac{\partial^2 F}{\partial\nu^2}(Y|X;\boj,\ej,\nuj)&=\frac{1}{4}\exp((\boj)^{\top}X)\frac{\partial^4\pi}{\partial h_2^4}(Y|h_2(X,\ej),\nuj),
\end{align*}
and
\begin{align*}
    \frac{\partial^2 F}{\partial\beta_1^{(u)}\partial\eta^{(v_2)}}(Y|X;\boj,\ej,\nuj)&=X^{(u)}\frac{\partial h_2}{\partial\eta^{(v_2)}}(X,\ej)\exp((\boj)^{\top}X)\frac{\partial\pi}{\partial h_2}(Y|h_2(X,\ej),\nuj),\\
    \frac{\partial^2 F}{\partial\beta_1^{(u)}\partial\nu}(Y|X;\boj,\ej,\nuj)&=\frac{1}{2}X^{(u)}\exp((\boj)^{\top}X)\frac{\partial^2\pi}{\partial h_2^2}(Y|h_2(X,\ej),\nuj),\\
    \frac{\partial^2 F}{\partial\eta^{(u_2)}\partial\nu}(Y|X;\boj,\ej,\nuj)&=\frac{1}{2}\frac{\partial h_2}{\partial\eta^{(u_2)}}(X,\ej)\exp((\boj)^{\top}X)\frac{\partial^3\pi}{\partial h_2^3}(Y|h_2(X,\ej),\nuj),
\end{align*}
for all $u_2,v_2\in[d_2]$. Putting the above results together, we can rewrite $B_{n,1}(Y|X)$ as 
\begin{align*}
    B_{n,1}(Y|X)=\sum_{j\in[k^*_2]:|\mathcal{V}_{2,j}|=1}\Big[B^{(j)}_{n,1,0}(X)\pi(Y|h_2(X,\ej),\nuj)+B^{(j)}_{n,1,1}(X)\frac{\partial\pi}{\partial h_2}(Y|h_2(X,\ej),\nuj)\\
    +B^{(j)}_{n,1,2}(X)\frac{\partial^2\pi}{\partial h_2^2}(Y|h_2(X,\ej),\nuj)\Big]+R_{n,3}(Y|X),
\end{align*}
where we denote
\begin{align*}
    B^{(j)}_{n,1,0}(X)&:=\sum_{i\in\mathcal{V}_{2,j}}\exp(\bzin)\sum_{u=1}^{d}(\dboijn)^{(u)}X^{(u)}\exp((\boj)^{\top}X),\\
    B^{(j)}_{n,1,1}(X)&:=\sum_{i\in\mathcal{V}_{2,j}}\exp(\bzin)\sum_{u_2=1}^{d_2}(\deijn)^{(u_2)}\frac{\partial h_2}{\partial\eta^{(u_2)}}(X,\ej)\exp((\boj)^{\top}X),\\
    B^{(j)}_{n,1,2}(X)&:=\sum_{i\in\mathcal{V}_{2,j}}\exp(\bzin)\frac{1}{2}(\dnuijn)\exp((\boj)^{\top}X),
\end{align*}
for all $j\in[k^*_2]$ such that $|\mathcal{V}_{2,j}|=1$. Analogously, we can represent the term $B_{n,2}(Y|X)$ as
\begin{align*}
    B_{n,2}(Y|X)=\sum_{j\in[k^*_2]:|\mathcal{V}_{2,j}|=1}\sum_{\rho=0}^{4}B^{(j)}_{n,2,\rho}(X)\frac{\partial^{\rho}\pi}{\partial h_2^{\rho}}(Y|h_2(X,\ej),\nuj)+R_{n,4}(Y|X),
\end{align*}
where we define
\begin{align*}
    B^{(j)}_{n,2,0}(X)&:=\sum_{i\in\mathcal{V}_{2,j}}\exp(\bzin)\Bigg[\sum_{u=1}^{d}(\dboijn)^{(u)}X^{(u)}+\sum_{u,v=1}^{d}\frac{(\dboijn)^{(u)}(\dboijn)^{(v)}}{1+1_{\{u=v\}}}X^{(u)}X^{(v)}\Bigg]\exp((\boj)^{\top}X),\\
    B^{(j)}_{n,2,1}(X)&:=\sum_{i\in\mathcal{V}_{2,j}}\exp(\bzin)\Bigg[\sum_{u_2=1}^{d_2}(\deijn)^{(u_2)}\frac{\partial h_2}{\partial\eta^{(u_2)}}(X,\ej)+\sum_{u_2,v_2=1}^{d_2}\frac{(\deijn)^{(u_2)}(\deijn)^{(v_2)}}{1+1_{\{u_2=v_2\}}}\frac{\partial^2h_2}{\partial\eta^{(u_2)}\partial\eta^{(v_2)}}(X,\ej)\\
    &+\sum_{u=1}^{d}\sum_{v_2=1}^{d_2}(\dboijn)^{(u)}(\deijn)^{(v_2)}X^{(u)}\frac{\partial h_2}{\partial\eta^{(v_2)}}(X,\ej)\Bigg]\exp((\boj)^{\top}X),\\
    B^{(j)}_{n,2,2}(X)&:=\sum_{i\in\mathcal{V}_{2,j}}\exp(\bzin)\Bigg[\frac{1}{2}(\dnuijn)+\sum_{u_2,v_2=1}^{d_2}\frac{(\deijn)^{(u_2)}(\deijn)^{(v_2)}}{1+1_{\{u_2=v_2\}}}\frac{\partial h_2}{\partial\eta^{(u_2)}}(X,\ej)\frac{\partial h_2}{\partial\eta^{(v_2)}}(X,\ej)\\
    &\hspace{4cm}+\sum_{u=1}^{d}\frac{1}{2}(\dboijn)^{(u)}(\dnuijn)X^{(u)}\Bigg]\exp((\boj)^{\top}X),\\
    B^{(j)}_{n,2,3}(X)&:=\sum_{i\in\mathcal{V}_{2,j}}\exp(\bzin)\sum_{u_2=1}^{d_2}\frac{1}{2}(\deijn)^{(u_2)}(\dnuijn)\frac{\partial h_2}{\partial\eta^{(u_2)}}(X,\ej)\exp((\boj)^{\top}X),\\
    B^{(j)}_{n,2,4}(X)&:=\sum_{i\in\mathcal{V}_{2,j}}\exp(\bzin)\frac{1}{8}(\dnuijn)^2\exp((\boj)^{\top}X),
\end{align*}
for all $j\in[k^*_2]$ such that $|\mathcal{V}_{2,j}|>1$.

\textbf{Stage 1.2.2:} In this step, we decompose the term $C_{n}(Y|X)$:
\begin{align*}
    C_{n}(Y|X)&=\sum_{j\in[k^*_2]:|\mathcal{V}_{2,j}|=1}\sum_{i\in\mathcal{V}_{2,j}}\exp(\bzin)[H(Y|X;\boin)-H(Y|X;\boj)]\\
    &+\sum_{j\in[k^*_2]:|\mathcal{V}_{2,j}|>1}\sum_{i\in\mathcal{V}_{2,j}}\exp(\bzin)[H(Y|X;\boin)-H(Y|X;\boj)]\\
    &:=C_{n,1}(Y|X) + C_{n,2}(Y|X).
\end{align*}
By means of the first-order and second-order Taylor expansions to the function $H(Y|X;\boin)$ around the point $\boj$, we get
\begin{align*}
    C_{n,1}(Y|X)&=\sum_{j\in[k^*_2]:|\mathcal{V}_{2,j}|=1}\sum_{i\in\mathcal{V}_{2,j}}\exp(\bzin)\sum_{u=1}^{d}(\dboijn)^{(u)}X^{(u)}H(Y|X;\boj)+R_{n,5}(Y|X),\\
    C_{n,2}(Y|X)&=\sum_{j\in[k^*_2]:|\mathcal{V}_{2,j}|>1}\sum_{i\in\mathcal{V}_{2,j}}\exp(\bzin)\Bigg[\sum_{u=1}^{d}(\dboijn)^{(u)}X^{(u)}H(Y|X;\boj)\\
    &\hspace{1.5cm}+\sum_{u,v=1}^{d}\frac{(\dboijn)^{(u)}(\dboijn)^{(v)}}{1+1_{\{u=v\}}}X^{(u)}X^{(v)}H(Y|X;\boj)\Bigg]+R_{n,6}(Y|X),
\end{align*}
where $R_{n,5}(Y|X)$ and $R_{n,6}(Y|X)$ are the Taylor remainders such that $R_{n,5}(Y|X)/\mathcal{D}_{1n}\to0$ and $R_{n,6}(Y|X)/\mathcal{D}_{1n}\to0$ as $n\to\infty$.

Putting the above decompositions together, we can view $A_{n,0}(Y|X)/\mathcal{D}_{1n}$, $[A_{n,1}(Y|X)-R_{n,1}(Y|X)]/\mathcal{D}_{1n}$, $[A_{n,2}(Y|X)-R_{n,2}(Y|X)]/\mathcal{D}_{1n}$, $[B_{n,1}(Y|X)-R_{n,3}(Y|X)]/\mathcal{D}_{1n}$, $[B_{n,2}(Y|X)-R_{n,4}(Y|X)]/\mathcal{D}_{1n}$, $[C_{n,1}(Y|X)-R_{n,5}(Y|X)]/\mathcal{D}_{1n}$, $[C_{n,2}(Y|X)-R_{n,6}(Y|X)]/\mathcal{D}_{1n}$, and $E_n(Y|X)/\mathcal{D}_{1n}$ as a combination of elements of the following sets
\begin{align*}
    \mathcal{S}_{0,j}&:=\{\pi(Y|h_1(X,\kj),\tj)\},\\
    \mathcal{S}_{1,j}&:=\Bigg\{\frac{\partial h_1}{\partial\kappa^{(u_1)}}(X,\kj)\frac{\partial\pi}{\partial h_1}(Y|h_1(X,\kj),\tj), \ \frac{\partial^2 h_1}{\partial\kappa^{(u_1)}\partial\kappa^{(v_1)}}(X,\kj)\frac{\partial\pi}{\partial h_1}(Y|h_1(X,\kj),\tj):u_1,v_1\in[d_1]\Bigg\},\\
    \mathcal{S}_{2,j}&:=\Bigg\{\frac{\partial^2\pi}{\partial h_1^2}(Y|h_1(X,\kj),\tj), \ \frac{\partial h_1}{\partial\kappa^{(u_1)}}(X,\kj)\frac{\partial h_1}{\partial\kappa^{(v_1)}}(X,\kj)\frac{\partial^2\pi}{\partial h_1^2}(Y|h_1(X,\kj),\tj):u_1,v_1\in[d_1]\Bigg\},\\
    \mathcal{S}_{3,j}&:=\Bigg\{ \frac{\partial h_1}{\partial\kappa^{(u_1)}}(X,\kj)\frac{\partial^3\pi}{\partial h_1^3}(Y|h_1(X,\kj),\tj):u_1,v_1\in[d_1]\Bigg\},\\
    \mathcal{S}_{4,j}&:=\Bigg\{ \frac{\partial^4\pi}{\partial h_1^4}(Y|h_1(X,\kj),\tj):u_1,v_1\in[d_1]\Bigg\},
\end{align*}
for all $j\in[k^*_1]$, and
\begin{align*}
    \mathcal{T}_{0,j}&:=\{F(Y|X;\boj,\ej,\nuj), \ X^{(u_2)}F(Y|X;\boj,\ej,\nuj), \ X^{(u_2)}X^{(v_2)}F(Y|X;\boj,\ej,\nuj):u_2,v_2\in[d_2]\},\\
    \mathcal{T}_{1,j}&:=\Bigg\{\frac{\partial h_2}{\partial\eta^{(u_2)}}(X,\ej)F_1(Y|X;\boj,\ej,\nuj), \ \frac{\partial^2 h_2}{\partial\eta^{(u_2)}\partial\eta^{(v_2)}}(X,\ej)F_1(Y|X;\boj,\ej,\nuj),\\
    &\hspace{1cm}X^{(u_2)}\frac{\partial h_2}{\partial\eta^{(u_2)}}(X,\ej)F_1(Y|X;\boj,\ej,\nuj):u_2,v_2\in[d_2]\Bigg\},\\
    \mathcal{T}_{2,j}&:=\Bigg\{F_2(Y|X;\boj,\ej,\nuj), \ \frac{\partial h_2}{\partial\eta^{(u_2)}}(X,\ej)\frac{\partial h_2}{\partial\eta^{(v_2)}}(X,\ej)F_2(Y|X;\boj,\ej,\nuj),\\
    &\hspace{1cm}X^{(u_2)}F_2(Y|X;\boj,\ej,\nuj):u_2,v_2\in[d_2]\Bigg\},\\
    \mathcal{T}_{3,j}&:=\Bigg\{\frac{\partial h_2}{\partial\eta^{(u_2)}}(X,\ej)F_3(Y|X;\boj,\ej,\nuj):u_2\in[d_2]\Bigg\},\\
    \mathcal{T}_{4,j}&:=\{F_4(Y|X;\boj,\ej,\nuj)\},\\
    \mathcal{T}_{5,j}&:=\{H(Y|X;\boj), \ X^{(u)}H(Y|X;\boj), \ X^{(u)}X^{(v)}H(Y|X;\boj):u,v\in[d]\},
\end{align*}
where we denote
\begin{align*}
    F_{\rho}(Y|X;\boj,\ej,\nuj):=\exp((\boj)^{\top}X)\frac{\partial^{\rho}\pi}{\partial h_1^{\rho}}(Y|h_2(X,\ej),\nuj),
\end{align*}
for all $\rho\in[4]$ and $j\in[k^*_2]$.

\vspace{0.5 em}
\noindent
\textbf{Stage 2 - Non-vanishing coefficients:} In this stage, we show by contradiction that not all the coefficients in the representations of $A_{n,0}(Y|X)/\mathcal{D}_{1n}$, $[A_{n,1}(Y|X)-R_{n,1}(Y|X)]/\mathcal{D}_{1n}$, $[A_{n,2}(Y|X)-R_{n,2}(Y|X)]/\mathcal{D}_{1n}$, $[B_{n,1}(Y|X)-R_{n,3}(Y|X)]/\mathcal{D}_{1n}$, $[B_{n,2}(Y|X)-R_{n,4}(Y|X)]/\mathcal{D}_{1n}$, $[C_{n,1}(Y|X)-R_{n,5}(Y|X)]/\mathcal{D}_{1n}$, $[C_{n,2}(Y|X)-R_{n,6}(Y|X)]/\mathcal{D}_{1n}$, and $E_n(Y|X)/\mathcal{D}_{1n}$ converge to zero as $n\to\infty$. In particular, we assume that all those coefficients go to zero. Then, by looking into the coefficients of the terms:
\begin{itemize}
    \item $\pi(Y|h_1(X,\kj),\tj)$ for $j\in[k^*_1]$, we have $\frac{1}{\mathcal{D}_{1n}}\cdot\sum_{j=1}^{k^*_1}\Big|\sum_{i\in\mathcal{V}_{1,j}}\oin-\oj\Big|\to0$;
    \item $\frac{\partial h_1}{\partial\kappa^{(u_1)}}(X,\kj)\frac{\partial\pi}{\partial h_1}(Y|h_1(X,\kj),\tj)$ for $j\in[k^*_1]:|\mathcal{V}_{1,j}|=1$ and $u_1\in[d_1]$, we have
    \begin{align*}
        \frac{1}{\mathcal{D}_{1n}}\cdot\sum_{j\in[k^*_1]:|\mathcal{V}_{1,j}|=1}\sum_{i\in\mathcal{V}_{1,j}}\oin\|\dkijn\|\to0;
    \end{align*}
    \item $\frac{\partial^2\pi}{\partial h_1^2}(Y|h_1(X,\kj),\tj)$ for $j\in[k^*_1]:|\mathcal{V}_{1,j}|=1$, we have
    \begin{align*}
        \frac{1}{\mathcal{D}_{1n}}\cdot\sum_{j\in[k^*_1]:|\mathcal{V}_{1,j}|=1}\sum_{i\in\mathcal{V}_{1,j}}\oin|\dtijn|\to0;
    \end{align*}
    \item $\big[\frac{\partial h_1}{\partial\kappa^{(u_1)}}(X,\kj)\big]^2\frac{\partial^2\pi}{\partial h_1^2}(Y|h_1(X,\kj),\tj)$ for $j\in[k^*_1]:|\mathcal{V}_{1,j}|>1$ and $u_1\in[d_1]$, we have
    \begin{align*}
        \frac{1}{\mathcal{D}_{1n}}\cdot\sum_{j\in[k^*_1]:|\mathcal{V}_{1,j}|>1}\sum_{i\in\mathcal{V}_{1,j}}\oin\|\dkijn\|^2\to0;
    \end{align*}
    \item $\frac{\partial^4\pi}{\partial h_1^4}(Y|h_1(X,\kj),\tj)$ for $j\in[k^*_1]:|\mathcal{V}_{1,j}|>1$ and $u_1\in[d_1]$, we have
    \begin{align*}
        \frac{1}{\mathcal{D}_{1n}}\cdot\sum_{j\in[k^*_1]:|\mathcal{V}_{1,j}|>1}\sum_{i\in\mathcal{V}_{1,j}}\oin|\dtijn|^2\to0;
    \end{align*}
    \item $F(Y|X;\boj,\ej,\nuj)$ for $j\in[k^*_2]$, we have $\frac{1}{\mathcal{D}_{1n}}\cdot\sum_{j=1}^{k^*_2}\Big|\sum_{i\in\mathcal{V}_{2,j}}\exp(\bzin)-\exp(\boj)\Big|\to0$;
    \item $X^{(u)}F(Y|X;\boj,\ej,\nuj)$ for $j\in[k^*_2]:|\mathcal{V}_{2,j}|=1$ and $u\in[d]$, we have
    \begin{align*}
        \frac{1}{\mathcal{D}_{1n}}\cdot\sum_{j\in[k^*_2]:|\mathcal{V}_{2,j}|=1}\sum_{i\in\mathcal{V}_{2,j}}\exp(\bzin)\|\dboijn\|\to0;
    \end{align*}
    \item $\frac{\partial h_2}{\partial\eta^{(u_2)}}F_1(Y|X;\boj,\ej,\nuj)$ for $j\in[k^*_2]:|\mathcal{V}_{2,j}|=1$ and $u_2\in[d_2]$, we have
    \begin{align*}
        \frac{1}{\mathcal{D}_{1n}}\cdot\sum_{j\in[k^*_2]:|\mathcal{V}_{2,j}|=1}\sum_{i\in\mathcal{V}_{2,j}}\exp(\bzin)\|\deijn\|\to0;
    \end{align*}
    \item $F_2(Y|X;\boj,\ej,\nuj)$ for $j\in[k^*_2]:|\mathcal{V}_{2,j}|=1$, we have
    \begin{align*}
        \frac{1}{\mathcal{D}_{1n}}\cdot\sum_{j\in[k^*_2]:|\mathcal{V}_{2,j}|=1}\sum_{i\in\mathcal{V}_{2,j}}\exp(\bzin)|\dnuijn|\to0;
    \end{align*}
    \item $X^{(u)}X^{(v)}F(Y|X;\boj,\ej,\nuj)$ for $j\in[k^*_2]:|\mathcal{V}_{2,j}|>1$ and $u,v\in[d]$, we have
    \begin{align*}
        \frac{1}{\mathcal{D}_{1n}}\cdot\sum_{j\in[k^*_2]:|\mathcal{V}_{2,j}|>1}\sum_{i\in\mathcal{V}_{2,j}}\exp(\bzin)\|\dboijn\|^2\to0;
    \end{align*}
    \item $\big[\frac{\partial^2 h_2}{\partial\eta^{(u_2)}}(X,\ej)\big]^2F_2(Y|X;\boj,\ej,\nuj)$ for $j\in[k^*_2]:|\mathcal{V}_{2,j}|>1$ and $u_2\in[d_2]$, we have
    \begin{align*}
        \frac{1}{\mathcal{D}_{1n}}\cdot\sum_{j\in[k^*_2]:|\mathcal{V}_{2,j}|>1}\sum_{i\in\mathcal{V}_{2,j}}\exp(\bzin)\|\deijn\|\to0;
    \end{align*}
    \item $F_4(Y|X;\boj,\ej,\nuj)$ for $j\in[k^*_2]:|\mathcal{V}_{2,j}|>1$, we have
    \begin{align*}
        \frac{1}{\mathcal{D}_{1n}}\cdot\sum_{j\in[k^*_2]:|\mathcal{V}_{2,j}|>1}\sum_{i\in\mathcal{V}_{2,j}}\exp(\bzin)|\dnuijn|^2\to0;
    \end{align*}
\end{itemize}
Taking the sum of the above limits, we deduce $1=\frac{1}{\mathcal{D}_{1n}}\cdot\mathcal{D}_{1n}\to0$ as $n\to\infty$, which is a contradiction. Thus, at least one among the coefficients in the representations of $A_{n,0}(Y|X)/\mathcal{D}_{1n}$, $[A_{n,1}(Y|X)-R_{n,1}(Y|X)]/\mathcal{D}_{1n}$, $[A_{n,2}(Y|X)-R_{n,2}(Y|X)]/\mathcal{D}_{1n}$, $[B_{n,1}(Y|X)-R_{n,3}(Y|X)]/\mathcal{D}_{1n}$, $[B_{n,2}(Y|X)-R_{n,4}(Y|X)]/\mathcal{D}_{1n}$, $[C_{n,1}(Y|X)-R_{n,5}(Y|X)]/\mathcal{D}_{1n}$, $[C_{n,2}(Y|X)-R_{n,6}(Y|X)]/\mathcal{D}_{1n}$, and $E_n(Y|X)/\mathcal{D}_{1n}$ does not converge to zero.

\vspace{0.5 em}
\noindent
\textbf{Stage 3 - Fatou's lemma contradiction:} In this stage, we use the Fatou's lemma to show a contradiction to the result of Stage 2. For that purpose, let us denote $m_n$ as the maximum of the absolute values of the coefficients in the representations of $A_{n,0}(Y|X)/\mathcal{D}_{1n}$, $[A_{n,1}(Y|X)-R_{n,1}(Y|X)]/\mathcal{D}_{1n}$, $[A_{n,2}(Y|X)-R_{n,2}(Y|X)]/\mathcal{D}_{1n}$, $[B_{n,1}(Y|X)-R_{n,3}(Y|X)]/\mathcal{D}_{1n}$, $[B_{n,2}(Y|X)-R_{n,4}(Y|X)]/\mathcal{D}_{1n}$, $[C_{n,1}(Y|X)-R_{n,5}(Y|X)]/\mathcal{D}_{1n}$, $[C_{n,2}(Y|X)-R_{n,6}(Y|X)]/\mathcal{D}_{1n}$, and $E_n(Y|X)/\mathcal{D}_{1n}$. It follows from the result of Stage 2 that $1/m_n\not\to\infty$ as $n\to\infty$. In addition, we also denote
\begin{align*}
    \frac{1}{m_n\mathcal{D}_{1n}}\cdot\sum_{i\in\mathcal{V}_{1,j}}\oin(\dkijn)^{(u_1)}\to s^{(u_1)}_{1,j},& \quad \frac{1}{m_n\mathcal{D}_{1n}}\cdot\sum_{i\in\mathcal{V}_{1,j}}\oin(\dtijn)\to s_{2,j},\\
    \frac{1}{m_n\mathcal{D}_{1n}}\cdot\sum_{i\in\mathcal{V}_{1,j}}\oin(\dkijn)^{(u_1)}(\dkijn)^{(v_1)}\to s^{(u_1v_1)}_{3,j},& \quad \frac{1}{m_n\mathcal{D}_{1n}}\cdot\sum_{i\in\mathcal{V}_{1,j}}\oin(\dtijn)^2\to s_{4,j},\\
    \frac{1}{m_n\mathcal{D}_{1n}}\cdot\sum_{i\in\mathcal{V}_{1,j}}\oin(\dkijn)^{(u_1)}(\dtijn)\to s^{(u_1)}_{5,j},& \quad \frac{1}{m_n\mathcal{D}_{1n}}\cdot\Big(\sum_{i\in\mathcal{V}_{1,j}}\oin-\oj\Big)\to s_{0,j},\\
\end{align*}
for all $j\in[k^*_1]$ and
\begin{align*}
    \frac{1}{m_n\mathcal{D}_{1n}}\cdot\Big(\sum_{i\in\mathcal{V}_{2,j}}\exp(\bzin)-\exp(\bzj)\Big)\to t_{0,j},& \quad \frac{1}{m_n\mathcal{D}_{1n}}\cdot\sum_{i\in\mathcal{V}_{2,j}}\exp(\bzin)(\dboijn)^{(u)}\to t^{(u)}_{1,j},\\
    \frac{1}{m_n\mathcal{D}_{1n}}\cdot\sum_{i\in\mathcal{V}_{2,j}}\exp(\bzin)(\deijn)^{(u_2)}\to t^{(u_2)}_{2,j},& \quad \frac{1}{m_n\mathcal{D}_{1n}}\cdot\sum_{i\in\mathcal{V}_{2,j}}\exp(\bzin)(\dnuijn)\to t_{3,j},\\
    \frac{1}{m_n\mathcal{D}_{1n}}\cdot\sum_{i\in\mathcal{V}_{2,j}}\exp(\bzin)(\dboijn)^{(u)}(\dboijn)^{(v)}\to t^{(uv)}_{4,j},& \quad \frac{1}{m_n\mathcal{D}_{1n}}\cdot\sum_{i\in\mathcal{V}_{2,j}}\exp(\bzin)(\deijn)^{(u_2)}(\deijn)^{(v_2)}\to t^{(u_2v_2)}_{5,j},\\
    \frac{1}{m_n\mathcal{D}_{1n}}\cdot\sum_{i\in\mathcal{V}_{2,j}}\exp(\bzin)(\dnuijn)^2\to t_{6,j},& \quad \frac{1}{m_n\mathcal{D}_{1n}}\cdot\sum_{i\in\mathcal{V}_{2,j}}\exp(\bzin)(\dboijn)^{(u)}(\deijn)^{(v_2)}\to t^{(uv_2)}_{7,j},\\
    \frac{1}{m_n\mathcal{D}_{1n}}\cdot\sum_{i\in\mathcal{V}_{2,j}}\exp(\bzin)(\dboijn)^{(u)}(\dnuijn)\to t^{(u)}_{8,j},& \quad \frac{1}{m_n\mathcal{D}_{1n}}\cdot\sum_{i\in\mathcal{V}_{2,j}}\exp(\bzin)(\deijn)^{(u_2)}(\dnuijn)\to t^{(u_2)}_{9,j},
\end{align*}
for all $j\in[k^*_2]$ as $n\to\infty$. Due to the result of Stage 2, at least one among the above limits is different from zero. Recall from equation~\eqref{eq:expectation_zero} that we have
\begin{align*}
    \bbE_X[V(f_{G^n_1,G^n_2}(\cdot|X),f_{G^*_1,G^*_2}(\cdot|X))]/\mathcal{D}_{1n}\to0,
\end{align*}
Furthermore, according to the Fatou's lemma, we get
\begin{align*}
    \lim_{n\to\infty}\dfrac{\bbE_X[V(f_{G^n_1,G^n_2}(\cdot|X),f_{G^*_1,G^*_2}(\cdot|X))]}{m_n\mathcal{D}_{1n}}\geq\int\liminf_{n\to\infty}\dfrac{|f_{G^n_1,G^n_2}(Y|X)-f_{G^*_1,G^*_2}(Y|X)|}{2m_n\mathcal{D}_{1n}}\dint (X,Y).
\end{align*}
Then, we deduce $[f_{G^n_1,G^n_2}(Y|X)-f_{G^*_1,G^*_2}(Y|X)]/[m_n\mathcal{D}_{1n}]\to0$ as $n\to\infty$ for almost surely $(X,Y)$. Since the input space is bounded and the parameter space is compact, the quantity $\sum_{j = 1}^{k^*_2} \exp((\beta_{1j}^{*})^{\top} X + \beta_{0j}^{*})$ is bounded. Thus, we also have
\begin{align*}
    \Big[\sum_{j = 1}^{k^*_2} \exp((\beta_{1j}^{*})^{\top} X + \beta_{0j}^{*})\Big][f_{G^n_1,G^n_2}(Y|X)-f_{G^*_1,G^*_2}(Y|X)]/[m_n\mathcal{D}_{1n}]\to0,
\end{align*}
implying that 
\begin{align*}
    \frac{1}{2}\Big[\sum_{j = 1}^{k^*_2} \exp((\beta_{1j}^{*})^{\top} X + \beta_{0j}^{*})\Big]\cdot\dfrac{q_{G^n_1}(Y|X)-q_{G^*_1}(Y|X)}{m_n\mathcal{D}_{1n}}+\frac{1}{2}\dfrac{Q_n(Y|X)}{m_n\mathcal{D}_{1n}}\to0.
\end{align*}
as $n\to\infty$ for almost surely $(X,Y)$. From the decomposition of the terms $q_{G^n_1}(Y|X)-q_{G^*_1}(Y|X)$ and $Q_n(Y|X)$ in Stage 1, we have
\begin{align}
    \frac{1}{2}\Big[\sum_{j = 1}^{k^*_2} \exp((\beta_{1j}^{*})^{\top} X + \beta_{0j}^{*})\Big]\cdot\dfrac{A_{n,2}(Y|X)+A_{n,1}(Y|X)+A_{n,0}(Y|X)}{m_n\mathcal{D}_{1n}}\nonumber\\
    \label{eq:zero_limit}
    +\frac{1}{2}\dfrac{B_{n,1}(Y|X)+B_{n,2}(Y|X)-C_{n,1}(Y|X)-C_{n,2}(Y|X)+E_{n}(Y|X)}{m_n\mathcal{D}_{1n}}\to0.
\end{align}
Denote $F_{\rho,j}(Y|X):=F_{\rho}(Y|X;\boj,\ej,\nuj)$ and $H_{j}(Y|X):=H(Y|X,\boj)$, we have
\begin{align*}
    &\lim_{n\to\infty}\frac{A_{n,0}(Y|X)}{m_n\mathcal{D}_{1n}}=\sum_{j=1}^{k^*_1}s_{0,j}\pi(Y|h_1(X,\kappa^*_j),\tau^*_j),\\
    &\lim_{n\to\infty}\frac{A_{n,1}(Y|X)}{m_n\mathcal{D}_{1n}}=\sum_{j\in[k^*_1]:|\mathcal{V}_{1,j}|=1}\Big[\sum_{u_1=1}^{d_1}s_{1,j}^{(u_1)}\frac{\partial h_1}{\partial\kappa^{(u_1)}}(X,\kj)\frac{\partial\pi}{\partial h_1}(Y|h_1(X,\kj),\tj)\\
    &\hspace{10cm}+\frac{1}{2}s_{2,j}\frac{\partial^2\pi}{\partial h_1^2}(Y|h_1(X,\kj),\tj)\Big],\\
    &\lim_{n\to\infty}\frac{A_{n,2}(Y|X)}{m_n\mathcal{D}_{1n}}=\sum_{j\in[k^*_1]:|\mathcal{V}_{1,j}|>1}\Big[\Big(\sum_{u_1=1}^{d_1}s_{1,j}^{(u_1)}\frac{\partial h_1}{\partial\kappa^{(u_1)}}(X,\kj)+\sum_{u_1,v_1=1}^{d_1}\frac{s_{3,j}^{(u_1v_1)}}{1+1_{\{u_1=v_1\}}}\frac{\partial^2h_1}{\partial\kappa^{(u_1)}\partial\kappa^{(v_1)}}(X,\kj)\Big)\\
    &\times \frac{\partial\pi}{\partial h_1}(Y|h_1(X,\kj),\tj)+\Big(\frac{1}{2}s_{2,j}+\sum_{u_1,v_1=1}^{d_1}\frac{s_{3,j}^{(u_1v_1)}}{1+1_{\{u_1=v_1\}}}\frac{\partial h_1}{\partial\kappa^{(u_1)}}(X,\kj)\frac{\partial h_1}{\partial\kappa^{(v_1)}}(X,\kj)\Big)\frac{\partial^2\pi}{\partial h_1^2}(Y|h_1(X,\kj),\tj)\\
    &\hspace{3cm}+\Big(\frac{1}{2}\sum_{u_1=1}^{d_1}s_{5,j}^{(u_1)}\frac{\partial h_1}{\partial\kappa^{(u_1)}}(X,\kj)\Big)\frac{\partial^3\pi}{\partial h_1^3}(Y|h_1(X,\kj),\tj)+\frac{1}{8}s_{4,j}\frac{\partial^4\pi}{\partial h_1^4}(Y|h_1(X,\kj),\tj)\Big],
\end{align*}
and
\begin{align*}
    &\lim_{n\to\infty}\frac{B_{n,1}(Y|X)}{m_n\mathcal{D}_{1n}}=\sum_{j\in[k^*_2]:|\mathcal{V}_{2,j}|=1}\Big[\sum_{u=1}^{d}t_{1,j}^{(u)}X^{(u)}F_{0,j}(Y|X)+\sum_{u_2=1}^{d_2}t_{2,j}^{(u_2)}\frac{\partial h_2}{\partial\eta^{(u_2)}}(X,\ej)F_{1,j}(Y|X)+\frac{1}{2}t_{3,j}F_{2,j}(Y|X)\Big],\\
    &\lim_{n\to\infty}\frac{B_{n,2}(Y|X)}{m_n\mathcal{D}_{1n}}=\sum_{j\in[k^*_2]:|\mathcal{V}_{2,j}|>1}\Big[\Big(\sum_{u=1}^{d}t_{1,j}^{(u)}X^{(u)}+\sum_{u,v=1}^{d}t_{4,j}^{(uv)}X^{(u)}X^{(v)}\Big)F_{0,j}(Y|X)\\
    &+\Big(\sum_{u_2=1}^{d_2}t_{2,j}^{(u_2)}\frac{\partial h_2}{\partial\eta^{(u_2)}}(X,\ej)+\sum_{u_2,v_2=1}^{d_2}t_{5,j}^{(u_2v_2)}\frac{\partial^2 h_2}{\partial\eta^{(u_2)}\partial\eta^{(v_2)}}(X,\ej)+\sum_{u=1}^{d}\sum_{v_2=1}^{d_2}t_{7,j}^{(uv_2)}X^{(u)}\frac{\partial h_2}{\partial\eta^{(v_2)}}(X,\ej)\Big)F_{1,j}(Y|X)\\
    &+\Big(\frac{1}{2}t_{3,j}+\sum_{u_2,v_2=1}^{d_2}t_{5,j}^{(u_2v_2)}\frac{\partial h_2}{\partial\eta^{(u_2)}}(X,\ej)\frac{\partial h_2}{\partial\eta^{(v_2)}}(X,\ej)+\sum_{u=1}^{d}\frac{1}{2}t_{8,j}^{(u)}X^{(u)}\Big)F_{2,j}(Y|X)\\
    &+\Big(\sum_{u_2=1}^{d_2}\frac{1}{2}t_{9,j}^{(u_2)}\frac{\partial h_2}{\partial\eta^{(u_2)}}(X,\ej)\Big)F_{3,j}(Y|X)+\frac{1}{8}t_{6,j}F_{4,j}(Y|X)\Big],
\end{align*}
and
\begin{align*}
    &\lim_{n\to\infty}\frac{C_{n,1}(Y|X)}{m_n\mathcal{D}_{1n}}=\sum_{j\in[k^*_2]:|\mathcal{V}_{2,j}|=1}\sum_{u=1}^{d}t_{1,j}^{(u)}X^{(u)}H_{j}(Y|X),\\
    &\lim_{n\to\infty}\frac{C_{n,2}(Y|X)}{m_n\mathcal{D}_{1n}}=\sum_{j\in[k^*_2]:|\mathcal{V}_{2,j}|>1}\Big(\sum_{u=1}^{d}t_{1,j}^{(u)}X^{(u)}+\sum_{u,v=1}^{d}t_{4,j}^{(uv)}X^{(u)}X^{(v)}\Big)H_{j}(Y|X),\\
    &\lim_{n\to\infty}\frac{E_{n}(Y|X)}{m_n\mathcal{D}_{1n}}=\sum_{j=1}^{k^*_2}t_{0,j}[F_{0,j}(Y|X)-H_{j}(Y|X)].
\end{align*}
It is worth noting that for almost every $X$, the set
\begin{align*}
    \Bigg\{\Big[\sum_{j = 1}^{k^*_2} \exp((\beta_{1j}^{*})^{\top} X + \beta_{0j}^{*})\Big]\frac{\partial^{\rho}\pi}{\partial h_1^{\rho}}(Y|h_1(X,\kj),\tj):0\leq\rho\leq 4, \  j\in[k^*_1]\Bigg\}\\
    \cup~\Big\{F_{\rho}(Y|X;\boj,\ej,\nuj), \ H(Y|X;\boj):0\leq\rho\leq 4, \  j\in[k^*_2]\Big\}
\end{align*}
is linearly independent w.r.t $Y$. Therefore, it follows that the coefficients of those terms in the limit in equation~\eqref{eq:zero_limit} become zero. 

For $j\in[k^*_1]$, by looking at the coefficient of the term $\Big[\sum_{j = 1}^{k^*_2} \exp((\beta_{1j}^{*})^{\top} X + \beta_{0j}^{*})\Big]\pi(Y|h_1(X,\kj),\tj)$, we have $s_{0,j}=0$.

For $j\in[k^*_1]$ such that $|\mathcal{V}_{1,j}|=1$, by considering the coefficients of 
\begin{itemize}
    \item $\Big[\sum_{j = 1}^{k^*_2} \exp((\beta_{1j}^{*})^{\top} X + \beta_{0j}^{*})\Big]\frac{\partial\pi}{\partial h_1}(Y|h_1(X,\kj),\tj)$, we have $\sum_{u_1=1}^{d_1}s_{1,j}^{(u_1)}\frac{\partial h_1}{\partial\kappa^{(u_1)}}(X,\kj)=0$ for almost every $X$. Since the expert function $h_1$ is strongly identifiable, we get $s_{1,j}^{(u_1)}=0$ for all $u_1\in[d_1]$;
    \item $\Big[\sum_{j = 1}^{k^*_2} \exp((\beta_{1j}^{*})^{\top} X + \beta_{0j}^{*})\Big]\frac{\partial^2\pi}{\partial h_1^2}(Y|h_1(X,\kj),\tj)$, we have $s_{2,j}=0$.
\end{itemize}
For $j\in[k^*_1]$ such that $|\mathcal{V}_{1,j}|>1$, by considering the coefficients of 
\begin{itemize}
    \item $\Big[\sum_{j = 1}^{k^*_2} \exp((\beta_{1j}^{*})^{\top} X + \beta_{0j}^{*})\Big]\frac{\partial\pi}{\partial h_1}(Y|h_1(X,\kj),\tj)$, we have
    \begin{align*}
        \sum_{u_1=1}^{d_1}s_{1,j}\frac{\partial h_1}{\partial\kappa^{(u_1)}}(X,\kj)+\sum_{u_1,v_1=1}^{d_1}\frac{s_{3,j}^{(u_1v_1)}}{1+1_{\{u_1=v_1\}}}\frac{\partial^2h_1}{\partial\kappa^{(u_1)}\partial\kappa^{(v_1)}}(X,\kj)=0,
    \end{align*}
    for almost every $X$. Since the expert function $h_1$ satisfies the strong identifiability condition, we get $s_{1,j}^{(u_1)}=s_{3,j}^{(u_1v_1)}=0$ for all $u_1,v_1\in[d_1]$;
    \item $\Big[\sum_{j = 1}^{k^*_2} \exp((\beta_{1j}^{*})^{\top} X + \beta_{0j}^{*})\Big]\frac{\partial^2\pi}{\partial h_1^2}(Y|h_1(X,\kj),\tj)$, we have
    \begin{align*}
        \frac{1}{2}s_{2,j}+\sum_{u_1,v_1=1}^{d_1}\frac{s_{3,j}^{(u_1v_1)}}{1+1_{\{u_1=v_1\}}}\frac{\partial h_1}{\partial\kappa^{(u_1)}}(X,\kj)\frac{\partial h_1}{\partial\kappa^{(v_1)}}(X,\kj)=0,
    \end{align*}
    for almost every $X$. Since $s_{3,j}^{(u_1v_1)}=0$ for all $u_1,v_1\in[d_1]$, we deduce $s_{2,j}=0$;
    \item $\Big[\sum_{j = 1}^{k^*_2} \exp((\beta_{1j}^{*})^{\top} X + \beta_{0j}^{*})\Big]\frac{\partial^3\pi}{\partial h_1^3}(Y|h_1(X,\kj),\tj)$, we have $\frac{1}{2}\sum_{u_1=1}^{d_1}s_{5,j}^{(u_1)}\frac{\partial h_1}{\partial\kappa^{(u_1)}}(X,\kj)=0$, for almost every $X$. As the expert function $h_1$ meets the strong identifiability condition, we get $s_{5,j}^{(u_1)}=0$ for all $u_1\in[d_1]$;
    \item $\Big[\sum_{j = 1}^{k^*_2} \exp((\beta_{1j}^{*})^{\top} X + \beta_{0j}^{*})\Big]\frac{\partial^4\pi}{\partial h_1^4}(Y|h_1(X,\kj),\tj)$, we have $s_{4,j}=0$.
\end{itemize}
For $j\in[k^*_2]$ such that $|\mathcal{V}_{2,j}|=1$, by considering the coefficients of 
\begin{itemize}
    \item $F_{0,j}(Y|X)$, we have $t_{0,j}+\sum_{u=1}^{d}t_{1,j}^{(u)}X^{(u)}=0$, for almost every $X$. Then, we deduce $t_{0,j}=t_{1,j}^{(u)}=0$ for all $u\in[d]$;
    \item $F_{1,j}(Y|X)$, we have $\sum_{u_2=1}^{d_2}t_{2,j}^{(u_2)}\frac{\partial h_2}{\partial\eta^{(u_2)}}(X,\ej)$, for almost every $X$. As the expert function $h_2$ is strongly identifiable, we get $t_{2,j}^{(u_2)}=0$ for all $u_2\in[d_2]$;
    \item $F_{2,j}(Y|X)$, we have $t_{3,j}=0$.
\end{itemize}
For $j\in[k^*_2]$ such that $|\mathcal{V}_{2,j}|>1$, by considering the coefficients of 
\begin{itemize}
    \item $F_{0,j}(Y|X)$, we have $t_{0,j}+\sum_{u=1}^{d}t_{1,j}^{(u)}X^{(u)}+\sum_{u,v=1}^{d}t_{4,j}^{(uv)}X^{(u)}X^{(v)}=0$, for almost surely $X$. Then, we get $t_{0,j}=t_{1,j}^{(u)}=t_{4,j}^{(uv)}$ for all $u,v\in[d]$.
    \item $F_{1,j}(Y|X)$, we have
    \begin{align*}
        \sum_{u_2=1}^{d_2}t_{2,j}^{(u_2)}\frac{\partial h_2}{\partial\eta^{(u_2)}}(X,\ej)+\sum_{u_2,v_2=1}^{d_2}t_{5,j}^{(u_2v_2)}\frac{\partial^2 h_2}{\partial\eta^{(u_2)}\partial\eta^{(v_2)}}(X,\ej)+\sum_{u=1}^{d}\sum_{v_2=1}^{d_2}t_{7,j}^{(uv_2)}X^{(u)}\frac{\partial h_2}{\partial\eta^{(v_2)}}(X,\ej)=0,
    \end{align*}
    for almost every $X$. As the expert function $h_2$ meets the strong identifiability condition, we get $t_{2,j}^{(u_2)}=t_{5,j}^{(u_2v_2)}=t_{7,j}^{(uv_2)}=0$ for all $u_2,v_2\in[d_2]$ and $u\in[d]$;
    \item $F_{2,j}(Y|X)$, we have 
    \begin{align*}
        \frac{1}{2}t_{3,j}+\sum_{u_2,v_2=1}^{d_2}t_{5,j}^{(u_2v_2)}\frac{\partial h_2}{\partial\eta^{(u_2)}}(X,\ej)\frac{\partial h_2}{\partial\eta^{(v_2)}}(X,\ej)+\sum_{u=1}^{d}\frac{1}{2}t_{8,j}^{(u)}X^{(u)}=0,
    \end{align*}
    for almost every $X$. Since $t_{5,j}^{(u_2v_2)}=0$ for all $u_2,v_2\in[d_2]$, we deduce $\frac{1}{2}t_{3,j}+\sum_{u=1}^{d}\frac{1}{2}t_{8,j}^{(u)}X^{(u)}=0$, for almost every $X$. Then, we get $t_{3,j}=t_{8,j}^{(u)}=0$ for all $u_2,v_2\in[d_2]$ and $u\in[d]$;
    \item $F_{3,j}(Y|X)$, we have $\sum_{u_2=1}^{d_2}\frac{1}{2}t_{9,j}^{(u_2)}\frac{\partial h_2}{\partial\eta^{(u_2)}}(X,\ej)=0$, for almost every $X$. As the expert function $h_2$ is strongly identifiable, we get $t_{9,j}^{(u_2)}$ for all $u_2\in[d_2]$;
    \item $F_{4,j}(Y|X)$, we have $t_{6,j}=0$.
\end{itemize}
Putting the above results together, we have (i) $s_{0,j}=s_{1,j}^{(u_1)}=s_{2,j}=s_{3,j}^{(u_1v_1)}=s_{4,j}=s_{5,j}^{(u_1)}=0$ for all $j\in[k^*_1]$ and $u_1,v_1\in[d_1]$; (ii) $t_{0,j}=t_{1,j}^{(u)}=t_{2,j}^{(u_2)}=t_{3,j}=t_{4,j}^{(uv)}=t_{5,j}^{(u_2v_2)}=t_{6,j}=t_{7,j}^{uv_2}=t_{8,j}^{(u)}=t_{9,j}^{(u_2)}=0$ for all $j\in[k^*_2]$, $u,v\in[d]$ and $u_2,v_2\in[d_2]$. This contradicts to the fact that at least one among them is non-zero. Consequently, we achieve the local part in equation~\eqref{eq:local_part}, that is,
\begin{align*}
    \lim_{\varepsilon\to0}\inf_{(G_1,G_2)\in\mathcal{G}_{k_1,k_2}(\Theta):\mathcal{D}_1((G_1,G_2),(G^*_1,G^*_2))\leq\varepsilon}\dfrac{\bbE_X[V(f_{G_1,G_2}(\cdot|X),f_{G^*_1,G^*_2}(\cdot|X))]}{\mathcal{D}_1((G_1,G_2),(G^*_1,G^*_2))}>0.
\end{align*}
The local part indicates that there exists a positive constant $\varepsilon'$ such that 
\begin{align*}
    \inf_{(G_1,G_2)\in\mathcal{G}_{k_1,k_2}(\Theta):\mathcal{D}_1((G_1,G_2),(G^*_1,G^*_2))\leq\varepsilon}\dfrac{\bbE_X[V(f_{G_1,G_2}(\cdot|X),f_{G^*_1,G^*_2}(\cdot|X))]}{\mathcal{D}_1((G_1,G_2),(G^*_1,G^*_2))}>0.
\end{align*}
\textbf{Proof for the global part~\eqref{eq:global_part}:} Thus, it is sufficient to establish the global part 
\begin{align*}
    \inf_{(G_1,G_2)\in\mathcal{G}_{k_1,k_2}(\Theta):\mathcal{D}_1((G_1,G_2),(G^*_1,G^*_2))>\varepsilon'}\dfrac{\bbE_X[V(f_{G_1,G_2}(\cdot|X),f_{G^*_1,G^*_2}(\cdot|X))]}{\mathcal{D}_1((G_1,G_2),(G^*_1,G^*_2))}>0.
\end{align*}
Suppose that the global part does not hold, then there exists a sequence of mixing measure pairs $(\tilde{G}^n_1,\tilde{G}^n_2)$ satisfying $\mathcal{D}_1((\tilde{G}^n_1,\tilde{G}^n_2),(G^*_1,G^*_2))>\varepsilon'$ and $\lim_{n\to\infty}\frac{\bbE_X[V(f_{\tilde{G}^n_1,\tilde{G}^n_2}(\cdot|X),f_{G^*_1,G^*_2}(\cdot|X))]}{\mathcal{D}_1((\tilde{G}^n_1,\tilde{G}^n_2),(G^*_1,G^*_2))}=0$. In other words, we have 
\begin{align*}
    \lim_{n\to\infty}\bbE_X[V(f_{\tilde{G}^n_1,\tilde{G}^n_2}(\cdot|X),f_{G^*_1,G^*_2}(\cdot|X))]=0.
\end{align*}
Recall that the parameter space $\Theta$ is compact, then we can replace the sequence $(\tilde{G}^n_1,\tilde{G}^n_2)$ by one of its subsequences which converges to some pair of mixing measures $(\tilde{G}_1,\tilde{G}_2)$. Due to the fact that $\mathcal{D}_1((\tilde{G}^n_1,\tilde{G}^n_2),(G^*_1,G^*_2))>\varepsilon'$, we get  $\mathcal{D}_1((\tilde{G}_1,\tilde{G}_2),(G^*_1,G^*_2))>\varepsilon'$. Next, by applying the Fatou's lemma, we have
\begin{align*}
    0=\lim_{n\to\infty}\bbE_X[V(f_{\tilde{G}^n_1,\tilde{G}^n_2}(\cdot|X),f_{G^*_1,G^*_2}(\cdot|X))]&\geq\frac{1}{2}\int\liminf_{n\to\infty}\Big|f_{\tilde{G}^n_1,\tilde{G}^n_2}(Y|X),f_{G^*_1,G^*_2}(Y|X)\Big|\dint(X,Y)\\
    &=\frac{1}{2}\int\Big|f_{\tilde{G}_1,\tilde{G}_2}(Y|X)-f_{G^*_1,G^*_2}(Y|X)\Big|\dint(X,Y).
\end{align*}
The above result implies that $f_{\tilde{G}_1,\tilde{G}_2}(Y|X)=f_{G^*_1,G^*_2}(Y|X)$ for almost surely $(X,Y)$. According to Proposition~\ref{prop:identifiability}, we deduce $(\tilde{G}_1,\tilde{G}_2)\equiv(G^*_1,G^*_2)$, indicating that $\mathcal{D}_1((\tilde{G}_1,\tilde{G}_2),(G^*_1,G^*_2))=0$. This contradicts the fact that $\mathcal{D}_1((\tilde{G}_1,\tilde{G}_2),(G^*_1,G^*_2))>\varepsilon'>0$. Hence, we obtain the global part~\eqref{eq:global_part} and complete the proof.


\subsection{Proof of Theorem~\ref{theorem:linear_experts}}
\label{appendix:linear_experts}
By employing arguments used in Appendix~\ref{appendix:strongly_identifiable_experts}, it is sufficient to establish the local part
\begin{align}
    \label{eq:local_part_linear}
    \lim_{\varepsilon\to0}\inf_{(G_1,G_2)\in\mathcal{G}_{k_1,k_2}(\Theta):\mathcal{D}_2((G_1,G_2),(G^*_1,G^*_2))\leq\varepsilon}\dfrac{\bbE_X[V(f_{G_1,G_2}(\cdot|X),f_{G^*_1,G^*_2}(\cdot|X))]}{\mathcal{D}_2((G_1,G_2),(G^*_1,G^*_2))}>0,
\end{align}
and the global part
\begin{align}
    \label{eq:global_part_linear}
    \inf_{(G_1,G_2)\in\mathcal{G}_{k_1,k_2}(\Theta):\mathcal{D}_2((G_1,G_2),(G^*_1,G^*_2))>\varepsilon'}\dfrac{\bbE_X[V(f_{G_1,G_2}(\cdot|X),f_{G^*_1,G^*_2}(\cdot|X))]}{\mathcal{D}_2((G_1,G_2),(G^*_1,G^*_2))}>0.
\end{align}
in this appendix. As the global part~\eqref{eq:global_part_linear} can be derived similarly to Appendix~\ref{appendix:strongly_identifiable_experts}, we omit its proof here. Thus, we will focus on showing only the local part~\eqref{eq:local_part_linear}. Assume by contrary that the local part is not true. Then, there exists a sequence of mixing measure pairs $(G^n_1,G^n_2)$ taking the form $G^n_1:=\sum_{i=1}^{k^n_1}\oin\delta_{(\koin,\kzin,\tin)}$, $G^n_2:=\sum_{i=1}^{k^n_2}\exp(\bzin)\delta_{(\boin,\eoin,\ezin,\nuin)}$ for $n\in\mathbb{N}$ such that $\mathcal{D}_{2n}:=\mathcal{D}_{2}((G^n_1,G^n_2),(G^*_1,G^*_2))\to0$ and
\begin{align}
    \label{eq:expectation_zero_linear}
    \bbE_X[V(f_{G^n_1,G^n_2}(\cdot|X),f_{G^*_1,G^*_2}(\cdot|X))]/\mathcal{D}_{2n}\to0,
\end{align}
as $n\to\infty$. Here, we may assume WLOG that the number of shared experts and routed experts $k^n_1$, $k^n_2$ and Voronoi cells $\mathcal{V}_{1,j}=\mathcal{V}_{1,j}(G^n_1)$, $\mathcal{V}_{2,j}=\mathcal{V}_{2,j}(G^n_2)$ do not change with the sample size $n$. Then, the Voronoi loss $\mathcal{D}_{2n}$ can be rewritten as
\begin{align}
    \label{eq:loss_2n}
    &\mathcal{D}_{2n}=\sum_{j=1}^{k^*_1}\Big|\sum_{i\in\mathcal{V}_{1,j}}\oin-\oj\Big|+\sum_{j\in[k^*_2]:|\mathcal{V}_{2,j}|>1}\Big|\sum_{i\in\mathcal{V}_{2,j}}\exp(\bzin)-\exp(\bzj)\Big|\nonumber\\
    &+\sum_{j\in[k^*_1]:|\mathcal{V}_{1,j}|=1}\sum_{i\in\mathcal{V}_{1,j}}\oin(\|\dkoijn\|+|\dkzijn|+|\dtijn|)\nonumber\\
    &+\sum_{j\in[k^*_1]:|\mathcal{V}_{1,j}|>1}\sum_{i\in\mathcal{V}_{1,j}}\oin(\|\dkoijn\|^2+|\dkzijn|^{\brj}+|\dtijn|^{\brj/2})\nonumber\\
    &+\sum_{j\in[k^*_2]:|\mathcal{V}_{2,j}|=1}\sum_{i\in\mathcal{V}_{2,j}}\exp(\bzin)(\|\dboijn\|+\|\deoijn\|+|\dezijn|+|\dnuijn|)\nonumber\\
    &+\sum_{j\in[k^*_2]:|\mathcal{V}_{2,j}|>1}\sum_{i\in\mathcal{V}_{2,j}}\exp(\bzin)(\|\dboijn\|^{\trj}+\|\deoijn\|^{\trj/2}+|\dezijn|^{\trj}+|\dnuijn|^{\trj/2}),
\end{align}
where we denote $\dkoijn:=\koin-\koj$, $\dkzijn:=\kzin-\kzj$,  $\deoijn:=\eoin-\eoj$, and $\dezijn:=\ezin-\ezj$. Since $\mathcal{D}_{2n}\to0$ as $n\to\infty$, then the above formulation indicates that as $n\to\infty$, we have
$\sum_{i\in\mathcal{V}_{1,j}}\oin\to\oj$, $(\koin,\kzin,\tin)\to(\koj,\kzj,\tj)$ as $n\to\infty$ for all $i\in\mathcal{V}_{1,j}$ and $j\in[k^*_1]$. Furthermore, we also have $\sum_{i\in\mathcal{V}_{2,j}}\exp(\bzin)-\exp(\bzj)$, $(\boin,\eoin,\ezin,\nuin)\to(\boj,\eoj,\ezj,\nuj)$ as $n\to\infty$ for all $i\in\mathcal{V}_{2,j}$ and $j\in[k^*_2]$.

\vspace{0.5 em}
\noindent
Next, we divide the rest of this proof into three main steps:

\vspace{0.5 em}
\noindent
\textbf{Stage 1 - Density Decomposition:} In this stage, we aim to decompose the density discrepancy $f_{G^n_1,G^n_2}(Y|X)-f_{G^*_1,G^*_2}(Y|X)$. For ease of presentation, we denote
\begin{align*}
    q_{G^n_1}(Y|X)&:=\sum_{i=1}^{k^n_1}\omega^n_{i}\pi(Y|(\koin)^{\top}X+\kzin,\tau^n_{i}),\\
    q_{G^*_1}(Y|X)&:=\sum_{i=1}^{k^*_1}\omega^*_{i}\pi(Y|(\koj)^{\top}X+\kzj,\tau^*_{i}),\\
    p_{G^n_2}(Y|X)&:=\sum_{i = 1}^{k^n_2} \frac{\exp((\beta_{1i}^{n})^{\top} X + \beta_{0i}^{n})}{\sum_{j = 1}^{k^n_2} \exp((\beta_{1j}^{n})^{\top} X + \beta_{0j}^{n})}\cdot \pi(Y|(\eoin)^{\top}X+\ezin, \nu_{i}^{n}),\\
    p_{G^*_2}(Y|X)&:=\sum_{i = 1}^{k^*_2} \frac{\exp((\beta_{1i}^{*})^{\top} X + \beta_{0i}^{*})}{\sum_{j = 1}^{k^*_2} \exp((\beta_{1j}^{*})^{\top} X + \beta_{0j}^{*})}\cdot \pi(Y|(\eoj)^{\top}X+\ezj, \nu_{i}^{*}).
\end{align*}
Given the above notations, we get
\begin{align*}
    f_{G^n_1,G^n_2}(Y|X)-f_{G^*_1,G^*_2}(Y|X)=\frac{1}{2}\left[(q_{G^n_1}(Y|X)-q_{G^*_1}(Y|X))+(p_{G^n_2}(Y|X)-p_{G^*_2}(Y|X))\right].
\end{align*}
\textbf{Stage 1.1:} Firstly, we decompose the term $q_{G^n_1}(Y|X)-q_{G^*_1}(Y|X)$ as
\begin{align*}
    q_{G^n_1}(Y|X)-q_{G^*_1}(Y|X)&=\sum_{j\in[k^*_1]:|\mathcal{V}_{1,j}|=1}\sum_{i\in\mathcal{V}_{1,j}}\oin[\pi(Y|(\koin)^{\top}X+\kzin,\tin)-\pi(Y|(\koj)^{\top}X+\kzj,\tj)]\\
    &+\sum_{j\in[k^*_1]:|\mathcal{V}_{1,j}|>1}\sum_{i\in\mathcal{V}_{1,j}}\oin[\pi(Y|(\koin)^{\top}X+\kzin,\tin)-\pi(Y|(\koj)^{\top}X+\kzj,\tj)]\\
    &+\sum_{j=1}^{k^*_1}\Big(\sum_{i\in\mathcal{V}_{1,j}}\oin-\oj\Big)\pi(Y|(\koj)^{\top}X+\kzj,\tau^*_j)\\
    &:=A_{n,1}(Y|X)+A_{n,2}(Y|X)+A_{n,0}(Y|X).
\end{align*}
By applying the first-order Taylor expansion to the function $\pi(Y|(\koin)^{\top}X+\kzin,\tau^n_i)$ around the point $(\koj,\kzj,\tj)$, the term $A_{n,1}(Y|X)$ is rewritten as
\begin{align*}
    A_{n,1}(Y|X)&=\sum_{j\in[k^*_1]:|\mathcal{V}_{1,j}|=1}\sum_{i\in\mathcal{V}_{1,j}}\sum_{|\alpha|=1}\frac{\oin}{\alpha!}(\dkoijn)^{\alpha_1}(\dkzijn)^{\alpha_2}(\dtijn)^{\alpha_3}\\
    &\hspace{4cm}\times \frac{\partial^{|\alpha_1|+\alpha_2+\alpha_3}\pi}{\partial\kappa_1^{\alpha_1}\partial\kappa_2^{\alpha_2}\partial\tau^{\alpha_3}}(Y|(\koj)^{\top}X+\kzj,\tj)+R_{n,1}(Y|X)\\
    &=\sum_{j\in[k^*_1]:|\mathcal{V}_{1,j}|=1}\sum_{i\in\mathcal{V}_{1,j}}\sum_{|\alpha|=1}\frac{\oin}{2^{\alpha_3}\alpha!}(\dkoijn)^{\alpha_1}(\dkzijn)^{\alpha_2}(\dtijn)^{\alpha_3}\\
    &\hspace{4cm}\times X^{\alpha_1}\frac{\partial^{|\alpha_1|+\alpha_2+2\alpha_3}\pi}{\partial h_1^{|\alpha_1|+\alpha_2+2\alpha_3}}(Y|(\koj)^{\top}X+\kzj,\tj)+R_{n,1}(Y|X)\\
    &=\sum_{j\in[k^*_1]:|\mathcal{V}_{1,j}|=1}\sum_{|\alpha_1|=0}^{1}\sum_{\ell=1_{\{|\alpha_1|=0\}}}^{2(1-|\alpha_1|)}A^{(j)}_{n,\alpha_1,\ell}\cdot X^{\alpha_1}\frac{\partial^{|\alpha_1|+\ell}\pi}{\partial h_1^{|\alpha_1|+\ell}}(Y|(\koj)^{\top}X+\kzj,\tj)+R_{n,1}(Y|X),
\end{align*}
where $R_{n,1}(Y|X)$ is a Taylor remainder such that $R_{n,1}(Y|X)/\mathcal{D}_{2n}\to0$ as $n\to\infty$, and
\begin{align*}
    A^{(j)}_{n,\alpha_1,\ell}&:=\sum_{i\in\mathcal{V}_{1,j}}\sum_{\alpha_2+2\alpha_3=\ell}\frac{\oin}{2^{\alpha_3}\alpha!}(\dkoijn)^{\alpha_1}(\dkzijn)^{\alpha_2}(\dtijn)^{\alpha_3},
\end{align*}
for all $j\in[k^*_1]$, $\alpha_1\in\mathbb{N}^{d}$ and $\ell\in\mathbb{N}$ such that $(\alpha_1,\ell)\neq (0_d,0)$. Meanwhile, by applying the Taylor expansion of the order $\brj:=r_1(|\mathcal{V}_{1,j}|)$ to the function $\pi(Y|(\koin)^{\top}X+\kzin,\tau^n_i)$ around the point $(\koj,\kzj,\tj)$, we rewrite the term $A_{n,2}(Y|X)$ as
\begin{align*}
    A_{n,2}(Y|X)&=\sum_{j\in[k^*_1]:|\mathcal{V}_{1,j}|>1}\sum_{|\alpha_1|=0}^{\brj}\sum_{\ell=1_{\{|\alpha_1|=0\}}}^{2(\brj-|\alpha_1|)}A^{(j)}_{n,\alpha_1,\ell}\cdot X^{\alpha_1}\frac{\partial^{|\alpha_1|+\ell}\pi}{\partial h_1^{|\alpha_1|+\ell}}(Y|(\koj)^{\top}X+\kzj,\tj)+R_{n,2}(Y|X),
\end{align*}
where $R_{n,2}(Y|X)$ is a Taylor remainder such that $R_{n,2}(Y|X)/\mathcal{D}_{2n}\to$ as $n\to\infty$.

\vspace{0.5 em}
\noindent
\textbf{Stage 1.2:} Next, we attempt to decompose the term $Q_n(Y|X):=\Big[\sum_{j = 1}^{k^*_2} \exp((\beta_{1j}^{*})^{\top} X + \beta_{0j}^{*})\Big]\cdot[p_{G^n_2}(Y|X)-p_{G^*_2}(Y|X)]$. By denoting $F(Y|X;\beta_1,\eta_1,\eta_0,\nu):=\exp(\beta_{1}^{\top}X)\pi(Y|(\eta_1)^{\top}X+\eta_0,\nu)$ and $H(Y|X;\beta_1):=\exp(\beta_1^{\top}X)p_{G_2}(Y|X)$, we can represent $Q_n(Y|X)$ as
\begin{align*}
    Q_n(Y|X)&=\sum_{j=1}^{k^*_2}\sum_{i\in\mathcal{V}_{2,j}}\exp(\bzin)[F(Y|X;\boin,\eoin,\ezin,\nuin)-F(Y|X;\boj,\eoj,\ezj,\nuj)]\\
    &-\sum_{j=1}^{k^*_2}\sum_{i\in\mathcal{V}_{2,j}}\exp(\bzin)[H(Y|X;\boin)-H(Y|X;\boj)]\\
    &+\sum_{j=1}^{k^*_2}\Big(\sum_{i\in\mathcal{V}_{2,j}}\exp(\bzin)-\exp(\bzj)\Big)[F(Y|X;\boj,\eoj,\ezj,\nuj)-H(Y|X;\boj)]\\
    &:=B_{n}(Y|X)-C_{n}(Y|X)+E_{n}(Y|X).
\end{align*}
\textbf{Stage 1.2.1:} In this step, we decompose the term $B_{n}(Y|X)$:
\begin{align*}
    B_{n}(Y|X)&=\sum_{j\in[k^*_2]:|\mathcal{V}_{2,j}|=1}\sum_{i\in\mathcal{V}_{2,j}}\exp(\bzin)[F(Y|X;\boin,\eoin,\ezin,\nuin)-F(Y|X;\boj,\eoj,\ezj,\nuj)]\\
    &+\sum_{j\in[k^*_2]:|\mathcal{V}_{2,j}|>1}\sum_{i\in\mathcal{V}_{2,j}}\exp(\bzin)[F(Y|X;\boin,\eoin,\ezin,\nuin)-F(Y|X;\boj,\eoj,\ezj,\nuj)]\\
    &:=B_{n,1}(Y|X) + B_{n,2}(Y|X).
\end{align*}
By applying the first-order Taylor expansion to the function $F(Y|X;\boin,\eoin,\ezin,\nuin)$ around the point $(\boj,\eoj,\ezj,\nuj)$, we have
\begin{align*}
    B_{n,1}(Y|X)&=\sum_{j\in[k^*_2]:|\mathcal{V}_{2,j}|=1}\sum_{i\in\mathcal{V}_{2,j}}\exp(\bzin)\sum_{|\alpha|=1}\frac{1}{\alpha!}(\dboijn)^{\alpha_1}(\deoijn)^{\alpha_2}(\dezijn)^{\alpha_3}(\dnuijn)^{\alpha_4}\\
    &\hspace{3cm}\times \frac{\partial^{|\alpha_1|+|\alpha_2|+\alpha_3+\alpha_4} F}{\partial\beta_1^{\alpha_1}\partial\eta_1^{\alpha_2}\partial\eta_0^{\alpha_3}\partial\nu^{\alpha_4}}(Y|X;\boj,\eoj,\ezj,\nuj)+R_{n,3}(Y|X)\\
    &=\sum_{j\in[k^*_2]:|\mathcal{V}_{2,j}|=1}\sum_{i\in\mathcal{V}_{2,j}}\sum_{|\alpha|=1}\frac{\exp(\bzin)}{2^{\alpha_4}\alpha!}(\dboijn)^{\alpha_1}(\deoijn)^{\alpha_2}(\dezijn)^{\alpha_3}(\dnuijn)^{\alpha_4}\\
    &\hspace{2cm}\times X^{\alpha_1+\alpha_2}\exp((\boj)^{\top}X)\frac{\partial^{|\alpha_2|+\alpha_3+2\alpha_4} \pi}{\partial h_2^{|\alpha_2|+\alpha_3+2\alpha_4}}(Y|(\eoj)^{\top}X+\ezj,\nuj)+R_{n,3}(Y|X)\\
    &=\sum_{j\in[k^*_2]:|\mathcal{V}_{2,j}|=1}\sum_{|\ell_1|+\ell_2=1}^{2}B^{(j)}_{n,\ell_1,\ell_2}\cdot X^{\ell_1}\exp((\boj)^{\top}X)\frac{\partial^{\ell_2} \pi}{\partial h_2^{\ell_2}}(Y|(\eoj)^{\top}X+\ezj,\nuj)+R_{n,3}(Y|X),
\end{align*}
where $R_{n,3}(Y|X)$ is the Taylor remainder such that $R_{n,3}(Y|X)/\mathcal{D}_{2n}\to0$, and
\begin{align*}
    B^{(j)}_{n,\ell_1,\ell_2}&:=\sum_{i\in\mathcal{V}_{2,j}}\sum_{\alpha\in\mathcal{I}_{\ell_1,\ell_2}}\frac{\exp(\bzin)}{2^{\alpha_4}\alpha!}(\dboijn)^{\alpha_1}(\deoijn)^{\alpha_2}(\dezijn)^{\alpha_3}(\dnuijn)^{\alpha_4},
\end{align*}
for all $j\in[k^*_2]$, $\ell_1\in\mathbb{N}^{d}$, and $\ell_2\in\mathbb{N}$ such that $(\ell_1,\ell_2)\neq (0_d,0)$, where we define
\begin{align*}
    \mathcal{I}_{\ell_1,\ell_2}:=\{\alpha=(\alpha_i)_{i=1}^{4}\in\mathbb{N}^{d}\times\mathbb{N}^{d}\times\mathbb{N}\times\mathbb{N}:\alpha_1+\alpha_2=\ell_1, \alpha_3+2\alpha_4=\ell_2-|\alpha_2|\}.
\end{align*}
By applying the Taylor expansion of the order $\trj:=r_2(|\mathcal{V}_{2,j}|)$ to the function $F(Y|X;\boin,\eoin,\ezin,\nuin)$ around the point $(\boj,\eoj,\ezj,\nuj)$, we have
\begin{align*}
    B_{n,2}(Y|X)&=\sum_{j\in[k^*_2]:|\mathcal{V}_{2,j}|=1}\sum_{|\ell_1|+\ell_2=1}^{2\trj}B^{(j)}_{n,\ell_1,\ell_2}\cdot X^{\ell_1}\exp((\boj)^{\top}X)\frac{\partial^{\ell_2} \pi}{\partial h_2^{\ell_2}}(Y|(\eoj)^{\top}X+\ezj,\nuj)+R_{n,4}(Y|X),
\end{align*}
where $R_{n,4}(Y|X)$ is the Taylor remainder such that $R_{n,4}(Y|X)/\mathcal{D}_{2n}\to0$.

\vspace{0.5 em}
\noindent
\textbf{Stage 1.2.2:} In this step, we decompose the term $C_{n}(Y|X)$:
\begin{align*}
    C_{n}(Y|X)&=\sum_{j\in[k^*_2]:|\mathcal{V}_{2,j}|=1}\sum_{i\in\mathcal{V}_{2,j}}\exp(\bzin)[H(Y|X;\boin)-H(Y|X;\boj)]\\
    &+\sum_{j\in[k^*_2]:|\mathcal{V}_{2,j}|>1}\sum_{i\in\mathcal{V}_{2,j}}\exp(\bzin)[H(Y|X;\boin)-H(Y|X;\boj)]\\
    &:=C_{n,1}(Y|X) + C_{n,2}(Y|X).
\end{align*}
By means of the first-order and second-order Taylor expansions to the function $H(Y|X;\boin)$ around the point $\boj$, the term $C_{n,1}(Y|X)$ can be represented as
\begin{align*}
    C_{n,1}(Y|X)&=\sum_{j\in[k^*_2]:|\mathcal{V}_{2,j}|=1}\sum_{i\in\mathcal{V}_{2,j}}\exp(\bzin)\sum_{|\gamma|=1}\frac{1}{\gamma!}(\dboijn)^{\gamma}\frac{\partial^{|\gamma|}H}{\partial\beta_1^{\gamma}}(Y|X;\boj)+R_{n,5}(Y|X)\\
    &=\sum_{j\in[k^*_2]:|\mathcal{V}_{2,j}|=1}\sum_{i\in\mathcal{V}_{2,j}}\sum_{|\gamma|=1}\frac{\exp(\bzin)}{\gamma!}(\dboijn)^{\gamma}\cdot X^{\gamma}\exp((\boj)^{\top}X)p_{G^n_2}(Y|X)+R_{n,5}(Y|X)\\
    &=\sum_{j\in[k^*_2]:|\mathcal{V}_{2,j}|=1}\sum_{|\gamma|=1}C^{(j)}_{n,\gamma}\cdot X^{\gamma}\exp((\boj)^{\top}X)p_{G^n_2}(Y|X)+R_{n,5}(Y|X),
\end{align*}
where $R_{n,5}(Y|X)$ is the Taylor remainder such that $R_{n,5}(Y|X)/\mathcal{D}_{2n}\to0$, and
\begin{align*}
    C^{(j)}_{n,\gamma}&:=\sum_{i\in\mathcal{V}_{2,j}}\frac{\exp(\bzin)}{\gamma!}(\dboijn)^{\gamma},
\end{align*}
for all $j\in[k^*_2]$ and $\gamma\in\mathbb{N}^{d}\setminus\{0_d\}$. Analogously, by applying the second-order Taylor expansion to the function $H(Y|X;\boin)$ around the point $\boj$, we represent the term $C_{n,2}(Y|X)$ as
\begin{align*}
    C_{n,2}(Y|X)&=\sum_{j\in[k^*_2]:|\mathcal{V}_{2,j}|=1}\sum_{|\gamma|=1}^{2}C^{(j)}_{n,\gamma}\cdot X^{\gamma}\exp((\boj)^{\top}X)p_{G^n_2}(Y|X)+R_{n,6}(Y|X),
\end{align*}
where $R_{n,6}(Y|X)$ is the Taylor remainder such that $R_{n,6}(Y|X)/\mathcal{D}_{2n}\to0$.

\vspace{0.5 em}
\noindent
Combining the above decompositions of $A_n(Y|X)$, $B_n(Y|X)$, and $C_n(Y|X)$ together, we obtain
\begin{align}
    &\Big[\sum_{j = 1}^{k^*_2} \exp((\beta_{1j}^{*})^{\top} X + \beta_{0j}^{*})\Big]\cdot[f_{G^n_1,G^n_2}(Y|X)-f_{G^*_1,G^*_2}(Y|X)]\nonumber\\
    &=\Big[\sum_{j = 1}^{k^*_2} \exp((\beta_{1j}^{*})^{\top} X + \beta_{0j}^{*})\Big]\frac{1}{2}\sum_{j\in[k^*_1]}\sum_{|\alpha_1|=0}^{\brj}\sum_{\ell=0}^{2(\brj-|\alpha_1|)}A^{(j)}_{n,\alpha_1,\ell}\cdot X^{\alpha_1}\frac{\partial^{|\alpha_1|+\ell}\pi}{\partial h_1^{|\alpha_1|+\ell}}(Y|(\koj)^{\top}X+\kzj,\tj)\nonumber\\
    &+\frac{1}{2}\sum_{j\in[k^*_2]}\sum_{|\ell_1|+\ell_2=0}^{2\trj}B^{(j)}_{n,\ell_1,\ell_2}\cdot X^{\ell_1}\exp((\boj)^{\top}X)\frac{\partial^{\ell_2} \pi}{\partial h_2^{\ell_2}}(Y|(\eoj)^{\top}X+\ezj,\nuj)\nonumber\\
    &-\frac{1}{2}\sum_{j\in[k^*_2]}\sum_{|\gamma|=0}^{1+1_{\{|\mathcal{V}_{2,j}|>1\}}}C^{(j)}_{n,\gamma}\cdot X^{\gamma}\exp((\boj)^{\top}X)p_{G^n_2}(Y|X)\nonumber\\
    &+\frac{1}{2}\Big[\sum_{j = 1}^{k^*_2} \exp((\beta_{1j}^{*})^{\top} X + \beta_{0j}^{*})\Big][R_{n,1}(Y|X)+R_{n,2}(Y|X)]\nonumber\\
    \label{eq:f-f_decompose_linear}
    &+\frac{1}{2}[R_{n,3}(Y|X)+R_{n,4}(Y|X)-R_{n,5}(Y|X)-R_{n,6}(Y|X)],
\end{align}
with a convention that $r_{1,j}=1$ for $j\in[k^*_1]:|\mathcal{V}_{1,j}|=1$ and $r_{2,j}=1$ for $j\in[k^*_2]:|\mathcal{V}_{2,j}|$, where we define
\begin{align*}
    A^{(j)}_{n,0_d,0}&:=\sum_{i\in\mathcal{V}_{1,j}}\oin-\oj, \qquad j\in[k^*_1]\\
    B^{(j)}_{n,0_d,0}&:=\sum_{i\in\mathcal{V}_{2,j}}\exp(\bzin)-\exp(\bzj), \qquad j\in[k^*_2]\\
    C^{(j)}_{n,0_d}&:=\sum_{i\in\mathcal{V}_{2,j}}\exp(\bzin)-\exp(\bzj), \qquad j\in[k^*_2].
\end{align*}
\textbf{Stage 2 - Non-vanishing coefficients:} In this stage, we demonstrate that at least one among the terms $A^{(j)}_{n,\alpha_1,\ell}/\mathcal{D}_{2n}$, $B^{(j)}_{n,\ell_1,\ell_2}/\mathcal{D}_{2n}$, and $C^{(j)}_{n,\gamma}/\mathcal{D}_{2n}$ does not converge to zero as $n\to\infty$. In particular, we assume that all these terms go to zero. Then, by looking at the terms $A^{(j)}_{n,\alpha_1,\ell}$,
\begin{itemize}
    \item For $j\in[k^*_1]$ and $|\alpha_1|=\ell=0$, we have $\frac{1}{\mathcal{D}_{1n}}\cdot\sum_{j=1}^{k^*_1}\Big|\sum_{i\in\mathcal{V}_{1,j}}\oin-\oj\Big|\to0$;
    \item For $j\in[k^*_1]:|\mathcal{V}_{1,j}|=1$, $\alpha_1\in\mathbb{N}^{d}:|\alpha_1|=1$ and $\ell=0$, we have
    \begin{align*}
        \frac{1}{\mathcal{D}_{2n}}\cdot\sum_{j\in[k^*_1]:|\mathcal{V}_{1,j}|=1}\sum_{i\in\mathcal{V}_{1,j}}\oin\|\dkoijn\|\to0;
    \end{align*}
    \item For $j\in[k^*_1]:|\mathcal{V}_{1,j}|=1$, $\alpha_1=0_d$ and $\ell=1$, we have
    \begin{align*}
        \frac{1}{\mathcal{D}_{2n}}\cdot\sum_{j\in[k^*_1]:|\mathcal{V}_{1,j}|=1}\sum_{i\in\mathcal{V}_{1,j}}\oin|\dkzijn|\to0;
    \end{align*}
    \item For $j\in[k^*_1]:|\mathcal{V}_{1,j}|=1$, $\alpha_1\in\mathbb{N}^{d}:|\alpha_1|=1$ and $\ell=2$ we have
    \begin{align*}
        \frac{1}{\mathcal{D}_{2n}}\cdot\sum_{j\in[k^*_1]:|\mathcal{V}_{1,j}|=1}\sum_{i\in\mathcal{V}_{1,j}}\oin|\dtijn|\to0;
    \end{align*}
    \item  For $j\in[k^*_1]:|\mathcal{V}_{1,j}|>1$, $\alpha_1=2e_u$, where $e_u\in\mathbb{N}^{d}$ is a one-hot vector with the $u$-th entry being one while other entries being zero, for $u\in[d]$, we have
    \begin{align*}
        \frac{1}{\mathcal{D}_{2n}}\cdot\sum_{j\in[k^*_1]:|\mathcal{V}_{1,j}|>1}\sum_{i\in\mathcal{V}_{1,j}}\oin\|\dkoijn\|^2\to0;
    \end{align*}
\end{itemize}
Next, by considering the terms $B^{(j)}_{n,\ell_1,\ell_2}$
\begin{itemize}
    \item For $j\in[k^*_2]$ and $|\ell_1|=\ell_2=0$, we have $\frac{1}{\mathcal{D}_{2n}}\cdot\sum_{j=1}^{k^*_2}\Big|\sum_{i\in\mathcal{V}_{2,j}}\exp(\bzin)-\exp(\boj)\Big|\to0$;
    \item For $j\in[k^*_2]:|\mathcal{V}_{2,j}|=1$, $\ell_1=e_u$ for $u\in[d]$, and $\ell_2=0$, we have
    \begin{align*}
        \frac{1}{\mathcal{D}_{2n}}\cdot\sum_{j\in[k^*_2]:|\mathcal{V}_{2,j}|=1}\sum_{i\in\mathcal{V}_{2,j}}\exp(\bzin)\|\dboijn\|\to0;
    \end{align*}
    \item For $j\in[k^*_2]:|\mathcal{V}_{2,j}|=1$, $\ell_1=e_u$ for $u\in[d]$, and $\ell_2=1$, we have
    \begin{align*}
        \frac{1}{\mathcal{D}_{2n}}\cdot\sum_{j\in[k^*_2]:|\mathcal{V}_{2,j}|=1}\sum_{i\in\mathcal{V}_{2,j}}\exp(\bzin)\|\deoijn\|\to0;
    \end{align*}
    \item For $j\in[k^*_2]:|\mathcal{V}_{2,j}|=1$, $\ell_1=0_d$ and $\ell_2=1$, we have
    \begin{align*}
        \frac{1}{\mathcal{D}_{2n}}\cdot\sum_{j\in[k^*_2]:|\mathcal{V}_{2,j}|=1}\sum_{i\in\mathcal{V}_{2,j}}\exp(\bzin)|\dezijn|\to0;
    \end{align*}
    \item For $j\in[k^*_2]:|\mathcal{V}_{2,j}|=1$, $\ell_1=0_d$, and $\ell_2=2$ we have
    \begin{align*}
        \frac{1}{\mathcal{D}_{2n}}\cdot\sum_{j\in[k^*_2]:|\mathcal{V}_{2,j}|=1}\sum_{i\in\mathcal{V}_{2,j}}\exp(\bzin)|\dnuijn|\to0;
    \end{align*}
\end{itemize}
Taking the sum of the above limits, we deduce 
\begin{align*}
    \frac{1}{\mathcal{D}_{2n}}\cdot\Bigg[&\sum_{j=1}^{k^*_1}\Big|\sum_{i\in\mathcal{V}_{1,j}}\oin-\oj\Big|+\sum_{j\in[k^*_2]:|\mathcal{V}_{2,j}|>1}\Big|\sum_{i\in\mathcal{V}_{2,j}}\exp(\bzin)-\exp(\bzj)\Big|\nonumber\\
    &+\sum_{j\in[k^*_1]:|\mathcal{V}_{1,j}|=1}\sum_{i\in\mathcal{V}_{1,j}}\oin(\|\dkoijn\|+|\dkzijn|+|\dtijn|)+\sum_{j\in[k^*_1]:|\mathcal{V}_{1,j}|>1}\sum_{i\in\mathcal{V}_{1,j}}\oin\|\dkoijn\|^2\nonumber\\
    &+\sum_{j\in[k^*_2]:|\mathcal{V}_{2,j}|=1}\sum_{i\in\mathcal{V}_{2,j}}\exp(\bzin)(\|\dboijn\|+\|\deoijn\|+|\dezijn|+|\dnuijn|)\Bigg]\to0,
\end{align*}
as $n\to\infty$. From the formulation of the Voronoi loss $\mathcal{D}_{2n}$ in equation~\eqref{eq:loss_2n}, it follows that
\begin{align}
    \label{eq:non_zero_limit}
    &\frac{1}{\mathcal{D}_{2n}}\Bigg[\sum_{j\in[k^*_1]:|\mathcal{V}_{1,j}|>1}\sum_{i\in\mathcal{V}_{1,j}}\oin(|\dkzijn|^{\brj}+|\dtijn|^{\brj/2})\nonumber\\
    &+\sum_{j\in[k^*_2]:|\mathcal{V}_{2,j}|>1}\sum_{i\in\mathcal{V}_{2,j}}\exp(\bzin)(\|\dboijn\|^{\trj}+\|\deoijn\|^{\trj/2}+|\dezijn|^{\trj}+|\dnuijn|^{\trj/2})\Bigg]\not\to0,
\end{align}
as $n\to\infty$. Then, we consider two following cases:

\textbf{Case I:}
$\frac{1}{\mathcal{D}_{2n}}\sum_{j\in[k^*_1]:|\mathcal{V}_{1,j}|>1}\sum_{i\in\mathcal{V}_{1,j}}\oin(|\dkzijn|^{\brj}+|\dtijn|^{\brj/2})\not\to0$ as $n\to\infty$.

In this case, there exists some index $j'\in[k^*_1]:|\mathcal{V}_{1,j'}|>1$ such that
\begin{align}
    \label{eq:non_zero_limit_1}
    \frac{1}{\mathcal{D}_{2n}}\cdot\sum_{i\in\mathcal{V}_{1,j'}}\oin(|\Delta\kappa^n_{0ij'}|^{r_{1,j'}}+|\Delta\tau^n_{ij'}|^{r_{1,j'}/2})\not\to0,
\end{align}
as $n\to\infty$. WLOG, we may assume that $j'=1$. Recall that the term $A^{(j)}_{n,\alpha_1,\ell}/\mathcal{D}_{2n}\to0$ as $n\to\infty$ for all $0\leq|\alpha_1|\leq \brj$ and $0\leq\ell\leq2(\brj-|\alpha_1|)$. Then, by dividing the ratio $A^{(1)}_{n,0_d,\ell}$ by the left hand side of equation~\eqref{eq:non_zero_limit_1}, we get
\begin{align}
    \label{eq:zero_limit_1}
    \dfrac{\sum_{i\in\mathcal{V}_{1,1}}\sum_{\alpha_2+2\alpha_3=\ell}\frac{\oin}{2^{\alpha_3}\alpha_2!\alpha_3!}(\Delta\kappa_{1i1})^{\alpha_2}(\Delta\tau_{i1})^{\alpha_3}}{\sum_{i\in\mathcal{V}_{1,1}}\oin(|\Delta\kappa^n_{0i1}|^{r_{1,1}}+|\Delta\tau^n_{i1}|^{r_{1,1}/2})}\to0,
\end{align}
as $n\to\infty$ for all $0\leq\ell\leq2r_{1,1}$. 

\vspace{0.5 em}
\noindent
Let us denote $M_{n,1}:=\max\{|\Delta\kappa^n_{0i1}|, |\Delta\tau^n_{i1}|:i\in\mathcal{V}_{1,1}\}$ and $W_{n,1}:=\max\{\oin:i\in\mathcal{V}_{1,1}\}$. Since the sequence $(\oin/W_{n,1})_{n}$ is bounded below, we can replace it by its subsequence that admits the limit $s_{1i}^2:=\lim_{n\to\infty}\oin/W_{n,1}>0$. It should be noted that at least one among the terms $s^2_{1i}$, for $i\in\mathcal{V}_{1,1}$, is equal to 1. Next, we denote $(\Delta\kappa^n_{0i1})/M_{n,1}\to s_{2i}$ and $(\Delta\tau^n_{i1})/[2M_{n,1}^2]\to s_{3i}$ for all $i\in\mathcal{V}_{1,1}$. Similarly, at least one of each of the $s_{2i}$ and $s_{3i}$ is equal to 1 or $-1$. Then, by dividing both the numerators and the denominators of the left hand side of equation~\eqref{eq:zero_limit_1} by $W_{n,1}M_{n,1}^{\ell}$, we obtain the following system of polynomial equations:
\begin{align*}
    \sum_{i\in\mathcal{V}_{1,1}}\sum_{\alpha_2+2\alpha_3=\ell}\frac{s^2_{1i}~s_{2i}^{\alpha_2}~s_{3i}^{\alpha_3}}{\alpha_2!\alpha_3!}=0, \qquad 1\leq \ell\leq r_{1,1}.
\end{align*}
According to the definition of the term $r_{1,1}$, the above system does not admit any non-trivial solutions, which is a contradiction. Thus, Case I cannot occur.

\vspace{0.5 em}
\noindent
\textbf{Case II:} $\frac{1}{\mathcal{D}_{2n}}\sum_{j\in[k^*_2]:|\mathcal{V}_{2,j}|>1}\sum_{i\in\mathcal{V}_{2,j}}\exp(\bzin)(\|\dboijn\|^{\trj}+\|\deoijn\|^{\trj/2}+|\dezijn|^{\trj}+|\dnuijn|^{\trj/2})\not\to0$ as $n\to\infty$.

In this case, we can find some index $j'\in[k^*_2]:|\mathcal{V}_{2,j'}|>1$ such that 
\begin{align}
    \label{eq:non_zero_limit_2}
    \frac{1}{\mathcal{D}_{2n}}\cdot\sum_{i\in\mathcal{V}_{2,j'}}\exp(\bzin)(\|\Delta\beta^n_{1ij'}\|^{r_{2,j'}}+\|\Delta\eta^n_{1ij'}\|^{r_{2,j'}/2}+|\Delta\eta^n_{0ij'}|^{r_{2,j'}}+|\Delta\nu^n_{ij'}|^{r_{2,j'}/2})\not\to0,
\end{align}
as $n\to\infty$. WLOG, we may assume that $j'=1$. Recall that the term $B^{(j)}_{n,\ell_1,\ell_2}/\mathcal{D}_{2n}\to0$ as $n\to\infty$ for all $j\in[k^*_2]$ and $(\ell_1,\ell_2)\in\mathbb{N}^{d}\times\mathbb{N}:0\leq|\ell_1|+\ell_2\leq 2r_{2,j}$. Then, by dividing the ratio $B^{(1)}_{n,\ell_1,\ell_2}$ by the left hand side of equation~\eqref{eq:non_zero_limit_2}, we get
\begin{align}
    \label{eq:zero_limit_2}
    \dfrac{\sum_{i\in\mathcal{V}_{2,1}}\sum_{\alpha\in\mathcal{I}_{\ell_1,\ell_2}}\frac{\exp(\bzin)}{2^{\alpha_4}\alpha!}(\Delta\beta^n_{1i1})^{\alpha_1}(\Delta\eta^n_{1i1})^{\alpha_2}(\Delta\eta^n_{0i1})^{\alpha_3}(\Delta\nu^n_{i1})^{\alpha_4}}{\sum_{i\in\mathcal{V}_{2,1}}\exp(\bzin)(\|\Delta\beta^n_{1i1}\|^{r_{2,1}}+\|\Delta\eta^n_{1i1}\|^{r_{2,1}/2}+|\Delta\eta^n_{0i1}|^{r_{2,1}}+|\Delta\nu^n_{i1}|^{r_{2,1}/2})}\to0,
\end{align}
as $n\to\infty$ for all $(\ell_1,\ell_2)\in\mathbb{N}^{d}\times\mathbb{N}: 0\leq|\ell_1|+\ell_2\leq2r_{2,1}$.

\vspace{0.5 em}
\noindent
Let us denote $M_{n,2}:=\max\{\|\Delta\beta^n_{1i1}\|, \|\Delta\eta^n_{1i1}\|, |\Delta\eta^n_{0i1}|, |\Delta\nu^n_{i1}|:i\in\mathcal{V}_{2,1}\}$ and $W_{n,2}:=\max\{\exp(\bzin):i\in\mathcal{V}_{2,1}\}$. Since the sequence $(\exp(\bzin)/W_{n,2})_{n}$ is bounded below, we can replace it by its subsequence that admits the limit $t_{5i}^2:=\lim_{n\to\infty}\exp(\bzin)/W_{n,2}>0$. It should be noted that at least one among the terms $t^2_{5i}$, for $i\in\mathcal{V}_{2,1}$, is equal to 1. Next, we denote
\begin{align*}
    (\Delta\beta^n_{1i1})/M_{n,2}\to t_{1i}&, \qquad (\Delta\eta^n_{1i1})/M_{n,2}^2\to t_{2i},\\
    (\Delta\eta^n_{0i1})/M_{n,2}\to t_{3i}&, \qquad (\Delta\nu^n_{i1})/[2M_{n,2}^2]\to t_{4i},
\end{align*}
for all $i\in\mathcal{V}_{2,1}$. Similarly, at least one of each of the $t_{1i}$, $t_{2i}$, $t_{3i}$, and $t_{4i}$, is equal to 1 or $-1$. Then, by dividing both the numerators and the denominators of the left hand side of equation~\eqref{eq:zero_limit_2} by $W_{n,2}M_{n,2}^{|\ell_1|+\ell_2}$, we obtain the following system of polynomial equations:
\begin{align*}
    \sum_{i\in\mathcal{V}_{2,1}}\sum_{\alpha\in\mathcal{I}_{\ell_1,\ell_2}}\frac{1}{\alpha!}\cdot t^2_{5i}~t_{1i}^{\alpha_1}~t_{2i}^{\alpha_2}~t_{3i}^{\alpha_3}~t_{4i}^{\alpha_4}=0, \qquad 1\leq |\ell_1|+\ell_2\leq r_{2,1}.
\end{align*}
According to the definition of the term $r_{2,1}$, the above system does not admit any non-trivial solutions, which is a contradiction. Thus, Case II cannot occur.

\vspace{0.5 em}
\noindent
The fact that both Case I and Case II cannot occur contradicts the result of equation~\eqref{eq:non_zero_limit}. Thus, not all the terms $A^{(j)}_{n,\alpha_1,\ell}/\mathcal{D}_{2n}$, $B^{(j)}_{n,\ell_1,\ell_2}/\mathcal{D}_{2n}$, and $C^{(j)}_{n,\gamma}/\mathcal{D}_{2n}$ converge to zero as $n\to\infty$.

\vspace{0.5 em}
\noindent
\textbf{Stage 3 - Fatou's lemma contradiction:} We denote by $m_n$ the maximum of the absolute values of the ratios $A^{(j)}_{n,\alpha_1,\ell}/\mathcal{D}_{2n}$, $B^{(j)}_{n,\ell_1,\ell_2}/\mathcal{D}_{2n}$, and $C^{(j)}_{n,\gamma}/\mathcal{D}_{2n}$. It follows from the result of Stage that $1/m_n\not\to\infty$ as $n\to\infty$. Then, by means of the Fatou's lemma, we have
\begin{align*}
    \lim_{n\to\infty}\dfrac{\bbE_X[V(f_{G^n_1,G^n_2}(\cdot|X),f_{G^*_1,G^*_2}(\cdot|X))]}{m_n\mathcal{D}_{2n}}\geq\int\liminf_{n\to\infty}\dfrac{|f_{G^n_1,G^n_2}(Y|X)-f_{G^*_1,G^*_2}(Y|X)|}{2m_n\mathcal{D}_{2n}}\dint (X,Y).
\end{align*}
Then, we deduce $[f_{G^n_1,G^n_2}(Y|X)-f_{G^*_1,G^*_2}(Y|X)]/[m_n\mathcal{D}_{1n}]\to0$ as $n\to\infty$ for almost surely $(X,Y)$. Since the input space is bounded and the parameter space is compact, the quantity $\sum_{j = 1}^{k^*_2} \exp((\beta_{1j}^{*})^{\top} X + \beta_{0j}^{*})$ is bounded. Thus, we also have
\begin{align*}
    \frac{1}{m_n\mathcal{D}_{2n}}\Big[\sum_{j = 1}^{k^*_2} \exp((\beta_{1j}^{*})^{\top} X + \beta_{0j}^{*})\Big][f_{G^n_1,G^n_2}(Y|X)-f_{G^*_1,G^*_2}(Y|X)]\to0,
\end{align*}
as $n\to\infty$ for almost surely $(X,Y)$. Let us denote
\begin{align*}
    &\frac{1}{m_n\mathcal{D}_{2n}}A^{(j)}_{n,\alpha_1,\ell}\to a^{(j)}_{\alpha_1,\ell},\\
    &\frac{1}{m_n\mathcal{D}_{2n}}B^{(j)}_{n,\ell_1,\ell_2}\to b^{(j)}_{\ell_1,\ell_2},\\
    &\frac{1}{m_n\mathcal{D}_{2n}}C^{(j)}_{n,\gamma}\to c^{(j)}_{\gamma},
\end{align*}
as $n\to\infty$ with a note that at least one among them is non-zero. From equation~\eqref{eq:f-f_decompose_linear}, we deduce
\begin{align*}
    &\Big[\sum_{j = 1}^{k^*_2} \exp((\beta_{1j}^{*})^{\top} X + \beta_{0j}^{*})\Big]\frac{1}{2}\sum_{j\in[k^*_1]}\sum_{|\alpha_1|=0}^{\brj}\sum_{\ell=0}^{2(\brj-|\alpha_1|)}a^{(j)}_{\alpha_1,\ell}\cdot X^{\alpha_1}\frac{\partial^{|\alpha_1|+\ell}\pi}{\partial h_1^{|\alpha_1|+\ell}}(Y|(\koj)^{\top}X+\kzj,\tj)\nonumber\\
    &+\frac{1}{2}\sum_{j\in[k^*_2]}\sum_{|\ell_1|+\ell_2=0}^{2\trj}b^{(j)}_{\ell_1,\ell_2}\cdot X^{\ell_1}\exp((\boj)^{\top}X)\frac{\partial^{\ell_2} \pi}{\partial h_2^{\ell_2}}(Y|(\eoj)^{\top}X+\ezj,\nuj)\nonumber\\
    &-\frac{1}{2}\sum_{j\in[k^*_2]}\sum_{|\gamma|=0}^{1+1_{\{|\mathcal{V}_{2,j}|>1\}}}c^{(j)}_{\gamma}\cdot X^{\gamma}\exp((\boj)^{\top}X)p_{G^*_2}(Y|X)\Big]\to0,
\end{align*}
as $n\to\infty$ for almost surely $(X,Y)$. Since the set
\begin{align*}
    &\Bigg\{X^{\alpha_1}\frac{\partial^{|\alpha_1|+\ell}\pi}{\partial h_1^{|\alpha_1|+\ell}}(Y|(\koj)^{\top}X+\kzj,\tj):j\in[k^*_1], 0\leq|\alpha_1|\leq\brj,0\leq\ell\leq2(\brj-|\alpha_1|)\Bigg\}\\
    &\cup\Bigg\{X^{\ell_1}\exp((\boj)^{\top}X)\frac{\partial^{\ell_2} \pi}{\partial h_2^{\ell_2}}(Y|(\eoj)^{\top}X+\ezj,\nuj),\ X^{\gamma}\exp((\boj)^{\top}X)p_{G^*_2}(Y|X):\\
    &j\in[k^*_2],0\leq|\ell_1|+\ell_2\leq2\trj, 0\leq|\gamma|\leq 2\Bigg\}
\end{align*}
is linearly independent w.r.t ..., we obtain $a^{(j)}_{\alpha_1,\ell}$ for all $j\in[k^*_1]$, $\alpha_1\in\mathbb{N}^{d}$, $\ell\in\mathbb{N}$, and $b^{(j)}_{\ell_1,\ell_2}=c^{(j)}_{\gamma}=0$ for all $j\in[k^*_2]$, $(\ell_1,\ell_2)\in\mathbb{N}^{d}\times\mathbb{N}$, $\gamma\in\mathbb{N}^{d}$. This result contradicts the fact that not all the terms $a^{(j)}_{\alpha_1,\ell}$, $b^{(j)}_{\ell_1,\ell_2}$, and $c^{(j)}_{\gamma}$ equal zero. Hence, we achieve the local part in equation~\eqref{eq:local_part_linear} and complete the proof.

\subsection{Proof of Theorem~\ref{theorem:strongly_identifiable_experts_sigmoid}}
\label{appendix:strongly_identifiable_experts_sigmoid}
By leveraging the proof framework in Appendix~\ref{appendix:strongly_identifiable_experts}, we also focus on demonstrating the local part
\begin{align}
    \label{eq:local_part_sigmoid}
    \lim_{\varepsilon\to0}\inf_{(G_1,G_2)\in\mathcal{G}_{k_1,k_2}(\Theta):\mathcal{D}_3((G_1,G_2),(G^*_1,G^*_2))\leq\varepsilon}\dfrac{\bbE_X[V(g_{G_1,G_2}(\cdot|X),g_{G^*_1,G^*_2}(\cdot|X))]}{\mathcal{D}_3((G_1,G_2),(G^*_1,G^*_2))}>0,
\end{align}
and the global part
\begin{align}
    \label{eq:global_part_sigmoid}
    \inf_{(G_1,G_2)\in\mathcal{G}_{k_1,k_2}(\Theta):\mathcal{D}_3((G_1,G_2),(G^*_1,G^*_2))>\varepsilon'}\dfrac{\bbE_X[V(g_{G_1,G_2}(\cdot|X),g_{G^*_1,G^*_2}(\cdot|X))]}{\mathcal{D}_3((G_1,G_2),(G^*_1,G^*_2))}>0.
\end{align}
in this appendix. Note that since the global part~\eqref{eq:global_part_sigmoid} can be argued in a similar fashion to Appendix~\ref{appendix:strongly_identifiable_experts}, its derivation is omitted here. Therefore, it is sufficient to establish the local part~\eqref{eq:local_part_sigmoid}. Suppose that the local part does not hold. Then, there exists a sequence of mixing measure pairs $(G^n_1,G^n_2)$ taking the form $G^n_1:=\sum_{i=1}^{k^n_1}\oin\delta_{(\kin,\tin)}$, $G^n_2:=\sum_{i=1}^{k^n_2}\sigma(\bzin)\delta_{(\boin,\ein,\nuin)}$ for $n\in\mathbb{N}$ such that $\mathcal{D}_{3n}:=\mathcal{D}_{3}((G^n_1,G^n_2),(G^*_1,G^*_2))\to0$ and
\begin{align}
    \label{eq:expectation_zero_sigmoid}
    \bbE_X[V(g_{G^n_1,G^n_2}(\cdot|X),g_{G^*_1,G^*_2}(\cdot|X))]/\mathcal{D}_{3n}\to0,
\end{align}
as $n\to\infty$. Here, we may assume WLOG that the number of shared experts and routed experts $k^n_1$, $k^n_2$ and Voronoi cells $\mathcal{V}_{1,j}=\mathcal{V}_{1,j}(G^n_1)$, $\mathcal{V}_{2,j}=\mathcal{V}_{2,j}(G^n_2)$ do not change with the sample size $n$. Then, the Voronoi loss $\mathcal{D}_{3n}$ can be rewritten as
\begin{align}
    \label{eq:loss_3n}
    &\mathcal{D}_{3n}=\sum_{j=1}^{k^*_1}\Big|\sum_{i\in\mathcal{V}_{1,j}}\oin-\oj\Big|+\sum_{j\in[k^*_2]:|\mathcal{V}_{2,j}|>1}\Big|\sum_{i\in\mathcal{V}_{2,j}}\sigma(\bzin)-\sigma(\bzj)\Big|\nonumber\\
    &+\sum_{j\in[k^*_1]:|\mathcal{V}_{1,j}|=1}\sum_{i\in\mathcal{V}_{1,j}}\oin(\|\dkijn\|+|\dtijn|)+\sum_{j\in[k^*_2]:|\mathcal{V}_{2,j}|=1}\sum_{i\in\mathcal{V}_{2,j}}(\|\dboijn\|+|\dbzijn|+\|\deijn\|+|\dnuijn|)\nonumber\\
    &+\sum_{j\in[k^*_1]:|\mathcal{V}_{1,j}|>1}\sum_{i\in\mathcal{V}_{1,j}}\oin(\|\dkijn\|^2+|\dtijn|^2)+\sum_{j\in[k^*_2]:|\mathcal{V}_{2,j}|>1}\sum_{i\in\mathcal{V}_{2,j}}(\|\dboijn\|^2+\|\deijn\|^2+|\dnuijn|^2),
\end{align}
where we denote $\dbzijn:=\beta^n_{0i}-\bzj$. Since $\mathcal{D}_{3n}\to0$ as $n\to\infty$, then the above formulation indicates that as $n\to\infty$, we have
\begin{itemize}
    \item For $j\in[k^*_1]$ and $i\in\mathcal{V}_{1,j}$: $\sum_{i\in\mathcal{V}_{1,j}}\oin\to\oj$, $(\kin,\tin)\to(\kj,\tj)$;
    \item For $j\in[k^*_2]:|\mathcal{V}_{2,j}|=1$ and $i\in\mathcal{V}_{2,j}$: $(\boin,\bzin,\ein,\nuin)\to(\boj,\bzj,\ej,\nuj)$;
    \item For $j\in[k^*_2]:|\mathcal{V}_{2,j}|>1$ and $i\in\mathcal{V}_{2,j}$: $\sum_{i\in\mathcal{V}_{2,j}}\sigma(\bzin)-\sigma(\bzj)$, $(\boin,\ein,\nuin)\to(\boj,\ej,\nuj)$. 
\end{itemize}
Now, we divide the proof into three main stages:

\vspace{0.5 em}
\noindent
\textbf{Stage 1 - Density Decomposition:} In this stage, we aim to decompose the density discrepancy $g_{G^n_1,G^n_2}(Y|X)-g_{G^*_1,G^*_2}(Y|X)$. For ease of presentation, we denote
\begin{align*}
    q_{G^n_1}(Y|X)&:=\sum_{i=1}^{k^n_1}\omega^n_{i}\pi(Y|h_1(X,\kappa^n_{i}),\tau^n_{i}),\\
    q_{G^*_1}(Y|X)&:=\sum_{i=1}^{k^*_1}\omega^*_{i}\pi(Y|h_1(X,\kappa^*_{i}),\tau^*_{i}),\\
    p_{G^n_2}(Y|X)&:=\sum_{i = 1}^{k^n_2} \frac{\sigma((\beta_{1i}^{n})^{\top} X + \beta_{0i}^{n})}{\sum_{j = 1}^{k^n_2} \sigma((\beta_{1j}^{n})^{\top} X + \beta_{0j}^{n})}\cdot \pi(Y|h_2(X,\eta^n_{i}), \nu_{i}^{n}),\\
    p_{G^*_2}(Y|X)&:=\sum_{i = 1}^{k^*_2} \frac{\sigma((\beta_{1i}^{*})^{\top} X + \beta_{0i}^{*})}{\sum_{j = 1}^{k^*_2} \sigma((\beta_{1j}^{*})^{\top} X + \beta_{0j}^{*})}\cdot \pi(Y|h_2(X,\eta^*_{i}), \nu_{i}^{*}).
\end{align*}
Given the above notations, we get
\begin{align*}
    g_{G^n_1,G^n_2}(Y|X)-g_{G^*_1,G^*_2}(Y|X)=\frac{1}{2}\left[(q_{G^n_1}(Y|X)-q_{G^*_1}(Y|X))+(p_{G^n_2}(Y|X)-p_{G^*_2}(Y|X))\right].
\end{align*}
\textbf{Stage 1.1:} Firstly, we decompose the term $q_{G^n_1}(Y|X)-q_{G^*_1}(Y|X)$ as
\begin{align*}
    q_{G^n_1}(Y|X)-q_{G^*_1}(Y|X)&=\sum_{j\in[k^*_1]:|\mathcal{V}_{1,j}|=1}\sum_{i\in\mathcal{V}_{1,j}}\oin[\pi(Y|h_1(X,\kin),\tin)-\pi(Y|h_1(X,\kj),\tj)]\\
    &+\sum_{j\in[k^*_1]:|\mathcal{V}_{1,j}|>1}\sum_{i\in\mathcal{V}_{1,j}}\oin[\pi(Y|h_1(X,\kin),\tin)-\pi(Y|h_1(X,\kj),\tj)]\\
    &+\sum_{j=1}^{k^*_1}\Big(\sum_{i\in\mathcal{V}_{1,j}}\oin-\oj\Big)\pi(Y|h_1(X,\kappa^*_j),\tau^*_j)\\
    &:=A_{n,1}(Y|X)+A_{n,2}(Y|X)+A_{n,0}(Y|X).
\end{align*}
By using the same arguments as in Stage 1.1 in Appendix~\ref{appendix:strongly_identifiable_experts}, the term $A_{n,1}(Y|X)$ is rewritten as
\begin{align*}
    A_{n,1}(Y|X)&=\sum_{j\in[k^*_1]:|\mathcal{V}_{1,j}|=1}\sum_{\rho=1}^{2}A^{(j)}_{n,1,\rho}(X)\frac{\partial^{\rho}\pi}{\partial h_1^{\rho}}(Y|h_1(X,\kj),\tj)+R_{n,1}(Y|X),
\end{align*}
where $R_{n,1}(Y|X)$ is a Taylor remainder such that $R_{n,1}(Y|X)/\mathcal{D}_{3n}\to$ as $n\to\infty$, and
\begin{align*}
    A^{(j)}_{n,1,1}(X)&:=\sum_{i\in\mathcal{V}_{1,j}}\oin\sum_{u_1=1}^{d_1}(\dkijn)^{(u_1)}\frac{\partial h_1}{\partial\kappa^{(u_1)}}(X,\kj),\\
    A^{(j)}_{n,1,2}(X)&:=\sum_{i\in\mathcal{V}_{1,j}}\oin\frac{1}{2}(\dtijn),
\end{align*}
for all $j\in[k^*_1]$ such that $|\mathcal{V}_{1,j}|=1$. Meanwhile, we can represent $A_{n,2}(Y|X)$ as
\begin{align*}
    A_{n,2}(Y|X)&=\sum_{j\in[k^*_1]:|\mathcal{V}_{1,j}|>1}\sum_{\rho=1}^{4}A^{(j)}_{n,1,\rho}(X)\frac{\partial^{\rho}\pi}{\partial h_1^{\rho}}(Y|h_1(X,\kj),\tj)+R_{n,2}(Y|X),
\end{align*}
where $R_{n,2}(Y|X)$ is a Taylor remainder such that $R_{n,2}(Y|X)/\mathcal{D}_{3n}\to$ as $n\to\infty$, and
\begin{align*}
    A^{(j)}_{n,2,1}(X)&:=\sum_{i\in\mathcal{V}_{1,j}}\oin\Big(\sum_{u_1=1}^{d_1}(\dkijn)^{(u_1)}\frac{\partial h_1}{\partial\kappa^{(u_1)}}(X,\kj)+\sum_{u_1,v_1=1}^{d_1}\frac{(\dkijn)^{(u_1)}(\dkijn)^{(v_1)}}{1+1_{\{u_1=v_1\}}}\frac{\partial^2h_1}{\partial\kappa^{(u_1)}\partial\kappa^{(v_1)}}(X,\kj)\Big),\\
    A^{(j)}_{n,2,2}(X)&:=\sum_{i\in\mathcal{V}_{1,j}}\oin\Big(\frac{1}{2}(\dtijn)+\sum_{u_1,v_1=1}^{d_1}\frac{(\dkijn)^{(u_1)}(\dkijn)^{(v_1)}}{1+1_{\{u_1=v_1\}}}\frac{\partial h_1}{\partial\kappa^{(u_1)}}(X,\kj)\frac{\partial h_1}{\partial\kappa^{(v_1)}}(X,\kj)\Big),\\
    A^{(j)}_{n,2,3}(X)&:=\sum_{i\in\mathcal{V}_{1,j}}\oin\sum_{u_1=1}^{d_1}\frac{1}{2}(\dkijn)^{(u_1)}(\dtijn)\frac{\partial h_1}{\partial\kappa^{(u_1)}}(X,\kj),\\
    A^{(j)}_{n,2,4}(X)&:=\sum_{i\in\mathcal{V}_{1,j}}\oin\frac{1}{8}(\dtijn)^2,
\end{align*}
for all $j\in[k^*_1]$ such that $|\mathcal{V}_{1,j}|>1$. 

\vspace{0.5 em}
\noindent
\textbf{Stage 1.2:} Next, we attempt to decompose the term $Q_n(Y|X):=\Big[\sum_{j = 1}^{k^*_2} \sigma((\beta_{1j}^{*})^{\top} X + \beta_{0j}^{*})\Big]\cdot[p_{G^n_2}(Y|X)-p_{G^*_2}(Y|X)]$ as
\begin{align*}
    Q_n(Y|X)&=\sum_{j=1}^{k^*_2}\Big[\sum_{i\in\mathcal{V}_{2,j}}\sigma((\beta_{1i}^{n})^{\top} X + \beta_{0i}^{n})\pi(Y|h_2(X,\ein),\nuin)-\sigma((\beta_{1j}^{*})^{\top} X + \beta_{0j}^{*})\pi(Y|h_2(X,\ej),\nuj)\Big]\\
    &-\sum_{j=1}^{k^*_2}\Big[\sum_{i\in\mathcal{V}_{2,j}}\sigma((\beta_{1i}^{n})^{\top} X + \beta_{0i}^{n})-\sigma((\beta_{1j}^{*})^{\top} X + \beta_{0j}^{*})\Big]p_{G^n_2}(Y|X)\\
    &:=B_{n}(Y|X)-C_{n}(Y|X).
\end{align*}
\textbf{Stage 1.2.1:} In this step, we decompose the term $B_{n}(Y|X)$ with a note that $\boj=0_d$ for all $j\in[k^*_2]:|\mathcal{V}_{2,j}|>1$:
\begin{align*}
   B_{n}(Y|X)&=\sum_{j\in[k^*_2]:|\mathcal{V}_{2,j}|=1}\Big[\sum_{i\in\mathcal{V}_{2,j}}\sigma((\beta_{1i}^{n})^{\top} X + \beta_{0i}^{n})\pi(Y|h_2(X,\ein),\nuin)-\sigma((\beta_{1j}^{*})^{\top} X + \beta_{0j}^{*})\pi(Y|h_2(X,\ej),\nuj)\Big]\\
   &+\sum_{j\in[k^*_2]:|\mathcal{V}_{2,j}|>1}\Big[\sum_{i\in\mathcal{V}_{2,j}}\sigma((\beta_{1i}^{n})^{\top} X + \beta_{0i}^{n})\pi(Y|h_2(X,\ein),\nuin)-\sigma(\beta_{0j}^{*})\pi(Y|h_2(X,\ej),\nuj)\Big]\\
   &=\sum_{j\in[k^*_2]:|\mathcal{V}_{2,j}|=1}\Big[\sum_{i\in\mathcal{V}_{2,j}}\sigma((\beta_{1i}^{n})^{\top} X + \beta_{0i}^{n})\pi(Y|h_2(X,\ein),\nuin)-\sigma((\beta_{1j}^{*})^{\top} X + \beta_{0j}^{*})\pi(Y|h_2(X,\ej),\nuj)\Big]\\
   &+\sum_{j\in[k^*_2]:|\mathcal{V}_{2,j}|>1}\sum_{i\in\mathcal{V}_{2,j}}\Big[\sigma((\beta_{1i}^{n})^{\top} X + \beta_{0i}^{n})\pi(Y|h_2(X,\ein),\nuin)-\sigma(\beta_{0i}^{n})\pi(Y|h_2(X,\ej),\nuj)\Big]\\
   &+\sum_{j\in[k^*_2]:|\mathcal{V}_{2,j}|>1}\Big[\sum_{i\in\mathcal{V}_{2,j}}\sigma(\beta_{0i}^{n})-\sigma(\beta_{0j}^{*})\Big]\pi(Y|h_2(X,\ej),\nuj)\\
   &:=B_{n,1}(Y|X)+B_{n,2}(Y|X)+B_{n,0}(Y|X).
\end{align*}
Denote $\psi(X;\beta_1,\beta_0):=\sigma(\beta_{1}^{\top}X+\beta_{0})$. By applying the first-order Taylor expansion to the function $\psi(X,\boin,\bzin)\pi(Y|h_2(X,\ein),\nuin)$ around the point $(\boj,\bzj,\ein,\nuin)$, we have
\begin{align*}
    B_{n,1}(Y|X)&=\sum_{j\in[k^*_2]:|\mathcal{V}_{2,j}|=1}\sum_{i\in\mathcal{V}_{2,j}}\sum_{|\alpha|=1}\frac{1}{\alpha!}(\dboijn)^{\alpha_1}(\dbzijn)^{\alpha_2}(\deijn)^{\alpha_3}(\dnuijn)^{\alpha_4}\\
    &\hspace{2cm}\times\frac{\partial^{|\alpha_1|+\alpha_2}\psi}{\partial\beta_1^{\alpha_1}\partial\beta_0^{\alpha_2}}(X;\boj,\bzj)\frac{\partial^{|\alpha_3|+\alpha_4}\pi}{\partial\eta^{\alpha_3}\partial\nu^{\alpha_4}}(Y|h_2(X,\ej),\nuj)+R_{n,3}(Y|X)\\
    &=\sum_{j\in[k^*_2]:|\mathcal{V}_{2,j}|=1}\sum_{\rho=0}^{2}B^{(j)}_{n,1,\rho}(X)\cdot\frac{\partial^{\rho}\pi}{\partial h_2^{\rho}}(Y|h_2(X,\ej),\nuj)+R_{n,3}(Y|X),
\end{align*}
where $R_{n,3}(Y|X)$ is a Taylor remainder such that $R_{n,3}(Y|X)/\mathcal{D}_{3n}\to0$ as $n\to\infty$ and
\begin{align*}
    B^{(j)}_{n,1,0}(X)&:=\sum_{i\in\mathcal{V}_{2,j}}\Big[\sum_{u=1}^{d}(\dboijn)^{(u)}\frac{\partial\psi}{\partial\beta_1^{(u)}}(X;\boj,\bzj)+(\dbzijn)\frac{\partial\psi}{\partial\beta_0}(X;\boj,\bzj)\Big],\\
    B^{(j)}_{n,1,1}(X)&:=\sum_{i\in\mathcal{V}_{2,j}}\sum_{u_2=1}^{d_2}(\deijn)^{(u_2)}\frac{\partial h_2}{\partial\eta^{(u_2)}}(X,\ej)\psi(X;\boj,\bzj),\\
    B^{(j)}_{n,1,2}(X)&:=\sum_{i\in\mathcal{V}_{2,j}}\frac{1}{2}(\dnuijn)\psi(X;\boj,\bzj),
\end{align*}
for all $j\in[k^*_2]$ such that $|\mathcal{V}_{2,j}|=1$. Next, by means of the second-order Taylor expansion to the function $\psi(X;\boj,\bzin)\pi(Y|h_2(X,\ej),\nuj)$ around the point $(\boj,\ej,\nuj)$ with a note that $\boj=0_d$ for all $j\in[k^*_2]:|\mathcal{V}_{2,j}|>1$, we decompose the term $B_{n,2}(Y|X)$ as  
\begin{align*}
    B_{n,2}(Y|X)&=\sum_{j\in[k^*_2]:|\mathcal{V}_{2,j}|>1}\sum_{i\in\mathcal{V}_{2,j}}\sum_{|\alpha|=1}^{2}\frac{1}{\alpha!}(\dboijn)^{\alpha_1}(\deijn)^{\alpha_2}(\dnuijn)^{\alpha_3}\\
    &\hspace{2cm}\times\frac{\partial^{|\alpha_1|}\psi}{\partial\beta_1^{\alpha_1}}(X;0_d,\bzin)\frac{\partial^{|\alpha_2|+\alpha_3}\pi}{\partial\eta^{\alpha_2}\partial\nu^{\alpha_3}}(Y|h_2(X,\ej),\nuj)+R_{n,4}(Y|X)\\
    &=\sum_{j\in[k^*_2]:|\mathcal{V}_{2,j}|>1}\sum_{\rho=0}^{4}B^{(j)}_{n,2,\rho}(X)\cdot\frac{\partial^{\rho}\pi}{\partial h_2^{\rho}}(Y|h_2(X,\ej),\nuj)+R_{n,4}(Y|X),
\end{align*}
where $R_{n,4}(Y|X)$ is a Taylor remainder such that $R_{n,4}(Y|X)/\mathcal{D}_{3n}\to0$ as $n\to\infty$ and
\begin{align*}
    B^{(j)}_{n,2,0}(X)&:=\sum_{i\in\mathcal{V}_{2,j}}\Big[\sum_{u=1}^{d}(\dboijn)^{(u)}\frac{\partial\psi}{\partial\beta_1^{(u)}}(X;0_d,\bzin)+\sum_{u,v=1}^{d}\frac{(\dboijn)^{(u)}(\dboijn)^{(v)}}{1+1_{\{u=v\}}}\frac{\partial^2\psi}{\partial\beta_1^{(u)}\partial\beta_1^{(v)}}(X;0_d,\bzin)\Big],\\
    B^{(j)}_{n,2,1}(X)&:=\sum_{i\in\mathcal{V}_{2,j}}\Big[\sum_{u_2=1}^{d_2}(\deijn)^{(u_2)}\frac{\partial h_2}{\partial\eta^{(u_2)}}(X.\ej)\psi(X;0_d,\bzin)+\sum_{u_2,v_2=1}^{d_2}\frac{(\deijn)^{(u_2)}(\deijn)^{(v_2)}}{1+1_{\{u_2=v_2\}}}\frac{\partial^2h_2}{\partial\eta^{(u_2)}\partial\eta^{(v_2)}}(X,\ej)\\
    &\times\psi(X;0_d,\bzin)+\sum_{u=1}^{d}\sum_{u_2=1}^{d_2}(\dboijn)^{(u)}(\deijn)^{(u_2)}\frac{\partial h_2}{\partial\eta^{(u_2)}}(X.\ej)\frac{\partial\psi}{\partial\beta_1^{(u)}}(X;0_d,\bzin)\Big],\\
    B^{(j)}_{n,2,2}(X)&:=\sum_{i\in\mathcal{V}_{2,j}}\Big[\frac{1}{2}(\dnuijn)\psi(X;0_d,\bzin)+\sum_{u=1}^{d}(\dboijn)^{(u)}\frac{1}{2}(\dnuijn)\frac{\partial\psi}{\partial\beta_1^{(u)}}(X;0_d,\bzin)\\
    &\hspace{2cm}+\sum_{u_2,v_2=1}^{d_2}\frac{(\deijn)^{(u_2)}(\deijn)^{(v_2)}}{1+1_{\{u_2=v_2\}}}\frac{\partial^2h_2}{\partial\eta^{(u_2)}\partial\eta^{(v_2)}}(X,\ej)\psi(X;0_d,\bzin)\Big],\\
    B^{(j)}_{n,2,3}(X)&:=\sum_{i\in\mathcal{V}_{2,j}}\Big[\sum_{u_2=1}^{d_2}(\deijn)^{(u_2)}\frac{1}{2}(\dnuijn)\frac{\partial h_2}{\partial\eta^{(u_2)}}(X,\ej)\psi(X;0_d,\bzin)\Big],\\
    B^{(j)}_{n,2,4}(X)&:=\sum_{i\in\mathcal{V}_{2,j}}\frac{1}{8}(\dnuijn)^2\psi(X;0_d,\bzin),
\end{align*}
for all $j\in[k^*_2]$ such that $|\mathcal{V}_{2,j}|>1$.

\vspace{0.5 em}
\noindent
\textbf{Stage 1.2.2:} In this step, we decompose the term $C_{n}(Y|X)$ as
\begin{align*}
    C_{n}(Y|X)&=\sum_{j\in[k^*_2]:|\mathcal{V}_{2,j}|=1}\Big[\sum_{i\in\mathcal{V}_{2,j}}\psi(X;\boin,\bzin)-\psi(X;\boj,\bzj)\Big]p_{G^n_2}(Y|X)\\
    &+\sum_{j\in[k^*_2]:|\mathcal{V}_{2,j}|>1}\Big[\sum_{i\in\mathcal{V}_{2,j}}\psi(X;\boin,\bzin)-\psi(X;\boj,\bzj)\Big]p_{G^n_2}(Y|X)\\
    &=\sum_{j\in[k^*_2]:|\mathcal{V}_{2,j}|=1}\Big[\sum_{i\in\mathcal{V}_{2,j}}\psi(X;\boin,\bzin)-\psi(X;\boj,\bzj)\Big]p_{G^n_2}(Y|X)\\
    &+\sum_{j\in[k^*_2]:|\mathcal{V}_{2,j}|>1}\sum_{i\in\mathcal{V}_{2,j}}\Big[\psi(X;\boin,\bzin)-\psi(X;0_d,\bzin)\Big]p_{G^n_2}(Y|X)\\
    &+\sum_{j\in[k^*_2]:|\mathcal{V}_{2,j}|>1}\Big[\sum_{i\in\mathcal{V}_{2,j}}\psi(X;0_d,\bzin)-\psi(X;0_d,\bzj)\Big]p_{G^n_2}(Y|X)\\
    &:=C_{n,1}(Y|X)+C_{n,2}(Y|X)+C_{n,0}(Y|X).
\end{align*}
By applying the first-order Taylor expansion to the function $\psi(X;\boin,\bzin)$ around the point $(\boj,\bzj)$, we have
\begin{align*}
    C_{n,1}(Y|X)&=\sum_{j\in[k^*_2]:|\mathcal{V}_{2,j}|=1}\sum_{i\in\mathcal{V}_{2,j}}\Big[\sum_{u=1}^{d}(\dboijn)^{(u)}\frac{\partial\psi}{\beta_1^{(u)}}(X;\boj,\bzj)+(\dbzijn)\frac{\partial\psi}{\partial\beta_0}(X;\boj,\bzj)\Big]p_{G^n_2}(Y|X)\\
    &\hspace{12cm}+R_{n,5}(Y|X),
\end{align*}
where $R_{n,5}(Y|X)$ is a Taylor remainder such that $R_{n,5}(Y|X)/\mathcal{D}_{3n}\to0$ as $n\to\infty$. Next, by means of the second-order Taylor expansion to the function $\psi(X;\boin,\bzin)$ around the point $\beta^*_{1j}=0_d$ for $j\in[k^*_2]:|\mathcal{V}_{2,j}|>1$, we have
\begin{align*}
    C_{n,2}(Y|X)&=\sum_{j\in[k^*_2]:|\mathcal{V}_{2,j}|>1}\sum_{i\in\mathcal{V}_{2,j}}\Big[\sum_{u=1}^{d}(\dboijn)^{(u)}\frac{\partial\psi}{\partial\beta_1^{(u)}}(X;0_d,\bzin)\\
    &\hspace{1cm}+\sum_{u,v=1}^{d}\frac{(\dboijn)^{(u)}(\dboijn)^{(v)}}{1+1_{\{u=v\}}}\frac{\partial^2\psi}{\partial\beta_1^{(u)}\partial\beta_1^{(v)}}(X;0_d,\bzin)\Big]p_{G^n_2}(Y|X)+R_{n,6}(Y|X),
\end{align*}
where $R_{n,6}(Y|X)$ is a Taylor remainder such that $R_{n,6}(Y|X)/\mathcal{D}_{3n}\to0$ as $n\to\infty$.

\vspace{0.5 em}
\noindent
Combining the above decompositions, we can view $A_{n,0}(Y|X)/\mathcal{D}_{3n}$, $[A_{n,1}(Y|X)-R_{n,1}(Y|X)]/\mathcal{D}_{3n}$, $[A_{n,2}(Y|X)-R_{n,2}(Y|X)]/\mathcal{D}_{3n}$, $B_{n,0}(Y|X)/\mathcal{D}_{3n}$, $[B_{n,1}(Y|X)-R_{n,3}(Y|X)]/\mathcal{D}_{3n}$, $[B_{n,2}(Y|X)-R_{n,4}(Y|X)]/\mathcal{D}_{3n}$, $C_{n,0}(Y|X)/\mathcal{D}_{3n}$, $[C_{n,1}(Y|X)-R_{n,5}(Y|X)]/\mathcal{D}_{3n}$ and $[C_{n,2}(Y|X)-R_{n,6}(Y|X)]/\mathcal{D}_{3n}$ as a combination of elements from the following sets
\begin{align*}
    \mathcal{S}_{0,j}&:=\{\pi(Y|h_1(X,\kj),\tj)\},\\
    \mathcal{S}_{1,j}&:=\Bigg\{\frac{\partial h_1}{\partial\kappa^{(u_1)}}(X,\kj)\frac{\partial\pi}{\partial h_1}(Y|h_1(X,\kj),\tj), \ \frac{\partial^2 h_1}{\partial\kappa^{(u_1)}\partial\kappa^{(v_1)}}(X,\kj)\frac{\partial\pi}{\partial h_1}(Y|h_1(X,\kj),\tj):u_1,v_1\in[d_1]\Bigg\},\\
    \mathcal{S}_{2,j}&:=\Bigg\{\frac{\partial^2\pi}{\partial h_1^2}(Y|h_1(X,\kj),\tj), \ \frac{\partial h_1}{\partial\kappa^{(u_1)}}(X,\kj)\frac{\partial h_1}{\partial\kappa^{(v_1)}}(X,\kj)\frac{\partial^2\pi}{\partial h_1^2}(Y|h_1(X,\kj),\tj):u_1,v_1\in[d_1]\Bigg\},\\
    \mathcal{S}_{3,j}&:=\Bigg\{ \frac{\partial h_1}{\partial\kappa^{(u_1)}}(X,\kj)\frac{\partial^3\pi}{\partial h_1^3}(Y|h_1(X,\kj),\tj):u_1,v_1\in[d_1]\Bigg\},\\
    \mathcal{S}_{4,j}&:=\Bigg\{ \frac{\partial^4\pi}{\partial h_1^4}(Y|h_1(X,\kj),\tj):u_1,v_1\in[d_1]\Bigg\},
\end{align*}
for all $j\in[k^*_1]$, and
\begin{align*}
    \mathcal{T}_{0,j}&:=\Bigg\{\pi(Y|h_2(X,\ej),\nuj),\ \frac{\partial\psi}{\partial\beta_1^{(u)}}(X;0_d,\bzin)\pi(Y|h_2(X,\ej),\nuj), \\
    &\hspace{4cm}\frac{\partial^2\psi}{\partial\beta_1^{(u)}\partial\beta_1^{(v)}}(X;0_d,\bzin)\pi(Y|h_2(X,\ej),\nuj):u,v\in[d]\Bigg\},\\
    \mathcal{T}_{1,j}&:=\Bigg\{\frac{\partial h_2}{\partial\eta^{(u_2)}}(X.\ej)\psi(X;0_d,\bzin)\frac{\partial\pi}{\partial h_2}(Y|h_2(X,\ej),\nuj), \\
    &\hspace{4cm}\frac{\partial h_2}{\partial\eta^{(u_2)}}(X.\ej)\frac{\partial\psi}{\partial\beta_1^{(u)}}(X;0_d,\bzin)\frac{\partial\pi}{\partial h_2}(Y|h_2(X,\ej),\nuj):u\in[d], \ u_2\in[d_2]\Bigg\},\\
    \mathcal{T}_{2,j}&:=\Bigg\{\psi(X;0_d,\bzin)\frac{\partial^2\pi}{\partial h_2^2}(Y|h_2(X,\ej),\nuj), \ \frac{\partial\psi}{\partial\beta_1^{(u)}}(X;0_d,\bzin)\frac{\partial^2\pi}{\partial h_2^2}(Y|h_2(X,\ej),\nuj),\\
    &\hspace{3cm}\frac{\partial^2h_2}{\partial\eta^{(u_2)}\partial\eta^{(v_2)}}(X,\ej)\psi(X;0_d,\bzin)\frac{\partial^2\pi}{\partial h_2^2}(Y|h_2(X,\ej),\nuj):u\in[d], \ u_2,v_2\in[d_2]\Bigg\},\\
    \mathcal{T}_{3,j}&:=\Bigg\{\frac{\partial h_2}{\partial\eta^{(u_2)}}(X,\ej)\psi(X;0_d,\bzin)\frac{\partial^3\pi}{\partial h_2^3}(Y|h_2(X,\ej),\nuj):u_2\in[d_2]\Bigg\},\\
    \mathcal{T}_{4,j}&:=\Bigg\{\psi(X;0_d,\bzin)\frac{\partial^4\pi}{\partial h_2^4}(Y|h_2(X,\ej),\nuj)\Bigg\},\\
    \mathcal{T}_{5,j}&:=\Bigg\{\frac{\partial\psi}{\beta_1^{(u)}}(X;\boj,\bzj)p_{G^n_2}(Y|X), \ \frac{\partial\psi}{\partial\beta_0}(X;\boj,\bzj)p_{G^n_2}(Y|X), \ \frac{\partial\psi}{\partial\beta_1^{(u)}}(X;0_d,\bzin)p_{G^n_2}(Y|X),\\
    &\hspace{4cm}\frac{\partial^2\psi}{\partial\beta_1^{(u)}\partial\beta_1^{(v)}}(X;0_d,\bzin)p_{G^n_2}(Y|X):u\in[d]\Bigg\},
\end{align*}
for all $j\in[k^*_2]$.

\vspace{0.5 em}
\noindent
\textbf{Stage 2 - Non-vanishing coefficients:} In this stage, we demonstrate that not all the coefficients in the representations of $A_{n,0}(Y|X)/\mathcal{D}_{3n}$, $[A_{n,1}(Y|X)-R_{n,1}(Y|X)]/\mathcal{D}_{3n}$, $[A_{n,2}(Y|X)-R_{n,2}(Y|X)]/\mathcal{D}_{3n}$, $B_{n,0}(Y|X)/\mathcal{D}_{3n}$, $[B_{n,1}(Y|X)-R_{n,3}(Y|X)]/\mathcal{D}_{3n}$, $[B_{n,2}(Y|X)-R_{n,4}(Y|X)]/\mathcal{D}_{3n}$, $C_{n,0}(Y|X)/\mathcal{D}_{3n}$, $[C_{n,1}(Y|X)-R_{n,5}(Y|X)]/\mathcal{D}_{3n}$ and $[C_{n,2}(Y|X)-R_{n,6}(Y|X)]/\mathcal{D}_{3n}$ go to zero when $n\to\infty$. Assume by contrary that all these coefficients converge to zero. By using the same arguments as in Stage 2 in Appendix~\ref{appendix:strongly_identifiable_experts}, we have
\begin{align*}
    \frac{1}{\mathcal{D}_{3n}}\Big[\sum_{j=1}^{k^*_1}\Big|\sum_{i\in\mathcal{V}_{1,j}}\oin-\oj\Big|&+\sum_{j\in[k^*_1]:|\mathcal{V}_{1,j}|=1}\sum_{i\in\mathcal{V}_{1,j}}\oin(\|\dkijn\|+|\dtijn|)\\
    &+\sum_{j\in[k^*_1]:|\mathcal{V}_{1,j}|>1}\sum_{i\in\mathcal{V}_{1,j}}\oin(\|\dkijn\|^2+|\dtijn|^2)\Big]\to0,
\end{align*}
as $n\to\infty$. Additionally, by considering the coefficients of the terms:
\begin{itemize}
    \item $\frac{\partial\psi}{\partial\beta_1^{(u)}}(X;\boj,\bzj)\pi(Y|h_2(X,\ej),\nuj)$ for $j\in[k^*_2]:|\mathcal{V}_{2,j}|=1$, we get
    \begin{align*}
        \frac{1}{\mathcal{D}_{3n}}\sum_{j\in[k^*_2]:|\mathcal{V}_{2,j}|=1}\sum_{i\in\mathcal{V}_{2,j}}\|\dboijn\|\to0;
    \end{align*}
    \item $\frac{\partial\psi}{\partial\beta_0}(X;\boj,\bzj)\pi(Y|h_2(X,\ej),\nuj)$ for $j\in[k^*_2]:|\mathcal{V}_{2,j}|=1$, we get
    \begin{align*}
        \frac{1}{\mathcal{D}_{3n}}\sum_{j\in[k^*_2]:|\mathcal{V}_{2,j}|=1}\sum_{i\in\mathcal{V}_{2,j}}|\dbzijn|\to0;
    \end{align*}
    \item $\frac{\partial h_2}{\partial\eta^{(u_2)}}(X,\ej)\psi(X;\boj,\bzj)\frac{\partial\pi}{\partial h_2}(Y|h_2(X,\ej),\nuj)$ for $j\in[k^*_2]:|\mathcal{V}_{2,j}|=1$, we get
    \begin{align*}
        \frac{1}{\mathcal{D}_{3n}}\sum_{j\in[k^*_2]:|\mathcal{V}_{2,j}|=1}\sum_{i\in\mathcal{V}_{2,j}}\|\deijn\|\to0;
    \end{align*}
    \item $\psi(X;\boj,\bzj)\frac{\partial\pi}{\partial h_2}(Y|h_2(X,\ej),\nuj)$ for $j\in[k^*_2]:|\mathcal{V}_{2,j}|=1$, we get
    \begin{align*}
        \frac{1}{\mathcal{D}_{3n}}\sum_{j\in[k^*_2]:|\mathcal{V}_{2,j}|=1}\sum_{i\in\mathcal{V}_{2,j}}|\dnuijn|\to0;
    \end{align*}
    \item $\pi(Y|h_2(X,\ej),\nuj)$ for $j\in[k^*_2]:|\mathcal{V}_{2,j}|>1$, we get
    \begin{align*}
        \frac{1}{\mathcal{D}_{3n}}\sum_{j\in[k^*_2]:|\mathcal{V}_{2,j}|>1}\Big|\sum_{i\in\mathcal{V}_{2,j}}\sigma(\bzin)-\sigma(\bzj)\Big|\to0;
    \end{align*}
    \item $\frac{\partial^2\psi}{\partial\beta_1^{(u)}\partial\beta_1^{(v)}}(X;0_d,\bzin)\pi(Y|h_2(X,\ej),\nuj)$ for $j\in[k^*_2]:|\mathcal{V}_{2,j}|>1$, we get
    \begin{align*}
         \frac{1}{\mathcal{D}_{3n}}\sum_{j\in[k^*_2]:|\mathcal{V}_{2,j}|>1}\sum_{i\in\mathcal{V}_{2,j}}\|\dboijn\|^2\to0;
    \end{align*}
    \item $\frac{\partial^2h_2}{\partial\eta^{(u_2)}\partial\eta^{(v_2)}}(X,\ej)\psi(X;0_d,\bzin)\frac{\partial^2\pi}{\partial h_2^2}(Y|h_2(X,\ej),\nuj)$ for $j\in[k^*_2]:|\mathcal{V}_{2,j}|>1$, we get
    \begin{align*}
         \frac{1}{\mathcal{D}_{3n}}\sum_{j\in[k^*_2]:|\mathcal{V}_{2,j}|>1}\sum_{i\in\mathcal{V}_{2,j}}\|\deijn\|^2\to0;
    \end{align*}
    \item $\psi(X;0_d,\bzin)\frac{\partial^4\pi}{\partial h_2^4}(Y|h_2(X,\ej),\nuj)$ for $j\in[k^*_2]:|\mathcal{V}_{2,j}|>1$, we get
    \begin{align*}
         \frac{1}{\mathcal{D}_{3n}}\sum_{j\in[k^*_2]:|\mathcal{V}_{2,j}|>1}\sum_{i\in\mathcal{V}_{2,j}}|\dnuijn|^2\to0.
    \end{align*}
\end{itemize}
Putting the above limits together, we deduce $1=\frac{\mathcal{D}_{3n}}{\mathcal{D}_{3n}}\to0$ as $n\to\infty$, which is a contradiction. Therefore, at least one among the coefficients in the representations of $A_{n,0}(Y|X)/\mathcal{D}_{3n}$, $[A_{n,1}(Y|X)-R_{n,1}(Y|X)]/\mathcal{D}_{3n}$, $[A_{n,2}(Y|X)-R_{n,2}(Y|X)]/\mathcal{D}_{3n}$, $B_{n,0}(Y|X)/\mathcal{D}_{3n}$, $[B_{n,1}(Y|X)-R_{n,3}(Y|X)]/\mathcal{D}_{3n}$, $[B_{n,2}(Y|X)-R_{n,4}(Y|X)]/\mathcal{D}_{3n}$, $C_{n,0}(Y|X)/\mathcal{D}_{3n}$, $[C_{n,1}(Y|X)-R_{n,5}(Y|X)]/\mathcal{D}_{3n}$ and $[C_{n,2}(Y|X)-R_{n,6}(Y|X)]/\mathcal{D}_{3n}$ does not go to zero.

\vspace{0.5 em}
\noindent
\textbf{Stage 3 - Fatou's lemma contradiction:} In this stage, we use the Fatou's lemma to show a contradiction to the result of Stage 2. For that purpose, let us denote $m_n$ as the maximum of the absolute values of the coefficients in the representations of $A_{n,0}(Y|X)/\mathcal{D}_{3n}$, $[A_{n,1}(Y|X)-R_{n,1}(Y|X)]/\mathcal{D}_{3n}$, $[A_{n,2}(Y|X)-R_{n,2}(Y|X)]/\mathcal{D}_{3n}$, $B_{n,0}(Y|X)/\mathcal{D}_{3n}$, $[B_{n,1}(Y|X)-R_{n,3}(Y|X)]/\mathcal{D}_{3n}$, $[B_{n,2}(Y|X)-R_{n,4}(Y|X)]/\mathcal{D}_{3n}$, $C_{n,0}(Y|X)/\mathcal{D}_{3n}$, $[C_{n,1}(Y|X)-R_{n,5}(Y|X)]/\mathcal{D}_{3n}$ and $[C_{n,2}(Y|X)-R_{n,6}(Y|X)]/\mathcal{D}_{3n}$. It follows from the result of Stage 2 that $1/m_n\not\to\infty$ as $n\to\infty$. In addition, we also denote
\begin{align*}
    \frac{1}{m_n\mathcal{D}_{3n}}\cdot\sum_{i\in\mathcal{V}_{1,j}}\oin(\dkijn)^{(u_1)}\to s^{(u_1)}_{1,j},& \quad \frac{1}{m_n\mathcal{D}_{3n}}\cdot\sum_{i\in\mathcal{V}_{1,j}}\oin(\dtijn)\to s_{2,j},\\
    \frac{1}{m_n\mathcal{D}_{3n}}\cdot\sum_{i\in\mathcal{V}_{1,j}}\oin(\dkijn)^{(u_1)}(\dkijn)^{(v_1)}\to s^{(u_1v_1)}_{3,j},& \quad \frac{1}{m_n\mathcal{D}_{3n}}\cdot\sum_{i\in\mathcal{V}_{1,j}}\oin(\dtijn)^2\to s_{4,j},\\
    \frac{1}{m_n\mathcal{D}_{3n}}\cdot\sum_{i\in\mathcal{V}_{1,j}}\oin(\dkijn)^{(u_1)}(\dtijn)\to s^{(u_1)}_{5,j},& \quad \frac{1}{m_n\mathcal{D}_{3n}}\cdot\Big(\sum_{i\in\mathcal{V}_{1,j}}\oin-\oj\Big)\to s_{0,j},\\
\end{align*}
for all $j\in[k^*_1]$ and
\begin{align*}
    \frac{1}{m_n\mathcal{D}_{3n}}\cdot\sum_{i\in\mathcal{V}_{2,j}}(\dbzijn)\to t_{0,j}, j\in[k^*_2]:|\mathcal{V}_{2,j}|=1,& \quad \frac{1}{m_n\mathcal{D}_{3n}}\cdot\sum_{i\in\mathcal{V}_{2,j}}(\dboijn)^{(u)}\to t^{(u)}_{1,j},\\
    \frac{1}{m_n\mathcal{D}_{3n}}\cdot\sum_{i\in\mathcal{V}_{2,j}}(\deijn)^{(u_2)}\to t^{(u_2)}_{2,j},& \quad \frac{1}{m_n\mathcal{D}_{3n}}\cdot\sum_{i\in\mathcal{V}_{2,j}}(\dnuijn)\to t_{3,j},\\
\end{align*}
for all $j\in[k^*_2]:|\mathcal{V}_{2,j}|=1$, and
\begin{align*}
    \frac{1}{m_n\mathcal{D}_{3n}}\cdot\Big(\sum_{i\in\mathcal{V}_{2,j}}\sigma(\bzin)-\sigma(\bzj)\Big)\to t_{0,j},& \quad \frac{1}{m_n\mathcal{D}_{3n}}\cdot(\dboijn)^{(u)}\to t^{(u)}_{1,j,i},\\
    \frac{1}{m_n\mathcal{D}_{3n}}\cdot(\deijn)^{(u_2)}\to t^{(u_2)}_{2,j,i},& \quad \frac{1}{m_n\mathcal{D}_{3n}}\cdot(\dnuijn)\to t_{3,j,i},\\
    \frac{1}{m_n\mathcal{D}_{3n}}\cdot(\dboijn)^{(u)}(\dboijn)^{(v)}\to t^{(uv)}_{4,j,i},& \quad \frac{1}{m_n\mathcal{D}_{3n}}\cdot=(\deijn)^{(u_2)}(\deijn)^{(v_2)}\to t^{(u_2v_2)}_{5,j,i},\\
    \frac{1}{m_n\mathcal{D}_{3n}}\cdot(\dnuijn)^2\to t_{6,j,i},& \quad \frac{1}{m_n\mathcal{D}_{3n}}\cdot(\dboijn)^{(u)}(\deijn)^{(v_2)}\to t^{(uv_2)}_{7,j,i},\\
    \frac{1}{m_n\mathcal{D}_{3n}}\cdot(\dboijn)^{(u)}(\dnuijn)\to t^{(u)}_{8,j,i},& \quad \frac{1}{m_n\mathcal{D}_{3n}}\cdot(\deijn)^{(u_2)}(\dnuijn)\to t^{(u_2)}_{9,j,i},
\end{align*}
for all $j\in[k^*_2]:|\mathcal{V}_{2,j}|>1$ as $n\to\infty$.
Due to the result of Stage 2, at least one among the above limits is non-zero. Recall from equation~\eqref{eq:expectation_zero_sigmoid} that we get
\begin{align*}
    \bbE_X[V(g_{G^n_1,G^n_2}(\cdot|X),g_{G^*_1,G^*_2}(\cdot|X))]/\mathcal{D}_{3n}\to0,
\end{align*}
Moreover, by means of the Fatou's lemma, we have
\begin{align*}
    \lim_{n\to\infty}\dfrac{\bbE_X[V(g_{G^n_1,G^n_2}(\cdot|X),g_{G^*_1,G^*_2}(\cdot|X))]}{m_n\mathcal{D}_{3n}}\geq\int\liminf_{n\to\infty}\dfrac{|g_{G^n_1,G^n_2}(Y|X)-g_{G^*_1,G^*_2}(Y|X)|}{2m_n\mathcal{D}_{3n}}\dint (X,Y).
\end{align*}
Then, we deduce $[g_{G^n_1,G^n_2}(Y|X)-g_{G^*_1,G^*_2}(Y|X)]/[m_n\mathcal{D}_{3n}]\to0$ as $n\to\infty$ for almost surely $(X,Y)$. Since the input space is bounded and the parameter space is compact, the quantity $\sum_{j = 1}^{k^*_2} \sigma((\beta_{1j}^{*})^{\top} X + \beta_{0j}^{*})$ is bounded. Thus, we also have
\begin{align*}
    \Big[\sum_{j = 1}^{k^*_2} \sigma((\beta_{1j}^{*})^{\top} X + \beta_{0j}^{*})\Big][g_{G^n_1,G^n_2}(Y|X)-g_{G^*_1,G^*_2}(Y|X)]/[m_n\mathcal{D}_{3n}]\to0,
\end{align*}
implying that 
\begin{align*}
    \frac{1}{2}\Big[\sum_{j = 1}^{k^*_2} \sigma((\beta_{1j}^{*})^{\top} X + \beta_{0j}^{*})\Big]\cdot\dfrac{q_{G^n_1}(Y|X)-q_{G^*_1}(Y|X)}{m_n\mathcal{D}_{3n}}+\frac{1}{2}\dfrac{Q_n(Y|X)}{m_n\mathcal{D}_{3n}}\to0.
\end{align*}
as $n\to\infty$ for almost surely $(X,Y)$. From the decomposition of the terms $q_{G^n_1}(Y|X)-q_{G^*_1}(Y|X)$ and $Q_n(Y|X)$ in Stage 1, we have
\begin{align}
    \frac{1}{2}\Big[\sum_{j = 1}^{k^*_2} \sigma((\beta_{1j}^{*})^{\top} X + \beta_{0j}^{*})\Big]\cdot\dfrac{A_{n,2}(Y|X)+A_{n,1}(Y|X)+A_{n,0}(Y|X)}{m_n\mathcal{D}_{3n}}\nonumber\\
    \label{eq:zero_limit_sigmoid}
    +\frac{1}{2}\dfrac{B_{n,1}(Y|X)+B_{n,2}(Y|X)+B_{n,3}(Y|X)-C_{n,1}(Y|X)-C_{n,2}(Y|X)-C_{n,3}(Y|X)}{m_n\mathcal{D}_{3n}}\to0.
\end{align}
We have
\begin{align*}
    &\lim_{n\to\infty}\frac{A_{n,0}(Y|X)}{m_n\mathcal{D}_{3n}}=\sum_{j=1}^{k^*_1}s_{0,j}\pi(Y|h_1(X,\kappa^*_j),\tau^*_j),\\
    &\lim_{n\to\infty}\frac{A_{n,1}(Y|X)}{m_n\mathcal{D}_{3n}}=\sum_{j\in[k^*_1]:|\mathcal{V}_{1,j}|=1}\Big[\sum_{u_1=1}^{d_1}s_{1,j}^{(u_1)}\frac{\partial h_1}{\partial\kappa^{(u_1)}}(X,\kj)\frac{\partial\pi}{\partial h_1}(Y|h_1(X,\kj),\tj)\\
    &\hspace{10cm}+\frac{1}{2}s_{2,j}\frac{\partial^2\pi}{\partial h_1^2}(Y|h_1(X,\kj),\tj)\Big],\\
    &\lim_{n\to\infty}\frac{A_{n,2}(Y|X)}{m_n\mathcal{D}_{3n}}=\sum_{j\in[k^*_1]:|\mathcal{V}_{1,j}|>1}\Big[\Big(\sum_{u_1=1}^{d_1}s_{1,j}^{(u_1)}\frac{\partial h_1}{\partial\kappa^{(u_1)}}(X,\kj)+\sum_{u_1,v_1=1}^{d_1}\frac{s_{3,j}^{(u_1v_1)}}{1+1_{\{u_1=v_1\}}}\frac{\partial^2h_1}{\partial\kappa^{(u_1)}\partial\kappa^{(v_1)}}(X,\kj)\Big)\\
    &\times \frac{\partial\pi}{\partial h_1}(Y|h_1(X,\kj),\tj)+\Big(\frac{1}{2}s_{2,j}+\sum_{u_1,v_1=1}^{d_1}\frac{s_{3,j}^{(u_1v_1)}}{1+1_{\{u_1=v_1\}}}\frac{\partial h_1}{\partial\kappa^{(u_1)}}(X,\kj)\frac{\partial h_1}{\partial\kappa^{(v_1)}}(X,\kj)\Big)\frac{\partial^2\pi}{\partial h_1^2}(Y|h_1(X,\kj),\tj)\\
    &\hspace{3cm}+\Big(\frac{1}{2}\sum_{u_1=1}^{d_1}s_{5,j}^{(u_1)}\frac{\partial h_1}{\partial\kappa^{(u_1)}}(X,\kj)\Big)\frac{\partial^3\pi}{\partial h_1^3}(Y|h_1(X,\kj),\tj)+\frac{1}{8}s_{4,j}\frac{\partial^4\pi}{\partial h_1^4}(Y|h_1(X,\kj),\tj)\Big],
\end{align*}
and
\begin{align*}
    &\lim_{n\to\infty}\frac{B_{n,0}(Y|X)}{m_n\mathcal{D}_{3n}}=\sum_{j\in[k^*_2]:|\mathcal{V}_{2,j}|>1}t_{0,j}\pi(Y|h_2(X,\ej),\nuj),\\
    &\lim_{n\to\infty}\frac{B_{n,1}(Y|X)}{m_n\mathcal{D}_{3n}}=\sum_{j\in[k^*_2]:|\mathcal{V}_{2,j}|=1}\Big[\Big(\sum_{u=1}^{d}t_{1,j}^{(u)}\frac{\partial\psi}{\partial\beta_1^{(u)}}(X;\boj,\bzj)+t_{0,j}\frac{\partial\psi}{\partial\beta_0}(X;\boj,\bzj)\Big)\pi(Y|h_2(X,\ej),\nuj)\\
    &+\sum_{u_2=1}^{d_2}t_{2,j}^{(u_2)}\frac{\partial h_2}{\partial\eta^{(u_2)}}(X,\ej)\psi(X;\boj,\bzj)\frac{\partial\pi}{\partial h_2}(Y|h_2(X,\ej),\nuj)+\frac{1}{2}t_{3,j}\psi(X;\boj,\bzj)\frac{\partial^2\pi}{\partial h_2^2}(Y|h_2(X,\ej),\nuj)\Big],\\
    &\lim_{n\to\infty}\frac{B_{n,2}(Y|X)}{m_n\mathcal{D}_{3n}}=\sum_{j\in[k^*_2]:|\mathcal{V}_{2,j}|>1}\sum_{i\in\mathcal{V}_{2,j}}\Big[\Big(\sum_{u,v=1}^{d}\frac{t_{4,j,i}^{(uv)}}{1+1_{\{u=v\}}}\frac{\partial^2\psi}{\partial\beta_1^{(u)}\partial\beta_1^{(v)}}(X;0_d,\bar{\beta}_{0i})\\
    &+\sum_{u=1}^{d}t_{1,j,i}^{(u)}\frac{\partial\psi}{\partial\beta_1^{(u)}}(X;0_d,\bar{\beta}_{0i})\Big)\pi(Y|h_2(X,\ej),\nuj)+\Big(\sum_{u=1}^{d}\sum_{u_2=1}^{d_2}t_{7,j,i}^{(uu_2)}\frac{\partial h_2}{\partial\eta^{(u_2)}}(X.\ej)\frac{\partial\psi}{\partial\beta_1^{(u)}}(X;0_d,\bar{\beta}_{0i})\\
    &+\sum_{u_2=1}^{d_2}t_{2,j,i}^{(u_2)}\frac{\partial h_2}{\partial\eta^{(u_2)}}(X.\ej)\psi(X;0_d,\bar{\beta}_{0i})+\sum_{u_2,v_2=1}^{d_2}t_{5,j,i}^{(u_2v_2)}\frac{\partial^2h_2}{\partial\eta^{(u_2)}\partial\eta^{v_2)}}(X,\ej)\psi(X;0_d,\bar{\beta}_{0i})\Big)\frac{\partial\pi}{\partial h_2}(Y|h_2(X,\ej),\nuj)\\
    &+\Big(\sum_{u=1}^{d}\frac{1}{2}t_{8,j,i}^{(u)}\frac{\partial\psi}{\partial\beta_1^{(u)}}(X;0_d,\bar{\beta}_{0i})+\frac{1}{2}t_{3,j,i}\psi(X;0_d,\bar{\beta}_{0i})+\sum_{u_2,v_2=1}^{d_2}\frac{t_{5,j,i}^{(u_2v_2)}}{1+1_{\{u_2=v_2\}}}\frac{\partial h_2}{\partial\eta^{(u_2)}}(X,\ej)\frac{\partial h_2}{\partial\eta^{(v_2)}}(X,\ej)\\
    &\times\psi(X;0_d,\bar{\beta}_{0i})\Big)\frac{\partial^2\pi}{\partial h_2^2}(Y|h_2(X,\ej),\nuj)+\sum_{u_2=1}^{d_2}\frac{1}{2}t_{9,j,i}^{(u_2)}\frac{\partial h_2}{\partial\eta^{(u_2)}}(X,\ej)\psi(X;0_d,\bar{\beta}_{0i})\frac{\partial^3\pi}{\partial h_2^3}(Y|h_2(X,\ej),\nuj)\\
    &\hspace{8cm}+\frac{1}{8}t_{6,j,i}\psi(X;0_d,\bar{\beta}_{0i})\frac{\partial^4\pi}{\partial h_2^4}(Y|h_2(X,\ej),\nuj)\Big],
\end{align*}
and
\begin{align*}
    &\lim_{n\to\infty}\frac{C_{n,0}(Y|X)}{m_n\mathcal{D}_{3n}}=\sum_{j\in[k^*_2]:|\mathcal{V}_{2,j}|>1}t_{0,j}p_{G^*_2}(Y|X),\\
    &\lim_{n\to\infty}\frac{C_{n,1}(Y|X)}{m_n\mathcal{D}_{3n}}=\sum_{j\in[k^*_2]:|\mathcal{V}_{2,j}|=1}\Big[\sum_{u=1}^{d}t_{1,j}^{(u)}\frac{\partial\psi}{\beta_1^{(u)}}(X;\boj,\bzj)+t_{0,j}\frac{\partial\psi}{\partial\beta_0}(X;\boj,\bzj)\Big]p_{G^*_2}(Y|X),\\
    &\lim_{n\to\infty}\frac{C_{n,2}(Y|X)}{m_n\mathcal{D}_{3n}}=\sum_{j\in[k^*_2]:|\mathcal{V}_{2,j}|>1}\sum_{i\in\mathcal{V}_{2,j}}\Big[\sum_{u=1}^{d}t_{1,j,i}^{(u)}\frac{\partial\psi}{\beta_1^{(u)}}(X;0_d,\bar{\beta}_{0i})\\
    &\hspace{6cm}+\sum_{u,v=1}^{d}\frac{t_{4,j,i}^{(uv)}}{1+1_{\{u=v\}}}\frac{\partial^2\psi}{\partial\beta_1^{(u)}\partial\beta_1^{(v)}}(X;0_d,\bar{\beta}_{0i})\Big]p_{G^*_2}(Y|X).
\end{align*}
Note that for almost every $X$, the set
\begin{align*}
    &\Bigg\{\Big[\sum_{j = 1}^{k^*_2} \sigma((\beta_{1j}^{*})^{\top} X + \beta_{0j}^{*})\Big]\frac{\partial^{\rho}\pi}{\partial h_1^{\rho}}(Y|h_1(X,\kj),\tj):0\leq\rho\leq 4, \  j\in[k^*_1]\Bigg\}\\
    \cup~&\Bigg\{\frac{\partial\psi}{\partial\beta_1^{(u)}}(X;\boj,\bzj)\pi(Y|h_2(X,\ej),\nuj), \ \frac{\partial\psi}{\partial\beta_0}(X;\boj,\bzj)\pi(Y|h_2(X,\ej),\nuj), \\
    &\quad \frac{\partial\psi}{\partial\beta_1^{(u)}}(X;\boj,\bzj)p_{G^*_2}(Y|X), \ \frac{\partial\psi}{\partial\beta_0}(X;\boj,\bzj)p_{G^*_2}(Y|X), \ \psi(X;\boj,\bzj)\frac{\partial\pi}{\partial h_2}(Y|h_2(X,\ej),\nuj), \\
    &\hspace{4cm}\psi(X;\boj,\bzj)\frac{\partial^2\pi}{\partial h_2^2}(Y|h_2(X,\ej),\nuj)
    :u\in[d], \ j\in[k^*_2]:|\mathcal{V}_{2,j}|=1\Bigg\}\\
    \cup~&\Bigg\{\frac{\partial\psi}{\partial\beta_1^{(u)}}(X;0_d,\bar{\beta}_{0i})\pi(Y|h_2(X,\ej),\nuj), \ \frac{\partial^2\psi}{\partial\beta_1^{(u)}\partial\beta_1^{(v)}}(X;0_d,\bar{\beta}_{0i})\pi(Y|h_2(X,\ej),\nuj), \ \pi(Y|h_2(X,\ej),\nuj)\\
    &\quad \frac{\partial\psi}{\partial\beta_1^{(u)}}(X;0_d,\bar{\beta}_{0i})\frac{\partial\pi}{\partial h_2}(Y|h_2(X,\ej),\nuj), \ \psi(X;0_d,\bar{\beta}_{0i})\frac{\partial\pi}{\partial h_2}(Y|h_2(X,\ej),\nuj),\\
    &\quad \frac{\partial\psi}{\partial\beta_1^{(u)}}(X;0_d,\bar{\beta}_{0i})\frac{\partial^2\pi}{\partial h_2^2}(Y|h_2(X,\ej),\nuj), \ \psi(X;0_d,\bar{\beta}_{0i})\frac{\partial^2\pi}{\partial h_2^2}(Y|h_2(X,\ej),\nuj),\\
    &\quad \psi(X;0_d,\bar{\beta}_{0i})\frac{\partial^3\pi}{\partial h_2^3}(Y|h_2(X,\ej),\nuj), \ \psi(X;0_d,\bar{\beta}_{0i})\frac{\partial^4\pi}{\partial h_2^4}(Y|h_2(X,\ej),\nuj),\\
    &\quad \frac{\partial\psi}{\partial\beta_1^{(u)}}(X;0_d,\bar{\beta}_{0i})p_{G^*_2}(Y|X), \ \frac{\partial^2\psi}{\partial\beta_1^{(u)}\partial\beta_1^{(v)}}(X;0_d,\bar{\beta}_{0i})p_{G^*_2}(Y|X):u,v\in[d], j\in[k^*_2]:|\mathcal{V}_{2,j}|>1, i\in\mathcal{V}_{2,j}\Bigg\}
\end{align*}
is linearly independent w.r.t $Y$, implying that the coefficients of those terms in the limit in equation~\eqref{eq:zero_limit_sigmoid} are equal to zero. 

\vspace{0.5 em}
\noindent
For $j\in[k^*_1]$, by looking at the coefficient of the term $\Big[\sum_{j = 1}^{k^*_2} \sigma((\beta_{1j}^{*})^{\top} X + \beta_{0j}^{*})\Big]\pi(Y|h_1(X,\kj),\tj)$, we have $s_{0,j}=0$. 

\vspace{0.5 em}
\noindent
For $j\in[k^*_1]$ such that $|\mathcal{V}_{1,j}|=1$, by considering the coefficients of 
\begin{itemize}
    \item $\Big[\sum_{j = 1}^{k^*_2} \sigma((\beta_{1j}^{*})^{\top} X + \beta_{0j}^{*})\Big]\frac{\partial\pi}{\partial h_1}(Y|h_1(X,\kj),\tj)$, we have $\sum_{u_1=1}^{d_1}s_{1,j}^{(u_1)}\frac{\partial h_1}{\partial\kappa^{(u_1)}}(X,\kj)=0$ for almost every $X$. Since the expert function $h_1$ is strongly identifiable, we get $s_{1,j}^{(u_1)}=0$ for all $u_1\in[d_1]$;
    \item $\Big[\sum_{j = 1}^{k^*_2} \sigma((\beta_{1j}^{*})^{\top} X + \beta_{0j}^{*})\Big]\frac{\partial^2\pi}{\partial h_1^2}(Y|h_1(X,\kj),\tj)$, we have $s_{2,j}=0$.
\end{itemize}
\noindent
For $j\in[k^*_1]$ such that $|\mathcal{V}_{1,j}|>1$, by taking into account the coefficients of 
\begin{itemize}
    \item $\Big[\sum_{j = 1}^{k^*_2} \sigma((\beta_{1j}^{*})^{\top} X + \beta_{0j}^{*})\Big]\frac{\partial\pi}{\partial h_1}(Y|h_1(X,\kj),\tj)$, we have
    \begin{align*}
        \sum_{u_1=1}^{d_1}s_{1,j}\frac{\partial h_1}{\partial\kappa^{(u_1)}}(X,\kj)+\sum_{u_1,v_1=1}^{d_1}\frac{s_{3,j}^{(u_1v_1)}}{1+1_{\{u_1=v_1\}}}\frac{\partial^2h_1}{\partial\kappa^{(u_1)}\partial\kappa^{(v_1)}}(X,\kj)=0,
    \end{align*}
    for almost every $X$. Since the expert function $h_1$ satisfies the strong identifiability condition, we get $s_{1,j}^{(u_1)}=s_{3,j}^{(u_1v_1)}=0$ for all $u_1,v_1\in[d_1]$;
    \item $\Big[\sum_{j = 1}^{k^*_2} \sigma((\beta_{1j}^{*})^{\top} X + \beta_{0j}^{*})\Big]\frac{\partial^2\pi}{\partial h_1^2}(Y|h_1(X,\kj),\tj)$, we have
    \begin{align*}
        \frac{1}{2}s_{2,j}+\sum_{u_1,v_1=1}^{d_1}\frac{s_{3,j}^{(u_1v_1)}}{1+1_{\{u_1=v_1\}}}\frac{\partial h_1}{\partial\kappa^{(u_1)}}(X,\kj)\frac{\partial h_1}{\partial\kappa^{(v_1)}}(X,\kj)=0,
    \end{align*}
    for almost every $X$. Since $s_{3,j}^{(u_1v_1)}=0$ for all $u_1,v_1\in[d_1]$, we deduce $s_{2,j}=0$;
    \item $\Big[\sum_{j = 1}^{k^*_2} \sigma((\beta_{1j}^{*})^{\top} X + \beta_{0j}^{*})\Big]\frac{\partial^3\pi}{\partial h_1^3}(Y|h_1(X,\kj),\tj)$, we have $\frac{1}{2}\sum_{u_1=1}^{d_1}s_{5,j}^{(u_1)}\frac{\partial h_1}{\partial\kappa^{(u_1)}}(X,\kj)=0$, for almost every $X$. As the expert function $h_1$ meets the strong identifiability condition, we get $s_{5,j}^{(u_1)}=0$ for all $u_1\in[d_1]$;
    \item $\Big[\sum_{j = 1}^{k^*_2} \sigma((\beta_{1j}^{*})^{\top} X + \beta_{0j}^{*})\Big]\frac{\partial^4\pi}{\partial h_1^4}(Y|h_1(X,\kj),\tj)$, we have $s_{4,j}=0$.
\end{itemize}
For $j\in[k^*_2]$ such that $|\mathcal{V}_{2,j}|=1$, by considering the coefficients of 
\begin{itemize}
    \item $\frac{\partial\psi}{\partial\beta_1^{(u)}}(X;\boj,\bzj)\pi(Y|h_2(X,\ej),\nuj)$, we have $t_{1,j}^{(u)}=0$ for all $u\in[d]$;
    \item $\frac{\partial\psi}{\partial\beta_0}(X;\boj,\bzj)\pi(Y|h_2(X,\ej),\nuj)$, we have $t_{0,j}=0$;
    \item $\psi(X;\boj,\bzj)\frac{\partial\pi}{\partial h_2}(Y|h_2(X,\ej),\nuj)$, we have $\sum_{u_2=1}^{d_2}t_{2,j}^{(u_2)}\frac{\partial h_2}{\partial\eta^{(u_2)}}(X,\ej)=0$. Since the expert function $h_2$ is strongly identifiable, we deduce $t_{2,j}^{(u_2)}=0$ for all $u_2\in[d_2]$;
    \item $\psi(X;\boj,\bzj)\frac{\partial^2\pi}{\partial h_2^2}(Y|h_2(X,\ej),\nuj)$, we have $t_{3,j}=0$.
\end{itemize}
For $j\in[k^*_2]$ such that $|\mathcal{V}_{2,j}|>1$, by considering the coefficients of 
\begin{itemize}
    \item $\pi(Y|h_2(X,\ej),\nuj)$, we have $t_{0,j}=0$;
    \item $\frac{\partial\psi}{\partial\beta_1^{(u)}}(X;0_d,\bar{\beta}_{0i})\pi(Y|h_2(X,\ej),\nuj)$, we have $t_{1,j,i}^{(u)}=0$ for all $u\in[d]$ and $i\in\mathcal{V}_{2,j}$;
    \item $\frac{\partial^2\psi}{\partial\beta_1^{(u)}\partial\beta_1^{(v)}}(X;0_d,\bar{\beta}_{0i})\pi(Y|h_2(X,\ej),\nuj)$, we have $t_{4,j,i}^{(uv)}=0$ for all $u,v\in[d]$ and $i\in\mathcal{V}_{2,j}$;
    \item $\psi(X;0_d,\bar{\beta}_{0i})\frac{\partial\pi}{\partial h_2}(Y|h_2(X,\ej),\nuj)$, we have
    \begin{align*}
        \sum_{u_2=1}^{d_2}t_{2,j,i}^{(u_2)}\frac{\partial h_2}{\partial\eta^{(u_2)}}(X,\ej)+\sum_{u_2,v_2=1}^{d_2}t_{5,j,i}^{(u_2v_2)}\frac{\partial^2 h_2}{\partial\eta^{(u_2)}\partial\eta^{(v_2)}}(X,\ej)=0.
    \end{align*}
    As the expert function $h_2$ satisfies the strong identifiability condition, we deduce $t_{2,j,i}^{(u_2)}=t_{5,j,i}^{(u_2v_2)}=0$ for all $u_2,v_2\in[d_2]$ and $i\in\mathcal{V}_{2,j}$;
    \item $\frac{\partial\psi}{\partial\beta_1^{(u)}}(X;0_d,\bar{\beta}_{0i})\frac{\partial\pi}{\partial h_2}(Y|h_2(X,\ej),\nuj)$, we have $\sum_{u_2=1}^{d_2}t_{7,j,i}^{(uu_2)}\frac{\partial h_2}{\partial\eta^{(u_2)}}(X,\ej)=0$. Since the expert function $h_2$ is strongly identifiable, we deduce $t_{7,j,i}^{(uu_2)}=0$ for all $u\in[d]$, $u_2\in[d_2]$ and $i\in\mathcal{V}_{2,j}$;
    \item $\psi(X;0_d,\bar{\beta}_{0i})\frac{\partial^2\pi}{\partial h_2^2}(Y|h_2(X,\ej),\nuj)$, we have
    \begin{align*}
        \frac{1}{2}t_{3,j,i}+\sum_{u_2,v_2=1}^{d_2}\frac{t_{5,j,i}^{(u_2v_2)}}{1+1_{\{u_2=v_2\}}}\frac{\partial h_2}{\partial\eta^{(u_2)}}(X,\ej)\frac{\partial h_2}{\partial\eta^{(v_2)}}(X,\ej)=0.
    \end{align*}
    Note that $t_{5,j,i}^{(u_2v_2)}=0$ for all $u_2,v_2\in[d_2]$ and $i\in\mathcal{V}_{2,j}$, we deduce $t_{3,j,i}=0$ for all $i\in\mathcal{V}_{2,j}$;
    \item $\frac{\partial\psi}{\partial\beta_1^{(u)}}(X;0_d,\bar{\beta}_{0i})\frac{\partial^2\pi}{\partial h_2^2}(Y|h_2(X,\ej),\nuj)$, we have $t_{8,j,i}^{(u)}=0$ for all $u\in[d]$ and $i\in\mathcal{V}_{2,j}$;
    \item $\psi(X;0_d,\bar{\beta}_{0i})\frac{\partial^3\pi}{\partial h_2^3}(Y|h_2(X,\ej),\nuj)$, we have $\sum_{u_2=1}^{d_2}\frac{1}{2}t_{9,j,i}^{(u_2)}\frac{\partial h_2}{\partial\eta^{(u_2)}}(X,\ej)=0$. Since the expert function $h_2$ meets the strong identifiability, we deduce $t_{9,j,i}^{(u_2)}$ for all $u_2\in[d_2]$ and $i\in\mathcal{V}_{2,j}$;
    \item $\psi(X;0_d,\bar{\beta}_{0i})\frac{\partial^4\pi}{\partial h_2^4}(Y|h_2(X,\ej),\nuj)$, we have $t_{6,j,i}=0$ for all $i\in\mathcal{V}_{2,j}$.
\end{itemize}
Putting the above results together, we have (i) $s_{0,j}=s_{1,j}^{(u_1)}=s_{2,j}=s_{3,j}^{(u_1v_1)}=s_{4,j}=s_{5,j}^{(u_1)}=0$ for all $j\in[k^*_1]$ and $u_1,v_1\in[d_1]$; (ii) $t_{0,j}=t_{1,j}^{(u)}=t_{2,j}^{(u_2)}=t_{3,j}=0$ for all $j\in[k^*_2]:|\mathcal{V}_{2,j}|=1$, $u\in[d]$ and $u_2\in[d_2]$; (iii) $t_{0,j}=t_{1,j,i}^{(u)}=t_{2,j,i}^{(u_2)}=t_{3,j,i}=t_{4,j,i}^{(uv)}=t_{5,j,i}^{(u_2v_2)}=t_{6,j,i}=t_{7,j,i}^{uv_2}=t_{8,j,i}^{(u)}=t_{9,j,i}^{(u_2)}$ for all $j\in[k^*_2]:|\mathcal{V}_{2,j}|>1$, $u,v\in[d]$ and $u_2,v_2\in[d_2]$. This contradicts to the fact that at least one among them is non-zero. Consequently, we achieve the local part in equation~\eqref{eq:local_part_sigmoid} and complete the proof.

\subsection{Proof of Theorem~\ref{theorem:weakly_identifiable_experts_sigmoid}}
\label{appendix:weakly_identifiable_experts_sigmoid}
Note that it is sufficient to demonstrate that
\begin{align*}
    \inf_{(G_1,G_2)\in\mathcal{G}_{k_1,k_2}(\Theta)}\dfrac{\bbE_X[V(g_{G_1,G_2}(\cdot|X),g_{G^*_1,\check{G}_2}(\cdot|X))]}{\mathcal{D}_4((G_1,G_2),(G^*_1,\check{G}_2))}>0,
\end{align*}
for any pair of mixing measures $(G^*_1,\check{G}_2)\in\check{\mathcal{G}}_{k^*_1,k_2}(\Theta)$. For that purpose, given an arbitrary mixing measure $\check{G}_2:=\sum_{i=1}^{k_2}\sigma(\check{\beta}_{0i})\delta_{(\check{\beta}_{1i},\check{\eta}_{i},\check{\nu}_{i})}$, we need to establish its local part 
\begin{align}
    \label{eq:local_part_sigmoid_dense}
    \lim_{\varepsilon\to0}\inf_{(G_1,G_2)\in\mathcal{G}_{k_1,k_2}(\Theta):\mathcal{D}_4((G_1,G_2),(G^*_1,\check{G}_2))\leq\varepsilon}\dfrac{\bbE_X[V(g_{G_1,G_2}(\cdot|X),g_{G^*_1,\check{G}_2}(\cdot|X))]}{\mathcal{D}_4((G_1,G_2),(G^*_1,\check{G}_2))}>0,
\end{align}
and its global part
\begin{align}
    \label{eq:global_part_sigmoid_dense}
    \inf_{(G_1,G_2)\in\mathcal{G}_{k_1,k_2}(\Theta):\mathcal{D}_4((G_1,G_2),(G^*_1,\check{G}_2))>\varepsilon'}\dfrac{\bbE_X[V(g_{G_1,G_2}(\cdot|X),g_{G^*_1,\check{G}_2}(\cdot|X))]}{\mathcal{D}_4((G_1,G_2),(G^*_1,\check{G}_2))}>0.
\end{align}
Since the global part~\eqref{eq:global_part_sigmoid_dense} can be demonstrated analogously to that in Appendix~\ref{appendix:strongly_identifiable_experts}, we will focus only on proving the local part~\eqref{eq:local_part_sigmoid_dense} in this appendix. Assume by contrary that the above local part is not true. Then, we can find a sequence $(G^n_1,G^n_2)$ of the form $G^n_1:=\sum_{i=1}^{k^n_1}\oin\delta_{(\kin,\tin)}$, $G^n_2:=\sum_{i=1}^{k^n_2}\sigma(\bzin)\delta_{(\boin,\ein,\nuin)}$ for $n\in\mathbb{N}$ satisfying $\mathcal{D}_{4n}:=\mathcal{D}_{4}((G^n_1,G^n_2),(G^*_1,\check{G}_2))\to0$ and
\begin{align}
    \label{eq:expectation_zero_sigmoid_dense}
    \bbE_X[V(g_{G^n_1,G^n_2}(\cdot|X),g_{G^*_1,\check{G}_2}(\cdot|X))]/\mathcal{D}_{4n}\to0,
\end{align}
as $n\to\infty$. Moreover, we may assume WLOG that the number of shared experts $k^n_1$, the number of routed experts $k^n_2$, and Voronoi cells $\mathcal{V}_{1,j}=\mathcal{V}_{1,j}(G^n_1)$, $\mathcal{V}_{2,j}=\mathcal{V}_{2,j}(G^n_2)$ are independent of the sample size $n$. In addition, since $G^n_2$ and $\check{G}_2$ have the same number of atoms $k_2$, we may assume WLOG that the Voronoi cell $\mathcal{V}_{2,j}$ admits only one element, that is, $\mathcal{V}_{2,j}=\{j\}$ for all $j\in[k_2]$. Thus, we can represent the Voronoi loss $\mathcal{D}_{4n}$ as
\begin{align}
    \label{eq:loss_4n}
    &\mathcal{D}_{4n}=\sum_{j=1}^{k^*_1}\Big|\sum_{i\in\mathcal{V}_{1,j}}\oin-\oj\Big|+\sum_{i=1}^{k^*_2}(\|\cdboin\|+|\cdbzin|+\|\cdein\|+|\cdnuin|)\nonumber\\
    &\hspace{1cm}+\sum_{j\in[k^*_1]:|\mathcal{V}_{1,j}|=1}\sum_{i\in\mathcal{V}_{1,j}}\oin(\|\dkijn\|+|\dtijn|)+\sum_{j\in[k^*_1]:|\mathcal{V}_{1,j}|>1}\sum_{i\in\mathcal{V}_{1,j}}\oin(\|\dkijn\|^2+|\dtijn|^2),
\end{align}
where we denote $\cdboin:=\boin-\check{\beta}_{1i}$, $\cdbzin:=\bzin-\check{\beta}_{0i}$, $\cdein:=\ein-\check{\eta}_{i}$, and $\cdnuin:=\nuin-\check{\nu}_{i}$ for all $i\in[k_2]$.
Recall that $\mathcal{D}_{4n}\to0$ as $n\to\infty$, then equation~\eqref{eq:loss_4n} implies that as $n\to\infty$, we have
\begin{itemize}
    \item For $j\in[k^*_1]$ and $i\in\mathcal{V}_{1,j}$: $\sum_{i\in\mathcal{V}_{1,j}}\oin\to\oj$, $(\kin,\tin)\to(\kj,\tj)$;
    \item For $i\in[k^*_2]$: $(\boin,\bzin,\ein,\nuin)\to(\cboi,\cbzi,\cei,\cnui)$.
\end{itemize}
Now, we divide the proof into three main stages:

\vspace{0.5 em}
\noindent
\textbf{Stage 1 - Density Decomposition:} In this step, we reuse the following decomposition of the density discrepancy $g_{G^n_1,G^n_2}(Y|X)-g_{G^*_1,G^*_2}(Y|X)$ in Appendix~\ref{appendix:strongly_identifiable_experts_sigmoid}
\begin{align*}
    g_{G^n_1,G^n_2}(Y|X)-g_{G^*_1,G^*_2}(Y|X)=\frac{1}{2}\left[(q_{G^n_1}(Y|X)-q_{G^*_1}(Y|X))+(p_{G^n_2}(Y|X)-p_{G^*_2}(Y|X))\right],
\end{align*}
where we denote
\begin{align*}
    q_{G^n_1}(Y|X)&:=\sum_{i=1}^{k^n_1}\omega^n_{i}\pi(Y|h_1(X,\kappa^n_{i}),\tau^n_{i}),\\
    q_{G^*_1}(Y|X)&:=\sum_{i=1}^{k^*_1}\omega^*_{i}\pi(Y|h_1(X,\kappa^*_{i}),\tau^*_{i}),\\
    p_{G^n_2}(Y|X)&:=\sum_{i = 1}^{k^n_2} \frac{\sigma((\beta_{1i}^{n})^{\top} X + \beta_{0i}^{n})}{\sum_{j = 1}^{k^n_2} \sigma((\beta_{1j}^{n})^{\top} X + \beta_{0j}^{n})}\cdot \pi(Y|h_2(X,\eta^n_{i}), \nu_{i}^{n}),\\
    p_{\check{G}_2}(Y|X)&:=\sum_{i = 1}^{k_2} \frac{\sigma((\cboi)^{\top} X + \cbzi)}{\sum_{j = 1}^{k_2} \sigma((\check{\beta}_{1j})^{\top} X + \check{\beta}_{0j})}\cdot \pi(Y|h_2(X,\cei), \cnui).
\end{align*}
\textbf{Stage 1.1:} We also utilize the decomposition of the term $q_{G^n_1}(Y|X)-q_{G^*_1}(Y|X)$ in Appendix~\ref{appendix:strongly_identifiable_experts_sigmoid} as follows:
\begin{align*}
    q_{G^n_1}(Y|X)-q_{G^*_1}(Y|X)&=\sum_{j\in[k^*_1]:|\mathcal{V}_{1,j}|=1}\sum_{i\in\mathcal{V}_{1,j}}\oin[\pi(Y|h_1(X,\kin),\tin)-\pi(Y|h_1(X,\kj),\tj)]\\
    &+\sum_{j\in[k^*_1]:|\mathcal{V}_{1,j}|>1}\sum_{i\in\mathcal{V}_{1,j}}\oin[\pi(Y|h_1(X,\kin),\tin)-\pi(Y|h_1(X,\kj),\tj)]\\
    &+\sum_{j=1}^{k^*_1}\Big(\sum_{i\in\mathcal{V}_{1,j}}\oin-\oj\Big)\pi(Y|h_1(X,\kappa^*_j),\tau^*_j)\\
    &:=A_{n,1}(Y|X)+A_{n,2}(Y|X)+A_{n,0}(Y|X).
\end{align*}
Above, the quantity $A_{n,1}(Y|X)$ is expanded as
\begin{align*}
    A_{n,1}(Y|X)&=\sum_{j\in[k^*_1]:|\mathcal{V}_{1,j}|=1}\sum_{\rho=1}^{2}A^{(j)}_{n,1,\rho}(X)\frac{\partial^{\rho}\pi}{\partial h_1^{\rho}}(Y|h_1(X,\kj),\tj)+R_{n,1}(Y|X),
\end{align*}
where $R_{n,1}(Y|X)$ is a Taylor remainder such that $R_{n,1}(Y|X)/\mathcal{D}_{4n}\to$ as $n\to\infty$, and
\begin{align*}
    A^{(j)}_{n,1,1}(X)&:=\sum_{i\in\mathcal{V}_{1,j}}\oin\sum_{u_1=1}^{d_1}(\dkijn)^{(u_1)}\frac{\partial h_1}{\partial\kappa^{(u_1)}}(X,\kj),\\
    A^{(j)}_{n,1,2}(X)&:=\sum_{i\in\mathcal{V}_{1,j}}\oin\frac{1}{2}(\dtijn),
\end{align*}
for all $j\in[k^*_1]$ such that $|\mathcal{V}_{1,j}|=1$. In addition, we can rewrite $A_{n,2}(Y|X)$ as
\begin{align*}
    A_{n,2}(Y|X)&=\sum_{j\in[k^*_1]:|\mathcal{V}_{1,j}|>1}\sum_{\rho=1}^{4}A^{(j)}_{n,1,\rho}(X)\frac{\partial^{\rho}\pi}{\partial h_1^{\rho}}(Y|h_1(X,\kj),\tj)+R_{n,2}(Y|X),
\end{align*}
where $R_{n,2}(Y|X)$ is a Taylor remainder such that $R_{n,2}(Y|X)/\mathcal{D}_{4n}\to$ as $n\to\infty$, and
\begin{align*}
    A^{(j)}_{n,2,1}(X)&:=\sum_{i\in\mathcal{V}_{1,j}}\oin\Big(\sum_{u_1=1}^{d_1}(\dkijn)^{(u_1)}\frac{\partial h_1}{\partial\kappa^{(u_1)}}(X,\kj)+\sum_{u_1,v_1=1}^{d_1}\frac{(\dkijn)^{(u_1)}(\dkijn)^{(v_1)}}{1+1_{\{u_1=v_1\}}}\frac{\partial^2h_1}{\partial\kappa^{(u_1)}\partial\kappa^{(v_1)}}(X,\kj)\Big),\\
    A^{(j)}_{n,2,2}(X)&:=\sum_{i\in\mathcal{V}_{1,j}}\oin\Big(\frac{1}{2}(\dtijn)+\sum_{u_1,v_1=1}^{d_1}\frac{(\dkijn)^{(u_1)}(\dkijn)^{(v_1)}}{1+1_{\{u_1=v_1\}}}\frac{\partial h_1}{\partial\kappa^{(u_1)}}(X,\kj)\frac{\partial h_1}{\partial\kappa^{(v_1)}}(X,\kj)\Big),\\
    A^{(j)}_{n,2,3}(X)&:=\sum_{i\in\mathcal{V}_{1,j}}\oin\sum_{u_1=1}^{d_1}\frac{1}{2}(\dkijn)^{(u_1)}(\dtijn)\frac{\partial h_1}{\partial\kappa^{(u_1)}}(X,\kj),\\
    A^{(j)}_{n,2,4}(X)&:=\sum_{i\in\mathcal{V}_{1,j}}\oin\frac{1}{8}(\dtijn)^2,
\end{align*}
for all $j\in[k^*_1]$ such that $|\mathcal{V}_{1,j}|>1$. 

\vspace{0.5 em}
\noindent
\textbf{Stage 1.2:} Next, we decompose the term $Q_n(Y|X):=\Big[\sum_{j = 1}^{k_2} \sigma((\check{\beta}_{1j})^{\top} X + \check{\beta}_{0j})\Big]\cdot[p_{G^n_2}(Y|X)-p_{\check{G}_2}(Y|X)]$ as
\begin{align*}
    Q_n(Y|X)&=\sum_{i=1}^{k_2}\Big[\sigma((\beta_{1i}^{n})^{\top} X + \beta_{0i}^{n})\pi(Y|h_2(X,\ein),\nuin)-\sigma((\cboi)^{\top} X + \cbzi)\pi(Y|h_2(X,\cei),\cnui)\Big]\\
    &-\sum_{i=1}^{k_2}\Big[\sigma((\beta_{1i}^{n})^{\top} X + \beta_{0i}^{n})-\sigma((\cboi)^{\top} X + \cbzi)\Big]p_{G^n_2}(Y|X)\\
    &=\sum_{i=1}^{k_2}\Big[\psi(X;\beta_{1i}^{n},\beta_{0i}^{n})\pi(Y|h_2(X,\ein),\nuin)-\psi(X;\cboi,\cbzi)\pi(Y|h_2(X,\cei),\cnui)\Big]\\
    &-\sum_{i=1}^{k_2}\Big[\psi(X;\beta_{1i}^{n},\beta_{0i}^{n})-\psi(X;\cboi,\cbzi)\Big]p_{G^n_2}(Y|X)\\
    &:=B_{n}(Y|X)-C_{n}(Y|X),
\end{align*}
where we denote $\psi(X;\beta_1,\beta_0):=\sigma(\beta_1^{\top}X+\beta_0)$.

\vspace{0.5 em}
\noindent
\textbf{Stage 1.2.1:} In this step, we decompose $B_n(Y|X)$ by applying the first-order Taylor expansion to the function $\psi(X;\beta_{1i}^{n},\beta_{0i}^{n})\pi(Y|h_2(X,\ein),\nuin)$ around the point $(\cboi,\cbzi,\cei,\cnui)$ as follows:
\begin{align*}
   B_{n}(Y|X)&=\sum_{i=1}^{k_2}\sum_{|\alpha|=1}(\cdboin)^{\alpha_1}(\cdbzin)^{\alpha_2}(\cdein)^{\alpha_3}(\cdnuin)^{\alpha_4}\\
   &\hspace{3cm}\times\frac{\partial^{|\alpha_1|+\alpha_2}\psi}{\partial\beta_1^{\alpha_1}\partial\beta_0^{\alpha_2}}(X;\cboi,\cbzi)\frac{\partial^{|\alpha_3|+\alpha_4}\pi}{\partial\eta^{\alpha_3}\partial\nu^{\alpha_4}}(Y|h_2(X,\cei),\cnui)+R_{n,3}(Y|X)\\
   &=\sum_{i=1}^{k_2}\sum_{\rho=0}^{2}B^{(i)}_{n,\rho}(X)\frac{\partial^{\rho}\pi}{\partial h_2^{\rho}}(Y|h_2(X,\cei),\cnui)+R_{n,3}(Y|X),
\end{align*}
where $R_{n,3}(Y|X)$ is a Taylor remainder such that $R_{n,3}(Y|X)/\mathcal{D}_{4n}\to$ as $n\to\infty$, and
\begin{align*}
    B_{n,0}^{(i)}&:=\sum_{u=1}^{d}(\cdboin)^{(u)}\frac{\partial\psi}{\partial\beta_1^{(u)}}(X;\cboi,\cbzi)+(\cdbzin)\frac{\partial\psi}{\partial\beta_0}(X;\cboi,\cbzi),\\
    B_{n,1}^{(i)}&:=\sum_{u_2=1}^{d_2}(\cdein)^{(u_2)}\frac{\partial h_2}{\partial\eta^{(u_2)}}(X,\cei)\psi(X;\cboi,\cbzi),\\
    B_{n,2}^{(i)}&:=\frac{1}{2}(\cdnuin)\psi(X;\cboi,\cbzi),
\end{align*}
for all $i\in[k_2]$.

\vspace{0.5 em}
\noindent
\textbf{Stage 1.2.2:} Next, we proceed to decompose $C_{n}(Y|X)$ by applying the first-order Taylor expansion to the function $\psi(X;\boin,\bzin)$ around the point $(\cboi,\cbzi)$ as
\begin{align*}
    C_{n}(Y|X)&=\sum_{i=1}^{k_2}\sum_{|\alpha|=1}(\cdboin)^{\alpha_1}(\cdbzin)^{\alpha_2}\frac{\partial^{|\alpha_1|+\alpha_2}\psi}{\partial\beta_1^{\alpha_1}\partial\beta_0^{\alpha_2}}(X;\cboi,\cbzi)p_{G^n_2}(Y|X)+R_{n,4}(Y|X)\\
   &=\sum_{i=1}^{k_2}\Big[\sum_{u=1}^{d}(\cdboin)^{(u)}\frac{\partial\psi}{\partial\beta_1^{(u)}}(X;\cboi,\cbzi)+(\cdbzin)\frac{\partial\psi}{\partial\beta_0}(X;\cboi,\cbzi)\Big]p_{G^n_2}(Y|X)+R_{n,4}(Y|X),
\end{align*}
where $R_{n,4}(Y|X)$ is a Taylor remainder such that $R_{n,4}(Y|X)/\mathcal{D}_{4n}\to$ as $n\to\infty$.

\vspace{0.5 em}
\noindent
Combining the above decompositions, we can view $A_{n,0}(Y|X)/\mathcal{D}_{4n}$, $[A_{n,1}(Y|X)-R_{n,1}(Y|X)]/\mathcal{D}_{4n}$, $[A_{n,2}(Y|X)-R_{n,2}(Y|X)]/\mathcal{D}_{4n}$, $[B_{n}(Y|X)-R_{n,3}(Y|X)]/\mathcal{D}_{4n}$, $[C_{n}(Y|X)-R_{n,4}(Y|X)]/\mathcal{D}_{4n}$ as a combination of elements from the following sets
\begin{align*}
    \mathcal{S}_{0,j}&:=\{\pi(Y|h_1(X,\kj),\tj)\},\\
    \mathcal{S}_{1,j}&:=\Bigg\{\frac{\partial h_1}{\partial\kappa^{(u_1)}}(X,\kj)\frac{\partial\pi}{\partial h_1}(Y|h_1(X,\kj),\tj), \ \frac{\partial^2 h_1}{\partial\kappa^{(u_1)}\partial\kappa^{(v_1)}}(X,\kj)\frac{\partial\pi}{\partial h_1}(Y|h_1(X,\kj),\tj):u_1,v_1\in[d_1]\Bigg\},\\
    \mathcal{S}_{2,j}&:=\Bigg\{\frac{\partial^2\pi}{\partial h_1^2}(Y|h_1(X,\kj),\tj), \ \frac{\partial h_1}{\partial\kappa^{(u_1)}}(X,\kj)\frac{\partial h_1}{\partial\kappa^{(v_1)}}(X,\kj)\frac{\partial^2\pi}{\partial h_1^2}(Y|h_1(X,\kj),\tj):u_1,v_1\in[d_1]\Bigg\},\\
    \mathcal{S}_{3,j}&:=\Bigg\{ \frac{\partial h_1}{\partial\kappa^{(u_1)}}(X,\kj)\frac{\partial^3\pi}{\partial h_1^3}(Y|h_1(X,\kj),\tj):u_1,v_1\in[d_1]\Bigg\},\\
    \mathcal{S}_{4,j}&:=\Bigg\{ \frac{\partial^4\pi}{\partial h_1^4}(Y|h_1(X,\kj),\tj):u_1,v_1\in[d_1]\Bigg\},
\end{align*}
for all $j\in[k^*_1]$, and
\begin{align*}
    \mathcal{T}_{0,j}&:=\Bigg\{\frac{\partial\psi}{\partial\beta_1^{(u)}}(X;\cboi,\cbzi)\pi(Y|h_2(X,\cei),\cnui), \ \frac{\partial\psi}{\partial\beta_0}(X;\cboi,\cbzi)\pi(Y|h_2(X,\cei),\cnui),\\
    &\hspace{2cm}\frac{\partial\psi}{\partial\beta_1^{(u)}}(X;\cboi,\cbzi)p_{G^n_2}(Y|X), \ \frac{\partial\psi}{\partial\beta_0}(X;\cboi,\cbzi)p_{G^n_2}(Y|X):u\in[d]\Bigg\},\\
    \mathcal{T}_{1,j}&:=\Bigg\{\frac{\partial h_2}{\partial\eta^{(u_2)}}(X,\ej)\psi(X;\cboi,\cbzi)\frac{\partial\pi}{\partial h_2}(Y|h_2(X,\cei),\cnui):u\in[d], \ u_2\in[d_2]\Bigg\},\\
    \mathcal{T}_{2,j}&:=\Bigg\{\psi(X;\cboi,\cbzi)\frac{\partial^2\pi}{\partial h_2^2}(Y|h_2(X,\cei),\cnui)\Bigg\},
\end{align*}
for all $j\in[k^*_2]$.

\vspace{0.5 em}
\noindent
\textbf{Stage 2 - Non-vanishing coefficients:} In this stage, we show that at least one among the coefficients in the representations of $A_{n,0}(Y|X)/\mathcal{D}_{4n}$, $[A_{n,1}(Y|X)-R_{n,1}(Y|X)]/\mathcal{D}_{4n}$, $[A_{n,2}(Y|X)-R_{n,2}(Y|X)]/\mathcal{D}_{4n}$, $[B_{n}(Y|X)-R_{n,3}(Y|X)]/\mathcal{D}_{4n}$, $[C_{n}(Y|X)-R_{n,4}(Y|X)]/\mathcal{D}_{4n}$ does not converge to zero when $n\to\infty$. Suppose that all these coefficients go to zero. By using the same arguments as in Stage 2 in Appendix~\ref{appendix:strongly_identifiable_experts}, we have
\begin{align*}
    \frac{1}{\mathcal{D}_{4n}}\Big[\sum_{j=1}^{k^*_1}\Big|\sum_{i\in\mathcal{V}_{1,j}}\oin-\oj\Big|&+\sum_{j\in[k^*_1]:|\mathcal{V}_{1,j}|=1}\sum_{i\in\mathcal{V}_{1,j}}\oin(\|\dkijn\|+|\dtijn|)\\
    &+\sum_{j\in[k^*_1]:|\mathcal{V}_{1,j}|>1}\sum_{i\in\mathcal{V}_{1,j}}\oin(\|\dkijn\|^2+|\dtijn|^2)\Big]\to0,
\end{align*}
as $n\to\infty$. Additionally, by considering the coefficients of the terms:
\begin{itemize}
    \item $\frac{\partial\psi}{\partial\beta_1^{(u)}}(X;\cboi,\cbzi)\pi(Y|h_2(X,\cei),\cnui)$ for $i\in[k_2]$, we get $\frac{1}{\mathcal{D}_{4n}}\sum_{i=1}^{k_2}\|\dboijn\|\to0$;
    \item $\frac{\partial\psi}{\partial\beta_0}(X;\cboi,\cbzi)\pi(Y|h_2(X,\cei),\cnui)$ for $i\in[k_2]$, we get $\frac{1}{\mathcal{D}_{4n}}\sum_{i=1}^{k_2}|\dbzijn|\to0$;
    \item $\frac{\partial h_2}{\partial\eta^{(u_2)}}(X,\cei)\psi(X;\cboi,\cbzi)\frac{\partial\pi}{\partial h_2}(Y|h_2(X,\cei),\cnui)$ for $i\in[k_2]$, we get $\frac{1}{\mathcal{D}_{4n}}\sum_{i=1}^{k_2}\|\deijn\|\to0$;
    \item $\psi(X;\cboi,\cbzi)\frac{\partial\pi}{\partial h_2}(Y|h_2(X,\cei),\cnui)$ for $i\in[k_2]$, we get $ \frac{1}{\mathcal{D}_{4n}}\sum_{i=1}^{k_2}|\dnuijn|\to0$.
\end{itemize}
Taking the summation of the above limits, we deduce $1=\frac{\mathcal{D}_{4n}}{\mathcal{D}_{4n}}\to0$ as $n\to\infty$, which is a contradiction. Thus, not all the coefficients in the representations of $A_{n,0}(Y|X)/\mathcal{D}_{4n}$, $[A_{n,1}(Y|X)-R_{n,1}(Y|X)]/\mathcal{D}_{4n}$, $[A_{n,2}(Y|X)-R_{n,2}(Y|X)]/\mathcal{D}_{4n}$, $[B_{n}(Y|X)-R_{n,3}(Y|X)]/\mathcal{D}_{4n}$, $[C_{n}(Y|X)-R_{n,4}(Y|X)]/\mathcal{D}_{4n}$ converge to zero as $n\to\infty$.

\vspace{0.5 em}
\noindent
\textbf{Stage 3 - Fatou's lemma contradiction:} In this stage, we attempto to show a contradiction to the result of Stage 2 using the Fatou's lemma. Firstly, we denote $m_n$ as the maximum of the absolute values of the coefficients in the representations of $A_{n,0}(Y|X)/\mathcal{D}_{4n}$, $[A_{n,1}(Y|X)-R_{n,1}(Y|X)]/\mathcal{D}_{4n}$, $[A_{n,2}(Y|X)-R_{n,2}(Y|X)]/\mathcal{D}_{4n}$, $[B_{n}(Y|X)-R_{n,3}(Y|X)]/\mathcal{D}_{4n}$, $[C_{n}(Y|X)-R_{n,4}(Y|X)]/\mathcal{D}_{4n}$. The result of Stage 2 implies that $1/m_n\not\to\infty$ as $n\to\infty$. In addition, we also denote
\begin{align*}
    \frac{1}{m_n\mathcal{D}_{4n}}\cdot\sum_{i\in\mathcal{V}_{1,j}}\oin(\dkijn)^{(u_1)}\to s^{(u_1)}_{1,j},& \quad \frac{1}{m_n\mathcal{D}_{4n}}\cdot\sum_{i\in\mathcal{V}_{1,j}}\oin(\dtijn)\to s_{2,j},\\
    \frac{1}{m_n\mathcal{D}_{4n}}\cdot\sum_{i\in\mathcal{V}_{1,j}}\oin(\dkijn)^{(u_1)}(\dkijn)^{(v_1)}\to s^{(u_1v_1)}_{3,j},& \quad \frac{1}{m_n\mathcal{D}_{4n}}\cdot\sum_{i\in\mathcal{V}_{1,j}}\oin(\dtijn)^2\to s_{4,j},\\
    \frac{1}{m_n\mathcal{D}_{4n}}\cdot\sum_{i\in\mathcal{V}_{1,j}}\oin(\dkijn)^{(u_1)}(\dtijn)\to s^{(u_1)}_{5,j},& \quad \frac{1}{m_n\mathcal{D}_{4n}}\cdot\Big(\sum_{i\in\mathcal{V}_{1,j}}\oin-\oj\Big)\to s_{0,j},\\
\end{align*}
for all $j\in[k^*_1]$ and
\begin{align*}
   \frac{1}{m_n\mathcal{D}_{4n}}\cdot(\cdbzin)\to t_{0,i},& \quad \frac{1}{m_n\mathcal{D}_{4n}}\cdot(\cdboin)^{(u)}\to t^{(u)}_{1,i},\\
    \frac{1}{m_n\mathcal{D}_{4n}}\cdot(\cdein)^{(u_2)}\to t^{(u_2)}_{2,i},& \quad \frac{1}{m_n\mathcal{D}_{4n}}\cdot(\cdnuin)\to t_{3,i},\\
\end{align*}
for all $i\in[k_2]$. Due to the result of Stage 2, at least one among the above limits is different from zero. Recall from equation~\eqref{eq:expectation_zero_sigmoid_dense} that we have
\begin{align*}
    \bbE_X[V(g_{G^n_1,G^n_2}(\cdot|X),g_{G^*_1,\check{G}_2}(\cdot|X))]/\mathcal{D}_{4n}\to0,
\end{align*}
Furthermore, according to the Fatou's lemma, we get
\begin{align*}
    \lim_{n\to\infty}\dfrac{\bbE_X[V(g_{G^n_1,G^n_2}(\cdot|X),g_{G^*_1,\check{G}_2}(\cdot|X))]}{m_n\mathcal{D}_{4n}}\geq\int\liminf_{n\to\infty}\dfrac{|g_{G^n_1,G^n_2}(Y|X)-g_{G^*_1,\check{G}_2}(Y|X)|}{2m_n\mathcal{D}_{4n}}\dint (X,Y).
\end{align*}
Then, it follows that $[g_{G^n_1,G^n_2}(Y|X)-g_{G^*_1,\check{G}_2}(Y|X)]/[m_n\mathcal{D}_{4n}]\to0$ as $n\to\infty$ for almost surely $(X,Y)$. As the input space is bounded and the parameter space is compact, the quantity $\sum_{j = 1}^{k_2} \sigma((\check{\beta}_{1j})^{\top} X + \check{\beta}_{0j})$ is bounded. Therefore, we deduce
\begin{align*}
    \Big[\sum_{j = 1}^{k_2} \sigma((\check{\beta}_{1j})^{\top} X + \check{\beta}_{0j})\Big][g_{G^n_1,G^n_2}(Y|X)-g_{G^*_1,\check{G}_2}(Y|X)]/[m_n\mathcal{D}_{4n}]\to0,
\end{align*}
as $n\to\infty$. This result indicates
\begin{align*}
    \frac{1}{2}\Big[\sum_{j = 1}^{k_2} \sigma((\check{\beta}_{1j})^{\top} X + \check{\beta}_{0j})\Big]\cdot\dfrac{q_{G^n_1}(Y|X)-q_{G^*_1}(Y|X)}{m_n\mathcal{D}_{4n}}+\frac{1}{2}\dfrac{Q_n(Y|X)}{m_n\mathcal{D}_{4n}}\to0.
\end{align*}
as $n\to\infty$ for almost surely $(X,Y)$. From the decomposition of the terms $q_{G^n_1}(Y|X)-q_{G^*_1}(Y|X)$ and $Q_n(Y|X)$ in Stage 1, we have
\begin{align}
     \label{eq:zero_limit_sigmoid_dense}
    \frac{1}{2}\Big[\sum_{j = 1}^{k_2} \sigma((\check{\beta}_{1j})^{\top} X + \check{\beta}_{0j})\Big]\cdot\dfrac{A_{n,2}(Y|X)+A_{n,1}(Y|X)+A_{n,0}(Y|X)}{m_n\mathcal{D}_{4n}}
    +\frac{1}{2}\dfrac{B_{n}(Y|X)-C_{n}(Y|X)}{m_n\mathcal{D}_{4n}}\to0.
\end{align}
We have
\begin{align*}
    &\lim_{n\to\infty}\frac{A_{n,0}(Y|X)}{m_n\mathcal{D}_{4n}}=\sum_{j=1}^{k^*_1}s_{0,j}\pi(Y|h_1(X,\kappa^*_j),\tau^*_j),\\
    &\lim_{n\to\infty}\frac{A_{n,1}(Y|X)}{m_n\mathcal{D}_{4n}}=\sum_{j\in[k^*_1]:|\mathcal{V}_{1,j}|=1}\Big[\sum_{u_1=1}^{d_1}s_{1,j}^{(u_1)}\frac{\partial h_1}{\partial\kappa^{(u_1)}}(X,\kj)\frac{\partial\pi}{\partial h_1}(Y|h_1(X,\kj),\tj)\\
    &\hspace{10cm}+\frac{1}{2}s_{2,j}\frac{\partial^2\pi}{\partial h_1^2}(Y|h_1(X,\kj),\tj)\Big],\\
    &\lim_{n\to\infty}\frac{A_{n,2}(Y|X)}{m_n\mathcal{D}_{4n}}=\sum_{j\in[k^*_1]:|\mathcal{V}_{1,j}|>1}\Big[\Big(\sum_{u_1=1}^{d_1}s_{1,j}^{(u_1)}\frac{\partial h_1}{\partial\kappa^{(u_1)}}(X,\kj)+\sum_{u_1,v_1=1}^{d_1}\frac{s_{3,j}^{(u_1v_1)}}{1+1_{\{u_1=v_1\}}}\frac{\partial^2h_1}{\partial\kappa^{(u_1)}\partial\kappa^{(v_1)}}(X,\kj)\Big)\\
    &\times \frac{\partial\pi}{\partial h_1}(Y|h_1(X,\kj),\tj)+\Big(\frac{1}{2}s_{2,j}+\sum_{u_1,v_1=1}^{d_1}\frac{s_{3,j}^{(u_1v_1)}}{1+1_{\{u_1=v_1\}}}\frac{\partial h_1}{\partial\kappa^{(u_1)}}(X,\kj)\frac{\partial h_1}{\partial\kappa^{(v_1)}}(X,\kj)\Big)\frac{\partial^2\pi}{\partial h_1^2}(Y|h_1(X,\kj),\tj)\\
    &\hspace{3cm}+\Big(\frac{1}{2}\sum_{u_1=1}^{d_1}s_{5,j}^{(u_1)}\frac{\partial h_1}{\partial\kappa^{(u_1)}}(X,\kj)\Big)\frac{\partial^3\pi}{\partial h_1^3}(Y|h_1(X,\kj),\tj)+\frac{1}{8}s_{4,j}\frac{\partial^4\pi}{\partial h_1^4}(Y|h_1(X,\kj),\tj)\Big],
\end{align*}
and
\begin{align*}
    &\lim_{n\to\infty}\frac{B_{n}(Y|X)}{m_n\mathcal{D}_{4n}}=\sum_{i=1}^{k_2}\Big[\Big(\sum_{u=1}^{d}t_{1,i}^{(u)}\frac{\partial\psi}{\partial\beta_1^{(u)}}(X;\cboi,\cbzi)+t_{0,i}\frac{\partial\psi}{\partial\beta_0}(X;\cboi,\cbzi)\Big)\pi(Y|h_2(X,\cei),\cnui)\\
    &\hspace{4cm}+\sum_{u_2=1}^{d_2}t_{2,i}^{(u_2)}\frac{\partial h_2}{\partial\eta^{(u_2)}}(X,\cei)\psi(X;\cboi,\cbzi)\frac{\partial\pi}{\partial h_2}(Y|h_2(X,\cei),\cnui)\\
    &\hspace{4cm}+\frac{1}{2}(\cdnuin)\psi(X;\cboi,\cbzi)\frac{\partial^2\pi}{\partial h_2^2}(Y|h_2(X,\cei),\cnui)\Big],\\
    &\lim_{n\to\infty}\frac{C_{n}(Y|X)}{m_n\mathcal{D}_{4n}}=\sum_{i=1}^{k_2}\Big[\sum_{u=1}^{d}t_{1,i}^{(u)}\frac{\partial\psi}{\partial\beta_1^{(u)}}(X;\cboi,\cbzi)+t_{0,i}\frac{\partial\psi}{\partial\beta_0}(X;\cboi,\cbzi)\Big]p_{\check{G}_2}(Y|X).
\end{align*}
Note that for almost every $X$, the set
\begin{align*}
    &\Bigg\{\Big[\sum_{j = 1}^{k_2} \sigma((\check{\beta}_{1j})^{\top} X + \check{\beta}_{0j})\Big]\frac{\partial^{\rho}\pi}{\partial h_1^{\rho}}(Y|h_1(X,\kj),\tj):0\leq\rho\leq 4, \  j\in[k^*_1]\Bigg\}\\
    \cup~&\Bigg\{\frac{\partial\psi}{\partial\beta_1^{(u)}}(X;\cboi,\cbzi)\pi(Y|h_2(X,\cei),\cnui), \ \frac{\partial\psi}{\partial\beta_0}(X;\cboi,\cbzi)\pi(Y|h_2(X,\cei),\cnui), \\
    &\quad \frac{\partial\psi}{\partial\beta_1^{(u)}}(X;\cboi,\cbzi)p_{\check{G}_2}(Y|X), \ \frac{\partial\psi}{\partial\beta_0}(X;\cboi,\cbzi)p_{\check{G}_2}(Y|X), \\
    &\quad\psi(X;\cboi,\cbzi)\frac{\partial\pi}{\partial h_2}(Y|h_2(X,\cei),\cnui), \ \psi(X;\cboi,\cbzi)\frac{\partial^2\pi}{\partial h_2^2}(Y|h_2(X,\cei),\cnui)
    :u\in[d], \ i\in[k_2]\Bigg\}
\end{align*}
is linearly independent w.r.t $Y$, implying that the coefficients of those terms in the limit in equation~\eqref{eq:zero_limit_sigmoid_dense} are equal to zero. 

\vspace{0.5 em}
\noindent
Since the expert function $h_1$ is strongly identifiable, then by employing the same arguments as in the Stage 3 of Appendix~\ref{appendix:strongly_identifiable_experts_sigmoid}, we get $s_{0,j}=s_{1,j}^{(u_1)}=s_{2,j}=s_{3,j}^{(u_1v_1)}=s_{4,j}=s_{5,j}^{(u_1)}=0$ for all $j\in[k^*_1]$ and $u_1,v_1\in[d_1]$.
For $i\in[k_2]$, by considering the coefficients of 
\begin{itemize}
    \item $\frac{\partial\psi}{\partial\beta_1^{(u)}}(X;\cboi,\cbzi)\pi(Y|h_2(X,\cei),\cnui)$, we get $t_{1,i}^{(u)}=0$ for all $u\in[d]$;
    \item $\frac{\partial\psi}{\partial\beta_0}(X;\cboi,\cbzi)\pi(Y|h_2(X,\cei),\cnui)$, we get $t_{0,i}=0$;
    \item $\psi(X;\cboi,\cbzi)\frac{\partial\pi}{\partial h_2}(Y|h_2(X,\cei),\cnui)$, we get $\sum_{u_2=1}^{d_2}t_{2,i}^{(u_2)}\frac{\partial h_2}{\partial\eta^{(u_2)}}(X,\cei)=0$. Since the expert function $h_2$ is weakly identifiable, we deduce $t_{2,i}^{(u_2)}=0$ for all $u_2\in[d_2]$;
    \item $\psi(X;\cboi,\cbzi)\frac{\partial^2\pi}{\partial h_2^2}(Y|h_2(X,\cei),\cnui)$, we get $t_{3,i}=0$.
\end{itemize}
From the above results, it follows that (i) $s_{0,j}=s_{1,j}^{(u_1)}=s_{2,j}=s_{3,j}^{(u_1v_1)}=s_{4,j}=s_{5,j}^{(u_1)}=0$ for all $j\in[k^*_1]$ and $u_1,v_1\in[d_1]$; (ii) $t_{0,i}=t_{1,i}^{(u)}=t_{2,i}^{(u_2)}=t_{3,i}=0$ for all $i\in[k_2]$, $u\in[d]$ and $u_2\in[d_2]$. This contradicts to the fact that not all of them equal to zero. As a consequence, we obtain the local part in equation~\eqref{eq:local_part_sigmoid_dense}. Hence, the proof is completed.

\section{Proof of Auxiliary Results}

\subsection{Proof of Proposition~\ref{prop:density_rate}}
\label{appendix:density_rate}
In this proof, we will leverage fundamental results on density estimation for M-estimators in \cite{vandeGeer-00}. Before streamlining our arguments, let us introduce some concepts from the empirical process theory adapted to the setting of the model~\eqref{eq:density}.

\vspace{0.5 em}
\noindent
Firstly, we denote by $\mathcal{F}_{k_1,k_2}(\Theta):=\{f_{G_1,G_2}(Y|X):(G_1,G_2)\in\mathcal{G}_{k_1,k_2}(\Theta)\}$ the set of conditional density functions of interest. Furthermore, we also consider two variants of this set defined as 
\begin{align*}
   \widetilde{\mathcal{F}}_{k_1,k_2}(\Theta)&:=\Big\{\frac{1}{2}f_{(G_1,G_2)}(Y|X)+\frac{1}{2}f_{(G_1,G_2)}(Y|X):(G^*_1,G^*_2)\in\mathcal{G}_{k_1,k_2}(\Theta)\Big\},\\
   \widetilde{\mathcal{F}}^{1/2}_{k_1,k_2}(\Theta)&:=\{\tilde{f}^{1/2}:\tilde{f}\in\widetilde{\mathcal{F}}_{k_1,k_2}(\Theta)\}.
\end{align*}
For any $\delta>0$, the Hellinger ball centered around the the true density $f_{G^*_1,G^*_2}(Y|X)$ and intersected with $\widetilde{\mathcal{F}}_{k_1,k_2}(\Theta)$ is defined as 
\begin{align*}
    \widetilde{\mathcal{F}}^{1/2}_{k_1,k_2}(\Theta,\delta):=\{p^{1/2}\in\widetilde{\mathcal{F}}^{1/2}_{k_1,k_2}(\Theta):h(p,f_{G^*_1,G^*_2})\leq\delta\}.
\end{align*}
The size of the above Hellinger ball is determined by the  quantity \cite{vandeGeer-00}
\begin{align}
    \label{eq:integral_hellinger_ball}
    \mathcal{J}_B(\delta,\widetilde{\mathcal{F}}^{1/2}_{k_1,k_2}(\Theta,\delta),\|\cdot\|_{2}):=\int_{\delta^2/2^{13}}^{\delta}H_B^{1/2}(t,\widetilde{\mathcal{F}}^{1/2}_{k_1,k_2}(\Theta,t),\|\cdot\|_{2})\dint t\vee\delta,
\end{align}
where $H_B(t,\widetilde{\mathcal{F}}^{1/2}_{k_1,k_2}(\Theta,t),\|\cdot\|_{2})$ stands for the bracketing entropy of $\widetilde{\mathcal{F}}^{1/2}_{k_1,k_2}(\Theta,t)$ under the $L^2(m)$-norm with $m$ being the Lebesgue measure, and $t\vee\delta:=\max\{t,\delta\}$. Equipped with these notations, we are ready to present a standard result on density estimation for M-estimators in the following lemma:
\begin{lemma}[Theorem 7.4, \cite{vandeGeer-00}]
    \label{lemma:vandegeer}
    Let $\delta \in (0,1)$ and take $\Psi(\delta)\geq\mathcal{J}_B(\delta,\widetilde{\mathcal{F}}^{1/2}_{k_1,k_2}(\Theta,\delta))$ such that $\Psi(\delta)/\delta^2$ is a non-increasing function of $\delta $. Then, for a universal constant $c$ and for some sequence $(\delta_n)$ satisfying $\sqrt{n}\delta^2_n\geq c\Psi(\delta_n)$, the following holds for all $\delta\geq\delta_n$:
    \begin{align*}
        \mathbb{P}\Big(\bbE_X\Big[h(f_{\widetilde{G}^n_1,\widetilde{G}^n_2}(\cdot|X),f_{G^*_1,G^*_2}(\cdot|X))>\delta\Big]\Big)\leq c\exp\Big(-\frac{n\delta^2}{c^2}\Big).
    \end{align*}
\end{lemma}
\noindent
Given the above result, we will provide below the proof for Proposition~\ref{prop:density_rate}.
\begin{proof}[Main proof of Proposition~\ref{prop:density_rate}]
    Since $\widetilde{\mathcal{F}}^{1/2}_{k_1,k_2}(\Theta,t)\subset\widetilde{\mathcal{F}}^{1/2}_{k_1,k_2}(\Theta)$ for any $t>0$, we have
    \begin{align}
        \label{eq:bracketing_inequality}
        H_B(t,\widetilde{\mathcal{F}}^{1/2}_{k_1,k_2}(\Theta,t),\|\cdot\|_{2})&\leq H_B(t,\widetilde{\mathcal{F}}^{1/2}_{k_1,k_2}(\Theta),\|\cdot\|_{2})=H_B(t/\sqrt{2},\widetilde{\mathcal{F}}_{k_1,k_2}(\Theta),h),
    \end{align}
    where the last equality is due to the relationship between the Hellinger distance $h$ and the $L^2$-norm. Note that for any two mixing measure pairs $(G_1,G_2)$ and $(G^{\prime}_1,G^{\prime}_2)$, Lemma 4.2 in \cite{vandeGeer-00} shows that
    \begin{align*}
        h^2\Big(\frac{1}{2}f_{G_1,G_2}+\frac{1}{2}f_{G^*_1,G^*_2},\frac{1}{2}f_{G^{\prime}_1,G^{\prime}_2}+\frac{1}{2}f_{G^*_1,G^*_2}\Big)\leq\frac{1}{2}h^2(f_{G_1,G_2},f_{G^{\prime}_1,G^{\prime}_2}),
    \end{align*}
    which yields that $H_B(t/\sqrt{2},\widetilde{\mathcal{F}}_{k_1,k_2}(\Theta),h)\leq H_B(t,\mathcal{F}_{k_1,k_2}(\Theta),h)$. This result together with equation~\eqref{eq:bracketing_inequality} implies that 
    \begin{align*}
        H_B(t,\widetilde{\mathcal{F}}^{1/2}_{k_1,k_2}(\Theta,t),\|\cdot\|_{2})\leq H_B(t,\mathcal{F}_{k_1,k_2}(\Theta),h).
    \end{align*}
    From the definition of the Hellinger ball size in equation~\eqref{eq:integral_hellinger_ball}, we have that
    \begin{align*}
\mathcal{J}_B(\delta,\widetilde{\mathcal{F}}^{1/2}_{k_1,k_2}(\Theta,\delta),\|\cdot\|_{2})&=\int_{\delta^2/2^{13}}^{\delta}H_B^{1/2}(t,\widetilde{\mathcal{F}}^{1/2}_{k_1,k_2}(\Theta,t),\|\cdot\|_{2})\dint t\vee\delta\\
        &\leq\int_{\delta^2/2^{13}}^{\delta}H_B^{1/2}(t,\mathcal{F}_{k_1,k_2}(\Theta),h)\dint t\vee\delta\\
        &\lesssim\int_{\delta^2/2^{13}}^{\delta}[\log(1/t)]^{1/2}\dint t\vee\delta,
    \end{align*}
    where the last inequality is due to  Lemma~\ref{lemma:bracketing_entropy_bound} below. Let $\Psi(\delta):=\delta\sqrt{\log(1/\delta)}$, it can be verified that $\Psi(\delta)/\delta^2$ is a non-increasing function of $\delta$. Furthermore, the above result indicates that $\Psi(\delta)\geq \mathcal{J}_B(\delta,\widetilde{\mathcal{F}}^{1/2}_{k_1,k_2}(\Theta,\delta),\|\cdot\|_{2})$. By considering the sequence $(\delta_n)$ defined as $\delta_n:=\sqrt{\log(n)/n}$, we have $\sqrt{n}\delta_n^2\geq c\Psi(\delta_n)$ for some universal constant $c>0$. Then, according to Lemma~\ref{lemma:vandegeer}, we get 
    \begin{align*}
        \mathbb{P}\Big(\bbE_X\Big[h(f_{\widetilde{G}^n_1,\widetilde{G}^n_2}(\cdot|X),f_{G^*_1,G^*_2}(\cdot|X))>C\sqrt{\log(n)/n}\Big]\Big)\lesssim \exp(-c\log(n)),
    \end{align*}
    for some universal constant $C$ depending on $\Theta$.
\end{proof}
\begin{lemma}
    \label{lemma:bracketing_entropy_bound}
    The following holds for any $0<\epsilon<1/2$:
    \begin{align*}
        H_B(\epsilon,\mathcal{F}_{k_1,k_2}(\Theta),h)\lesssim\log(1/\epsilon).
    \end{align*}
\end{lemma}
\begin{proof}[Proof of Lemma~\ref{lemma:bracketing_entropy_bound}]
    Recall that for any mixing measure pair $(G_1,G_2)$, we have
    \begin{align*}
        f_{G_1,G_2}(Y|X)=\frac{1}{2}\sum_{i=1}^{k_1}\omega_{i}\pi(Y|h_1(X,\kappa_{i}),\tau_{i})+\frac{1}{2}\sum_{i = 1}^{k_2} \frac{\exp((\beta_{1i})^{\top} X + \beta_{0i})}{\sum_{j = 1}^{k_2} \exp((\beta_{1j})^{\top} X + \beta_{0j})}\cdot \pi(Y|h_2(X,\eta_{i}), \nu_{i}). 
    \end{align*}
    Firstly, we will establish upper bounds for the Gaussian densities $\pi(Y|h_1(X,\kappa),\tau)$ and $\pi(Y|h_2(X,\eta), \nu)$, respectively. Indeed, since the expert function $h_1$ is bounded and the parameter space is compact, we have $|h_1(X,\kappa)|\leq M_1$ for all $X\in\mathcal{X}$ for some constant $M_1>0$, and $\ell_1\leq \tau\leq u_1$ for some $\ell_1,u_1>0$. Therefore, for $|Y|\geq 2M_1$, since $\frac{(Y-h_1(X,\kappa))^2}{2\tau}\geq\frac{Y^2}{8u_1}$ for all $X\in\mathcal{X}$, we have
    \begin{align*}
        \pi(Y|h_1(X,\kappa),\tau)=\frac{1}{\sqrt{2\boldsymbol{\pi}\tau}}\exp\Big(-\frac{(Y-h_1(X,\kappa))^2}{2\tau}\Big)\leq\frac{1}{\sqrt{2\boldsymbol{\pi}\ell_1}}\exp\Big(-\frac{Y^2}{8u_1}\Big).
    \end{align*}
    Next, for $|Y|<2M_1$, it follows that
    \begin{align*}
        \pi(Y|h_1(X,\kappa),\tau)=\frac{1}{\sqrt{2\boldsymbol{\pi}\tau}}\exp\Big(-\frac{(Y-h_1(X,\kappa))^2}{2\tau}\Big)\leq\frac{1}{\sqrt{2\boldsymbol{\pi}\tau}}\leq \frac{1}{\sqrt{2\boldsymbol{\pi}\ell_1}}.
    \end{align*}
    Combine the above results together, we deduce $ \pi(Y|h_1(X,\kappa),\tau)\leq E_1(Y|X)$ for all $(X,Y)$ where
    \begin{align*}
        E_1(Y|X):=\begin{cases}
            \frac{1}{\sqrt{2\boldsymbol{\pi}\ell_1}}\exp\Big(-\frac{Y^2}{8u_1}\Big), \quad \text{for } |Y|\geq 2M_1\\
            \frac{1}{\sqrt{2\boldsymbol{\pi}\ell_1}}, \hspace{2.6cm} \text{for } |Y|<2M_1.
        \end{cases}
    \end{align*}
    By arguing in similar fashion based on the assumptions that $|h_2(X,\eta)|\leq M_2$ for all $X\in\mathcal{X}$ for some constant $M_2>0$, and $\ell_2\leq\nu\leq u_2$ for some $\ell_2,u_2>0$, we also get $\pi(Y|h_2(X,\eta), \nu)\leq E_2(Y|X)$, where
    \begin{align*}
        E_2(Y|X):=\begin{cases}
            \frac{1}{\sqrt{2\boldsymbol{\pi}\ell_2}}\exp\Big(-\frac{Y^2}{8u_2}\Big), \quad \text{for } |Y|\geq 2M_2\\
            \frac{1}{\sqrt{2\boldsymbol{\pi}\ell_2}}, \hspace{2.6cm} \text{for } |Y|<2M_2.
        \end{cases}
    \end{align*}
    Now, let $\lambda\leq\epsilon$ be some constant that we will choose later, we denote $p_1,p_2,\ldots, p_N$ as an $\lambda$-cover of the set $\mathcal{F}_{k_1,k_2}(\Theta)$, where $N:=N(\lambda,\mathcal{F}_{k_1,k_2}(\Theta),\|\cdot\|_{\infty})$ stands for the $\lambda$-covering number of the set $\mathcal{F}_{k_1,k_2}(\Theta)$ under the $L^{\infty}$-norm. Then, we take into account the brackets $[p^L_i,p^U_i]$ given by
    \begin{align*}
        p^L_i(Y|X)&:=\max\{p_i(Y|X)-\lambda,0\},\\
        p^U_i(Y|X)&:=\max\{p_i(Y|X)+\lambda,E(Y|X)\},
    \end{align*}
    for all $i\in[N]$, where $E(Y|X):=\frac{1}{2}E_1(Y|X)+\frac{1}{2}E_2(Y|X)$. It can be justified that $\mathcal{F}_{k_1,k_2}(\Theta)\subseteq\cup_{i=1}^{N}[p^L_i,p^U_i]$ and $p^U_i(Y|X)-p^L_i(Y|X)\leq\min\{2\lambda,E(Y|X)\}$. Furthermore, we have
    \begin{align*}
        \|p^U_i-p^L_i\|_{2}=\Big(\int[p^U_i(Y|X)-p^L_i(Y|X)]^2\dint(X,Y)\Big)^{1/2}\leq2\lambda. 
    \end{align*}
    By definition of the bracketing entropy, we get
\begin{align*}
    H_B(2\lambda,\mathcal{F}_{k_1,k_2}(\Theta),\|\cdot\|_{2})\leq \log N=\log N(\lambda,\mathcal{F}_{k_1,k_2}(\Theta),\|\cdot\|_{\infty}).
\end{align*}
Thus, we need to derive an upper bound for the covering number $N(\lambda,\mathcal{F}_{k_1,k_2}(\Theta),\|\cdot\|_{\infty})$. Let us denote $\Delta:=\Delta_1\times\Delta_2$ and $\Omega:=\Omega_1\times\Omega_2$, where
\begin{align*}
    \Delta_1&:=\{\omega_i\in\mathbb{R}_+:(\omega,\kappa,\tau)\in\Theta_1\}, \\
    \Delta_2&:=\{(\kappa,\tau)\in\mathbb{R}^{d_1}\times\mathbb{R}_+:(\omega,\kappa,\tau)\in\Theta_1\},\\
    \Omega_1&:=\{(\beta_{0},\beta_{1})\in\mathbb{R}\times\mathbb{R}^{d}:(\beta_{0},\beta_{1},\eta,\nu)\in\Theta_2\},\\
    \Omega_2&:=\{(\eta,\nu)\in\mathbb{R}^{d_2}\times\mathbb{R}_+:(\beta_{0},\beta_{1},\eta,\nu)\in\Theta_2\}.
\end{align*}
Since $\Theta_1$ and $\Theta_2$ are compact, the sets $\Delta_1,\Delta_2$ and $\Omega_1,\Omega_2$ are also compact. Thus, there exist $\lambda$-covers $\Delta_{1,\lambda},\Delta_{2,\lambda}$ and $\Omega_{1,\lambda},\Omega_{2,\lambda}$ for those sets, respectively. Moreover, the cardinalities of those $\lambda$-covers are bounded as follows:
\begin{align*}
    |\Delta_{1,\lambda}|\leq\mathcal{O}(\lambda^{-k_1}),& \quad |\Delta_{2,\lambda}|\leq\mathcal{O}(\lambda^{-(d_1+1)k_1}),\\
    |\Omega_{1,\lambda}|\leq\mathcal{O}(\lambda^{-(d+1)k_2}),& \quad |\Omega_{2,\lambda}|\leq\mathcal{O}(\lambda^{-(d_2+1)k_2}).
\end{align*}
For each pair of mixing measure $(G_1,G_2)\in\mathcal{G}_{k_1,k_2}(\Theta)$, we consider two other mixing measure pairs $(G'_1,G'_2)$ and $(\overline{G}_1,\overline{G}_2)$ given by
\begin{align*}
    G'_1:=\sum_{i=1}^{k_1}\omega_i\delta_{(\bar{\kappa}_i,\bar{\tau}_i)},& \hspace{2cm} G'_2:=\sum_{i=1}^{k_2}\bar{\omega}_i\delta_{(\bar{\kappa}_i,\bar{\tau}_i)},\\
    \overline{G}_1:=\sum_{i=1}^{k_2}\exp(\beta_{0i})\delta_{(\beta_{1i},\bar{\eta}_i,\bar{\tau}_i)},& \hspace{2cm} \overline{G}_2:=\sum_{i=1}^{k_2}\exp(\bar{\beta}_{0i})\delta_{(\bar{\beta}_{1i},\bar{\eta}_i,\bar{\nu}_i)}.
\end{align*}
Above, $\bar{\omega}_i\in\Delta_{1,\lambda}$ is the closest point to $\omega_i$ in that set, $(\bar{\kappa}_i,\bar{\tau}_i)\in\Delta_{2,\lambda}$ is the closest point to $(\kappa_i,\tau_i)$ in that set, $(\bar{\beta}_{0i},\bar{\beta}_{1i})\in\Omega_{1,\lambda}$ is the closest point to $(\beta_{0i},\beta_{1i})$ in that set, $(\bar{\eta}_i,\bar{\nu}_i)\in\Omega_{2,\lambda}$ is the closest point to $(\eta_i,\nu_i)$ in that set. Subsequently, we aim to upper bound the term $\|f_{G_1,G_2}-f_{\overline{G}_1,\overline{G}_2}\|_{\infty}$. By the triangle inequality, we have
\begin{align*}
    \|f_{G_1,G_2}-f_{\overline{G}_1,\overline{G}_2}\|_{\infty}\leq \|f_{G_1,G_2}-f_{G'_1,G'_2}\|_{\infty}+\|f_{G'_1,G'_2}-f_{\overline{G}_1,\overline{G}_2}\|_{\infty}.
\end{align*}
We aim to upper bound the two terms in the above right hand sides, respectively. For ease of presentation, for any mixing measure pair $(G_1,G_2)$, we denote
\begin{align*}
    q_{G_1}(Y|X)&:=\sum_{i=1}^{k_1}\omega_i\pi(Y|h_1(X,\kappa_i),\tau_i),\\
    p_{G_2}(Y|X)&:=\sum_{i=1}^{k_2}\frac{\exp(\beta_{1i}^{\top}X+\beta_{0i})}{\sum_{j=1}^{k_2}\exp(\beta_{1j}^{\top}X+\beta_{0j})}\pi(Y|h_2(X,\eta_i),\nu_i).
\end{align*}
We start with bounding the term $\|f_{G_1,G_2}-f_{G'_1,G'_2}\|_{\infty}$ as follows:
\begin{align*}
    \|f_{G_1,G_2}-f_{G'_1,G'_2}\|_{\infty}\leq\frac{1}{2}\|q_{G_1}-q_{G'_1}\|_{\infty}+\frac{1}{2}\|p_{G_2}-p_{G'_2}\|_{\infty}.
\end{align*}
In particular, we have
\begin{align*}
    \|q_{G_1}-q_{G'_1}\|_{\infty}&=\sup_{(X,Y)\in\mathcal{X}\times\mathcal{Y}}\Bigg|\sum_{i=1}^{k_1}\omega_i\Big[\pi(Y|h_1(X,\kappa_i),\tau_i)-\pi(Y|h_1(X,\bar{\kappa}_i),\bar{\tau}_i)\Big]\Bigg|\\
    &\leq\sum_{i=1}^{k_1}\omega_i\sup_{(X,Y)\in\mathcal{X}\times\mathcal{Y}}~\Big|\pi(Y|h_1(X,\kappa_i),\tau_i)-\pi(Y|h_1(X,\bar{\kappa}_i),\bar{\tau}_i)\Big|\\
    &\lesssim\sum_{i=1}^{k_1}\omega_i(\|\kappa_i-\bar{\kappa}_i\|+|\tau_i-\bar{\tau}_i|)\\
    &\leq\sum_{i=1}^{k_1}\omega_i(\lambda+\lambda)=2\lambda\lesssim\lambda,
\end{align*}
and
\begin{align*}
    &\|p_{G_2}-p_{G'_2}\|_{\infty}\\
    &=\sup_{(X,Y)\in\mathcal{X}\times\mathcal{Y}}\Bigg|\sum_{i=1}^{k_2}\frac{\exp(\beta_{1i}^{\top}X+\beta_{0i})}{\sum_{j=1}^{k_2}\exp(\beta_{1j}^{\top}X+\beta_{0j})}\Big[\pi(Y|h_2(X,\eta_i),\nu_i)-\pi(Y|h_1(X,\bar{\eta}_i),\bar{\nu}_i)\Big]\Bigg|\\
    &\leq\sum_{i=1}^{k_2}\sup_{(X,Y)\in\mathcal{X}\times\mathcal{Y}}~\frac{\exp(\beta_{1i}^{\top}X+\beta_{0i})}{\sum_{j=1}^{k_2}\exp(\beta_{1j}^{\top}X+\beta_{0j})}\Big|\pi(Y|h_2(X,\eta_i),\nu_i)-\pi(Y|h_1(X,\bar{\eta}_i),\bar{\nu}_i)\Big|\\
    &\leq\sum_{i=1}^{k_2}\sup_{(X,Y)\in\mathcal{X}\times\mathcal{Y}}~\Big|\pi(Y|h_2(X,\eta_i),\nu_i)-\pi(Y|h_1(X,\bar{\eta}_i),\bar{\nu}_i)\Big|\\
    &\lesssim\sum_{i=1}^{k_2}(\|\eta_i-\bar{\eta}_i\|+|\nu_i-\bar{\nu}_i|)\leq\sum_{i=1}^{k_2}(\lambda+\lambda)\lesssim\lambda,
\end{align*}
which implies that
\begin{align}
    \label{eq:first_term_bound}
    \|f_{G_1,G_2}-f_{G'_1,G'_2}\|_{\infty}\lesssim\frac{1}{2}\lambda+\frac{1}{2}\lambda=\lambda.
\end{align}
Next, we continue with bounding the term $\|f_{G'_1,G'_2}-f_{\overline{G}_1,\overline{G}_2}\|_{\infty}$ as
\begin{align*}
    \|f_{G'_1,G'_2}-f_{\overline{G}_1,\overline{G}_2}\|_{\infty}\leq \frac{1}{2}\|q_{G'_1}-q_{\overline{G}_1}\|_{\infty}+\frac{1}{2}\|p_{G'_2}-p_{\overline{G}_2}\|_{\infty}.
\end{align*}
By looking into each term in the above right hand side, we have
\begin{align*}
    \|q_{G'_1}-q_{\overline{G}_1}\|_{\infty}&=\sup_{(X,Y)\in\mathcal{X}\times\mathcal{Y}}\Bigg|\sum_{i=1}^{k_1}[\omega_i-\bar{\omega}_i]\pi(Y|h_1(X,\bar{\kappa}_i),\bar{\tau}_i)\Bigg|\\
    &\leq\sum_{i=1}^{k_1}|\omega_i-\bar{\omega}_i|\sup_{(X,Y)\in\mathcal{X}\times\mathcal{Y}}~|\pi(Y|h_1(X,\bar{\kappa}_i),\bar{\tau}_i)|\\
    &\lesssim\sum_{i=1}^{k_1}|\omega_i-\bar{\omega}_i|\leq\sum_{i=1}^{k_1}\lambda\lesssim\lambda,
\end{align*}
and
\begin{align*}
    &\|p_{G'_2}-p_{\overline{G}_2}\|_{\infty}\\
    &=\sup_{(X,Y)\in\mathcal{X}\times\mathcal{Y}}\Bigg|\sum_{i=1}^{k_2}\Bigg[\frac{\exp(\beta_{1i}^{\top}X+\beta_{0i})}{\sum_{j=1}^{k_2}\exp(\beta_{1j}^{\top}X+\beta_{0j})}-\frac{\exp(\bar{\beta}_{1i}^{\top}X+\bar{\beta}_{0i})}{\sum_{j=1}^{k_2}\exp(\bar{\beta}_{1j}^{\top}X+\bar{\beta}_{0j})}\Bigg]\pi(Y|h_1(X,\bar{\kappa}_i),\bar{\tau}_i)\Bigg|\\
    &\leq\sum_{i=1}^{k_2}\sup_{(X,Y)\in\mathcal{X}\times\mathcal{Y}}~\Bigg|\frac{\exp(\beta_{1i}^{\top}X+\beta_{0i})}{\sum_{j=1}^{k_2}\exp(\beta_{1j}^{\top}X+\beta_{0j})}-\frac{\exp(\bar{\beta}_{1i}^{\top}X+\bar{\beta}_{0i})}{\sum_{j=1}^{k_2}\exp(\bar{\beta}_{1j}^{\top}X+\bar{\beta}_{0j})}\Bigg|\cdot|\pi(Y|h_1(X,\bar{\kappa}_i),\bar{\tau}_i)|\\
    &\lesssim\sum_{i=1}^{k_2}\sup_{(X,Y)\in\mathcal{X}\times\mathcal{Y}}~\Bigg|\frac{\exp(\beta_{1i}^{\top}X+\beta_{0i})}{\sum_{j=1}^{k_2}\exp(\beta_{1j}^{\top}X+\beta_{0j})}-\frac{\exp(\bar{\beta}_{1i}^{\top}X+\bar{\beta}_{0i})}{\sum_{j=1}^{k_2}\exp(\bar{\beta}_{1j}^{\top}X+\bar{\beta}_{0j})}\Bigg|\\
    &\lesssim\sum_{i=1}^{k_2}\sup_{X\in\mathcal{X}}~\Big(\|\beta_{1i}-\bar{\beta}_{1i}\|\cdot\|X\|+\|\beta_{0i}-\bar{\beta}_{0i}\|\Big)\\
    &\lesssim\sum_{i=1}^{k_2}\Big(\lambda\cdot\sup_{X\in\mathcal{X}}~\|X\|+\lambda\Big)\lesssim\lambda.
\end{align*}
Putting these bounds together, we deduce
\begin{align}
    \label{eq:second_term_bound}
    \|f_{G'_1,G'_2}-f_{\overline{G}_1,\overline{G}_2}\|_{\infty}\lesssim\frac{1}{2}\lambda+\frac{1}{2}\lambda=\lambda.
\end{align}
From equations~\eqref{eq:first_term_bound} and \eqref{eq:second_term_bound}, we obtain 
\begin{align*}
    \|f_{G_1,G_2}-f_{\overline{G}_1,\overline{G}_2}\|_{\infty}\leq\lambda+\lambda\lesssim\lambda.
\end{align*}
By definition of the covering number, we get 
\begin{align*}
    N(\lambda,\mathcal{F}_{k_1,k_2}(\Theta),\|\cdot\|_{\infty})&\leq|\Delta_{1,\lambda}|\cdot|\Delta_{2,\lambda}|\cdot|\Omega_{1,\lambda}|\cdot|\Omega_{2,\lambda}|\\
    &\leq\mathcal{O}(\lambda^{-k_1})\cdot\mathcal{O}(\lambda^{-(d_1+1)k_1})\cdot\mathcal{O}(\lambda^{-(d+1)k_2})\cdot\mathcal{O}(\lambda^{-(d_2+1)k_2})\\
    &\leq\mathcal{O}(\lambda^{-(d_1+2)k_1-(d_2+d+2)k_2}).
\end{align*}
As a result, we deduce
\begin{align*}
     H_B(2\lambda,\mathcal{F}_{k_1,k_2}(\Theta),\|\cdot\|_{2})\leq \log N(\lambda,\mathcal{F}_{k_1,k_2}(\Theta),\|\cdot\|_{\infty})\lesssim\log(1/\lambda).
\end{align*}
Let $\lambda=\epsilon/2$, we achieve the desired result that $H_B(\epsilon,\mathcal{F}_{k_1,k_2}(\Theta),\|\cdot\|_{2})\lesssim\log(1/\epsilon)$. Hence, the proof is completed.
\end{proof}

\subsection{Identifiability of DeepSeekMoE}

\begin{proposition}[Identifiability]
    \label{prop:identifiability}
    For any pair of mixing measures $(G_1,G_2)$, if the equation $f_{G_1,G_2}(Y|X)=f_{G^*_1,G^*_2}(Y|X)$ holds for almost surely $(X,Y)$, then we obtain $(G_1,G_2)\equiv(G^*_1,G^*_2)$.
\end{proposition}
\begin{proof}[Proof of Proposition~\ref{prop:identifiability}]
    First of all, we expand the equation $f_{G_1,G_2}(Y|X)=f_{G^*_1,G^*_2}(Y|X)$ for almost surely $(X,Y)$ as follows:
    \begin{align}
        \frac{1}{2}\sum_{i=1}^{k_1}\omega_i\pi(Y|h_1(X,\kappa_i),\tau_i)+\frac{1}{2}\sum_{i=1}^{k_2}\frac{\exp(\beta_{1i}^{\top}X+\beta_{0i})}{\sum_{j=1}^{k_2}\exp(\beta_{1j}^{\top}X+\beta_{0j})}\pi(Y|h_2(X,\eta_i),\nu_i)\nonumber\\
        \label{eq:identifiable_equation}
        =\frac{1}{2}\sum_{i=1}^{k^*_1}\omega^*_{i}\pi(Y|h_1(X,\kappa^*_{i}),\tau^*_{i})+\frac{1}{2}\sum_{i = 1}^{k^*_2} \frac{\exp((\beta_{1i}^{*})^{\top} X + \beta_{0i}^{*})}{\sum_{j = 1}^{k^*_2} \exp((\beta_{1j}^{*})^{\top} X + \beta_{0j}^{*})}\cdot \pi(Y|h_2(X,\eta^*_{i}), \nu_{i}^{*}).
    \end{align}
    Since the location-scale Gaussian mixtures are identifiable \cite{Teicher-63}, the above equation implies that $k_1+k_2=k^*_1+k^*_2$ and
    \begin{align*}
        \Bigg\{\omega_{i'},\frac{\exp(\beta_{1i}^{\top}X+\beta_{0i})}{\sum_{j=1}^{k_2}\exp(\beta_{1j}^{\top}X+\beta_{0j})}&:i'\in[k_1], \ i\in[k_2]\Bigg\}\\
        &=\Bigg\{\omega^*_{i'},\frac{\exp((\beta^*_{1i})^{\top}X+\beta^*_{0i})}{\sum_{j=1}^{k^*_2}\exp((\beta^*_{1j})^{\top}X+\beta^*_{0j})}:i'\in[k^*_1], \ i\in[k^*_2]\Bigg\},
    \end{align*}
    for almost surely $X$. As the weights $\omega_{i'}$ and $\omega^*_{i'}$ are independent of $X$ for all $i'\in[k^*_1]$, we deduce $k_1=k^*_1$ and $\{\omega_{i'}:i'\in[k^*_1]\}=\{\omega^*_{i'}:i'\in[k^*_1]\}$. For simplicity, we assume WLOG that $\omega_{i'}=\omega^*_{i'}$ for all $i'\in[k^*_1]$. Furthermore, we also get $k_2=k^*_2$ and
    \begin{align*}
        \Bigg\{\frac{\exp(\beta_{1i}^{\top}X+\beta_{0i})}{\sum_{j=1}^{k^*_2}\exp(\beta_{1j}^{\top}X+\beta_{0j})}&:i\in[k^*_2]\Bigg\}=\Bigg\{\frac{\exp((\beta^*_{1i})^{\top}X+\beta^*_{0i})}{\sum_{j=1}^{k^*_2}\exp((\beta^*_{1j})^{\top}X+\beta^*_{0j})}:i\in[k^*_2]\Bigg\},
    \end{align*}
    for almost surely $X$. Again, we assume WLOG that $\frac{\exp(\beta_{1i}^{\top}X+\beta_{0i})}{\sum_{j=1}^{k^*_2}\exp(\beta_{1j}^{\top}X+\beta_{0j})}=\frac{\exp((\beta^*_{1i})^{\top}X+\beta^*_{0i})}{\sum_{j=1}^{k^*_2}\exp((\beta^*_{1j})^{\top}X+\beta^*_{0j})}$ for almost surely $X$ for all $i\in[k^*_2]$. Due to the invariance to translation of the softmax function, this result indicates that $\beta_{1i}=\beta^*_{1i}+c_1$ and $\beta_{0i}=\beta^*_{0i}+c_0$ for some $c_1\in\mathbb{R}^{d}$ and $c_0\in\mathbb{R}$. Then, it follows from the assumption $\beta_{1k^*_2}=\beta^*_{1k^*_2}=0_{d}$ and $\beta_{0k^*_2}=\beta^*_{0k^*_2}=0$ that $c_1=0_{d}$ and $c_0=0$. Therefore, we obtain $\beta_{1i}=\beta^*_{1i}$ and $\beta_{0i}=\beta^*_{0i}$ for all $i\in[k^*_2]$. 

    \vspace{0.5 em}
\noindent
    Subsequently, we partition the index set $[k^*_1]$ into disjoint subsets $U_1,U_2,\ldots,U_{m_1}$ such that for each $\ell\in[m_1]$, we have (i) $\omega_i=\omega^*_{i'}$ for $i,i'\in U_{\ell}$ and (ii) $\omega_i\neq\omega^*_{i'}$ if $i$ and $i'$ dot not belong to the same set $U_{\ell}$. Similarly, we also partition the index set $[k^*_2]$ into disjoint subsets $V_1,V_2,\ldots,V_{m_2}$ such that for each $\ell\in[m_2]$, we have (i) $\exp(\beta_{0i})=\exp(\beta^*_{0i'})$ for $i,i'\in V_{\ell}$ and (ii) $\exp(\beta_{0i})\neq\exp(\beta^*_{0i'})$ if $i$ and $i'$ dot not belong to the same set $V_{\ell}$.
    As a consequence, we can rewrite equation~\eqref{eq:identifiable_equation} as
    \begin{align*}
        \frac{1}{2}\sum_{\ell=1}^{m_1}\sum_{i\in U_{\ell}}\omega_i\pi(Y|h_1(X,\kappa_i),\tau_i)+\frac{1}{2S}\sum_{\ell=1}^{m_2}\sum_{i\in V_{\ell}}\exp(\beta_{0i})\exp(\beta_{1i}^{\top}X)\pi(Y|h_2(X,\eta_i),\nu_i)\nonumber\\
        =\frac{1}{2}\sum_{\ell=1}^{m_1}\sum_{i\in U_{\ell}}\omega^*_{i}\pi(Y|h_1(X,\kappa^*_{i}),\tau^*_{i})+\frac{1}{2S}\sum_{\ell=1}^{m_2}\sum_{i\in V_{\ell}}\exp(\beta_{0i}^{*})\exp((\beta_{1i}^{*})^{\top} X ) \pi(Y|h_2(X,\eta^*_{i}), \nu_{i}^{*}),
    \end{align*}
    for almost surely $(X,Y)$, where we denote $S:=\sum_{j = 1}^{k^*_2} \exp((\beta_{1j}^{*})^{\top} X + \beta_{0j}^{*})$. The above equation implies that
    \begin{align*}
        \{(h_1(X,\kappa_i),\tau_i):i\in U_{\ell}\}&=\{(h_1(X,\kappa^*_i),\tau^*_i):i\in U_{\ell}\}, \quad \forall\ell\in[m_1]\\
        \{(h_2(X,\eta_i),\nu_i):i\in V_{\ell}\}&=\{(h_2(X,\eta^*_i),\nu^*_i):i\in V_{\ell}\}, \quad \forall\ell\in[m_2],
    \end{align*}
    for almost surely $X$. As the expert functions $h_1$ and $h_2$ are identifiable, we deduce 
    \begin{align*}
        \{(\kappa_i,\tau_i):i\in U_{\ell}\}&=\{(\kappa^*_i,\tau^*_i):i\in U_{\ell}\}, \quad \forall\ell\in[m_1]\\
        \{(\eta_i,\nu_i):i\in V_{\ell}\}&=\{(\eta^*_i,\nu^*_i):i\in V_{\ell}\}, \quad \forall\ell\in[m_2].
    \end{align*}
    Therefore, we obtain 
    \begin{align*}
        &G_1=\sum_{\ell=1}^{m_1}\sum_{i\in U_{\ell}}\omega_i\delta_{(\kappa_i,\tau_i)}=\sum_{\ell=1}^{m_1}\sum_{i\in U_{\ell}}\omega^*_i\delta_{(\kappa^*_i,\tau^*_i)}=G^*_1,\\
        &G_2=\sum_{\ell=1}^{m_2}\sum_{i\in V_{\ell}}\exp(\beta_{0i})\delta_{(\beta_{1i},\eta_i,\nu_i)}=\sum_{\ell=1}^{m_2}\sum_{i\in V_{\ell}}\exp(\beta^*_{0i})\delta_{(\beta^*_{1i},\eta^*_i,\nu^*_i)}=G^*_2.
    \end{align*}
    Hence, the proof is completed. 
\end{proof}

\section{Extended Theoretical Results for Sparse Gating MoE}
\label{appendix:sparse_gating}
In this appendix, we extend the convergence analysis of parameter and expert estimations presented in Theoreom~\ref{theorem:strongly_identifiable_experts} to the setting of a Top-$K$ sparse gating function. Our main arguments rely on fundamental techniques for dealing with the sparse gating function proposed in \cite{nguyen2024statistical}. Since the results of Theorems~\ref{theorem:linear_experts}, \ref{theorem:strongly_identifiable_experts_sigmoid}, and \ref{theorem:weakly_identifiable_experts_sigmoid} can be extended in a similar fashion, we will omit their extension here.

\vspace{0.5 em}
\noindent
\textbf{Problem setting:} Assume that $(X_{1}, Y_{1}), \ldots, (X_{n}, Y_{n}) \in \mathbb{R}^{d} \times \mathbb{R}$ are i.i.d. samples drawn from the softmax gating Gaussian mixture of experts of order $k_{*}$ whose conditional density function $s_{G^{*}_1,G^*_2}(y|x)$ is given by:
\begin{align}
    \label{eq:sparse_density}
    s_{G^*_1,G^*_2}(y|x):=\frac{1}{2}&\sum_{i=1}^{k^*_1}\omega^*_{i}\pi(y|h_1(x,\kappa^*_{i}),\tau^*_{i})\nonumber\\
    &+\frac{1}{2}\sum_{i = 1}^{k^*_2} \softmax(\topK((\boi)^{\top}x;\bzi)) \pi(y|h_2(x,\eta^*_{i}), \nu_{i}^{*}), 
\end{align}
where the pair of ground-truth mixing measures $(G^*_1,G^*_2)$ are given by $G^*_1 : = \sum_{i=1}^{k^*_1}\omega^*_i\delta_{(\kappa^*_{i},\tau^*_{i})}$ and $G^*_2:=\sum_{i = 1}^{k^*_2} \exp(\beta_{0i}^{*}) \delta_{(\beta_{1i}^{*}, \eta_{i}^{*}, \nu_{i}^{*})}$. Additionally, for any natural number $k$ and vectors $(v_i)_{i=1}^{k}$ and $(u_i)_{i=1}^{}$ in $\mathbb{R}^{k}$, the $\topK$ sparse function is defined as
\begin{align*}
   \topK(v_i,K;u_i):=\begin{cases}
        v_i +u_{i},\hspace{0.82cm} \text{if } v_i \text{ is in the top } K \text{ elements of } v;\\
        -\infty, \hspace{1.2cm} \text{otherwise},
    \end{cases}
\end{align*}
while the softmax function is formulated as $\softmax(v_i):={\exp(v_i)}/{\sum_{j=1}^{k}\exp(v_j)}$. 

\vspace{0.5 em}
\noindent
In practice, since the number of shared experts $k^*_1$ and routed experts $k^*2$ are typically unknown, we have to fit the ground-truth model~\eqref{eq:sparse_density} with $k_1>k^*_1$ shared experts and $k_2>k^*_2$ routed experts. Thus, some ground-truth shared experts and routed experts will be fitted by more than one estimated expert. As a result, since there are $K$ routed experts activated per input in the ground-truth density $s_{G^*_1,G^*_2}$, it is necessary to activate $\bar{K}>K$ experts in the density estimation in order to ensure its convergence to the true density. For that purpose, let us introduce the formulation of the density estimation as follows:
\begin{align*}
    \bar{s}_{G^n_1,G^n_2}(Y|X):=\frac{1}{2}&\sum_{i=1}^{k^n_1}\omega^n_{i}\pi(y|h_1(x,\kappa^n_{i}),\tau^n_{i})\nonumber\\
    &+\frac{1}{2}\sum_{i = 1}^{k^n_2} \softmax(\topKbar((\boin)^{\top}x;\bzin)) \pi(y|h_2(x,\eta^n_{i}), \nu_{i}^{n}),
\end{align*}
where $K<\bar{K}\leq k_2$ and the pair of mixing measure estimations $(G^n_1,G^n_2)$ are defined as
\begin{align}
    \label{eq:MLE_sparse}
    (\widehat{G}^n_1,\widehat{G}^n_2)\in\argmax_{(G_1,G_2)\in\mathcal{G}_{k_1,k_2}(\Theta)}\frac{1}{n}\sum_{i=1}^{n}\log(\bar{s}_{G_1,G_2}(Y_i|X_i)),
\end{align}
where the set of mixing measures $\mathcal{G}_{k_1,k_2}(\Theta):=\mathcal{G}_{k_1}(\Theta_1)\times\mathcal{G}_{k_2}(\Theta)$ is defined below equation~\eqref{eq:MLE}.

\vspace{0.5 em}
\noindent
\textbf{Input space partition w.r.t the true density.} In order that the density estimation $s_{G^n_1,G^n_2}$ converges to the true density $s_{G^*_1,G^*_2}$, we must ensure that for each input, the $\bar{K}$ routed experts activated in the density estimation converge to the $K$ routed experts activated in the true density. Since the activated experts vary with the input value, we need to partition the input space $\mathcal{X}$ into $M:=\binom{k^*_2}{K}$ regions $\mathcal{X}^*_{m}$ corresponding to $\binom{k^*_2}{K}$ choices of activated experts in the true density. For each $m\in[M]$, let us denote $\{m_1,m_2,\ldots,m_{K}\}$ as an $K$-element subset of the index set $[k^*_2]$, and $\{m_{K+1},\ldots,m_{k^*_2}\}:=[k^*_2]\setminus\{m_1,m_2,\ldots,m_{K}\}$. Then, the $m$-th region of the input space is defined as  
\begin{align*}
    \mathcal{X}^*_{m}:=\Big\{x\in\mathcal{X}:(\beta^*_{1i})^{\top}x\geq(\beta^*_{1i'})^{\top}x, \ \forall i\in\{m_1,m_2,\ldots,m_{K}\},i'\in\{m_{K+1},\ldots,m_{k^*_2}\}\Big\},
\end{align*}
for any $m\in[M]$. For example, suppose that $X\in\mathcal{X}^*_m$ where $m\in[M]$ such that $\{m_1,m_2,\ldots,m_K\}=\{1,2,\ldots,K\}$. Then, it follows that 
\begin{align*}
    \topK((\boi)^{\top}X;\bzi)=(\boi)^{\top}X+\bzi,
\end{align*}
for all $i\in[K]$. In other words, $h_2(X,\eta^*_1),h_2(X,\eta^*_2),\ldots,h_2(X,\eta^*_K)$ are the $K$ routed experts activated in the true density $s_{G^*_1,G^*_2}(y|x)$, which is reduced to
\begin{align}
    \label{eq:sparse_density_activated}
    s_{G^*_1,G^*_2}(y|x):=\frac{1}{2}&\sum_{i=1}^{k^*_1}\omega^*_{i}\pi(y|h_1(x,\kappa^*_{i}),\tau^*_{i})\nonumber\\
    &+\frac{1}{2}\sum_{i = 1}^{K} \frac{\exp((\beta_{1i}^{*})^{\top} x + \beta_{0i}^{*})}{\sum_{j = 1}^{k^*_2} \exp((\beta_{1j}^{*})^{\top} x + \beta_{0j}^{*})}\cdot \pi(y|h_2(x,\eta^*_{i}), \nu_{i}^{*}). 
\end{align}
\textbf{Input space partition w.r.t the density estimation.}
Next, with the same input $X\in\mathcal{X}^*_{m}$, we need to guarantee that the routed expert estimations converging to the above $K$ routed experts activated in the true density $s_{G^*_1,G^*_2}(Y|X)$ are also activated in the density estimation $s_{G^n_1,G^n_2}(Y|X)$. For that purpose, it is necessary to partition the input space with respect to the density estimation. In particular, we partition the input space into $\bar{M}:=\binom{k_2}{\bar{K}}$ regions $\bar{\mathcal{X}}_{m}$ corresponding to $\binom{k_2}{\bar{K}}$ choices of activated experts in the true density. For each $\bar{m}\in[\bar{M}]$, we denote $\{\bar{m}_1,\bar{m}_2,\ldots,\bar{m}_{\bar{K}}\}$ as an $\bar{K}$-element subset of the index set $[k_2]$, and $\{\bar{m}_{\bar{K}+1},\ldots,\bar{m}_{k_2}\}:=[k_2]\setminus\{\bar{m}_1,\bar{m}_2,\ldots,\bar{m}_{\bar{K}}\}$. 
Given these notations, we are ready to show that the input partition w.r.t the density estimation aligns with the input space partition w.r.t the true density in the following lemma whose proof will be provided in Appendix~\ref{appendix:partition_input_space}:
\begin{lemma}
    \label{lemma:partition_input_space}
    For any $j\in[k^*_2]$, $i\in\mathcal{V}_{2,j}$ and $\beta_{1i},\beta^*_{1j}\in\mathbb{R}^{d}$, assume that there exist sufficiently small $\varepsilon_j>0$ satisfying $\|\beta_{1i}-\beta^*_{1j}\|\leq\varepsilon_j$. Moreover, suppose that there exist $m\in[M]$ and $\bar{m}\in[\bar{M}]$ such that $\{\bar{m}_1,\bar{m}_2,\ldots,\bar{m}_{\bar{K}}\}=\mathcal{V}_{2,m_1}\cup\mathcal{V}_{2,m_2}\ldots\cup\mathcal{V}_{2,m_K}$. Then, for any $m\in[M]$, if the input region $\mathcal{X}^*_{m}$ has non-zero measure, we have $\mathcal{X}^*_m=\bar{\mathcal{X}}_{\bar{m}}$, 
    where
    \begin{align*}
    \bar{\mathcal{X}}_{\bar{m}}:=\Big\{x\in\mathcal{X}:(\beta_{1i})^{\top}x\geq(\beta_{1i'})^{\top}x, \ \forall i\in\{\bar{m}_1,\bar{m}_2,\ldots,\bar{m}_{\bar{K}}\},i'\in\{\bar{m}_{K+1},\ldots,\bar{m}_{k_2}\}\Big\}.
    \end{align*}
\end{lemma}
\noindent
Suppose that the expert estimation $h(X,\hat{\eta}^n_{i})$ converges to the ground-truth expert $h(X,\eta^*_j)$ for some $j\in[k^*_2]$ and $i\in\mathcal{V}_{2,j}$. Then, Lemma~\ref{lemma:partition_input_space} reveals that for almost surely $X$, if the expert $h(X,\eta^*_j)$ is activated in the true density, then the expert $h(X,\hat{\eta}^n_i)$ is also activated in the density estimation. Mathematically, we have $\topK((\boj)^{\top}X;\bzj)=(\boj)^{\top}X+\bzj$ occurs holds if and only if $\topKbar((\hat{\beta}^n_{1i})^{\top}X;\hat{\beta}^n_{0i})=(\hat{\beta}^n_{1i})^{\top}X+\hat{\beta}^n_{0i}$.

\vspace{0.5 em}
\noindent
\textbf{Density estimation convergence.} Given the above input partition w.r.t the density estimation, we exhibit in Proposition~\ref{prop:K_bar_bound} an interesting phenomenon that the density estimation $\bar{s}_{\widehat{G}^n_1,\widehat{G}^n_2}$ converges to the true density $s_{G^*_1,G^*_2}$ under the Total Variation distance only if the number of routed experts activated in the density estimation is bounded below as $\bar{K}\geq\max_{\{m_1,m_2,\ldots,m_K\}\subset[k^*_2]}\sum_{j=1}^{K}|\mathcal{V}_{2,m_j}|$. 
\begin{proposition}
    \label{prop:K_bar_bound}
    If $\bar{K}<\max_{\{m_1,m_2,\ldots,m_K\}\subset[k^*_2]}\sum_{j=1}^{K}|\mathcal{V}_{2,m_j}|$, then the following holds:
    \begin{align*}
        \inf_{(G_1,G_2)\in\mathcal{G}_{k_1,k_2}(\Theta)}\mathbb{E}_X[V(\bar{s}_{G_1,G_2}(\cdot|X),s_{G^*_1,G^*_2}(\cdot|X))]>0.
    \end{align*}
\end{proposition}
\noindent
Proof of Proposition~\ref{prop:K_bar_bound} will be provided in Appendix~\ref{appendix:K_bar_bound}. Following from the result of Proposition~\ref{prop:K_bar_bound}, we will assume $\max_{\{m_1,m_2,\ldots,m_K\}\subset[k^*_2]}\sum_{j=1}^{K}|\mathcal{V}_{2,m_j}|\leq \bar{K}\leq k_2$ in the rest of this appendix unless stating otherwise to ensure the convergence of density estimation. Next, by combining the above results and the arguments used to prove Proposition~\ref{prop:density_rate}, we arrive at the following density estimation rate.
 \begin{proposition}
    \label{prop:sparse_density_rate}
    The density estimation $\bar{s}_{\widehat{G}^n_1,\widehat{G}^n_2}(Y|X)$ converges to the true density $s_{G^*_1,G^*_2}(Y|X)$ at the following rate:
    \begin{align*}
        \bbE_X[V(\bar{s}_{\widehat{G}^n_1,\widehat{G}^n_2}(\cdot|X),s_{G^*_1,G^*_2}(\cdot|X))]=\mathcal{O}_P([\log(n)/n]^{\frac{1}{2}}).
    \end{align*}
\end{proposition}
\noindent
\textbf{Voronoi loss.} In align with the above input partition w.r.t the true density, we need to modify the formulation of the Voronoi loss previously defined in equation~\eqref{eq:loss_1} as follows:
\begin{align*}
    &\mathcal{D}_{5}((G_1,G_2),(G^*_1,G^*_2)):=\max_{\{m_1,\ldots,m_K\}\subset[k^*_2]}\Bigg\{\sum_{j=1}^{k^*_1}\Big|\sum_{i\in\mathcal{V}_{1,j}}\omega_{i}-\oj\Big|+\sum_{j=1}^{K}\Big|\sum_{i\in\mathcal{V}_{2,m_j}}\exp(\beta_{0i})-\exp(\bzj)\Big|\nonumber\\
    &+\sum_{\substack{j\in[k^*_1],\\|\mathcal{V}_{1,j}|=1}}\sum_{i\in\mathcal{V}_{1,j}}\omega_{i}(\|\Delta\kappa_{ij}\|+|\Delta\tau_{ij}|)+\sum_{\substack{j\in[K],\\ |\mathcal{V}_{2,m_j}|=1}}\sum_{i\in\mathcal{V}_{2,m_j}}\exp(\beta_{0i})(\|\Delta\beta_{1im_j}\|+\|\Delta\eta_{im_j}\|+|\Delta\nu_{im_j}|)\nonumber\\
    &+\sum_{\substack{j\in[k^*_1],\\|\mathcal{V}_{1,j}|>1}}\sum_{i\in\mathcal{V}_{1,j}}\omega_{i}(\|\Delta\kappa_{ij}\|^2+|\Delta\tau_{ij}|^2)+\sum_{\substack{j\in[K],\\ |\mathcal{V}_{2,m_j}|>1}}\sum_{i\in\mathcal{V}_{2,m_j}}\exp(\beta_{0i})(\|\Delta\beta_{1im_j}\|^2+\|\Delta\eta_{im_j}\|^2+|\Delta\nu_{im_j}|^2)\Bigg\}.
\end{align*}
The maximum operator in the above formulation helps capture the convergence behavior of the parameter estimation in different input regions partitioned w.r.t the true density. Given the loss function $\mathcal{D}_5(G,G_*)$, it is sufficient to establish parameter and expert estimation rates in the following theorem:
\begin{theorem}
    \label{theorem:strongly_identifiable_experts_sparse}
    Suppose that the expert functions $h_1$ and $h_2$ are strongly identifiable. Then, the lower bound $\bbE_X[V(\bar{s}_{G_1,G_2}(\cdot|X),s_{G^*_1,G^*_2}(\cdot|X))]\gtrsim \mathcal{D}_5((G_1,G_2),(G^*_1,G^*_2))$ holds for any $(G_1,G_2)\in\mathcal{G}_{k_1,k_2}(\Theta)$. As a consequence, we have 
    \begin{align*}
        \mathcal{D}_5(\widehat{G}^n_1,\widehat{G}^n_2),(G^*_1,G^*_2))=\mathcal{O}_P([\log(n)/n]^{\frac{1}{2}}).
    \end{align*}
\end{theorem}
\begin{proof}[Proof of Theorem~\ref{theorem:strongly_identifiable_experts_sparse}]
    Analogous to Appendix~\ref{appendix:strongly_identifiable_experts}, it suffices to derive the local part
\begin{align}
    \label{eq:local_part_sparse}
    \lim_{\varepsilon\to0}\inf_{(G_1,G_2)\in\mathcal{G}_{k_1,k_2}(\Theta):\mathcal{D}_5((G_1,G_2),(G^*_1,G^*_2))\leq\varepsilon}\dfrac{\bbE_X[V(\bar{s}_{G_1,G_2}(\cdot|X),s_{G^*_1,G^*_2}(\cdot|X))]}{\mathcal{D}_5((G_1,G_2),(G^*_1,G^*_2))}>0,
\end{align}
and the global part
\begin{align}
    \label{eq:global_part_sparse}
    \inf_{(G_1,G_2)\in\mathcal{G}_{k_1,k_2}(\Theta):\mathcal{D}_5((G_1,G_2),(G^*_1,G^*_2))>\varepsilon'}\dfrac{\bbE_X[V(\bar{s}_{G_1,G_2}(\cdot|X),s_{G^*_1,G^*_2}(\cdot|X))]}{\mathcal{D}_5((G_1,G_2),(G^*_1,G^*_2))}>0.
\end{align}
in this appendix. However, since the global part~\eqref{eq:global_part_linear} can be established in the same fashion as in Appendix~\ref{appendix:strongly_identifiable_experts}, its proof is omitted here. Thus, we will focus on showing only the local part~\eqref{eq:local_part_sparse}. Suppose that the local part does not hold. Then, we can find a sequence of mixing measure pairs $(G^n_1,G^n_2)$ of the form $G^n_1:=\sum_{i=1}^{k^n_1}\oin\delta_{(\koin,\kzin,\tin)}$, $G^n_2:=\sum_{i=1}^{k^n_2}\exp(\bzin)\delta_{(\boin,\eoin,\ezin,\nuin)}$ for $n\in\mathbb{N}$ satisfying $\mathcal{D}_{5n}:=\mathcal{D}_{5}((G^n_1,G^n_2),(G^*_1,G^*_2))\to0$ and
\begin{align}
    \label{eq:expectation_zero_sparse}
    \bbE_X[V(\bar{s}_{G^n_1,G^n_2}(\cdot|X),s_{G^*_1,G^*_2}(\cdot|X))]/\mathcal{D}_{5n}\to0,
\end{align}
as $n\to\infty$. Here, we may assume WLOG that the number of shared experts and routed experts $k^n_1$, $k^n_2$ and Voronoi cells $\mathcal{V}_{1,j}=\mathcal{V}_{1,j}(G^n_1)$, $\mathcal{V}_{2,j}=\mathcal{V}_{2,j}(G^n_2)$ do not change with the sample size $n$. WLOG, we may assume that the Voronoi loss $\mathcal{D}_{5n}$ is reduced to
\begin{align}
    \label{eq:loss_5n}
    &\mathcal{D}_{5n}=\sum_{j=1}^{k^*_1}\Big|\sum_{i\in\mathcal{V}_{1,j}}\oin-\oj\Big|+\sum_{j=1}^{K}\Big|\sum_{i\in\mathcal{V}_{2,j}}\exp(\bzin)-\exp(\bzj)\Big|\nonumber\\
    &+\sum_{\substack{j\in[k^*_1],\\|\mathcal{V}_{1,j}|=1}}\sum_{i\in\mathcal{V}_{1,j}}\oin(\|\dkijn\|+|\dtijn|)+\sum_{\substack{j\in[K],\\|\mathcal{V}_{2,j}|=1}}\sum_{i\in\mathcal{V}_{2,j}}\exp(\bzin)(\|\dboijn\|+\|\deijn\|+|\dnuijn|)\nonumber\\
    &+\sum_{\substack{j\in[k^*_1],\\|\mathcal{V}_{1,j}|>1}}\sum_{i\in\mathcal{V}_{1,j}}\oin(\|\dkijn\|^2+|\dtijn|^2)+\sum_{\substack{j\in[K],\\|\mathcal{V}_{2,j}|>1}}\sum_{i\in\mathcal{V}_{2,j}}\exp(\bzin)(\|\dboijn\|^2+\|\deijn\|^2+|\dnuijn|^2).
\end{align}
Recall that we partition the input space w.r.t the true density into $M=\binom{k^*_2}{K}$ regions. For each $m\in[M]$, we denote $\{m_1,m_2,\ldots,m_K\}$ as a subset of the index set $[k^*_2]$ and $\{m_{K+1},\ldots,m_{k^*_2}\}=[k^*_2]\setminus\{m_1,m_2,\ldots,m_K\}$. Then, the $m$-th region is given by
\begin{align*}
    \mathcal{X}^*_{m}:=\Big\{x\in\mathcal{X}:(\beta^*_{1i})^{\top}x\geq(\beta^*_{1i'})^{\top}x, \ \forall i\in\{m_1,m_2,\ldots,m_{K}\},i'\in\{m_{K+1},\ldots,m_{k^*_2}\}\Big\},
\end{align*}
for any $m\in[M]$. Let $\bar{K}\in\mathbb{N}$ such that $\max\{\sum_{j=1}^{K}|\mathcal{V}_{2,j}|:\{m_1,\ldots,m_K\}\subset[k^*_2]\}\leq\bar{K}\leq k_2$ and let $\bar{M}:=\binom{k_2}{\bar{K}}$. Next, for any $\bar{m}\in[\bar{M}]$, we denote $\{\bar{m}_1,\bar{m}_2,\ldots,\bar{m}_{\bar{K}}\}$ as a subset of the index set $[k_2]$ and $\{\bar{m}_{\bar{K}+1},\ldots,\bar{m}_{k_2}\}:=[k_2]\setminus\{\bar{m}_1,\bar{m}_2,\ldots,\bar{m}_{\bar{K}}\}$. Then, we partition the input space w.r.t the density estimation $s_{G^n_1,G^n_2}(Y|X)$ as $\mathcal{X}=\cup_{\bar{m}=1}^{\bar{M}}\mathcal{X}^n_{\bar{m}}$, where the $\bar{m}$-th region is defined as
\begin{align*}
    \mathcal{X}^n_{\bar{m}}:=\Big\{x\in\mathcal{X}:(\beta^n_{1i})^{\top}x\geq(\beta^n_{1i'})^{\top}x, \ \forall i\in\{\bar{m}_1,\bar{m}_2,\ldots,\bar{m}_{\bar{K}}\},i'\in\{\bar{m}_{\bar{K}+1},\ldots,\bar{m}_{k_2}\}\Big\}
\end{align*}
for any $\bar{m}\in[\bar{M}]$. Let $X\mathcal{X}^*_{m}$ for $m\in[M]$ such that $\{m_1,m_2,\ldots,m_K\}=\{1,2,\ldots,K\}$. If there does not exist $\bar{m}\in[\bar{M}]$ such that $\{\bar{m}_1,\bar{m}_2,\ldots,\bar{m}_{\bar{K}}\}=\mathcal{V}_{2,1}\cup\mathcal{V}_{2,2}\cup\ldots\cup\mathcal{V}_{2,K}$, then the ratio $\bbE_X[V(\bar{s}_{G^n_1,G^n_2}(\cdot|X),s_{G^*_1,G^*_2}(\cdot|X))]/\mathcal{D}_{5n}$ does not converge to zero, which contradicts the result in equation~\eqref{eq:expectation_zero_sparse}. Thus, we can find $\bar{m}\in[\bar{M}]$ such that $\{\bar{m}_1,\bar{m}_2,\ldots,\bar{m}_{\bar{K}}\}=\mathcal{V}_{2,1}\cup\mathcal{V}_{2,2}\cup\ldots\cup\mathcal{V}_{2,K}$.

\vspace{0.5 em}
\noindent
Since the Voronoi loss $\mathcal{D}_{5n}$ converges to zero, it follows that $\boin\to\boj$ for all $j\in[K]$ and $i\in\mathcal{V}_{2,j}$. Then, by means of Lemma~\ref{lemma:partition_input_space}, we deduce $\mathcal{X}^*_{m}=\mathcal{X}^n_{\bar{m}}$ for sufficiently large $n$, implying that $X\in\mathcal{X}^n_{\bar{m}}$. Therefore, we can represent the true density and the density estimation when the sample size $n$ is sufficiently large as follows:
\begin{align*}
    s_{G^*_1,G^*_2}(y|x):=\frac{1}{2}&\sum_{i=1}^{k^*_1}\omega^*_{i}\pi(y|h_1(x,\kappa^*_{i}),\tau^*_{i})+\frac{1}{2}\sum_{i = 1}^{K} \frac{\exp((\beta_{1i}^{*})^{\top} x + \beta_{0i}^{*})}{\sum_{j = 1}^{k^*_2} \exp((\beta_{1j}^{*})^{\top} x + \beta_{0j}^{*})}\cdot \pi(y|h_2(x,\eta^*_{i}), \nu_{i}^{*}),\\
    s_{G^n_1,G^n_2}(y|x):=\frac{1}{2}&\sum_{i=1}^{k^n_1}\omega^*_{i}\pi(y|h_1(x,\kappa^n_{i}),\tau^n_{i})+\frac{1}{2}\sum_{i = 1}^{\bar{K}} \frac{\exp((\beta_{1i}^{n})^{\top} x + \beta_{0i}^{n})}{\sum_{j = 1}^{\bar{K}} \exp((\beta_{1j}^{n})^{\top} x + \beta_{0j}^{n})}\cdot \pi(y|h_2(x,\eta^n_{i}), \nu_{i}^{n}).
\end{align*}
Given the above formulations, we can achieve the local part~\eqref{eq:local_part_sparse} by employing the same arguments used in Appendix~\ref{appendix:strongly_identifiable_experts}. Hence, the proof is completed.
\end{proof}

\subsection{Proof of Lemma~\ref{lemma:partition_input_space}}
\label{appendix:partition_input_space}
Let us consider $\varepsilon_j=N_j\eta$, where $\eta>0$ is some fixed constant, and $N_j>0$ will be chosen later. Since the input space $\mathcal{X}$ and the parameter space $\Theta$ are bounded, there exists a constant $c^*_{m}\geq 0$ such that
\begin{align}
    \label{eq:lowest_dist}
    \min_{x,j,j'}\Big[(\beta^*_{1j})^{\top}x-(\beta^*_{1j'})^{\top}x\Big]=c^*_{m}\eta,
\end{align}
where the above minimum is subject to $x\in\mathcal{X}^*_{m},j\in\{m_1,m_2,\ldots,m_K\}$ and $j'\in\{m_{K+1},\ldots,m_{k^*_2}\}$. We will show by contradiction that $c^*_{m}>0$. Suppose that $c^*_{m}=0$. For $x\in\mathcal{X}^*_{m}$, we may assume for any $1\leq i<j\leq k^*_2$ that
\begin{align*}
    (\beta^*_{1m_i})^{\top}x\geq (\beta^*_{1m_j})^{\top}x.
\end{align*}
As $c^*_{m}=0$, the result in equation~\eqref{eq:lowest_dist} indicates that $(\beta^*_{1m_{K}})^{\top}x-(\beta^*_{1m_{K+1}})^{\top}x=0$, or equivalently
\begin{align*}
    (\beta^*_{1m_K}-\beta^*_{1m_{K+1}})^{\top}x=0.
\end{align*}
In other words, $\mathcal{X}^*_{m}$ is a subset of
\begin{align*}
    \mathcal{N}:=\{x\in\mathcal{X}:(\beta^*_{1m_K}-\beta^*_{1m_{K+1}})^{\top}x=0\}.
\end{align*}
Since the difference $\beta_{1m_K}-\beta_{1m_{K+1}}$ is non-zero and the input $X$ follows a continuous distribution, then the set $\mathcal{N}$ has measure zero. Furthermore, as $\mathcal{X}^*_{m}\subseteq \mathcal{N}$, it follows that $\mathcal{X}^*_{m}$ also has measure zero, which contradicts the fact that it has non-zero measure. Thus, we must have $c^*_{m}>0$.

\vspace{0.5 em}
\noindent
Subsequently, let $x\in\mathcal{X}^*_{m}$ and $\bar{m}\in[\bar{M}]$ such that $\{\bar{m}_{1},\bar{m}_{2},\ldots,\bar{m}_{\bar{K}}\}=\mathcal{V}_{2,m_1}\cup\mathcal{V}_{2,m_2}\cup\ldots\cup\mathcal{V}_{2,m_K}$. We will demonstrate that $x\in\bar{\mathcal{X}}_{\bar{m}}$. Indeed, recall that the input space $\mathcal{X}$ is bounded, then we may assume that $\|x\|\leq B$ for any $x\in\mathcal{X}$, where $B>0$ is some constant. Then, for any $i\in\{\bar{m}_1,\bar{m}_2,\ldots,\bar{m}_{\bar{K}}\}$ and $i'\in\{\bar{m}_{\bar{K}+1},\ldots,\bar{m}_{k_2}\}$, we have
\begin{align*}
    {\beta}_{1i}^{\top}x&=({\beta}_{1i}-\beta^*_{1j})^{\top}x+(\beta^*_{1j})^{\top}x\\
    &\geq -N_j\eta B+(\beta^*_{1j'})^{\top}x+c^*_{m}\eta\\
    &=-N_j\eta B+c^*_{m}\eta+(\beta^*_{1j'}-{\beta}_{1i'})^{\top}x+{\beta}_{1i'}^{\top}x\\
    &\geq -2N_j\eta B+c^*_{m}\eta+{\beta}_{1i'}^{\top}x,
\end{align*}
where $j\in\{m_1,m_2,\ldots,m_K\}$ and $j'\in\{m_{K+1},\ldots,m_{k^*_2}\}$ such that $i\in\mathcal{V}_{2,j}$ and $i'\in\mathcal{V}_{2,j'}$. Note that if $N_j\leq\dfrac{c^*_{m}}{2B}$, then we obtain $x\in{\mathcal{X}}_{\bar{m}}$, which implies that $\mathcal{X}^*_{m}\subseteq\bar{\mathcal{X}}_{\bar{m}}$. 

\vspace{0.5 em}
\noindent
Analogously, assume that there exists some constant $c_{m}\geq 0$ such that
\begin{align*}
    \min_{x,j,j'}\Big[(\beta^*_{1j})^{\top}x-(\beta^*_{1j'})^{\top}x\Big]=c^*_{m}\eta,
\end{align*}
where the above minimum is subject to $x\in\bar{\mathcal{X}}_{\bar{m}}$, $i\in\{\bar{m}_1,\bar{m}_2,\ldots,\bar{m}_{\bar{K}}\}$ and $i'\in\{\bar{m}_{\bar{K}+1},\ldots,\bar{m}_{k}\}$. Then, if $N_j\leq \dfrac{c_{m}}{2B}$, we have $\bar{\mathcal{X}}_{\bar{m}}\subseteq\mathcal{X}^*_{m}$. Consequently, by setting $N_j=\dfrac{1}{2B}\min\{c^*_{m},c_{m}\}$, we reach the conclusion that $\bar{\mathcal{X}}_{\bar{m}}=\mathcal{X}^*_{m}$. Hence, the proof is completed.

\subsection{Proof of Proposition~\ref{prop:K_bar_bound}}
\label{appendix:K_bar_bound}
To begin with, we show that
\begin{align}
    \label{eq:prop_local}
    \lim_{\varepsilon\to0}\inf_{(G_1,G_2)\in\mathcal{G}_{k_1,k_2}(\Theta):\mathcal{D}_5((G_1,G_2),(G^*_1,G^*_2))\leq\varepsilon}\bbE_X[V(\bar{s}_{G_1,G_2}(\cdot|X),s_{G^*_1,G^*_2}(\cdot|X))]>0.
\end{align}
Suppose that the above inequality does not hold, then there exist a sequence of pairs of mixing measures $(G^n_1,G^n_2)$ in $\mathcal{G}_{k_1,k_2}(\Theta)$ given by $G^n_1=\sum_{i=1}^{k^n_1}\oin\delta_{(\kin,\tin)}$ and  $G^n_2=\sum_{i=1}^{k^n_2}\exp(\beta^n_{0i})\delta_{(\beta^n_{1i},\eta^n_i,\nu^n_i)}$ that satisfies $\mathcal{D}_5((G^n_1,G^n_2),(G^*_1,G^*_2))\to0$ and
\begin{align*}
    \bbE_X[V(\bar{s}_{G^n_1,G^n_2}(\cdot|X),s_{G^*_1,G^*_2}(\cdot|X))]\to0
\end{align*}
as $n\to\infty$. According to the Fatou's lemma, we have
\begin{align}
    \label{eq:prop_fatou}
    0&=\lim_{n\to\infty}\bbE_X[V(\bar{s}_{G^n_1,G^n_2}(\cdot|X),s_{G^*_1,G^*_2}(\cdot|X))]\nonumber\\
    &\geq \frac{1}{2}\int_{\mathcal{X}\times\mathcal{Y}}\liminf_{n\to\infty}|\bar{s}_{G^n_1,G^n_2}(Y|X)-s_{G^*_1,G^*_2}(Y|X)|\dint(X,Y),
\end{align}
implying that $\bar{s}_{G^n_1,G^n_2}(Y|X)-s_{G^*_1,G^*_2}(Y|X)\to0$ as $n\to\infty$ for almost surely $(X,Y)$. WLOG, we may assume that 
\begin{align*}
    \max_{\{m_1,m_2,\ldots,m_K\}}\sum_{j=1}^{K}|\mathcal{V}_{2,m_j}|=|\mathcal{V}_{2,1}|+|\mathcal{V}_{2,2}|+\ldots+|\mathcal{V}_{2,K}|.
\end{align*}
Let $X\in\mathcal{X}^*_{m}$, where $m\in[M]$ such that $\{m_1,m_2,\ldots,m_K\}=\{1,2,\ldots,K\}$. Since the Voronoi loss $\mathcal{D}_5((G^n_1,G^n_2),(G^*_1,G^*_2))$ goes to zero, it follows that $\beta^n_{1i}\to\beta^*_{1j}$ as $n\to\infty$ for any $j\in[k^*_2]$ and $i\in\mathcal{V}_{2,j}$. By means of Lemma~\ref{lemma:partition_input_space}, we deduce $X\in\bar{\mathcal{X}}_{\bar{m}}$, where $\bar{m}\in[\bar{q}]$ such that $\{\bar{m}_1,\bar{m}_2,\ldots,\bar{m}_{\bar{K}}\}=\mathcal{V}_{2,1}\cup\mathcal{V}_{2,2}\cup\ldots\cup\mathcal{V}_{2,K}$. 
However, as $\bar{K}<\sum_{j=1}^{K}|\mathcal{V}_{2,j}|$, the fact that $\{\bar{m}_1,\bar{m}_2,\ldots,\bar{m}_{\bar{K}}\}=\mathcal{V}_{2,1}\cup\mathcal{V}_{2,2}\cup\ldots\cup\mathcal{V}_{2,K}$ cannot occur. Thus, we obtain the result in equation~\eqref{eq:prop_local}. As a consequence, we can find a positive constant $\varepsilon'$ such that
\begin{align*}
    \inf_{(G_1,G_2)\in\mathcal{G}_{k_1,k_2}(\Theta):\mathcal{D}_5((G_1,G_2),(G^*_1,G^*_2))\leq\varepsilon'}\bbE_X[V(\bar{s}_{G_1,G_2}(\cdot|X),s_{G^*_1,G^*_2}(\cdot|X))]>0.
\end{align*}
Given the above result, it is sufficient to show that
\begin{align}
    \label{eq:prop_global}
    \inf_{(G_1,G_2)\in\mathcal{G}_{k_1,k_2}(\Theta):\mathcal{D}_5((G_1,G_2),(G^*_1,G^*_2))>\varepsilon'}\bbE_X[V(\bar{s}_{G_1,G_2}(\cdot|X),s_{G^*_1,G^*_2}(\cdot|X))]>0.
\end{align}
Assume by contrary that the inequality~\eqref{eq:prop_global} does not hold, then we can find a sequence $(\tilde{G}^n_1,\tilde{G}^n_2)\in\mathcal{G}_{k_1,k_2}(\Theta)$ such that $\mathcal{D}_5((\tilde{G}^n_1,\tilde{G}^n_2),(G^*_1,G^*_2))>\varepsilon'$ and
\begin{align*}
    \bbE_X[V(\bar{s}_{\tilde{G}^n_1,\tilde{G}^n_2}(\cdot|X),s_{G^*_1,G^*_2}(\cdot|X))]\to0.
\end{align*}
Again, by utilizing the Fatou's lemma as in equation~\eqref{eq:prop_fatou}, we get $\bar{s}_{\tilde{G}^n_1,\tilde{G}^n_2}(Y|X)-s_{G^*_1,G^*_2}(Y|X)\to 0$ as $n\to\infty$ for almost surely $(X,Y)$. Since the parameter space $\Theta$ is compact, we can substitute the sequence $(\tilde{G}^n_1,\tilde{G}^n_2)$ with its subsequence which converges to some pair of mixing measures $(\tilde{G}_1,\tilde{G}_2)$ in $\mathcal{G}_{k_1,k_2}(\Theta)$. This result leads to $\bar{s}_{\tilde{G}_1,\tilde{G}_2}(Y|X)=s_{G^*_1,G^*_2}(Y|X)$ for almost surely $(X,Y)$. As the Top-$K$ sparse gating MoE is identifiable, we deduce $(\tilde{G}_1,\tilde{G}_2)\equiv (G^*_1,G^*_2)$, or equivalently, $\mathcal{D}_5((\tilde{G}_1,\tilde{G}_2),(G^*_1,G^*_2))=0$.
On the other hand, due to the fact that $\mathcal{D}_5((\tilde{G}^n_1,\tilde{G}^n_2),(G^*_1,G^*_2))>\varepsilon'$ for any $n\in\mathbb{N}$, we obtain $\mathcal{D}_5((\tilde{G}_1,\tilde{G}_2),(G^*_1,G^*_2))>\varepsilon'>0$, which contradicts the previous result. Hence, we reach the result in equation~\eqref{eq:prop_global} and complete the proof.

\section{Experimental Details}  \label{appendix:experiment_settings}

\subsection{Language Modeling}  \label{appx:lm_setting}

\subsubsection{Datasets}


\noindent
\textbf{SlimPajama.} The SlimPajama \cite{cerebras2023slimpajama} dataset is a filtered and deduplicated corpus of the 1.2T token RedPajama dataset \cite{weber2024redpajama} designed for language model pretraining. It contains around 627B tokens across diverse sources. 

\vspace{0.5 em}
\noindent
\textbf{LAMBADA.} The LAMBADA \cite{paperno2016lambada} dataset evaluates a model's ability to predict the final word of a passage, requiring understanding of broad discourse context. Each instance comprises a narrative where the target word is only predictable when considering the entire passage, challenging models to perform deep contextual comprehension beyond sentence-level cues 

\vspace{0.5 em}
\noindent
\textbf{BLiMP.}  The Benchmark of Linguistic Minimal Pairs (BLiMP) \cite{warstadt2020blimp} assesses language models' grasp of English grammar through 67 sub-datasets, each containing 1,000 minimal pairs. These pairs differ subtly to test specific syntactic, morphological, or semantic phenomena, enabling fine-grained evaluation of linguistic competence 

\vspace{0.5 em}
\noindent
\textbf{Children's Book Test (CBT).} CBT \cite{hill2015goldilocks} measures a model's ability to utilize wider linguistic context by providing passages from children's books with a missing word to predict. The dataset distinguishes between predicting syntactic function words and semantically rich content words, emphasizing the importance of context in language understanding

\vspace{0.5 em}
\noindent
\textbf{HellaSwag.} HellaSwag \cite{zellers2019hellaswag} challenges models with sentence completion tasks that require commonsense reasoning. Each instance presents a context and multiple plausible continuations, with only one being correct. The dataset is adversarially filtered to be trivial for humans but difficult for models, highlighting gaps in machine commonsense understanding.

\vspace{0.5 em}
\noindent
\textbf{PIQA.} The Physical Interaction Question Answering (PIQA) \cite{bisk2020piqa} dataset tests models on physical commonsense reasoning. It comprises questions about everyday tasks, requiring knowledge of physical properties and affordances, challenging models to reason about the physical world without direct sensory experience.

\vspace{0.5 em}
\noindent
\textbf{ARC-Challenge.} The AI2 Reasoning Challenge (ARC) \cite{clark2018think} presents grade-school level multiple-choice science questions that necessitate reasoning and external knowledge. The Challenge set includes questions that are particularly difficult for models, serving as a benchmark for advanced question-answering capabilities .

\vspace{0.5 em}
\noindent
\textbf{OpenBookQA.} OpenBookQA \cite{lai-etal-2017-race} consists of multiple-choice questions derived from a curated set of science facts, resembling open-book exams. Answering requires combining the provided facts with external commonsense knowledge, testing a model's ability to integrate information from multiple sources.

\vspace{0.5 em}
\noindent
\textbf{RACE.} The Reading Comprehension Dataset from Examinations (RACE) \cite{sap2019socialiqa} contains passages and questions from English exams for Chinese middle and high school students. With nearly 100,000 questions, it evaluates a model's reading comprehension and reasoning skills across diverse topics.

\vspace{0.5 em}
\noindent
\textbf{SIQA.} Social IQa (SIQA) \cite{sap2019socialiqa} focuses on social commonsense reasoning, presenting questions about everyday social interactions. Models must infer motivations, reactions, and social dynamics, challenging their understanding of human social behavior.

\vspace{0.5 em}
\noindent
\textbf{CommonSenseQA.} CommonSenseQA \cite{talmor2018commonsenseqa} is a multiple-choice question-answering dataset that requires models to apply commonsense knowledge. Each question is designed to probe a specific aspect of commonsense reasoning, with distractor answers carefully crafted to be plausible yet incorrect.

\subsubsection{Model Settings, Training Settings and Evaluation}

\begin{table}[htbp]
  \centering
  \caption{Comprehensive Model Configurations for Experimental Evaluation. SMoE refers to settings applied for both Vanilla SMoE and SMoE Sigmoid Gating, whereas DeepSeek corresponds to configurations used for DeepSeek-V2 and DeepSeek-V3 models.}
  \label{tab:lm_model_setting}
  \setlength{\tabcolsep}{4pt}
  \renewcommand{\arraystretch}{1.2}

  \resizebox{\textwidth}{!}{%
    \begin{tabular}{*{14}{c}}
      \toprule
     \makecell{Scale} & \makecell{Model} & \makecell{\# params} & \makecell{\# act.\\params} & \makecell{\# trained \\ tokens} & \makecell{$d_{\mathrm{model}}$} & \makecell{H} & \makecell{$d_{\mathrm{head}}$} & \makecell{$N_E$} & \makecell{$K_r$} & \makecell{$N_s$} & \makecell{Expert \\ dim} & \makecell{\textbf{$N_{\mathrm{warmup}}$}} & \makecell{\textbf{$\kappa$}} \\
     \midrule
      
      \multirow{2}{*}{\textbf{Small}}
            & \textbf{SMoE}      & 158M & 36M  & 6.5B & 512  & 4 & 82  & 66  & 8  & 0 & 128  & 0    & 0.1  \\
            & \textbf{DeepSeek}  & 158M & 36M  & 6.5B & 512  & 4 & 82  & 64  & 6  & 2 & 128  & 0    & 0.1  \\ 
      \addlinespace
      \multirow{2}{*}{\textbf{Large}}
            & \textbf{SMoE}      & 679M & 131M & 26.2B  & 1024 & 4 & 128 & 66  & 8  & 0 & 256  & 4000 & 0.25 \\
            & \textbf{DeepSeek}  & 679M & 131M & 26.2B  & 1024 & 4 & 128 & 64  & 6  & 2 & 256  & 4000 & 0.25 \\
      \bottomrule
      \end{tabular}%
  }
\end{table}


\noindent
\textbf{Training datasets.} 
We conduct the experiments on language modeling using the popular SLimPajama \cite{cerebras2023slimpajama} dataset.
Due to the limited computational resource, we utilize only subsets of the SlimPajama \cite{cerebras2023slimpajama} dataset containing 6.5B and 26.2B tokens to train our 158M and 679M parameter models, respectively.

\vspace{0.5 em}
\noindent
\textbf{Model Settings.} Table~\ref{tab:lm_model_setting} summarizes the comprehensive set of hyperparameters and configurations for both scales and the two model variants evaluated in our experiments. All models employ a total of $N_r=66$ experts. For routing schemes, the baseline SMoE utilizes a top-8 expert routing strategy ($K_r = 8$), while the DeepSeek variants adopt a mixed routing approach comprising top-6 expert selection ($K_r = 6$) plus $N_s=2$ shared experts. To align with the fine-grained expert segmentation proposed in DeepSeekMoE~\cite{dai2024deepseekmoe}, we set the expert dimensionality to $1 / 4 \ d_{model}$ and increase the expert count to 66 instead of the common settings with 16 experts. Additionally, the number of attention heads is uniformly set to $H=4$ across both model scales. All models leverage Rotary Positional Embedding (RoPE)~\cite{su2024roformer}, PyTorch’s optimized attention implementation, and employ pre-layernorm Transformers.  To ensure balanced expert utilization, we use the standard load balancing loss defined in Switch Transformers \cite{fedus2022switch}.


\vspace{0.5 em}
\noindent
\textbf{Training Settings.} All models are trained in PyTorch using a batch size of $64$, context length of $1024$, and a learning rate of $2.5e-4$. We apply $4000$ linear warm-up steps specifically for the larger-scale models and utilize the AdamW optimizer~\cite{loshchilov2017decoupled} with its default hyperparameters and a weight decay of $0.01$. Gradient clipping is performed with threshold $\kappa$, and the precise number of linear warm-up steps ($N_{warmup}$) per model variant is provided in Table~\ref{tab:lm_model_setting}. We tokenize the input using SentencePiece~\cite{kudo2018sentencepiece}, configured with a vocabulary size of 8000 tokens, which is trained on a representative subset of the SlimPajama dataset~\cite{cerebras2023slimpajama}. 
This choice was intentionally made to align the vocabulary size with the limited representational capacity of our model, ensuring suitability and computational efficiency within our experimental constraints.

\vspace{0.5 em}
\noindent
\textbf{Evaluation.} We evaluate our model with the Perplexity score (PPL) and zero-shot performance with nine different downstream tasks: LAMBADA \cite{paperno2016lambada}, BLiMP \cite{warstadt2020blimp}, Children’s Book Test \cite{hill2015goldilocks}, HellaSwag \cite{zellers2019hellaswag}, PIQA\cite{bisk2020piqa},  ARC-Challenge \cite{clark2018think}, RACE \cite{lai-etal-2017-race}, SIQA \cite{sap2019socialiqa} and CommonSenseQA \cite{talmor2018commonsenseqa}. For LAMBADA, we use the detokenized version from OpenAI, and we evaluate the top-1 accuracy of the last word (it can span multiple tokens; here we use greedy decoding). For CBT, BLiMP, and RACE, we measure the accuracy of each task and report the average accuracy of the tasks.

\vspace{0.5 em}
\noindent
\textbf{Compute Resource.} All models are trained and evaluated on a single node equipped with 4 NVIDIA A100 80GB CoWoS HBM2e PCIe 4.0 employing data-parallelism.
\subsection{Vision Language Modeling} \label{appx:vlm_setting}


\subsubsection{Datasets}

\textbf{LLaVA-558K.} The LLaVA 558K \cite{liu2023llava} dataset is a curated subset of 558,000 image-text pairs derived from the LAION/CC/SBU dataset. It is designed for the pretraining stage of visual instruction tuning, facilitating the alignment between visual and language modalities. This dataset includes BLIP-generated captions and synthetic multimodal conversations, serving as a foundational resource for training models like LLaVA towards enhanced vision-language capabilities.

\vspace{0.5 em}
\noindent
\textbf{ALLaVA.} ALLaVA \cite{chen2024allava} is a large-scale synthetic dataset comprising approximately 1.3 million samples, generated using GPT-4V. It includes fine-grained image annotations and complex reasoning visual question-answering pairs. The dataset aims to bridge the performance gap between traditional large vision-language models and more resource-efficient lite versions by providing high-quality training data for visual instruction tuning.

\vspace{0.5 em}
\noindent
\textbf{LLaVA-665K.} The LLaVA-665K \cite{liu2024improved} dataset is an expanded and refined version of the original LLaVA instruction tuning dataset, containing 665,000 multimodal instruction-following samples. It integrates diverse sources such as VQAv2, GQA, OCR-VQA, and RefCOCO, among others, to enhance the model's performance across various vision-language tasks. This comprehensive dataset supports improved visual instruction tuning for models like LLaVA-1.5 \cite{liu2024improved}.

\vspace{0.5 em}
\noindent
\textbf{AI2D.} The AI2D (AI2 Diagrams) \cite{kembhavi2016diagram} dataset comprises over 5000 grade school science diagrams, annotated with more than 150,000 rich annotations and over 15000 corresponding multiple-choice questions. It serves as a resource for evaluating models' abilities in diagram understanding and visual reasoning within educational contexts. 

\vspace{0.5 em}
\noindent
\textbf{MMStar.} MMStar \cite{chen2024we} is a meticulously curated benchmark designed to evaluate large vision-language models (LVLMs) on vision-indispensable tasks. It includes 1,500 samples across six core capabilities and 18 detailed axes, ensuring each sample necessitates visual understanding and minimizes data leakage.

\vspace{0.5 em}
\noindent
\textbf{POPE.} The POPE (Polling-based Object Probing Evaluation) \cite{li2023evaluating} dataset is developed to assess object hallucination in LVLMs. It provides a systematic approach to evaluate the consistency of object descriptions generated by models, highlighting tendencies to generate objects not present in the input images.

\vspace{0.5 em}
\noindent
\textbf{ScienceQA.} ScienceQA \cite{lu2022learn} is a large-scale multimodal dataset featuring science questions enriched with lectures and explanations. It spans diverse subjects, including natural science, language science, and social science, aiming to evaluate models' abilities in multimodal reasoning and explanatory question answering.

\vspace{0.5 em}
\noindent
\textbf{TextVQA.} The TextVQA \cite{singh2019towards} dataset focuses on visual question answering tasks that require reading and reasoning about text within images. It contains 45,336 questions over 28,408 images, challenging models to integrate textual and visual information effectively.

\vspace{0.5 em}
\noindent
\textbf{GQA.} GQA (Graph Question Answering) \cite{hudson2019gqa} is a large-scale dataset designed for real-world visual reasoning and compositional question answering. It includes 22 million questions based on 113,000 images, each accompanied by scene graphs detailing objects, attributes, and relationships, facilitating structured reasoning evaluations. 

\vspace{0.5 em}
\noindent
\textbf{MME-RealWorld-Lite.} MME-RealWorld-Lite \cite{zhang2024mme} is a streamlined version of the MME-RealWorld benchmark, offering a subset of 50 samples per task to accelerate inference. It maintains the benchmark's focus on evaluating multimodal models in real-world scenarios with high-resolution images and complex tasks.

\vspace{0.5 em}
\noindent
\textbf{MMMU Pro.} MMMU Pro \cite{yue2024mmmu} is an enhanced benchmark for assessing multimodal models' understanding and reasoning across multiple disciplines. It filters out questions answerable by text-only models, augments candidate options, and introduces vision-only input settings, thereby rigorously evaluating models' true multimodal capabilities.

\vspace{0.5 em}
\noindent
\textbf{OCRBench.} OCRBench \cite{liu2024ocrbench} is a comprehensive evaluation benchmark for optical character recognition (OCR) capabilities in large multimodal models. It encompasses 29 datasets covering tasks like text recognition, scene text-centric VQA, document-oriented VQA, key information extraction, and handwritten mathematical expression recognition, providing a thorough assessment of OCR performance. 

\subsubsection{Model Settings, Training Settings and Evaluation}


\noindent
\textbf{Model Settings.} We embrace the vision-language pre-training task \cite{liu2024improved}, a challenging problem setting that enables effective model training with relatively limited data. We adopt the experiment setting in LIBMoE \cite{nguyen2024libmoe} with LLaVA architecture \cite{liu2023llava}, which includes three modules: pre-trained Large Language Model, pre-trained visual encoder, and randomly initialized MLP connector. We employ the pre-trained SigLIP (Patch14-224) \cite{zhai2023sigmoid} as the vision encoder, pre-trained Phi-3.5-mini-instruct \cite{abdin2024phi} as the LLM, and a randomly initialized MLP connector. In the Visual Instruction Tuning (VIT) stage, we adopt a sparse upcycling approach \cite{komatsuzaki2022sparse} and upcycle only the MLP Connector into 8 experts, employing a top-4 expert routing strategy, while the DeepSeek variants adopt a top-3 expert routing scheme with an additional shared expert. Thus, our model has approximately 4.4B parameters. 

\vspace{0.5 em}
\noindent
\textbf{Training Settings.} We follow LIBMoE \cite{nguyen2024libmoe} for the training settings. Specifically, our training recipe with three stages of training: pre-training, pre-finetuning, and Visual Instruction Tuning (VIT). In the first stage, we only pretrain the MLP connector for better alignment using LLaVA 558K dataset \cite{liu2023llava}. During the second pre-finetuning stage, we train all parameters using high-quality caption data with the ALLaVA \cite{chen2024allava} dataset with 708K samples, aiming to warm up the entire model. In the third stage, we upcycle the MLP Connector to MoE block and trained on visual instruction tuning data (a subset of LLaVA-665K \cite{liu2024improved} with 332K samples). The learning rate is set to $1e-3$ for pre-training the MLP connector and reduced to $2e-6$ for pre-finetuning and $4e-6$ for the final stage. All models are trained in PyTorch using a batch size of 4 and AdamW optimizer~\cite{loshchilov2017decoupled} with its default hyperparameters. We use Zero Redundancy Optimizer (ZeRO) \cite{rajbhandari2020zero} for memory optimization with Zero2 for the first stage and Zero3 for both pre-finetuning and VIT stages.

\vspace{0.5 em}
\noindent
\textbf{Evaluation.} Our model is evaluated under the zero-shot setting across a diverse set of benchmarks encompassing various vision-language capabilities, such as perception, reasoning, OCR, instruction following, etc. The benchmarks considered include AI2D \cite{kembhavi2016diagram}, MMStar \cite{chen2024we}, POPE \cite{li2023evaluating}, ScienceQA \cite{lu2022learn}, TextVQA \cite{singh2019towards}, GQA \cite{hudson2019gqa}, MME-RealWorld-Lite \cite{zhang2024mme}, MMMU Pro \cite{yue2024mmmu}, OCRBench \cite{liu2024ocrbench}.

\vspace{0.5 em}
\noindent
\textbf{Compute Resource.} All models are trained and evaluated on a single node equipped with 4 NVIDIA A100 80GB CoWoS HBM2e PCIe 4.0 employing data-parallelism.

\subsection{Training Time and Resource Allocation}

Table~\ref{tab:time_and_resource} summarizes the training time and resource utilization across all experimental settings.

\begin{table}[htbp]
  \centering
  \caption{Training Time and GPU Resource Allocation Across All Experimental Settings.}
  \label{tab:time_and_resource}
  \setlength{\tabcolsep}{4pt}
  \renewcommand{\arraystretch}{1.2}

  \resizebox{\textwidth}{!}{%
    \begin{tabular}{*{5}{c}}
      \toprule
        \multicolumn{3}{c}{\textbf{Model}} & \makecell{\textbf{Training Time} \\ (hours)} & \textbf{Resourse} \\
      \midrule
        \multirow{6}{*}{\makecell{Vision Language \\ Modeling}} & \multicolumn{2}{c}{Pre-Training} & 5.5 & 4xA100 \\
      \cmidrule{2-5}
         & \multicolumn{2}{c}{Pre-FineTuning} & 18 & 4xA100 \\
      \cmidrule{2-5}
         & \multirow{4}{*}{\makecell{Visual Instruction\\ Tuning}} & SMoE & 10 & 4xA100 \\
         &  & SMoE Sigmoid Gating & 10 & 4xA100 \\
         &  & DeepSeek-V2 & 10.5 & 4xA100 \\
         &  & DeepSeek-V3 & 10.5 & 4xA100 \\

      \midrule
        \multirow{8}{*}{\makecell{Language \\ Modeling}} & \multirow{4}{*}{158M parametes} & SMoE & 9.5 & 4xA100 \\
         &  & SMoE Sigmoid Gating & 10 & 4xA100 \\
         &  & DeepSeek-V2 & 10.5 & 4xA100 \\
         &  & DeepSeek-V3 & 10.5 & 4xA100 \\
      \cmidrule{2-5}
         & \multirow{4}{*}{679M parametes} & SMoE & 65 & 4xA100 \\
         &  & SMoE Sigmoid Gating & 65 & 4xA100 \\
         &  & DeepSeek-V2 & 71 & 4xA100 \\
         &  & DeepSeek-V3 & 71.5 & 4xA100 \\
      \bottomrule
      \end{tabular}%
   }
\end{table}

\newpage

\begin{figure}[t!]
    \centering
    \includegraphics[width=\linewidth]{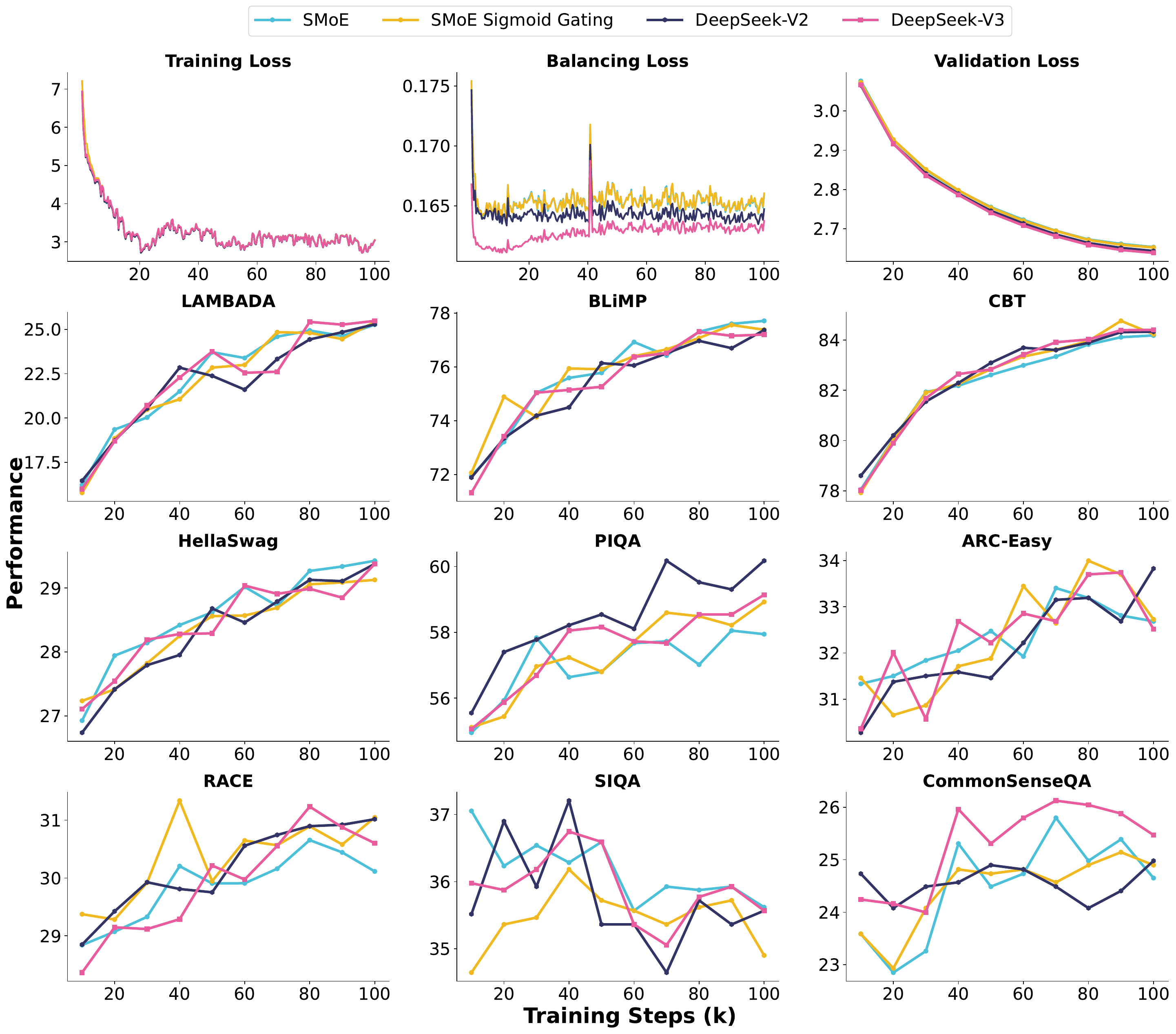}
    \caption{Benchmark curves during training in language modeling tasks for model with 158M parameters.}
    \label{fig:lm_each_benchmark_158m}
\end{figure}

\begin{figure}[t!]
    \centering
    \includegraphics[width=\linewidth]{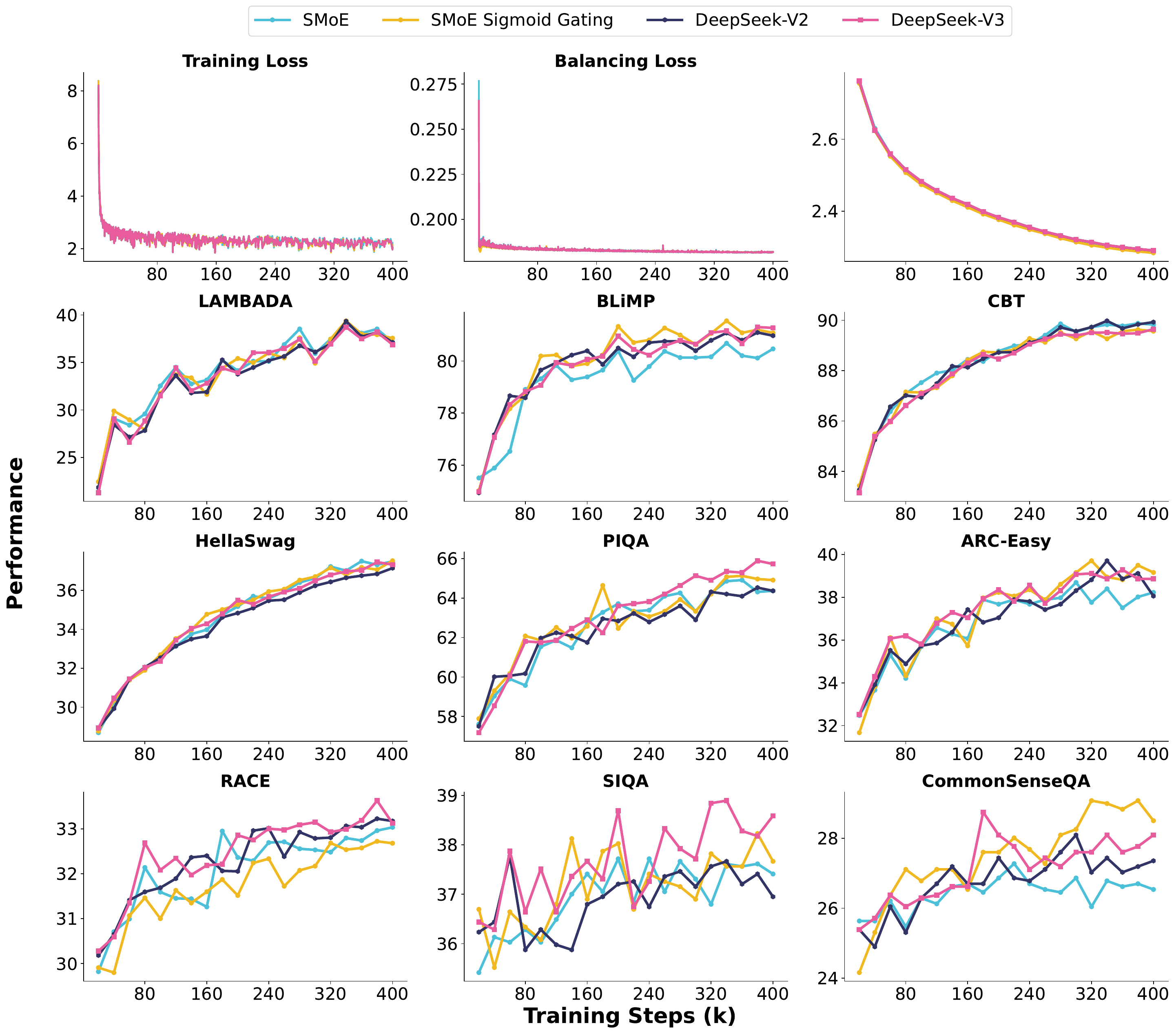}
    \caption{Benchmark curves during training in language modeling tasks for model with 679M parameters.}
    \label{fig:lm_each_benchmark_679m}
\end{figure}

\begin{figure}[t!]
    \centering
    \includegraphics[width=\linewidth]{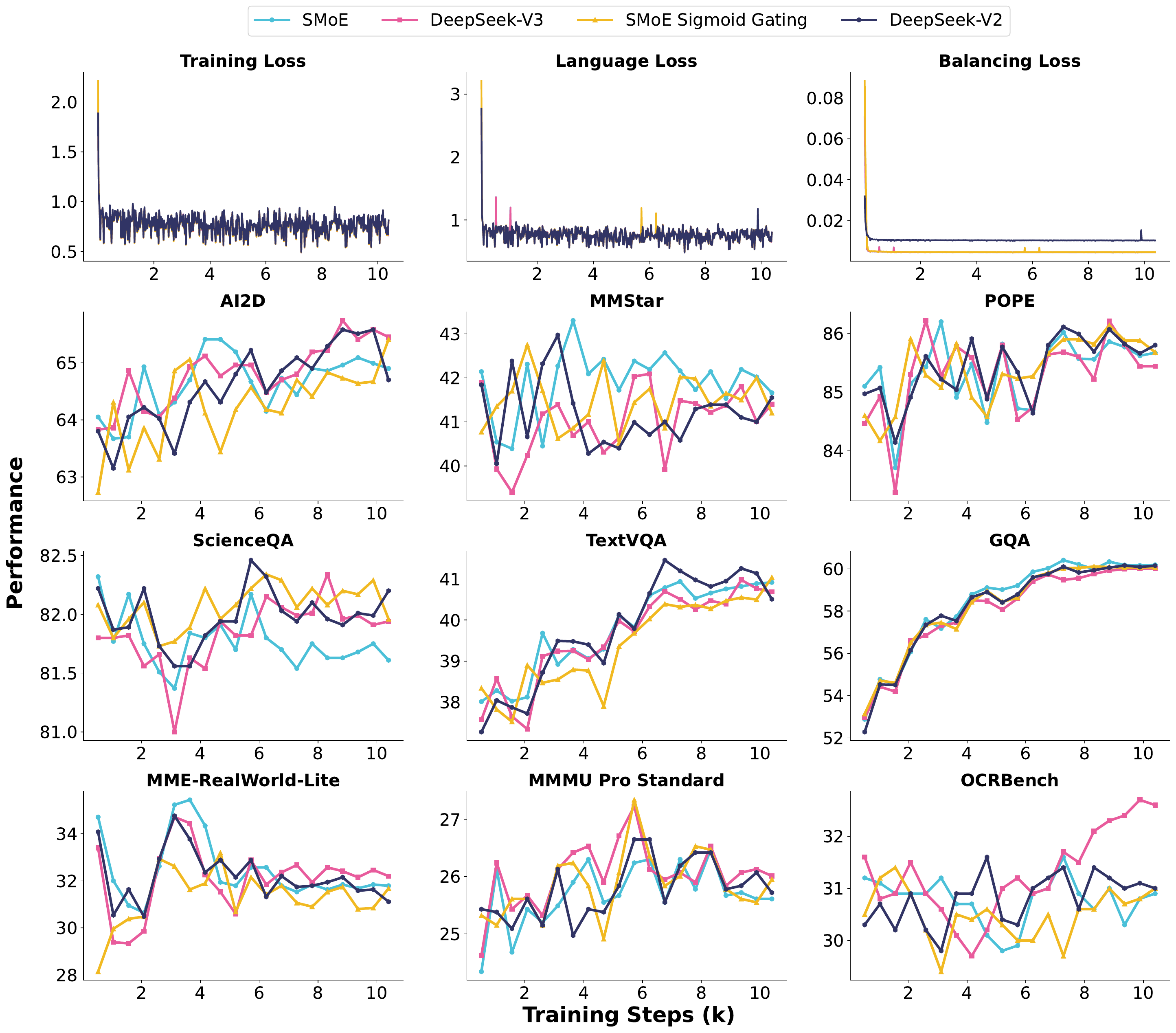}
    \caption{Benchmark curves during training in vision-language modeling tasks.}
    \label{fig:vlm_each_benchmark}
\end{figure}

\newpage
\bibliography{references}

@inproceedings{vaswani2017attention,
 author = {Vaswani, Ashish and Shazeer, Noam and Parmar, Niki and Uszkoreit, Jakob and Jones, Llion and Gomez, Aidan N and Kaiser, \L ukasz and Polosukhin, Illia},
 booktitle = {Advances in Neural Information Processing Systems},
 publisher = {Curran Associates, Inc.},
 title = {Attention is All you Need},
 volume = {30},
 year = {2017}
}

@inproceedings{chi_representation_2022,
	title = {On the Representation Collapse of Sparse Mixture of Experts},
	booktitle = {Advances in {Neural} {Information} {Processing} {Systems}},
	author = {Chi, Zewen and Dong, Li and Huang, Shaohan and Dai, Damai and Ma, Shuming and Patra, Barun and Singhal, Saksham and Bajaj, Payal and Song, Xia and Mao, Xian-Ling and Huang, Heyan and Wei, Furu},
	editor = {Oh, Alice H. and Agarwal, Alekh and Belgrave, Danielle and Cho, Kyunghyun},
	year = {2022},
}

@inproceedings{han2024fusemoe,
  title={FuseMoE: Mixture-of-Experts Transformers for Fleximodal Fusion},
  author={Han, Xing and Nguyen, Huy and Harris, Carl and Ho, Nhat and Saria, Suchi},
  booktitle = "Advances in Neural Information Processing Systems",
  year={2024}
}

@article{pham2024competesmoe,
      title={CompeteSMoE -- Effective Training of Sparse Mixture of Experts via Competition}, 
      author={Quang Pham and Giang Do and Huy Nguyen and TrungTin Nguyen and Chenghao Liu and Mina Sartipi and Binh T. Nguyen and Savitha Ramasamy and Xiaoli Li and Steven Hoi and Nhat Ho},
      journal={arXiv preprint arXiv:2402.02526},
      year={2024}
}

@inproceedings{
diep2025zero,
title={On Zero-Initialized Attention: Optimal Prompt and Gating Factor Estimation},
author={Nghiem Tuong Diep and Huy Nguyen and Chau Nguyen and Minh Le and Duy Minh Ho Nguyen and Daniel Sonntag and Mathias Niepert and Nhat Ho},
booktitle={Forty-second International Conference on Machine Learning},
year={2025}
}

@article{yan2025sigmoid,
      title={Sigmoid Self-Attention has Lower Sample Complexity than Softmax Self-Attention: A Mixture-of-Experts Perspective},
      author={Fanqi Yan and Huy Nguyen and Pedram Akbarian and Nhat Ho and Alessandro Rinaldo},
      journal={arXiv preprint arXiv:2502.00281},
      year={2025}
}

@inproceedings{
truong2025replora,
title={RepLo{RA}: Reparameterizing Low-rank Adaptation via the Perspective of Mixture of Experts},
author={Tuan Truong and Chau Nguyen and Huy Nguyen and Minh Le and Trung Le and Nhat Ho},
booktitle={Forty-second International Conference on Machine Learning},
year={2025}
}

@ARTICLE{nguyen2025convergence,
  author={Nguyen, Huy and Ho, Nhat and Rinaldo, Alessandro},
  journal={IEEE Transactions on Information Theory}, 
  title={Convergence Rates for Softmax Gating Mixture of Experts}, 
  year={2026},
  volume={72},
  number={2},
  pages={1276-1304},
  doi={10.1109/TIT.2025.3647061}}

@article{akbarian2024quadratic,
      title={Quadratic Gating Mixture of Experts: Statistical Insights into Self-Attention}, 
      author={Pedram Akbarian and Huy Nguyen and Xing Han and Nhat Ho},
      journal={arXiv preprint arXiv:2410.11222}, 
      year={2024}
}

@inproceedings{nguyen2025cosine,
    author = {Huy Nguyen and Pedram Akbarian and Trang Pham and Trang Nguyen and Shujian Zhang and Nhat Ho},
    title = {Statistical Advantages of Perturbing Cosine Router in Mixture of Experts},
    booktitle = {International Conference on Learning Representations},
    year = 2025
}

@inproceedings{
le2024mixture,
title={Mixture of Experts Meets Prompt-Based Continual Learning},
author={Minh Le and An Nguyen The and Huy Nguyen and Thien Trang Nguyen Vu and Huyen Trang Pham and Linh Ngo Van and Nhat Ho},
booktitle={The Thirty-eighth Annual Conference on Neural Information Processing Systems},
year={2024}
}

@inproceedings{
le2025revisiting,
title={Revisiting Prefix-tuning: Statistical Benefits of Reparameterization among Prompts},
author={Minh Le and Chau Nguyen and Huy Nguyen and Quyen Tran and Trung Le and Nhat Ho},
booktitle={The Thirteenth International Conference on Learning Representations},
year={2025},
url={https://openreview.net/forum?id=QjTSaFXg25}
}

@inproceedings{dai2024deepseekmoe,
    title = "{D}eep{S}eek{M}o{E}: Towards Ultimate Expert Specialization in Mixture-of-Experts Language Models",
    author = "Dai, Damai  and
      Deng, Chengqi  and
      Zhao, Chenggang  and
      Xu, R.x.  and
      Gao, Huazuo  and
      Chen, Deli  and
      Li, Jiashi  and
      Zeng, Wangding  and
      Yu, Xingkai  and
      Wu, Y.  and
      Xie, Zhenda  and
      Li, Y.k.  and
      Huang, Panpan  and
      Luo, Fuli  and
      Ruan, Chong  and
      Sui, Zhifang  and
      Liang, Wenfeng",
    booktitle = "Proceedings of the 62nd Annual Meeting of the Association for Computational Linguistics (Volume 1: Long Papers)",
    month = aug,
    year = "2024",
    publisher = "Association for Computational Linguistics",
    doi = "10.18653/v1/2024.acl-long.70",
    pages = "1280--1297",
}

@inproceedings{nguyen2024statistical,
      title={Statistical Perspective of Top-K Sparse Softmax Gating Mixture of Experts}, 
      author={Huy Nguyen and Pedram Akbarian and Fanqi Yan and Nhat Ho},
      booktitle={International Conference on Learning Representations},
      year={2024}
}

@inproceedings{
li2025cl,
title={Theory on Mixture-of-Experts in Continual Learning},
author={Hongbo Li and Sen Lin and Lingjie Duan and Yingbin Liang and Ness Shroff},
booktitle={The Thirteenth International Conference on Learning Representations},
year={2025}
}

@inproceedings{kwon_em_2020,
	series = {Proceedings of {Machine} {Learning} {Research}},
	title = {{EM} {Converges} for a {Mixture} of {Many} {Linear} {Regressions}},
	volume = {108},
	url = {https://proceedings.mlr.press/v108/kwon20a.html},
	booktitle = {Proceedings of the {Twenty} {Third} {International} {Conference} on {Artificial} {Intelligence} and {Statistics}},
	publisher = {PMLR},
	author = {Kwon, Jeongyeol and Caramanis, Constantine},
	editor = {Chiappa, Silvia and Calandra, Roberto},
	month = aug,
	year = {2020},
	pages = {1727--1736},
}

@inproceedings{liang_m3vit_2022,
	title = {M$^3${ViT}: {Mixture}-of-{Experts} {Vision} {Transformer} for {Efficient} {Multi}-task {Learning} with {Model}-{Accelerator} {Co}-design},
	booktitle = {{NeurIPS}},
	author = {Liang, Hanxue and Fan, Zhiwen and Sarkar, Rishov and Jiang, Ziyu and Chen, Tianlong and Zou, Kai and Cheng, Yu and Hao, Cong and Wang, Zhangyang},
	year = {2022},
}

@article{Teicher-63,
author = "Henry Teicher",
title = "Identifiability of finite mixtures",
journal = "Ann. Math. Statist.",
volume = "32",
year = "1963",
pages = "1265--1269"
}

@BOOK{vandeGeer-00,
AUTHOR = "Sara van de Geer",
TITLE = "Empirical processes in M-estimation",
PUBLISHER = "Cambridge University Press",
YEAR = "2000"
}

@article{Jacob_Jordan-1991,
	author="Robert A. Jacobs and Michael I. Jordan and Steven J. Nowlan and Geoffrey E. Hinton",
	title="Adaptive mixtures of local experts",
	journal="Neural Computation",
	volume="3",
	page="79-87",
	year="1991"
}

@article{ho2022gaussian,
  author  = {Nhat Ho and Chiao-Yu Yang and Michael I. Jordan},
  title   = {Convergence Rates for {G}aussian Mixtures of Experts},
  journal = {Journal of Machine Learning Research},
  year    = {2022},
  volume  = {23},
  number  = {323},
  pages   = {1--81},
}

@inproceedings{
ceron2024rl,
title={Mixtures of Experts Unlock Parameter Scaling for Deep {RL}},
author={Johan Samir Obando Ceron and Ghada Sokar and Timon Willi and Clare Lyle and Jesse Farebrother and Jakob Nicolaus Foerster and Gintare Karolina Dziugaite and Doina Precup and Pablo Samuel Castro},
booktitle={Forty-first International Conference on Machine Learning},
year={2024}
}

@INPROCEEDINGS{shazeer2017topk,
   AUTHOR = "Noam Shazeer and Azalia Mirhoseini and Krzysztof Maziarz and Andy Davis and Quoc Le and Geoffrey Hinton and Jeff Dean",
   TITLE = "Outrageously Large Neural Networks: The Sparsely-Gated Mixture-of-Experts Layer",
   BOOKTITLE = 	 "In International Conference on Learning Representations", 
   YEAR = 	 2017
}

@inproceedings{Du_Glam_MoE,
  title = "GLaM: Efficient Scaling of Language Models with Mixture-of-Experts",
  author = "Nan Du and Yanping Huang and Andrew M. Dai and Simon Tong and Dmitry Lepikhin and Yuanzhong Xu and Maxim Krikun and Yanqi Zhou and Adams Wei Yu and Orhan Firat and Barret Zoph and Liam Fedus and Maarten Bosma and Zongwei Zhou and Tao Wang and Yu Emma Wang and Kellie Webster and Marie Pellat and Kevin Robinson and Kathleen Meier-Hellstern and Toju Duke and Lucas Dixon and Kun Zhang and Quoc V Le and Yonghui Wu and Zhifeng Chen and Claire Cui",
  booktitle = "ICML",
  year = "2022"
}

@InProceedings{manole22refined,
  title = 	 {Refined Convergence Rates for Maximum Likelihood Estimation under Finite Mixture Models},
  author =       {Tudor Manole and Nhat Ho},
  booktitle = 	 {Proceedings of the 39th International Conference on Machine Learning},
  pages = 	 {14979--15006},
  year = 	 {2022},
  volume = 	 {162},
  series = 	 {Proceedings of Machine Learning Research},
  month = 	 {17--23 Jul},
  publisher =    {PMLR}
}

@article{boixadsera2025granularity,
      title={The power of fine-grained experts: Granularity boosts expressivity in Mixture of Experts}, 
      author={Enric Boix-Adsera and Philippe Rigollet},
      year={2025},
      Journal = {arxiv preprint arxiv 2505.06839}
}

@inproceedings{
wang2025expressivepower,
title={On the Expressive Power of Mixture-of-Experts for Structured Complex Tasks},
author={Mingze Wang and Weinan E},
booktitle={The Thirty-ninth Annual Conference on Neural Information Processing Systems},
year={2025}
}

@ARTICLE{zeevi1998approximation,
  author={Assaf Zeevi and Ron Meir and Vitaly Maiorov},
  journal={IEEE Transactions on Information Theory}, 
  title={Error bounds for functional approximation and estimation using mixtures of experts}, 
  year={1998},
  volume={44},
  number={3},
  pages={1010-1025}}

@inproceedings{
yun2024flexmoe,
title={Flex-MoE: Modeling Arbitrary Modality Combination via the Flexible Mixture-of-Experts},
author={Sukwon Yun and Inyoung Choi and Jie Peng and Yangfan Wu and Jingxuan Bao and Qiyiwen Zhang and Jiayi Xin and Qi Long and Tianlong Chen},
booktitle={The Thirty-eighth Annual Conference on Neural Information Processing Systems},
year={2024}
}

@article{deepseekv3,
  title={Deepseek-v3 technical report},
  author={DeepSeek-AI and others},
  journal={arXiv preprint arXiv:2412.19437},
  year={2024}
}

@article{qwen2025,
  title={Qwen2.5 Technical Report},
  author={Qwen and others},
  journal={arXiv preprint arXiv:2412.15115},
  year={2025}
}

@article{deepseekv2,
  title={DeepSeek-V2: A Strong, Economical, and Efficient Mixture-of-Experts Language Model},
  author={DeepSeek-AI and others},
  journal={arXiv preprint arXiv:2405.04434},
  year={2024}
}

@inproceedings{dwivedi2018misspecified,
 author = {Dwivedi, Raaz and Ho, Nhat and Khamaru, Koulik and Wainwright, Martin J and Jordan, Michael I},
 booktitle = {Advances in Neural Information Processing Systems},
 editor = {S. Bengio and H. Wallach and H. Larochelle and K. Grauman and N. Cesa-Bianchi and R. Garnett},
 pages = {},
 publisher = {Curran Associates, Inc.},
 title = {Theoretical guarantees for EM under misspecified Gaussian mixture models},
 url = {https://proceedings.neurips.cc/paper_files/paper/2018/file/acc21473c4525b922286130ffbfe00b5-Paper.pdf},
 volume = {31},
 year = {2018}
}

@article{geminiteam2024gemini15,
  title={Gemini 1.5: Unlocking multimodal understanding across millions of tokens of context},
  author={Gemini Team and Petko Georgiev and Ving Ian Lei and Ryan Burnell and Libin Bai and Anmol Gulati and Garrett Tanzer and Damien Vincent and Zhufeng Pan and others},
  journal={arXiv preprint arXiv:2403.05530},
  year={2024}
}

@article{grattafiori2024llama3,
  title={The Llama 3 Herd of Models},
  author={Aaron Grattafiori and Abhimanyu Dubey and Abhinav Jauhri and Abhinav Pandey and Abhishek Kadian and Ahmad Al-Dahle and Aiesha Letman and Akhil Mathur and others},
  journal={arXiv preprint arXiv:2407.21783},
  year={2024}
}

@article{jiang2024mixtral,
      title={Mixtral of Experts}, 
      author={Albert Q. Jiang and Alexandre Sablayrolles and Antoine Roux and Arthur Mensch and Blanche Savary and Chris Bamford and Devendra Singh Chaplot and Diego de las Casas and Emma Bou Hanna and Florian Bressand and Gianna Lengyel and Guillaume Bour and Guillaume Lample and Lélio Renard Lavaud and Lucile Saulnier and Marie-Anne Lachaux and Pierre Stock and Sandeep Subramanian and Sophia Yang and Szymon Antoniak and Teven Le Scao and Théophile Gervet and Thibaut Lavril and Thomas Wang and Timothée Lacroix and William El Sayed},
      year={2024},
      Journal = {arxiv preprint arxiv 2401.04088}
}

@inproceedings{lepikhin_gshard_2021,
	title = {{GS}hard: {Scaling} {Giant} {Models} with {Conditional} {Computation} and {Automatic} {Sharding}},
	booktitle = {International {Conference} on {Learning} {Representations}},
	author = {Dmitry Lepikhin and HyoukJoong Lee and Yuanzhong Xu and Dehao Chen and Orhan Firat and Yanping Huang and Maxim Krikun and Noam Shazeer and Zhifeng Chen},
	year = {2021},
}

@inproceedings{Riquelme2021scalingvision,
 author = {Carlos Riquelme and Joan Puigcerver and Basil Mustafa and Maxim Neumann and Rodolphe Jenatton and André Susano Pinto and Daniel Keysers and Neil Houlsby},
 booktitle = {Advances in Neural Information Processing Systems},
 pages = {8583--8595},
 publisher = {Curran Associates, Inc.},
 title = {Scaling Vision with Sparse Mixture of Experts},
 volume = {34},
 year = {2021}
}

@inproceedings{chen2022theory,
 author = {Chen, Zixiang and Deng, Yihe and Wu, Yue and Gu, Quanquan and Li, Yuanzhi},
 booktitle = {Advances in Neural Information Processing Systems},
 editor = {S. Koyejo and S. Mohamed and A. Agarwal and D. Belgrave and K. Cho and A. Oh},
 pages = {23049--23062},
 publisher = {Curran Associates, Inc.},
 title = {Towards Understanding the Mixture-of-Experts Layer in Deep Learning},
 volume = {35},
 year = {2022}
}

@inproceedings{nguyen2023demystifying,
      title={Demystifying Softmax Gating Function in {G}aussian Mixture of Experts}, 
      author={Huy Nguyen and TrungTin Nguyen and Nhat Ho},
      booktitle = "Advances in Neural Information Processing Systems",
      year={2023}
}

@inproceedings{nguyen2024general,
      title={A General Theory for Softmax Gating Multinomial Logistic Mixture of Experts}, 
      author={Huy Nguyen and Pedram Akbarian and TrungTin Nguyen and Nhat Ho},
      booktitle ="Proceedings of the 41st International Conference on Machine Learning",
      year={2024}
}

@article{faria2010regression,
author = {Susana Faria and Gilda Soromenho},
title = {Fitting mixtures of linear regressions},
journal = {Journal of Statistical Computation and Simulation},
volume = {80},
number = {2},
pages = {201-225},
year = {2010},
publisher = {Taylor & Francis}
}

@inproceedings{chow_mixture_expert_2023,
	title = {A {Mixture}-of-{Expert} {Approach} to {RL}-based {Dialogue} {Management}},
	url = {https://openreview.net/forum?id=4FBUihxz5nm},
	booktitle = {The {Eleventh} {International} {Conference} on {Learning} {Representations}},
	author = {Chow, Yinlam and Tulepbergenov, Azamat and Nachum, Ofir and Gupta, Dhawal and Ryu, Moonkyung and Ghavamzadeh, Mohammad and Boutilier, Craig},
	year = {2023},
}

@ARTICLE{mendes2011convergence,
  author={Mendes, Eduardo F. and Jiang, Wenxin},
  journal={Neural Computation}, 
  title={On Convergence Rates of Mixtures of Polynomial Experts}, 
  year={2012},
  volume={24},
  number={11},
  pages={3025-3051},
  doi={10.1162/NECO_a_00354}
}

@misc{cerebras2023slimpajama,
author = {Soboleva, Daria and Al-Khateeb, Faisal and Myers, Robert and Steeves, Jacob R and Hestness, Joel and Dey, Nolan},
title = {{SlimPajama: A 627B token cleaned and deduplicated version of RedPajama}},
year = 2023,
url = {https://huggingface.co/datasets/cerebras/SlimPajama-627B},
}

@inproceedings{liu2023llava,
author      = {Liu, Haotian and Li, Chunyuan and Wu, Qingyang and Lee, Yong Jae},
title       = {Visual Instruction Tuning},
booktitle   = {NeurIPS},
year        = {2023}
}

@article{dempster1977maximum,
  title={Maximum likelihood from incomplete data via the EM algorithm},
  author={Dempster, Arthur P and Laird, Nan M and Rubin, Donald B},
  journal={Journal of the royal statistical society: series B (methodological)},
  volume={39},
  number={1},
  pages={1--22},
  year={1977},
  publisher={Wiley Online Library}
}

@article{chamroukhi2009time,
  title={Time series modeling by a regression approach based on a latent process},
  author={Chamroukhi, Faicel and Sam{\'e}, Allou and Govaert, G{\'e}rard and Aknin, Patrice},
  journal={Neural Networks},
  volume={22},
  number={5-6},
  pages={593--602},
  year={2009},
  publisher={Elsevier}
}

@article{fedus2022switch,
  title={Switch transformers: Scaling to trillion parameter models with simple and efficient sparsity},
  author={Fedus, William and Zoph, Barret and Shazeer, Noam},
  journal={Journal of Machine Learning Research},
  volume={23},
  number={120},
  pages={1--39},
  year={2022}
}

@article{nguyen2024libmoe,
  title={LIBMoE: A Library for comprehensive benchmarking Mixture of Experts in Large Language Models},
  author={Nguyen, Nam V and Doan, Thong T and Tran, Luong and Nguyen, Van and Pham, Quang},
  journal={arXiv preprint arXiv:2411.00918},
  year={2024}
}

@article{paperno2016lambada,
  title={The LAMBADA dataset: Word prediction requiring a broad discourse context},
  author={Paperno, Denis and Kruszewski, Germ{\'a}n and Lazaridou, Angeliki and Pham, Quan Ngoc and Bernardi, Raffaella and Pezzelle, Sandro and Baroni, Marco and Boleda, Gemma and Fern{\'a}ndez, Raquel},
  journal={arXiv preprint arXiv:1606.06031},
  year={2016}
}

@article{warstadt2020blimp,
  title={BLiMP: The benchmark of linguistic minimal pairs for English},
  author={Warstadt, Alex and Parrish, Alicia and Liu, Haokun and Mohananey, Anhad and Peng, Wei and Wang, Sheng-Fu and Bowman, Samuel R},
  journal={Transactions of the Association for Computational Linguistics},
  volume={8},
  pages={377--392},
  year={2020},
  publisher={MIT Press One Rogers Street, Cambridge, MA 02142-1209, USA journals-info~…}
}

@article{hill2015goldilocks,
  title={The goldilocks principle: Reading children's books with explicit memory representations},
  author={Hill, Felix and Bordes, Antoine and Chopra, Sumit and Weston, Jason},
  journal={arXiv preprint arXiv:1511.02301},
  year={2015}
}

@article{zellers2019hellaswag,
  title={Hellaswag: Can a machine really finish your sentence?},
  author={Zellers, Rowan and Holtzman, Ari and Bisk, Yonatan and Farhadi, Ali and Choi, Yejin},
  journal={arXiv preprint arXiv:1905.07830},
  year={2019}
}

@inproceedings{bisk2020piqa,
  title={Piqa: Reasoning about physical commonsense in natural language},
  author={Bisk, Yonatan and Zellers, Rowan and Gao, Jianfeng and Choi, Yejin and others},
  booktitle={Proceedings of the AAAI conference on artificial intelligence},
  year={2020}
}

@article{clark2018think,
  title={Think you have solved question answering? try arc, the ai2 reasoning challenge},
  author={Clark, Peter and Cowhey, Isaac and Etzioni, Oren and Khot, Tushar and Sabharwal, Ashish and Schoenick, Carissa and Tafjord, Oyvind},
  journal={arXiv preprint arXiv:1803.05457},
  year={2018}
}

@inproceedings{lai-etal-2017-race,
    title = "{RACE}: Large-scale {R}e{A}ding Comprehension Dataset From Examinations",
    author = "Lai, Guokun  and
      Xie, Qizhe  and
      Liu, Hanxiao  and
      Yang, Yiming  and
      Hovy, Eduard",
    booktitle = "Proceedings of the 2017 Conference on Empirical Methods in Natural Language Processing",
    month = sep,
    year = "2017",
    address = "Copenhagen, Denmark",
    publisher = "Association for Computational Linguistics",
    url = "https://aclanthology.org/D17-1082",
    doi = "10.18653/v1/D17-1082",
    pages = "785--794",
}

@article{sap2019socialiqa,
  title={Socialiqa: Commonsense reasoning about social interactions},
  author={Sap, Maarten and Rashkin, Hannah and Chen, Derek and LeBras, Ronan and Choi, Yejin},
  journal={arXiv preprint arXiv:1904.09728},
  year={2019}
}

@article{talmor2018commonsenseqa,
  title={Commonsenseqa: A question answering challenge targeting commonsense knowledge},
  author={Talmor, Alon and Herzig, Jonathan and Lourie, Nicholas and Berant, Jonathan},
  journal={arXiv preprint arXiv:1811.00937},
  year={2018}
}

@article{abdin2024phi,
  title={Phi-3 technical report: A highly capable language model locally on your phone},
  author={Abdin, Marah and Aneja, Jyoti and Awadalla, Hany and Awadallah, Ahmed and Awan, Ammar Ahmad and Bach, Nguyen and Bahree, Amit and Bakhtiari, Arash and Bao, Jianmin and Behl, Harkirat and others},
  journal={arXiv preprint arXiv:2404.14219},
  year={2024}
}

@inproceedings{kembhavi2016diagram,
  title={A diagram is worth a dozen images},
  author={Kembhavi, Aniruddha and Salvato, Mike and Kolve, Eric and Seo, Minjoon and Hajishirzi, Hannaneh and Farhadi, Ali},
  booktitle={Computer Vision--ECCV 2016: 14th European Conference, Amsterdam, The Netherlands, October 11--14, 2016, Proceedings, Part IV 14},
  pages={235--251},
  year={2016},
  organization={Springer}
}

@article{chen2024we,
  title={Are we on the right way for evaluating large vision-language models?},
  author={Chen, Lin and Li, Jinsong and Dong, Xiaoyi and Zhang, Pan and Zang, Yuhang and Chen, Zehui and Duan, Haodong and Wang, Jiaqi and Qiao, Yu and Lin, Dahua and others},
  journal={arXiv preprint arXiv:2403.20330},
  year={2024}
}

@article{li2023evaluating,
  title={Evaluating object hallucination in large vision-language models},
  author={Li, Yifan and Du, Yifan and Zhou, Kun and Wang, Jinpeng and Zhao, Wayne Xin and Wen, Ji-Rong},
  journal={arXiv preprint arXiv:2305.10355},
  year={2023}
}

@article{lu2022learn,
  title={Learn to explain: Multimodal reasoning via thought chains for science question answering},
  author={Lu, Pan and Mishra, Swaroop and Xia, Tanglin and Qiu, Liang and Chang, Kai-Wei and Zhu, Song-Chun and Tafjord, Oyvind and Clark, Peter and Kalyan, Ashwin},
  journal={Advances in Neural Information Processing Systems},
  volume={35},
  pages={2507--2521},
  year={2022}
}

@inproceedings{singh2019towards,
  title={Towards vqa models that can read},
  author={Singh, Amanpreet and Natarajan, Vivek and Shah, Meet and Jiang, Yu and Chen, Xinlei and Batra, Dhruv and Parikh, Devi and Rohrbach, Marcus},
  booktitle={Proceedings of the IEEE/CVF conference on computer vision and pattern recognition},
  pages={8317--8326},
  year={2019}
}

@inproceedings{hudson2019gqa,
  title={Gqa: A new dataset for real-world visual reasoning and compositional question answering},
  author={Hudson, Drew A and Manning, Christopher D},
  booktitle={Proceedings of the IEEE/CVF conference on computer vision and pattern recognition},
  pages={6700--6709},
  year={2019}
}

@article{zhang2024mme,
  title={MME-RealWorld: Could Your Multimodal LLM Challenge High-Resolution Real-World Scenarios that are Difficult for Humans?},
  author={Zhang, Yi-Fan and Zhang, Huanyu and Tian, Haochen and Fu, Chaoyou and Zhang, Shuangqing and Wu, Junfei and Li, Feng and Wang, Kun and Wen, Qingsong and Zhang, Zhang and others},
  journal={arXiv preprint arXiv:2408.13257},
  year={2024}
}

@article{yue2024mmmu,
  title={Mmmu-pro: A more robust multi-discipline multimodal understanding benchmark},
  author={Yue, Xiang and Zheng, Tianyu and Ni, Yuansheng and Wang, Yubo and Zhang, Kai and Tong, Shengbang and Sun, Yuxuan and Yu, Botao and Zhang, Ge and Sun, Huan and others},
  journal={arXiv preprint arXiv:2409.02813},
  year={2024}
}

@inproceedings{zhai2023sigmoid,
  title={Sigmoid loss for language image pre-training},
  author={Zhai, Xiaohua and Mustafa, Basil and Kolesnikov, Alexander and Beyer, Lucas},
  booktitle={Proceedings of the IEEE/CVF international conference on computer vision},
  pages={11975--11986},
  year={2023}
}

@article{muennighoff2024olmoe,
  title={Olmoe: Open mixture-of-experts language models},
  author={Muennighoff, Niklas and Soldaini, Luca and Groeneveld, Dirk and Lo, Kyle and Morrison, Jacob and Min, Sewon and Shi, Weijia and Walsh, Pete and Tafjord, Oyvind and Lambert, Nathan and others},
  journal={arXiv preprint arXiv:2409.02060},
  year={2024}
}

@inproceedings{
ludziejewski2024scaling,
title={Scaling Laws for Fine-Grained Mixture of Experts},
author={Jan Ludziejewski and Jakub Krajewski and Kamil Adamczewski and Maciej Pi{\'o}ro and Micha{\l} Krutul and Szymon Antoniak and Kamil Ciebiera and Krystian Kr{\'o}l and Tomasz Odrzyg{\'o}{\'z}d{\'z} and Piotr Sankowski and Marek Cygan and Sebastian Jaszczur},
booktitle={ICLR 2024 Workshop on Mathematical and Empirical Understanding of Foundation Models},
year={2024}
}

@article{ho_convergence_2016,
	title = {Convergence rates of parameter estimation for some weakly identifiable finite mixtures},
	volume = {44},
	doi = {10.1214/16-AOS1444},
	number = {6},
	journal = {The Annals of Statistics},
	author = {Ho, Nhat and Nguyen, XuanLong},
	year = {2016},
	note = {Publisher: Institute of Mathematical Statistics and Bernoulli Society},
	pages = {2726 -- 2755},
}

@article{xue2024openmoe,
  title={Openmoe: An early effort on open mixture-of-experts language models},
  author={Xue, Fuzhao and Zheng, Zian and Fu, Yao and Ni, Jinjie and Zheng, Zangwei and Zhou, Wangchunshu and You, Yang},
  journal={arXiv preprint arXiv:2402.01739},
  year={2024}
}

@article{dai2022stablemoe,
  title={Stablemoe: Stable routing strategy for mixture of experts},
  author={Dai, Damai and Dong, Li and Ma, Shuming and Zheng, Bo and Sui, Zhifang and Chang, Baobao and Wei, Furu},
  journal={arXiv preprint arXiv:2204.08396},
  year={2022}
}

@article{weber2024redpajama,
  title={Redpajama: an open dataset for training large language models},
  author={Weber, Maurice and Fu, Dan and Anthony, Quentin and Oren, Yonatan and Adams, Shane and Alexandrov, Anton and Lyu, Xiaozhong and Nguyen, Huu and Yao, Xiaozhe and Adams, Virginia and others},
  journal={Advances in neural information processing systems},
  volume={37},
  pages={116462--116492},
  year={2024}
}

@article{liu2024ocrbench,
  title={OCRBench: on the hidden mystery of OCR in large multimodal models},
  author={Liu, Yuliang and Li, Zhang and Huang, Mingxin and Yang, Biao and Yu, Wenwen and Li, Chunyuan and Yin, Xu-Cheng and Liu, Cheng-Lin and Jin, Lianwen and Bai, Xiang},
  journal={Science China Information Sciences},
  volume={67},
  number={12},
  pages={220102},
  year={2024},
  publisher={Springer}
}

@article{su2024roformer,
  title={Roformer: Enhanced transformer with rotary position embedding},
  author={Su, Jianlin and Ahmed, Murtadha and Lu, Yu and Pan, Shengfeng and Bo, Wen and Liu, Yunfeng},
  journal={Neurocomputing},
  volume={568},
  pages={127063},
  year={2024},
  publisher={Elsevier}
}

@article{loshchilov2017decoupled,
  title={Decoupled weight decay regularization},
  author={Loshchilov, Ilya and Hutter, Frank},
  journal={arXiv preprint arXiv:1711.05101},
  year={2017}
}

@article{kudo2018sentencepiece,
  title={Sentencepiece: A simple and language independent subword tokenizer and detokenizer for neural text processing},
  author={Kudo, Taku and Richardson, John},
  journal={arXiv preprint arXiv:1808.06226},
  year={2018}
}

@inproceedings{liu2024improved,
  title={Improved baselines with visual instruction tuning},
  author={Liu, Haotian and Li, Chunyuan and Li, Yuheng and Lee, Yong Jae},
  booktitle={Proceedings of the IEEE/CVF Conference on Computer Vision and Pattern Recognition},
  pages={26296--26306},
  year={2024}
}

@article{chen2024allava,
  title={Allava: Harnessing gpt4v-synthesized data for lite vision-language models},
  author={Chen, Guiming Hardy and Chen, Shunian and Zhang, Ruifei and Chen, Junying and Wu, Xiangbo and Zhang, Zhiyi and Chen, Zhihong and Li, Jianquan and Wan, Xiang and Wang, Benyou},
  journal={arXiv preprint arXiv:2402.11684},
  year={2024}
}

@article{jain1984quantitative,
  title={A quantitative measure of fairness and discrimination},
  author={Jain, Rajendra K and Chiu, Dah-Ming W and Hawe, William R and others},
  journal={Eastern Research Laboratory, Digital Equipment Corporation, Hudson, MA},
  volume={21},
  number={1},
  year={1984}
}

@article{komatsuzaki2022sparse,
  title={Sparse upcycling: Training mixture-of-experts from dense checkpoints},
  author={Komatsuzaki, Aran and Puigcerver, Joan and Lee-Thorp, James and Ruiz, Carlos Riquelme and Mustafa, Basil and Ainslie, Joshua and Tay, Yi and Dehghani, Mostafa and Houlsby, Neil},
  journal={arXiv preprint arXiv:2212.05055},
  year={2022}
}

@inproceedings{rajbhandari2020zero,
  title={Zero: Memory optimizations toward training trillion parameter models},
  author={Rajbhandari, Samyam and Rasley, Jeff and Ruwase, Olatunji and He, Yuxiong},
  booktitle={SC20: International Conference for High Performance Computing, Networking, Storage and Analysis},
  pages={1--16},
  year={2020},
  organization={IEEE}
}
\bibliographystyle{abbrv}
\end{document}